\documentclass[10pt]{article}
\usepackage[dvipsnames]{xcolor}
\usepackage[utf8]{inputenc}
\usepackage{tikz}
\usepackage{subcaption}
\usepackage{float} 
\usepackage{amsmath}
\usepackage{amssymb}
\usepackage{natbib}

\usepackage{array}
\usepackage{bbm}
\usepackage{makeidx}
\usepackage{hyperref}
\usepackage[T1]{fontenc}
\usepackage[capitalise]{cleveref}
\usepackage{amsthm}
\usepackage{mathtools}
\usepackage{graphicx} 
\usepackage[LGR,T1]{fontenc}
\usepackage{siunitx}
\usepackage{mathrsfs}
\usepackage{scalerel}
\usepackage{forest}
\usepackage{dsfont}
\usepackage{algorithm}
\usepackage[noend]{algpseudocode}
\usepackage{lipsum}

\usepackage[title]{appendix}
\hypersetup{
    colorlinks,
    linkcolor={blue!80!black},
    citecolor={green!50!black},
}
\usepackage{accents}
\usepackage{apptools}
\usetikzlibrary{arrows, positioning, automata}
\usepackage{comment}
\usepackage[shortlabels]{enumitem}

\DeclareSymbolFont{rsfs}{U}{rsfs}{m}{n}
\DeclareSymbolFontAlphabet{\mathscrsfs}{rsfs}

\AtAppendix{\counterwithin{lemma}{section}}

\usepackage{mathtools}

\let\hat\widehat
\let\tilde\widetilde

\let\bar\overline

\renewcommand{\Vec}{\mathop{\mathrm{vec}}}

\theoremstyle{plain}
\def\arc{\mathsf{arc}}
\def\sfb{\mathsf{b}}

\newtheorem{theorem}{Theorem}[section]

\newtheorem{lemma}{Lemma}[section]

\newtheorem{assumption}{Assumption}[section]
\newtheorem{definition}{Definition}[section]
\newtheorem{remark}{Remark}[section]
\Crefname{equation}{}{}
\allowdisplaybreaks
\usepackage[top=1in,bottom=1in,left=1in,right=1in]{geometry}
\usepackage[utf8]{inputenc}

\def\dd{\mathrm{d}}

\def\regret{\mathrm{Regret}}
\def\top{\intercal}

\def\cP{\mathcal{P}}

\def\dist{\mathsf{D}}

\def\IC{\mathsf{IC}}
\def\<{\langle}
\def\>{\rangle}

\def\sfP{\mathsf{P}}

\def\cC{\mathcal{C}}
\def\cV{\mathcal{V}}

\def\cE{\mathcal{E}}
\def\cS{\mathcal{S}}
\def\cI{\mathcal{I}}
\def\cZ{\mathcal{Z}}

\def\cE{\mathcal{E}}

\def\cX{\mathcal{X}}
\def\cA{\mathcal{A}}
\def\cM{\mathcal{M}}

\def\cZ{\mathcal{Z}}
\def\cH{\mathcal{H}}
\def\cO{\mathcal{O}}
\def\cQ{\mathcal{Q}}

\def\BB{\mathbb{B}}

\def\EE{\mathbb{E}}

\def\NN{\mathbb{N}}

\def\PP{\mathbb{P}}

\def\RR{\mathbb{R}}
\def\SS{\mathbb{S}}

\def\ZZ{\mathbb{Z}}

\def\normal{{\mathsf{N}}}

\def\balpha{\boldsymbol{\alpha}}

\renewcommand{\P}{\mathbb{P}}
\newcommand{\E}{\mathbb{E}}
\newcommand{\R}{\mathbb{R}}

\newcommand{\eps}{\varepsilon}

\newcommand{\argmax}{\operatorname{argmax}}
\newcommand{\argmin}{\operatorname{argmin}}

\newcommand{\sign}{\operatorname{sign}}

\newcommand{\Unif}{\operatorname{Unif}}

\newcommand{\RN}[1]{%
  \textup{\uppercase\expandafter{\romannumeral#1}}%
}

\newcommand\iidsim{\stackrel{\mathclap{iid}}{\sim}}

\newcommand{\RNum}[1]{\uppercase\expandafter{\romannumeral #1\relax}}



\makeatletter
\newcommand*{\rom}[1]{\expandafter\@slowromancap\romannumeral #1@}
\makeatother
\algnewcommand{\LineComment}[1]{\Statex \(\triangleright\) #1}

\title{Learning to Lead: Incentivizing Strategic Agents in the Dark}
\author{
Yuchen Wu\thanks{School of Operations Research and Information Engineering, Cornell University;  \texttt{yw2867@cornell.edu} } 
\and  
Xinyi Zhong\thanks{Department of Statistics and Data Science, Yale University;  \texttt{xinyi.zhong.academic@outlook.com}} 
\and 
Zhuoran Yang\thanks{Department of Statistics and Data Science, Yale University;  \texttt{zhuoran.yang@yale.edu}} 
}
\date{}
\begin{document}
\maketitle

% \begin{abstract}

% We study an online generalized principal-agent model,
% where a principal and an agent repeatedly interact with each other, each seeking to maximize their individual utility. 
% Our goal is to design an online learning algorithm on behalf of the principal that minimizes her cumulative regret. 
% We consider a challenging setting where the agent possesses private types, takes unobservable actions, and receives private rewards. 
% The principal, on the other hand, only observes her own realized rewards and the agent's reported types, which may differ from the hidden true types. In addition, the agent is allowed to be non-myopic and may take strategic actions that are suboptimal in the present to mislead the principal’s learning process for future gains.

% To address these challenges, we propose a novel online learning algorithm based on optimistic planning that is subject to pessimistic constraints. 
% To estimate the agent's reward function, we introduce a sector test to recover the coordinates of unknown parameter vectors, and develop a matching procedure to assemble the outcomes of sector tests to reconstruct full parameter vectors. 
% Our algorithm achieves a regret bound of $\tilde{O}(\sqrt{T})$, offering the first near-optimal sample complexity upper bound in this setting.

% \end{abstract}

\begin{abstract}
    We study an online learning version of the generalized principal-agent model, where a principal interacts repeatedly with a strategic agent possessing private types, private rewards, and taking unobservable actions. The agent is non-myopic, optimizing a discounted sum of future rewards and may strategically misreport types to manipulate the principal's learning. The principal, observing only her own realized rewards and the agent's reported types, aims to learn an optimal coordination mechanism that minimizes strategic regret.
We develop the first provably sample-efficient algorithm for this challenging setting. Our approach features a novel pipeline that combines (i) a delaying mechanism to incentivize approximately myopic agent behavior, (ii) an innovative reward angle estimation framework that uses sector tests and a matching procedure to recover type-dependent reward functions, and (iii) a pessimistic-optimistic LinUCB algorithm that enables the principal to explore efficiently while respecting the agent's incentive constraints. We establish a near optimal $\tilde O(\sqrt{T}) $ regret bound for learning the principal's optimal policy, where $\tilde O(\cdot) $ omits logarithmic factors. Our results open up new avenues for designing robust online learning algorithms for a wide range of game-theoretic settings involving private types and strategic agents.
\end{abstract}

\tableofcontents

\section{Introduction}

The principal-agent model \citep{ross1973economic,grossman1992analysis, smith2004contract, laffont2009theory} is a fundamental framework for understanding decision-making processes with \emph{misaligned incentives} and \emph{information asymmetry}, with wide applications across various disciplines such as economics, finance, and computer science \citep{ratliff2018incentives,kamenica2012behavioral}. 
In this model, the principal represents an entity such as a service provider, a policy maker, or a firm, whose objective is to maximize certain system-level outcomes, such as revenue, social welfare, or efficiency. 
On the other hand, an agent, who could be a customer, an employee, or an individual participant, aims to optimize his utility based on his private preferences or information, which is not directly observable by the principal. 
To induce the optimal outcomes, the principal designs and commits to a mechanism, which could be a contract, an incentive scheme, or a policy, that aligns the agent's incentives with the principal's objectives.
The optimal mechanism and the agent's optimal strategy against it constitute the equilibrium of the principal-agent model, 
in certain settings also known as the Stackelberg equilibrium \citep{stackelberg1934marktform,von2010market}. 

While the principal-agent model has been extensively studied, most of the existing works focus on the characterization of the optimal mechanism and its computational aspects, 
assuming that the principal has certain knowledge of the environment. 
See, e.g., \cite{myerson1982optimal,kadan2017existence,gan2022optimal} and the references therein.
However, these analyses fall short in characterizing how the desired outcomes can be achieved through the interactions between the principal and the agent when such knowledge is absent, 
and the principal must learn the agent's preferences and behaviors through repeated interactions over time.

In this work, we study a generalized version of the principal-agent model \citep{myerson1982optimal} in an online learning setting, where a principal repeatedly interacts with a self-interested agent and aims to learn the optimal mechanism through these interactions. 
%In particular, we additionally allow two complications that are not present in the traditional principal-agent model. 
We allow for several complications that have not been jointly considered in the existing literature.
First, in each round, the agent has a random {private type} that is drawn from a prior distribution, both the type and the prior distribution are unknown to the principal. 
The private type of the agent encodes his private information, and affects the reward functions of both the principal and the agent. 
To learn information about the agent's private type, 
the principal asks the agent to \emph{report} his type in each round, 
while the agent's response might or might not agree with his true type. 
Second, the agent is strategic and forward-looking in the sense that he chooses actions and reported types to maximize his $\gamma$-discounted cumulative reward, where $\gamma \in (0, 1)$ is a discount factor unknown to the principal.
Moreover, as in the classical principal-agent model, the principal and agent have distinct reward functions. 
We consider a challenging setting where the principal has no knowledge of either function and, in each round, observes only her own realized reward. 
In particular, she has no access to the agent's reward or chosen action. 
Our objective is to design a low-regret learning algorithm on behalf of the principal, which generates a sequence of coordination mechanisms that interact with the agent for $T$ rounds, 
where $T$ itself may also be unknown to the principal.

Below, we summarize and discuss the key challenges of our setting: 
\begin{enumerate}
    \item[(i).] \emph{Private information and strategic response: } 
In our setting, the agent could manipulate both his actions and reported types to maximize his long-term cumulative reward. 
On the other hand, the principal does not have direct access to the agent's private types and actions, and might be misled by strategically manipulated reports. 
To address this challenge, 
the principal must design mechanisms that incentivize the agent to report myopically.
We comment that although the inclusion of a type-reporting step in principal-agent interactions increases analytical complexity, it may offer advantages to the principal by eliciting more information from the agent \citep{castiglioni2022bayesian,bernasconi2023optimal}.

\item[(ii).] \emph{Unknown reward functions: }
The principal lacks access to both the agent's reward function and his realized rewards.
As a result, she cannot directly infer the agent's preferences from data.
This restriction makes it difficult to understand how the agent responds to different coordination mechanisms, thus complicating the design of an optimal mechanism. 
Moreover, although the principal observes her own realized rewards, 
she lacks access to her underlying reward function and must infer it using the realized rewards in order to design an optimal mechanism. 
This limitation can arise even when the principal has a clear objective, due to the lack of an explicit functional description mapping outcomes to payoffs.
See \cite{dogan2023repeated, dogan2023estimating} for motivation behind considering unknown reward functions. 

\item[(iii)] \emph{Exploration-exploitation tradeoff:} 
Finally, as in other online learning problems, 
achieving a low regret requires the principal to strike a balance between exploring different coordination mechanisms to learn the reward function and the agent's type information, and exploiting the current best estimate of the agent's behavior to maximize her own reward.
The efficiency of such trade-off determines how quickly the principal converges to an optimal strategy.

\end{enumerate}

In this work, we address these challenges simultaneously.  
Specifically, we answer the following question:
\begin{center}
	\emph{Can we design a sample-efficient online learning algorithm for the generalized principal-agent model, when the agent is non-myopic and possesses private information? }
\end{center}

We provide an affirmative answer to this question by introducing a novel online learning algorithm with $\tilde \cO(\sqrt{T})$ regret, where $\tilde \cO(\cdot)$ suppresses constant and logarithmic factors in $T$. 
At a high level, our algorithm is a variant of the Linear Upper Confidence Bound (LinUCB) algorithm \citep{abbasi2011improved}, tailored to the principal-agent model setting.
Specifically, 
to update the principal's coordination mechanism in each round, 
our algorithm solves an \emph{optimistic planning problem} under a novel \emph{pessimistic constraint} that accounts for the agent's strategic behaviors. 
The pessimistic constraint is constructed based on estimates of the agent's reward function from past strategic interactions and quantifying the uncertainty in these estimates.
The construction of such a pessimistic constraint is specific to our model and plays a crucial role in achieving sublinear regret. 

To achieve sublinear regret, we must simultaneously address challenges (i)–(iii) outlined above. 
To this end, we propose a method that leverages a suite of novel algorithmic ideas designed to effectively tackle these challenges. 
Specifically, to mitigate the agent's strategic and forward-looking behaviors, 
we introduce a \emph{delaying scheme} that incentivizes the agent to act as if he were approximately myopic.
In addition, note that each agent's type in $\Theta$ is associated with a different reward function.
Instead of estimating each reward function separately, 
we propose to estimate the entire collection of $|\Theta|$ functions together.
Leveraging the structure of the agent's reward maximization problem, 
we represent these $|\Theta|$ reward functions as unit vectors on a sphere. 
The reward estimation task then reduces to recovering the spherical coordinates of these vectors, which we term the \emph{reward angles}.
These reward angles constitute a set of tractable parameters that can be efficiently inferred from the agent's strategic reports. 
To estimate these parameters in a sample-efficient manner, 
we introduce a novel detection subroutine called the \emph{sector test}, 
which identifies regions likely to contain the reward angles. 
Each sector test operates as a binary-search-style procedure that estimates a reward angle through iteratively partitioning the space of possible angles into sectors and determining the most probable sector based on the agent's reported types. 
The outputs of these sector tests are then assembled into full reward vectors via a matching algorithm. 
Combining these ingredients, we are able to estimate the agent's reward functions at an arbitrarily high precision with only a polylogarithmic number of samples. Finally, integrating these algorithmic ideas with the optimism principle \citep{auer2002finite, abbasi2011improved, bubeck2012regret} to address the exploration–exploitation tradeoff, we obtain a sample-efficient algorithm that achieves sublinear regret.

In summary, we propose a novel online learning algorithm on behalf of the principal for the generalized principal-agent model. Our approach simultaneously addresses several key challenges, including private information, non-myopic agents, and unknown reward functions. 
To the best of our knowledge, we offer the first provably sample-efficient algorithm for this setting.  

\subsection{Related works}

Online learning in principal-agent models has been extensively studied and is closely related to our work.
Among others, a substantial body of work is devoted to understanding and characterizing equilibrium strategies \citep{radner1981monitoring, rogerson1985repeated,spear1987repeated,abreu1990toward,plambeck2000performance}. 
These equilibrium strategies depend on the agent's private types and preferences, 
which in our setting are not directly observable and must be learned through repeated interactions.
\cite{conitzer2006} studies online learning in a principal-agent problem using both standard bandit algorithms and a gradient ascent method, under the assumption that the principal interacts with a new agent in each round and observes the agent's actions. 
\cite{dogan2023repeated, dogan2023estimating,scheid2024incentivized} consider settings of adverse selection where a myopic agent's preferences or reward functions are hidden from the principal, while their learning algorithms require access to the agent's actions. 
A line of work including \cite{vera2003structural,misra2005salesforce,misra2011structural,leezenios2012,gayle2015identifying,aswani2019data,kaynar2022estimating,mintz2023behavioral} studies the problem of estimating the agent's model in offline settings, where data are available a priori.

The generalized principal-agent model we consider was first introduced by \cite{myerson1982optimal} and has been further explored in \cite{kadan2017existence,gan2022optimal}.
These works focus on characterizing the optimal coordination mechanism in single-period static games, 
leaving open the question of how such mechanisms can be learned through repeated principal-agent interactions. 
We address this question in our work. 
As demonstrated in \cite{gan2022optimal}, the generalized principal-agent model subsumes many important special cases that have been studied separately in the literature, including contract design \citep{grossman1992analysis, smith2004contract, bolton2004contract}, information design \citep{kamenica2011bayesian, bergemann2019information, kamenica2019bayesian}, and Bayesian Stackelberg games \citep{stackelberg1934marktform, von2010market}. Online learning has been explored in each of these special cases, and the techniques developed therein share important similarities with our approach, as discussed below.

\vspace{5pt}
{\bf \noindent Online contract design.}
Contract design can be viewed as a special case of the principal-agent model, where the agent's reward is given by the sum of the cost associated with each action and the payment received from the principal.
The problem of learning an optimal contract in online settings has been studied in a flurry of recent works. See, e.g., \cite{ho2016adaptive, zhu2022sample,dogan2023repeated, dogan2023estimating,zuo2024new,chen2023learning, scheid2024incentivized, wu2024contractual} and the references therein. 
In particular, \cite{zuo2024new} studies linear contracts with binary outcomes and proposes an algorithm that achieves $\tilde{O}(\sqrt{T})$ regret.
\cite{ho2016adaptive, zhu2022sample} consider more general contract design problems by reducing them to bandit problems with potentially large action spaces, 
and develop algorithms that attain $\tilde{O}(T^{1 - \delta})$ regret, 
where $\delta > 0$ is a small constant depending on the complexity of the contract space.  \cite{zhu2022sample} also establishes a matching lower bound for their results.
In contrast, our algorithm achieves $\tilde{O}(\sqrt{T})$ regret by allowing the principal to obtain additional information from the agent through reported types. 
Furthermore, \cite{chen2023learning,wu2024contractual} propose methods based on binary search to estimate the agent's cost functions from data, which conceptually is similar to our sector test.
Their algorithms achieve $\tilde {\cO}(T^{2/3})$ regret for general models and $\tilde {\cO}(T^{1/2})$ regret in special settings, 
but rely on observing the agent's actions and assume no private types. 
In contrast, our method accommodates private agent types, which significantly complicates the learning task: instead of estimating a single parameter vector, we must collectively estimate a set of type-dependent vectors. 
This requires aggregating results from multiple sector tests using a matching procedure to reconstruct full reward vectors, adding an additional layer of complexity that is absent in the aforementioned works.

\vspace{5pt} 
{\bf \noindent Online information design.} 
Several recent works have investigated information design in online settings \citep{zu2021learning, wu2022sequential, bernasconi2022sequential, bernasconi2023optimal, chen2023persuading, feng2024rationality, bernasconi2024persuading,lin2024information,bacchiocchi2024online}. 
These works propose sample-efficient online algorithms that enable the sender to learn an approximately optimal signaling scheme. 
However, these algorithms typically assume that the sender observes the receiver's actions and the states of the environment, and are therefore not directly applicable to our setting. 

\vspace{5pt}
{\bf \noindent Online Stackelberg games.} 
Our work is also closely related to the line of research on learning Stackelberg game in online settings. 
See, e.g., \cite{letchford2009learning, balcan2015commitment, bai2021sample, zhao2023online, zhong2021can,haghtalab2022learning,  gan2023robust, chen2023actions, harris2024regret,balcan2025nearly} and the references therein. 
In contrast, we consider a more general framework that allows the principal to adopt personalized strategies for different agent's types.

%\begin{itemize}
%    \item Most of the computational works on principal-agent problems have focused on the basic setting
%in which the principal knows everything about the agent, i.e., they know both the probability
%distribution over outcomes and the cost associated with each agent’s action. 
%\end{itemize}

\subsection{Notation}
\label{sec:notation}
For $n \in \NN_+$, we define the set $[n] = \{1,2,\cdots, n\}$. For two sequences $\{x_n\}_{n \in \NN_+}$ and $\{y_n\}_{n \in \NN_+}$ of non-negative numbers, we use the notation $x_n \lesssim y_n$ to mean that there exists some universal constant $C > 0$ such that $x_n \leq C y_n$. 
For $\alpha \in \RR$, $m \in \RR_+$, we define $\cZ_m(\alpha) = \alpha + km$, where $k \in \ZZ$ is the unique integer satisfying $\alpha + km \in [0, m)$. 
For $\alpha, \beta \in \R$, we define 
\begin{align*}
	& \arc(\alpha, \beta) := \min \left\{ \cZ_{2\pi}(\beta - \alpha),  \, \cZ_{2\pi}(\alpha - \beta)\right\}.% \\
	%& \arc_0(\alpha, \beta) := \big| \alpha \, (\mbox{mod }\pi) - \beta \, (\mbox{mod }\pi) \big|. 
\end{align*}
Here $\arc(\cdot, \cdot)$ should be understood as measuring the arc length between two angles $\alpha$ and $\beta$ on the unit circle. %, while  $\arc_0(\cdot, \cdot)$ instead  measures the arc length on a semicircle. 
Throughout the paper, when we speak of angles we shall always consider it in the sense of modulo $2\pi$ if not specified otherwise.
For $x \in \RR^n$ and $1 \leq i < j \leq n$, we denote by $x_i$ the $i$-th coordinate of $x$ and $x_{i:j} \in \RR^{j - i + 1}$ the vector consisting of the $i$-th to the $j$-th entries of $x$. We say $x \geq 0$ if and only $x_i \geq 0$ for all $i \in [n]$. 
We use $a \vee b$ to represent the maximum of $a$ and $b$, and use $a \wedge b$ to represent the minimum of $a$ and $b$. For $d \in \NN_+$ and $R > 0$, we let $\SS^{d - 1}(R)$ be the $(d - 1)$-dimensional sphere in $\RR^d$ with radius $R$ that is centered at the origin. We write $\SS^{d - 1}(1) = \SS^{d - 1}$ for short. We also make the convention that $\SS^0(R) = \{\pm R\}$. 
For a set $\cC$, we let $\cP(\cC)$ denote the set of probability distributions over $\cC$.

\subsection{Roadmap}

The remainder of the paper is organized as follows.
We start by formulating the problem of online learning under a generalized principal-agent model in \cref{sec:problem}. 
We present a $\tilde{O}(\sqrt{T})$ regret upper bound and outline an algorithm that achieves this bound in \cref{sec:alg-pipeline}. 
We provide details on estimaitng the reward functions in \cref{sec:reward-function}, 
and conclude with a discussion of our contributions and potential directions for future work in \cref{sec:conclusion}.
\section{Problem formulation}
\label{sec:problem}

%\yw{We study online principal agent model with information asymmetry, with agent type... in particular... first sentence sell}

In this section, we formulate the problem of online learning under a {principal-agent model} with information asymmetry.
In words, our model specifies a dynamic game between two parties, namely, the \emph{principal (she)} and the \emph{agent (he)},
where the two players interact repeatedly and seek to maximize their own cumulative rewards. 
%We start by describing our problem formulation. 
%This can be viewed as a sequential game where the two players, namely, the \emph{principal (she)} and the \emph{agent (he)} repeatedly interact with each other and seek to maximize their own utilities. 
At the beginning of each round, the principal commits to an outcome-dependent acting rule, and dynamically adjusts her rule based on historical data. 
%The goal is to design rules that incentivize the agent to take actions that favor the principal's benefit. 
Taking the principal's perspective, our goal is to design game strategies that favor the principal's benefit. 

Our setting features two challenges: First, we do not assume the agent is myopic, and thus might act strategically to trick the principal in the current round for larger gains in the future.
We refer to \cite{edelman2007strategic} for empirical evidence of strategic players in online games. 
Second, the reward functions and values of the agent are kept private, thereby limiting the information available to the principal. These restrictions intertwine with each other and make the problem even more difficult to address.

\subsection{Generalized principal-agent model} \label{sec:generalized_PA}
We start by stating the one-shot version of our model, specifically referring to the \emph{generalized principal-agent model} \citep{myerson1982optimal,gan2022optimal}. 
%This model serves as a generalization of  the hidden-action principal-agent model  \citep{grossman1992analysis} in order to incorporate multiple agent's types.
%Such generalization enables the model to
This model captures a wide array of economic applications, including contract design \citep{grossman1992analysis,smith2004contract,bolton2004contract},   
information design \citep{kamenica2011bayesian, bergemann2019information,kamenica2019bayesian}, and Bayesian Stackelberg games \citep{stackelberg1934marktform,von2010market}.
Specifically, 
an instance of the generalized  principal-agent model is denoted by 
\begin{align}\label{eq:generalized_PA}
\mathscr{P} := (\Theta, \cX, \cA, \cO, f, U, V, \{F_{\theta, x, a}\}_{\theta \in \Theta, x \in \cX, a \in \cA}).
\end{align}
Here, we define $\Theta$ as the space of agent's types, $\cX$ as the space of principal's actions, $\cA$ as the space of agent's actions, and $\cO$ as the space of outcomes.
For ease of presentation, we assume that $\Theta$, $\cX$, and $\cA$  are all finite, while $\cO$ is a general measurable space. 
Without loss of generality, we write  $\Theta = \{1, 2, \cdots, |\Theta|\}$. 
We let $f \in \cP(\Theta)$ denote a probability distribution over the agent's type space $\Theta$, 
$U: \Theta  \times \cX \times \cA \times \cO \to \RR$  and  $V: \Theta  \times \cX \times \cA \times \cO \to \RR$ be the reward functions of the principal and the agent, respectively, and $ \{ F_{\theta, x, a} \in \cP(\cO) \}$  be a collection of outcome distributions over $\cO$. 
The generalized principal-agent model proceeds as follows. First, the private type $\theta$ of the agent is drawn from the type distribution $f$. Subsequently, given that the agent has a private type $\theta$ and reports a public type $\theta'$, the principal then takes action $x$, and the agent responds with a private action $a$. 
The two parties then receive rewards $U(\theta, x, a, o)$ and $V(\theta, x, a, o)$, respectively, where the outcome $o$ is sampled from the distribution $F_{\theta, x, a}$.
%First, the private type $\theta$ of the agent is drawn from the type distribution $f$, and the principal announces her policy. Then the agent strategically  reports a type (not necessarily $\theta$), and the principal takes action $x$ according to her announced policy. Next, the agent takes action $a$, which incurs an outcome sampled from the distribution $F_{\theta, x, a}$.
%The two parties then receive rewards $U(\theta, x, a, o)$ and $V(\theta, x, a, o)$, respectively.
%, and the outcome $o$ is sampled from $F_{\theta, x, a}$. 

\vspace{5pt}
{\noindent \bf Misaligned incentives and information asymmetry.} The generalized principal-agent model is a fundamental model for studying  decision-making processes with two key features  --- \emph{misaligned incentives} and \emph{information asymmetry} \citep{tsvetkov2014information,sargent2000recursive}. 
In particular, misaligned incentives means that the principal and agent have different reward functions, both seeking to maximize their individual gains. 
Information asymmetry, on the other hand, refers to the fact that the principal does not observe the agent's type and action, both remain private to the agent.
Games involving private information and uncontrollable actions face additional challenges, commonly referred to as the  \emph{adverse selection} and \emph{moral hazard} effects in the contract design literature
\citep{pauly1978overinsurance,cutler1998adverse, einav2013selection}.

\vspace{5pt}
{\bf \noindent Communication before game play.} From the principal's standpoint, our objective is to devise a mechanism (policy) that incentivizes the agent to take actions in  principal's favor, despite the presence of misaligned incentives and information asymmetry.
To achieve this goal, a \emph{communication stage} occurs between the principal and the agent before either party takes action. 
This stage involves two steps. 
In the first step, the principal  announces a \emph{public coordination mechanism} (policy) $\pi: \Theta  \times \cX \to [0,1]$ to the agent, where $\pi(\theta, \cdot) \in \cP(\cX)$ represents a probability distribution over the principal's action space $\cX$. 
By announcing $\pi$, the principal specifies how she will respond to each possible reported type from the agent.
In the second step, the agent reports a type $\theta' \in \Theta$ to the principal, which may or may not agree with his true type $\theta$.
After the communication stage, upon receiving the agent's reported type $\theta'$, the principal commits to the announced mechanism $\pi$ and selects a random action $x \in \cX$ following the distribution $\pi(\theta', \cdot)$. 
Note that here the principal substitutes the agent's reported type $\theta'$ into the announced mechanism $\pi$ to determine her action. 
After observing the principal's action $x$, the agent chooses an action $a \in \mathcal{A}$ in response. 
The action pair $(x, a)$ then triggers an outcome $o \in \cO$ that is sampled from $F_{\theta, x, a}$. 
Finally, the principal and the agent separately receive rewards $U(\theta, x, a, o)$ and $V(\theta, x, a, o)$, both depending on the true type $\theta$ of the agent. 
Throughout the interaction process, the agent's true type $\theta$, action $a$ and reward $V(\theta, x, a, o)$ are kept private from the principal. 
See Figure \ref{fig:Principal-agent-model}\footnote{Figure \ref{fig:Principal-agent-model} depicts the online version of the generalized principal-agent model. For an illustration of the one-shot model, we simply set $t = 1$. } for an illustration of the interaction process of the generalized principal-agent model.

\vspace{5pt}
{\bf \noindent Agent's problem --- best response to the principal.} 
Since the agent takes his action after observing the principal's action $x$, he may then selects an action to maximize his own reward. We refer to such action as his \emph{best response} in the current step. 
Specifically, if the true type of the agent is $\theta$, and the principal takes action $x$, then the expected reward of the agent when he takes action $a$ is given by 
\begin{align}\label{eq:reward-Vtxa}
    V(\theta, x, a) := \int V(\theta, x, a, o)\, \dd F_{\theta, x, a}(o). 
\end{align}
Here, the expectation is taken with respect to the randomness of the outcome $o$. 
A rational agent will choose the action that maximizes his expected reward, i.e., choosing an  action that maximizes $V(\theta, x, a)$. 
We define such a best response action  as 
\begin{align}\label{eq:best-response-action}
    a_{\theta, x} := \argmax_{a \in \cA} V(\theta, x, a). 
\end{align} 
When there is a tie, $a_{\theta, x}$ could be any action that attains the maximum.  
When the agent adopts the best response action $a_{\theta, x}$ in response to $(\theta, x)$, we can then compute the expected reward for both the principal and the agent, which are defined respectively as follows:
\begin{align} \label{eq:reward_func_br}
    V(\theta, x) := V(\theta, x, a_{\theta, x}), \qquad U(\theta, x) := \int U(\theta, x, a_{\theta, x}, o) \, \dd F_{\theta, x, a_{\theta, x}} (o).  
\end{align}
Here, $V(\theta, x)$ represents the expected reward of the agent, and $U(\theta, x)$ represents the expected reward of the principal, in a situation where the agent has type $\theta$, the principal takes action $x$, and the agent adopts the best response action.

Furthermore, upon receiving the principal's coordination mechanism $\pi$, the agent essentially gains control over the principal's action by potentially misreporting his type. 
Therefore, the agent's best response to the principal is hierarchical. 
Namely, the agent first determines the best reported type $\theta'$ to steer the principal's action distribution. Then, once the principal's action $x$  is revealed, the agent 
chooses the best response action $a_{\theta, x}$ according to his true type and the principal's action. 
To determine the best reported type, the agent solves the optimization problem 
\begin{align} \label{eq:report_type_br}
\max _{r \in \Theta} \sum_{x \in \cX} V(\theta, x) \pi (r, x).
\end{align}
Here, we take an expectation over the action distribution of the principal, which is induced by the agent's reported type $r$. 
A rational agent will report the solution of \eqref{eq:report_type_br} as his reported type $\theta'$ to the principal.

\vspace{5pt} 
{\bf \noindent Principal's problem --- bilevel optimization.} 
From the principal's perspective, 
assuming she possesses knowledge of the agent's reward function $V$ and type distribution $f$, and given that the agent is rational, she can then predict the best response action of the agent associated with every possible coordination mechanism $\pi$. 
Note that the principal does not have access to the agent's true type $\theta$, but only the type distribution. 
To find the optimal coordination mechanism, the principal solves the following bilevel optimization problem: 
\begin{align} \label{eq:OPT}
	\begin{split}
		  \mathop{\mathrm{maximize}}_{\Pi } ~~& \sum_{\theta \in \Theta} f(\theta) \sum_{x \in \cX} \pi (r_{\theta, \Pi}, x) U(\theta, x), \\
		% & \mbox{\hspace{1cm} where } r_{\theta, \Pi} = \argmax_{r \in \Theta}  \sum_{x \in \cX} V(\theta, x) \pi (r, x) \\ 
		\mbox{subject to }\,\,   &  r_{\theta,\Pi} \in \mathop{\mathrm{argmax}} _{r \in \Theta}\biggl\{  \sum_{x \in \cX} V(\theta, x) \pi (r, x)  \biggr\} ~~~~\qquad \forall \,\, \theta \in \Theta, \\
		& \,\,\,\,\,\,\,\,\,\,\, \sum_{x \in \cX} \pi(\theta, x) = 1, \qquad \pi(\theta, x) \geq 0, ~~ \qquad\forall \,\, \theta \in \Theta, \,\, x \in \cX.  
	\end{split}\tag{$\mbox{OPT}^{\ast}$}
	\end{align}
	%
	%\yw{consider instead the worst case? }
	%\yw{this is an assumption on the feasible policy, need to find $\Pi$ that induces $\theta$, necessary condition for learning. This is in fact a restriction on the algorithm. }
	%
Here, the 
	optimization is over the coordination mechanism $\Pi := (\pi(\theta, x))_{\theta \in \Theta, x \in \cX}$, and $r_{\theta, \Pi}$ denotes the optimal reported type of the agent given his true type $\theta$ and the announced mechanism $\Pi$. 
The first constraint itself is another optimization problem, meaning that the agent of type $\theta$ will always report the type  that leads to his maximum expected reward.  
% which  ensures that the agent of type $\theta$ will report his type to the principal in order to maximize his expected reward. 
	The second constraint ensures that the coordination mechanism $\pi(\theta, \cdot)$ is a probability distribution over the principal's action space for each $\theta$. 
As for the objective function, since the principal does not know the agent's true type, she evaluates her expected reward $\sum_{x \in \cX} \pi (r_{\theta, \Pi}, x) U(\theta, x)$ conditioning on every possible agent's true type $\theta$, 
and then takes the expectation over the type distribution~$f$.

It is important to note that the lower-level optimization problem in \eqref{eq:OPT} might have multiple solutions, i.e., there might exist ties. 
In this case, the agent  is free to choose any type within the tie to report.  
While such a choice will not change the agent's expected reward,  the principal’s expected reward is affected by the particular choice of tie-breaking.
In order to define the principal's maximum possible revenue in \eqref{eq:OPT}, we additionally require that the reported type $r_{\theta, \Pi}$ in \eqref{eq:OPT} is chosen in favor of the principal whenever there is a tie. 
This requirement aligns with conventions observed in the literature of Stackelberg games \citep{letchford2009learning, peng2019learning}.
We emphasize that this is not an assumption on the agent's behaviors, but rather a consequence of the principal's ability to select a coordination mechanism beforehand. 
Specifically, in \eqref{eq:OPT}, the principal can always narrow down the range of coordination mechanisms to ensure that a desired reported type $r_{\theta, \Pi}$ is strictly optimal by a small margin $\epsilon$. Further details can be found in \cite{letchford2009learning}.

Next, we state two regularity assumptions on the agent's reward function $V$. %in the context of the optimization problem  \eqref{eq:OPT}. %In words, we require that  the feasible region as a subset of the simplex has at least one ``interior point''.
 
% we impose an extra constraint on the feasible policy, which requires that the reported type $r_{\theta, \Pi}$ is unique for each $\theta \in \Theta$. This essentially  is a constraint imposed on the feasible coordination mechanism of the principal.  
%	We assume that such an optimization problem in \eqref{eq:OPT} is feasible when we restrict to  coordination mechanisms such that the reported type $r_{\theta, \Pi}$ is unique for each $\theta \in \Theta$.
	 
\begin{assumption}\label{assumption:feasible}
%		We assume that there exists a coordination mechanism  $\bar \pi  $ such that the first constraint in \eqref{eq:OPT} admits a  unique solution for each $\theta \in \Theta$.
  { We assume that there exist a coordination mechanism  $\bar \pi$ and a positive constant $\eps$, such that $\{\pi: \|\pi - \bar \pi\|_2 \leq \eps, \pi \geq 0, \langle\pi(\theta,\cdot ), \mathds{1}  \rangle = 1 \mbox{ for all }\theta \in \Theta \}$ is a subset of the feasible region of \eqref{eq:OPT}.}  
\end{assumption} 
This assumption specifies that the feasible set of \eqref{eq:OPT}, as a subset of an aggregation of $|\Theta|$ probability simplices over $\cX$, admits an interior point. 
Here, $\| \cdot \|_2$ is defined as the vector $\ell_2$-norm in $\RR^{|\Theta|\cdot |\cX|} $.

Inspecting the constraints of   \eqref{eq:OPT}, we note that the reward functions $(V(\theta, x))_{\theta \in \Theta, x \in \cX}$ appear in linear forms. 
Hence, the constraints (i.e., the feasible region) of \eqref{eq:OPT} remain invariant under scaling and shifting of the reward functions.
In the sequel, to simplify the presentation, we regard the agent's reward function as a collection of $|\Theta|$ vectors in $\RR^{|\cX|}$ with a unit norm, achieved through the following normalization.

\begin{definition}[Reward vectors]
	\label{def:normalized_vec}
	For ease of presentation, we  view the reward functions $V$ and $U$ defined in \eqref{eq:reward_func_br} as   collections of $|\Theta|$  vectors in $\RR^{|\cX|}$, dubbed the \emph{reward vectors}.
	Specifically, for each type $\theta \in \Theta$, we define  vectors $v_{\theta} := (V(\theta, x))_{x \in \cX} \in \RR^{|\cX|}$ and $u_{\theta} := (U(\theta, x))_{x \in \cX} \in \RR^{|\cX|}$.
Moreover, we define the normalized reward vectors for the agent:
	%
    % \begin{align}\label{eq:barv-theta}
    %     \bar{v}_{\theta} := (v_{\theta} - \langle  v_{\theta}, \mathds{1} _d\rangle / d) / \|v_{\theta} - \langle  v_{\theta}, \mathds{1}_d\rangle / d\|_2, \qquad \forall \theta \in \Theta. 
    % \end{align}
 \begin{align}\label{eq:barv-theta}
        \bar{v}_{\theta} := (v_{\theta} - \mathds{1}_d \cdot \langle  v_{\theta}, \mathds{1} _d\rangle / d) / \|v_{\theta} - \mathds{1}_d \cdot \langle  v_{\theta}, \mathds{1}_d\rangle / d\|_2, \qquad \forall \theta \in \Theta. 
    \end{align}
	Here, we write $d=|\cX|$ to simplify the notation, and $\mathds{1}_d$ stands for an all-one vector in $\RR^d$. 
\end{definition} 
We impose the following regularity conditions on the reward vectors. We immediately see that $\bar v_{\theta}$ is well-defined for all $\theta \in \Theta$ under Assumption \ref{assumption:not-all-one}. 
\begin{assumption}
\label{assumption:not-all-one}
    We impose the following assumptions on the agent's reward vectors: 
    \begin{enumerate}
        \item We assume that $v_{\theta}$ is not    parallel to  an all-one vector for all $\theta \in \Theta$. %That is, $\bar{v}_{\theta} \neq 0$ for all $\theta \in \Theta$. 
        \item The normalized reward vectors are distinct: There do not exist $\theta, \theta' \in \Theta$ and $\theta \neq \theta'$, such that $\bar{v}_{\theta} = \bar{v}_{\theta'}$.
        % , where $\bar{v}_{\theta} := (v_{\theta} - \langle  v_{\theta}, \mathds{1} _d\rangle / d) / \|v_{\theta} - \langle  v_{\theta}, \mathds{1}_d\rangle / d\|_2$. 
    \end{enumerate}
\end{assumption}

Assumption \ref{assumption:not-all-one} requires that the agent's reward functions do not degenerate. In particular, the first part of  Assumption \ref{assumption:not-all-one} ensures that the normalized reward vectors in \cref{eq:barv-theta} are well-defined. 
We remark that both Assumptions \ref{assumption:feasible} and  \ref{assumption:not-all-one} are standard regularity conditions commonly found in the literature on Stackelberg games \citep{letchford2009learning,peng2019learning}.

Moreover, we also impose the following regularity conditions on the model instance  $\mathscr{P}$\footnote{Recall this is defined in \cref{eq:generalized_PA}.}, which ensures that the reward functions are uniformly bounded, and the type distribution $f$ has a positive mass on each possible type $\theta$.  

\begin{assumption}\label{assumption:model}
	We assume that the principal-agent model $\mathscr{P}$ satisfies the following conditions:
	\begin{enumerate}
		\item Positive type distribution: For any $\theta\in \Theta$, we have   $f(\theta) > 0$. We define $f_{\min} := \min_{\theta \in \Theta} f(\theta)$. 
		\item Bounded reward: There exists a positive constant $B$, such that $\max\{\|U\|_{\infty}, \|V\|_{\infty} \}\leq B$.
        \item Agent's best action: For any $\theta \in \Theta$ and $x \in \cX$, we have $V(\theta, x, a_{\theta, x}) > \sup_{a \in \cA \backslash \{a_{\theta, x}\}} V(\theta, x, a)$, where we recall that $V(\theta, x, a)$ is defined in \cref{eq:reward-Vtxa}. 
	\end{enumerate} 
\end{assumption}

	%A direct consequence of Assumption \ref{assumption:feasible} 
	%is that $V (\theta, \cdot )$ given  in \eqref{eq:reward_func_br}, when viewed as a vector in $\RR^{|\cX|}$, can not be parallel to  an all-one vector for all $\theta \in \Theta$. 
	%Suppose there is $c \in \RR$ such that $V(\theta, x) = c$ for all $x \in \cX$ and some $\theta$. Then the maximizer of the function $\sum_{x \in \cX} V(\theta, x) \pi (r, x)$ is not unique despite the choice of $\pi$. 

We have now defined the one-shot generalized principal-agent model.
In the next section, we consider the online version of this model, characterized by sequential interactions between the principal and the agent.  A major distinction is that instead of assuming the principal has direct access to the model instance $\mathscr{P}$, in the online setting she needs to design an algorithm to learn the parameters and the optimal coordination mechanism based on observed data.

\subsection{The online setting of the generalized principal-agent model}
\label{sec:online-pa-model}

%We propose and investigate the \emph{online principal-agent model}.

In the online setting of the generalized principal-agent model, we assume that the principal and the agent of the game specified in Section \ref{sec:generalized_PA}  
interact in a sequential manner over $T$ rounds. 
We examine a challenging scenario for the principal, where the agent possesses complete knowledge, including the model parameters and the principal's algorithm, while the principal only knows a subset of this information. 
Specifically, before such a multi-round game starts, the principal first announces her algorithm $\mathscr{A}$ for computing the coordination mechanisms in each round, which subsequently becomes a part of the public knowledge. 
Then, for each $t\in [T]$, at the beginning of the $t$-th round, the agent observes his private type $\theta_t \in \Theta$ sampled independently from the type  distribution $f$. 
Within the $t$-th round, the principal and the agent engage in the following interaction process: 
(i) 
The principal  first announces a public coordination mechanism $\pi_t: \Theta  \times \cX \to [0,1]$ to the agent, which is generated by the announced algorithm $\mathscr{A}$ using historical data. (ii) The agent then reports a type $\theta_t' \in \Theta$ to the principal. 
(iii) Upon observing the agent's reported type $\theta_t'$, the principal plays an action $x_t \in \cX$ according to the distribution $\pi_t(\theta_t', \cdot)$. (iv) Then, the agent selects a private action $a_t \in \cA$ in response to the principal's action $x_t$. 
The action pair $(x_t, a_t)$ generates an  outcome $o_t \in \cO$ that is sampled from the distribution $F_{\theta_t, x_t, a_t}(\cdot)$. 
The outcome $o_t$ is observed by both the principal and the agent.
Finally, the principal and the agent separately receive their rewards $u_t = U(\theta_t, x_t, a_t, o_t)$ and $v_t = V(\theta_t, x_t, a_t, o_t)$, respectively. 
The goal of the principal is to design an algorithm $\mathscr{A}$ that generates $\{\pi_t\}_{t\in [T]}$ to maximize her expected cumulative reward over $T$ rounds, i.e., $\EE[ \sum_{t=1}^T u_t]$. 
An illustration of the principal-agent interacting process can be found as Figure \ref{fig:Principal-agent-model}.

\begin{figure}[ht]
	\centering
	\includegraphics[width=0.9\textwidth]{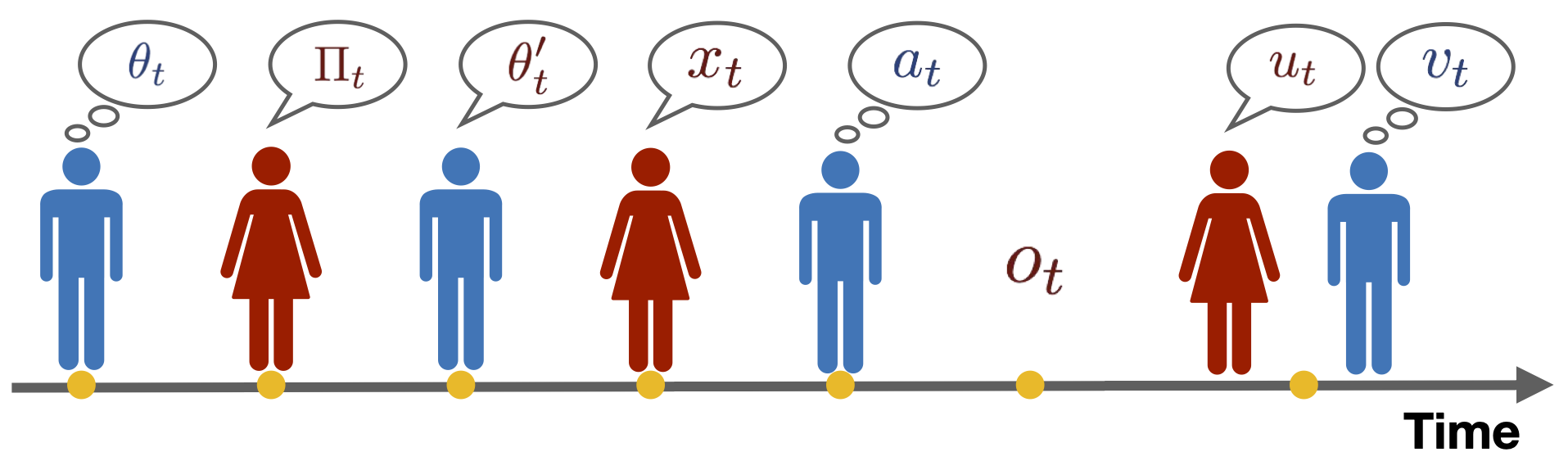}
	\caption{Illustration of the principal-agent interaction process in round $t$. Objects colored in blue are hidden from the principal. The agent on the other hand has access to all information. 
% In the above display, $U_t = U(\theta_t, x_t, a_t, o_t)$ and $V_t = V(\theta_t, x_t, a_t, o_t)$.
 }
    \label{fig:Principal-agent-model}
\end{figure}

%To summarize, an instance of the online principal-agent model is defined as 
%$$\cP := (\Theta, \cX, \cA, \cO, f, U, V, \{F_{\theta, x, a}\}_{\theta \in \Theta, x \in \cX, a \in \cA}).$$ 
%This consists of all the relevant spaces, a prior distribution of agent's type, the reward functions for two parties, and the outcome distributions conditioning on previous interactions. 
%Note that $\cP$ does not involve the principal or the agent's game strategy, and it is our task to design a strategy for the principal that favors her utility.   

%\yw{Put section 2.4 here}

\vspace{5pt} 
{\bf \noindent Strategic and discounted agent.}
In the static setting of the generalized principal-agent model as discussed in Section  \ref{sec:generalized_PA}, 
the agent's optimal strategy involves maximizing his expected reward based on the policy announced by the principal.
Adopting such strategy in the online setting means that, in each round, the agent reports a type and takes an action that maximizes his expected reward in the current round. 
Namely, in this case the agent is \emph{myopic}.
However, in the online setting characterized by sequential interactions, 
the agent's reported types and actions in the current round can impact his future rewards.  
As a consequence, the agent may opt for  strategic actions aimed at maximizing his cumulative rewards over the long term.
To capture such complicated forward-looking behavior, in the online setting we consider a \emph{strategic agent} who aims to maximize his $\gamma$-discounted cumulative reward for some discount factor $\gamma \in (0,1)$. 
In particular, before the beginning of the $t$-th round, principal has access to the historical data $\cH_t := \{\pi_\ell, \theta'_\ell, x_\ell, o_\ell, u_\ell: \ell \in [t - 1]\}$ and announces the coordination mechanism $\pi_t$ based on $\cH_t$. 
The agent then chooses his reported types and actions $\{\theta'_{t + \ell}, a_{t + \ell}\}_{\ell \geq 0}$ to maximize the following objective function: 
\begin{align}\label{eq:agent_objective_at_t}
	%\max_{\{ \theta'_{t+s}, a_{t+s} \}_{s\geq 0 } }
 \sum_{\ell = 0}^{T - t}  \E[\gamma^\ell \cdot V(\theta_{t + \ell}, x_{t + \ell}, a_{t + \ell}, o_{t + \ell}) \mid \cH_t, \theta_t, \pi_t  ].
\end{align} 
Here, the expectation is taken with respect to the randomness of everything in the future, including the randomness of the principal's future coordination mechanism, conditioning on  $\cH_t$, $\theta_t$, and $\pi_t$. In particular, the future mechanism $\pi_{t + \ell}$ is generated by the principal's algorithm $\mathscr{A} (\cH_{t+\ell})$, which itself involves future random variables that are not observed at round $t$. 
Moreover, the decision variables in \eqref{eq:agent_objective_at_t} involve all the reported types and agent actions starting from the $t$-th round, with $(\theta_{t + \ell}', a_{t + \ell})$ interpreted as functions of $(\cH_{t + \ell}, \theta_{t + \ell}, \pi_{t + \ell})$\footnote{i.e., optimization occurs over a space of functions.}. 
When this problem is solved, the agent will report the  type $\theta_t'$ that corresponds to an optimal solution. 
The principal then takes an action $x_t$ according to the announced mechanism $\pi_t(\theta_t', \cdot)$. 
Upon observing $x_t$, to determine the best action in response, the  agent solves the optimization problem \eqref{eq:agent_objective_at_t} again, but conditioning on $\{\cH_t , \theta_t, \pi_t, \theta_t', x_t\}$.  
We remark that the optimization problem \eqref{eq:agent_objective_at_t} implicitly depends on the principal's algorithm $\mathscr{A}$, i.e., the algorithm that the principal employs to update her coordination mechanisms.

%such a best response behavior of the agent in \eqref{eq:agent_objective_at_t} is with respect to the principal's algorithm $\mathscr{A}$, i.e., the method based on which the principal updates the coordination mechanisms. 

Meanwhile, the  discount factor $\gamma \in (0, 1)$ highlights the fact that the principal is more patient than the agent. For example, in the context of online advertising, sellers in general are willing to wait for long-term revenue, while buyers often wish to receive goods within a short period of time \citep{amin2013learning}.  Here, we assume the agent knows the value of the discount factor while the principal does not.

\vspace{5pt} 
{\bf \noindent Information asymmetry in the online setting.}   
Under the online generalized principal-agent model we just introduced, the principal and the agent have asymmetric access to information, commonly referred to as \emph{information asymmetry} in the literature \citep{tsvetkov2014information,sargent2000recursive}. We summarize the information available to the principal and the agent respectively as follows.

%We emphasize that under the assumptions of the online principal-agent model we proposed in Section \ref{sec:online-pa-model}, the principal and the agent have asymmetric access to the information. This phenomenon is often referred to as \emph{information asymmetry} in literature \citep{tsvetkov2014information,sargent2000recursive}. We summarize the information available to each party in this section to enhance clarity. 

%\yw{Sample $\pi_t$ from a distribution, distribution is a function of $\cH_t$. Algorithm map Ht to distribution over a set. }

\begin{itemize}
	\item [(i)] {\bf Information available to the principal.} 
	Regarding the problem instance $\mathscr{P}$, the principal only knows the spaces $\Theta, \cX, \cA$ and $\cO$.  
Moreover, right before the start of round $t \in [T]$, the principal has access to the historical data $\cH_t = \{\pi_\ell, \theta'_\ell, x_\ell, o_\ell, u_\ell: \ell \in [t - 1]\}$. Based on $\cH_t$, the principal then designs a mechanism $\pi_t$ using  $\mathscr{A}$. In particular, the algorithm $\mathscr{A}$ is a mapping from the historical data $\cH_t$ to a distribution $\mathscr{A}(\cH_t)$  over the mechanism space, and the principal announces a mechanism $\pi_t$ sampled from $\mathscr{A}(\cH_t)$. It is worth noting that the principal has access only to her own realized reward $u_\ell$. 
She does not know the agent's realized reward $v_\ell$ or the reward functions $U, V$. 
Finally, the principal is unaware of the total number of interactions $T$, and the discount factor $\gamma$. 
	
	\item [(ii)] {\bf Information available to the agent.} We assume the agent  knows everything regarding the problem instance $\mathscr{P}$, including the two reward functions. 
	He also has access to the full past observations $\bar\cH_t := \{\theta_\ell, \pi_\ell, \theta_\ell', x_\ell, a_\ell, o_\ell, v_\ell, u_\ell: \ell \in [t - 1]\}$ before the start of round $t$.
	In addition, he has knowledge of the principal's algorithm $\mathscr{A}$ that generates the mechanism.
	Namely, the agent is aware of the mapping that connects $\cH_t$ to $\mathscr{A}({\cH_t})$. 
    He additionally observes $\theta_t$ and $\pi_t$ during round $t$ before he takes any action in this round.  
    As a result, the agent is able to solve the problem in \eqref{eq:agent_objective_at_t} to maximize his $\gamma$-discounted reward. Finally, the agent knows the discount factor $\gamma$ and the number of interactions $T$  while the principal does not. 
\end{itemize}

In summary, in the online setting of the generalized principal-agent model, the agent has full knowledge of the underlying problem instance $\mathscr{P}$ and acts strategically to maximize his $\gamma$-discounted cumulative reward. 
The principal aims to maximize her expected cumulative reward over $T$ rounds in the presence of strategic interactions and information asymmetry. 
Our goal is to design an algorithm on behalf of the principal, such that her cumulative reward is maximized.  To assess the efficacy of the proposed algorithm, we will introduce the notation  of \emph{strategic regret} in the next section.

\subsection{Strategic regret and  myopic oracle benchmark}
\label{sec:oracle-benchmark}

Note that the types and actions reported by the agent as a solution to maximizing the objective function \eqref{eq:agent_objective_at_t} are contingent upon the principal's algorithm  $\mathscr{A}$.
Consequently, in order to calculate the maximum expected cumulative reward achievable by the principal, it is necessary to optimize over the space of all potential algorithms that the principal could employ. This task presents formidable complexity, and in most cases is both analytically and computationally intractable. 
To bypass this issue, we consider a more modest benchmark,  defined as the optimal reward of the principal under the assumption that the agent is \emph{myopic}.

Specifically, when the agent is myopic, then in each round $t$ he will strategically report a type and take an action  that maximize his expected reward in the current round. 
As the principal announces a mechanism $\pi_t$, the agent will report a type $\theta_t'$ via solving problem \eqref{eq:report_type_br}.
After observing the principal's action $x_t$,
he then takes a best response action $a_t$ according to   \eqref{eq:best-response-action}.
Recall that the bilevel optimization problem \eqref{eq:OPT} characterizes the principal's maximum expected reward in a single round assuming the agent is myoic and rationale.
We take the value of this optimization problem as the benchmark for the principal's learning algorithm. 
More precisely, we  denote by $u_{\ast}$ the value of  \eqref{eq:OPT}. 
We evaluate the performance of any principal's learning  algorithm via the notion of  \emph{pseudo-regret}, defined as 
\begin{align}\label{eq:p-regret}
	\regret(T) := T u_{\ast} -  \sum_{t = 1}^T \E\left[  U(\theta_t, x_t, a_t, o_t) \mid \cH_t  \right]. 
\end{align}
We note that the pseudo-regret we investigate here is slightly different from the realized regret (defined as $T u_{\ast} - \sum_{t = 1}^T u_t$). 
In particular, at each round, we consider the expectation of the realized regret associated with the current round conditioning on the past history. 
By Wald's equation the pseudo-regret and the realized regret have the same expectation \citep{audibert2009exploration}.
Hence, controlling the pseudo-regret automatically provides control for the realized regret in expectation. 
Moreover, the difference between these two regret notions is a sum of martingale differences. 
When  the reward function $U$ is upper bounded, by the Azuma-Hoeffding inequality, the regret difference is at most $\cO(\sqrt{T})$ with high probability.
Finally, we remark that \eqref{eq:OPT} is only used as an oracle benchmark to assess the performance of the principal's learning algorithm. During the online learning  process, the agent is not myopic and his strategic  behavior is characterized by~\eqref{eq:agent_objective_at_t}.

% As the randomness of $\theta_t$ and $o_t$ is uncontrollable from the principal's , the realized regret is a random variable.

% We remark that there is little purpose to deal with uncontrollable randomness caused by $\theta_t$ and $o_t$. Secondly,  Hence, controlling the pseudo-reward automatically provides control for the expected reward.  

\vspace{5pt} 
{\noindent \bf Truthful mechanism, revelation principle, and strategic regret.} 
Note that in the myopic oracle benchmark, agent's reported type $r_{\theta, \Pi}$ can be different from $\theta$, i.e., the agent may misreport his type. 
Next, we show that it is possible for the principal to consider a smaller class of mechanisms which ensures that $r_{\theta, \Pi} = \theta$ for all $\theta \in \Theta$. 
Such a mechanism is called a \emph{truthful} one.  
When we additionally require the mechanism $\Pi = (\pi(\theta, x)) _{\theta\in\Theta, x \in \cX}$ to be truthful in \eqref{eq:OPT}, the resulting optimization problem can be reformulated as a linear programming (LP) task:  
\begin{align}\label{eq:LP1}
\begin{split}
\mathop{\mathrm{maximize}}_{\Pi } 	\quad &  \sum_{\theta \in \Theta} f(\theta) \sum_{x \in \cX} \pi (\theta, x) U(\theta, x), \\
	\mbox{subject to} \quad &   \sum_{x \in \cX} V(\theta, x) \pi(\theta, x) \textcolor{black}{\geq} \sum_{x \in \cX} V(\theta, x) \pi(\theta', x), \qquad \forall \,\, \theta, \theta' \in \Theta, \,\,\theta \neq \theta', \\
	  &  \sum_{x \in \cX} \pi(\theta, x) = 1, \qquad \pi(\theta, x) \geq 0, \qquad \qquad\quad \forall \,\, \theta \in \Theta, \,\, x \in \cX.  
\end{split}\tag{$\mbox{LP}$}
\end{align}
%
%If a mechanism $\pi$ satisfies the constraints indicated in \eqref{eq:LP}, then we say this mechanism is \emph{incentive compatible}.
Here, \eqref{eq:LP1} differs from \eqref{eq:OPT} only in the first inequality constraint, which ensures that, when $\theta$ is the true type of the agent, truthfully reporting $\theta$ is no worse than reporting any other type $\theta'$, i.e., the agent has no incentive to misreport his type. 
We notice that the constraints of both \eqref{eq:OPT} and \eqref{eq:LP1} are  determined by the set of \emph{normalized reward vectors}.
From an algorithmic design perspective, the implication is that in order to estimate the feasible regions of \eqref{eq:OPT} and \eqref{eq:LP1}, it suffices to estimate $\{\bar{v}_{\theta}: \theta \in \Theta\}$ instead of $\{v_{\theta}: \theta \in \Theta\}$.

The class of truthful mechanisms constitutes a subset of all possible mechanisms.
Thus, a natural question is whether restricting to such a subset yields a strictly smaller optimal value for the principal.
We next establish the revelation principle \citep{myerson1981optimal, myerson1982optimal} for this problem, which asserts that the optimal values of the two optimization problems \eqref{eq:OPT} and \eqref{eq:LP1} coincide. 
In other words, it suffices to only consider truthful mechanisms in order to compute the myopic oracle benchmark. 
\begin{lemma}[Revelation principle]\label{lemma:revelation}
	Under Assumptions \ref{assumption:feasible} and \ref{assumption:not-all-one}, it holds that \eqref{eq:LP1} is also feasible, and there exists a coordination mechanism $\bar \pi$ and a positive constant $\delta$, such that $\{\pi: \|\pi - \bar \pi\|_2 \leq \delta, \pi \geq 0, \langle\pi(\theta,\cdot ), \mathds{1}  \rangle = 1 \mbox{ for all }\theta \in \Theta \}$ is a subset of the feasible region of \eqref{eq:LP1}.
    In addition, the two optimization problems \eqref{eq:LP1} and \eqref{eq:OPT} have the same optimal value.
\end{lemma} 
\begin{proof}[Proof of \cref{lemma:revelation}]
    See Appendix \ref{sec:proof-lemma:revelation} for a detailed proof.
\end{proof}
According to the revelation principle (Lemma \ref{lemma:revelation}), the optimal value $u_{\ast}$ in the definition of pseudo-regret (given in \eqref{eq:p-regret}) coincides with the optimal value of \eqref{eq:LP1}. Consequently, the pseudo-regret in \eqref{eq:p-regret} can be interpreted as the \emph{strategic regret}, denoting the difference between the principal's expected realized reward and what she would have achieved in expectation if she had employed the optimal strategy under the assumption that the agent acted truthfully.
%This notion is also closely related to the one-shot regret.
The notion of strategic regret was first introduced in \cite{amin2013learning}, and has since been frequently adopted as a benchmark in the presence of strategic players \citep{mohri2014optimal,golrezaei2019dynamic,kanoria2020dynamic}. Strategic regret serves as a simple while effective benchmark in the presence of complicated interactions. As we will see later, by incorporating sample delay into our algorithm, we can guarantee that the agent behaves approximately myopic, hence  justifying the adoption of strategic regret as the benchmark.

Moreover, using the normalized reward vectors introduced in \eqref{eq:barv-theta},  we can equivalently write \eqref{eq:LP1} as 
\begin{align}\label{eq:LP}
	\begin{split}
		\mathop{\mathrm{maximize}}_{\Pi } \quad & \sum_{\theta \in \Theta} f(\theta) \cdot \langle \Pi_{\theta}, u_{\theta} \rangle , \\
		\mbox{subject to}  \quad &  \langle \bar{v}_{\theta}, \Pi_{\theta} \rangle {\geq} \langle \bar{v}_{\theta}, \Pi_{\theta'} \rangle, \qquad \qquad\qquad ~~\forall \,\, \theta, \theta' \in \Theta, \,\,\theta \neq \theta', \\
		&   \langle \Pi_{\theta}, \mathds{1}_{d} \rangle  = 1, \qquad \pi(\theta, x) \geq 0, \qquad \forall \,\, \theta \in \Theta, \,\, x \in \cX.  
	\end{split}\tag{$\mbox{LP}^{\ast}$}
	\end{align}
	Here we write $\Pi_{\theta} = (\pi(\theta, x))_{x \in \cX }\in \RR^d$ for all $\theta \in \Theta$, and we recall that $d = |\cX|$. The inner product employed here denotes the standard Euclidean inner product in $\RR^d$.
The above formulation proves to be advantageous for presenting our results more effectively. Therefore, for the remainder of the paper, we adopt this formulation, which is equivalent to \eqref{eq:LP1} and has optimal value $u_{\ast}$.

We note that \eqref{eq:LP} is a linear programming with $\Pi$ being the decision variable. 
For \eqref{eq:LP}, observe that the feasible region and the objective function are determined by the agent's normalized reward vectors $\{\bar v_{\theta}: \theta \in \Theta\}$, the principal's reward vectors $\{u_{\theta}: \theta \in \Theta\}$, and the type distribution $f$, all three are unknown to the principal in our setting. 
{\color{black} It is worth noting that in a simplified scenario where $\{\bar{v}_{\theta}: \theta\in\Theta\}$ are known to the principal, the problem essentially reduces to a linear bandit problem with a compact action space\footnote{
{\color{black}
To see this, note that when the principal knows the agent's reward function, it means that she knows the feasible region of \eqref{eq:LP}. 
Suppose for now that the agent is myopic. 
By examining \eqref{eq:LP}, we see that at each round, the principal chooses her mechanism $\Pi_t$ from a feasible region determined by linear constraints. The expected reward she receives in that round is also a linear function of $\Pi_t$, hence the problem reduces to a linear bandit problem. Even if the agent is not myopic and aims to minimize the $\gamma$-discounted cumulative rewards in \eqref{eq:agent_objective_at_t}, the principal can adopt a delaying algorithm to induce approximately greedy behavior from the agent, as detailed in Section \ref{sec:delayed-observations}, thus reducing the problem to a linear bandit problem. 
} % 
}\citep{abbasi2011improved}.}
% It is worth noting that in a slightly simplified scenario where $\{\bar{v}_{\theta}: \theta\in\Theta\}$ are known to the principal, the problem essentially reduces to a linear bandit problem with a compact action space  \citep{abbasi2011improved}. 
% In particular, the coordination mechanism $\Pi$ announced by the principal corresponds to an ``action'', and the objective function $\sum_{\theta \in \Theta} f(\theta) \langle \Pi_{\theta}, u_{\theta} \rangle$ is linear in $\Pi$. 
In the linear bandit context, various effective approaches have been developed that come with nearly-optimal theoretical guarantees~\citep{audibert2009minimax,abbasi2011improved,bubeck2012regret,russo2014learning}.
However, we are facing a significantly more challenging scenario where the feasible region of \eqref{eq:LP} is also unknown and must be learned. 

The major contribution of our paper is the design of a principal's algorithm $\mathscr{A}$ that achieves sublinear regret with a square-root dependence on $T$. 
As detailed in the following sections, our algorithm comprises two steps: (i) we first estimate the feasible region of \eqref{eq:LP} based on the collected data, and (ii) we then run a classical linear bandit algorithm aiming to minimize the regret based on the estimated feasible region. 
The second step is relatively standard, 
and our main contribution is to develop a sample-efficient algorithm that accurately estimates the feasible region of \eqref{eq:LP}, which we state in Section \ref{sec:reward-function}.

Recall  that $d = |\cX|$. Note that if $d=1$, then the only feasible mechanism is a degenerate one: $\pi(\theta, x) = 1$ for all $\theta \in \Theta$. Hence, in this case the principal does not have any flexibility in choosing the mechanism to be implemented and the problem is automatically solved. As a result, $\regret(T) \equiv 0$ for all $T \in \NN_+$ when $d = 1$.   
In what follows, we focus on the non-degenerate case where $d \geq 2$.

To summarize, we consider an online setting of the generalized principal-agent model, where the agent acts strategically to maximize his $\gamma$-discounted cumulative reward as defined in \eqref{eq:agent_objective_at_t}, and the principal aims to maximize her expected cumulative rewards over $T$ rounds in the presence of strategic interactions and information asymmetry.
The principal's performance is assessed by the strategic regret in \eqref{eq:p-regret},  where the benchmark is specified by the optimal value of \eqref{eq:LP}. In particular, the principal has no knowledge of the reward functions of both the principal and the agent, and hence has to estimate them from the collected data.

\section{Algorithm pipeline and regret upper bound}
\label{sec:alg-pipeline}

In this section, we present the algorithm pipeline for $\mathscr{A}$ and an upper bound for the corresponding regret.
Note that if the reward vectors are given, 
% i.e., the feasible region of \eqref{eq:LP} is known, then 
minimizing the strategic regret defined in \eqref{eq:p-regret} reduces to solving a linear bandit problem, %
% To see this, suppose for a moment that the agent is myopic and his reward function is known to the principal. Suppose first that the agent is myopic. Inspecting \eqref{eq:LP}, we see that at each round, the principal chooses her mechanism $\Pi_t$ from a feasible region determined by linear constraints. The expected reward she receives as a consequence of $\Pi_t$ is also a linear function of $\Pi_t$, hence the problem reduces to a linear bandit one.
% Moreover, even when the agent is not myopic and aims to minimize the $\gamma$-discounted cumulative rewards given in \eqref{eq:agent_objective_at_t}, as we will show in Section \ref{sec:delayed-observations}, the principal can adopt a delaying algorithm to elicit approximately greedy behavior out of the agent, thus reducing the problem to a linear bandit one.
% },
for which many algorithms are known to achieve a square-root regret.
As we will show in the sequel, 
% when the feasible region of \eqref{eq:LP} is
{\color{black} when the reward vectors are}
unknown, we can devise an algorithm that achieves a $\tilde \cO(\sqrt{T})$-regret by actively estimating the normalized reward vectors, where $\tilde \cO(\cdot)$ hides a polylogarithmic factor.

Our algorithm involves a few intricate components designed to handle the strategic behavior of the agent. For ease of presentation,  we start with presenting an informal regret upper bound. 
\begin{theorem}[Informal]
\label{thm:informal}
    Suppose the interactions  between the principal and the agent lasts for $T$ rounds, where $T$ is a large positive integer, then there exists a principal's algorithm $\mathscr{A}$, such that the associated regret with high probability satisfies
\begin{align}\label{eq:regret_informal}
        \regret(T) \lesssim \sqrt{T} \cdot  \mathrm{polylog}(T), 
\end{align}
    where ``$\lesssim$'' hides a constant that depends only on the  parameters of the  model instance $\mathscr{P}$. 
\end{theorem}
A more rigorous statement of this theorem is deferred to \cref{thm:main}, where we also discuss the hidden constant in \eqref{eq:regret_informal}. 
To attain the above regret upper bound, we must address the following three challenges: 
(1) First, at each round, the reported type and the action taken by the agent are strategic,
with the objective of maximizing the $\gamma$-discounted cumulative reward \eqref{eq:agent_objective_at_t} that takes the future into account.
%according to \eqref{eq:agent_objective_at_t}, aiming at maximizing the $\gamma$-discounted cumulative rewards that takes the future into account. 
As a result, the principal needs to devise an algorithm that tames the strategic behavior of the agent in order to achieve a low regret. 
(2) Second, the feasible region of \eqref{eq:LP} is unknown. To achieve a low regret, the principal needs to estimate the normalized reward vectors in order to construct a mechanism that approximately solves \eqref{eq:LP}. Note that such an estimation problem differs from standard statistical estimation problems such as mean estimation or regression, as the realized rewards of the agent are private and not observable by the principal. 
(3) Third, in an online learning problem, achieving a low regret necessitates a balance between exploration and exploitation. 
To this end, 
in addition to constructing point estimates of the normalized reward vectors, more importantly, the principal needs to quantify the uncertainty of the estimates, and characterize how the error of estimating the constraint of \eqref{eq:LP} propagates to the regret.

As we will show in the sequel, our algorithm pipeline consists of several components that address the above challenges. 
Specifically, to resolve the first challenge,
we introduce a delaying algorithm that deliberately postpones the principal's incorporation of new observations into policy update. 
Thanks to the discount factor $\gamma$, with sufficient delay, the impact of agent's current round behavior  on his future revenue becomes negligible, and thus the agent is forced to take the action that approximately maximizes his gain in the present round.
To address the second challenge, we propose to estimate the normalized reward vectors in spherical coordinates, which we refer to as the \emph{reward angles}. We devise a novel algorithm that constructs a confidence set for the reward angles, whose details are deferred to \cref{sec:reward-function}. 
Finally, to tackle the third challenge, we propose to modify the optimism principle \citep{auer2002finite,sutton2018reinforcement} in the classical bandit literature with guidance from the pessimism principle \citep{jin2021pessimism}. We refer to the resulting algorithm as the \emph{pessimistic LinUCB}. 
In particular, when implementing the pessimistic LinUCB, we choose a coordination mechanism that  maximizes an optimistic estimate of the principal's reward function subject to a pessimistic estimate of the incentive compatibility constraint. 
Such a pessimistic modification is crucial for controlling the impact of the uncertainty that arises from estimating the constraints of \eqref{eq:LP} on the principal's regret.

% Second, the agent's reward is kept private and is not observable by the principal. 
% Therefore, acquiring an estimate of the feasible region of \eqref{eq:LP} is a crucial preliminary step before implementing any linear bandit algorithm.

{%\color{red} 
We establish a rigorous version of \cref{thm:informal} in the rest parts of this section. 
In \cref{sec:delayed-observations}, we present a delaying algorithm designed to tame the impact of the agent's strategic behavior.  
In \cref{sec:representation} we introduce the reformulation  of the normalized reward vectors into reward angles. 
We state in \cref{sec:complexity-est-angle} the algorithmic motivation and sample complexity for estimating the reward angles and constructing the corresponding confidence sets. 
In \cref{eq:pess_lin_ucb}, we describe the proposed pessimistic LinUCB algorithm and presents a regret upper bound. 
%In \cref{sec:pessimism}, we present a pessimistic estimate of the feasible region of \eqref{eq:LP}, derived from the constructed confidence set of the reward angles. 
%Additionally, in \cref{sec:episodic} we introduce the how to incorporate the doubling trick \cite{} to handle the unknown $T$ and in \cref{sec:delayed-linUCB} we introduce the pessimistic LinUCB algorithm. 
%Finally, in \Cref{sec:delayed-linUCB} we lay out the details the pessimistic LinUCB and establish its regret upper bound in \cref{sec:regret-analysis}.
}
  
% We summarize our algorithm pipeline in \cref{sec:episodic} that adopts an episodic structure. As an important component of our proposed algorithm, in \cref{sec:delayed-linUCB} we introduce the LinUCB algorithm with delayed feedback. Finally, we characterize the regret achieved by our algorithm in \cref{sec:regret-analysis}. 

\subsection{Delaying algorithm}
\label{sec:delayed-observations}

One crucial challenge in our setting arises from the difficulty in controlling the agent's strategic behavior, which requires considering agent's future rewards.  
Motivated by \cite{golrezaei2019dynamic}, 
we propose to untangle the complexity of the agent's strategic behavior via implementing a delaying algorithm. 
That is, when the principal designs and announces the coordination mechanism, she deliberately avoids incorporating the most recent observations.
Specifically, we define the notion of \emph{delay factor} as follows. 

\begin{definition}[Delay factor] \label{def:delay_factor}
Recall that for each $t \in [T]$, $\cH_t$ denotes the historical data accessible to the principal before the start of round $t$.
For any $t \in [T]$, we define the delay factor $\ell_t$ as the largest integer $i$, such that $\mathscr{A} (\cH_{t + i})$ only depends on $\cH_{t}$ but not $\cH_{t+1}$.
In other words, the $(t + \ell_t + 1)$-th round is the first round such that the observations generated in the $t$-th round are employed to design the principal's mechanism. 
That is to say, the information presented to the principal in the $t$-th round $\{ \pi_t, \theta_t', x_t, o_t, u_t\}$, only statistically affects $\pi_{t+\ell_t+1}$ and mechanisms that appear later. The mechanisms $\pi_{t+1}, \cdots, \pi_{t+\ell_t}$ on the other hand are independent of $\{ \pi_t, \theta_t', x_t, o_t, u_t\}$. 
	
\end{definition}

% suppose the principal chooses not to use observations from round $t$ and afterwards to construct her mechanism until round $t + \ell_t$, in the sense that
% %
% \begin{align}\label{eq:ell-t-sigma-alg}
%     %\pi_{t + i} \in \sigma (\cH_{t - 1}, \eps_{t + i}), \qquad i = 0, 1, \cdots, \ell_t - 1.
%     \mathscr{A} _{\cH_{t + i}} \mbox{ depends only on }\cH_{t - 1}, \qquad i = 0, 1, \cdots, \ell_t - 1.
% \end{align}
% %

A key implication of the delay factor is that, by choosing a sufficiently large $\ell_t$,  
the agent will approximately choose to report a type and take an action that maximizes his utility in the current round as if he were myopic, i.e., 
\begin{align}
\label{eq:thetatp-at-best-response}
	 \theta_t' \approx \argmax_{\theta' \in \Theta} \langle \pi_t(\theta', \cdot), v_{\theta_t} \rangle, \qquad a_t \approx \argmax_{a \in \cA}  V(\theta_t, x_t, a),  %\sum_{x \in \cX} \pi_t(\theta', x) \cdot \sup_{a \in \cA} \int V(\theta_t, x, a, o) \dd F_{\theta_t, x, a}(o), \\
    %& a_t \approx \argmax_{a \in \cA}  V(\theta_t, x_t, a),
\end{align}
where $\pi_t(\theta', \cdot) = (\pi_t(\theta', x))_{x \in \cX}$ is regarded as  a vector in $\RR^{|\cX|}$. 
The reason is that when the principal implements a delaying algorithm, the observations %coordination mechanism $\pi_t$ 
in the $t$-th round impacts the agent's discounted reward as described in \eqref{eq:agent_objective_at_t} in two ways -- directly through the immediate reward at round $t$ as $x_{t}$ is sampled from $\pi_t$, 
and indirectly through the future rewards starting from round $t+\ell_t+1$. 
The rewards between round $t+1$ and round $t+\ell_t$ are not affected by the agent's behavior in round $t$, %$\pi_t$, 
as by the definition of the delay factor they depend only on the historical data $\cH_t$ that is generated before round $t$. 
Hence, for a sufficiently large $\ell_t$, thanks to the discount factor $\gamma$, the impact of the future rewards on the agent's behavior becomes negligible, and the agent will approximately take the action that maximizes his immediate reward, i.e., being myopic.

%Let
%
%\begin{align*}
%	\cS_t = \left\{s \in \NN_+: s + \ell_s \geq t  \right\},
%\end{align*}
%
%which is the collection of accessible samples at round $t$. Namely, we assume 
%%
%\begin{align*}
%	\pi_t \in \sigma(\cH_{})
%\end{align*}
%Prior to formally stating our results, we find it useful to introduce several concepts that will be frequently applied in the subsequent analysis.
%With a slight abuse of notations we define
%

%
%For all $\theta \in \Theta$, we define 
%
%\begin{align*}
%	v_{\theta} := (V(\theta, x))_{x \in \cX} \in \RR^{|\cX|}, \qquad u_{\theta} = (U(\theta, x))_{x \in \cX} \in \RR^{|\cX|}
%\end{align*}
%
%as the agent's and the principal's utility vectors, respectively. In this paper, we might use the term reward vectors interchangeably. %In the above display, $U(\theta, x)$ is defined as follows: 
%
%\begin{align*}
%	a_{\theta, x} = \argmax_{a \in \cA} \int V(\theta, x, a, o) \,\dd F_a(o), \qquad U(\theta, x) = \int U(\theta, x, a_{\theta, x}, o) \, \dd F_{a_{\theta, x}} (o).
%\end{align*}
%
We make the above argument precise in \cref{lemma:best-response} below. Prior to stating the lemma, we introduce some notations that simplify the  presentation. 
We denote by $\Pi_t \in \RR^{|\Theta| \times |\cX|}$ the matrix representation of $\pi_t$. More precisely, we assume that the rows of $\Pi_t$ are indexed by $\Theta$ and the columns of $\Pi_t$ are indexed by $\cX$, with $(\Pi_t)_{\theta, x} = \pi_t(\theta, x)$.
In addition, we denote by $\Pi_{t, \theta} \in \RR^{d}$ the row of $\Pi_t$ that corresponds to index $\theta$, where we recall that $d = |\cX|$. 

%Suppose the principal chooses not to incorporate the observations from round $t$ into her mechanism design algorithm until round $t + \ell_t$ for some $\ell_t \in \NN_+$. 
%The following lemma is a straightforward consequence of the premise that the agent aims to maximize his $\gamma$-discounted reward, which we recall is defined in \cref{eq:agent_objective_at_t}: 
%
\begin{lemma}\label{lemma:best-response}
	 We assume that  the principal adopts a delaying algorithm with $\ell_t$ being the delay factor at the $t$-th round.
	  Then, under Assumptions \ref{assumption:feasible}, \ref{assumption:not-all-one}, and \ref{assumption:model},  if the agent reports a type $\theta_t'$ in round $t$, then $\theta_t'$ must satisfy the following  inequality:
% the agent might report $\theta_t'$ in round $t$ only if the following inequality is satisfied: 
	%
	\begin{align*}
		\langle \Pi_{t, \theta_t'}, \, v_{\theta_t} \rangle \geq \langle \Pi_{t, \theta_t^{\ast}}, \, v_{\theta_t} \rangle - \frac{2B \gamma^{\ell_t}}{1 - \gamma}. 
	\end{align*}
       Here  $\theta_t^{\ast} = \argmax_{\theta' \in \Theta} \langle \Pi_{t, \theta'}, v_{\theta_t} \rangle$.
        An equivalent formulation can be derived using the normalized reward vectors. Specifically, if the agent reports $\theta_t'$ in round $t$, then  
	\begin{align} \label{eq:approximately_myopic2}
	\langle \Pi_{t, \theta_t'}, \, \bar v_{\theta_t} \rangle \geq \langle \Pi_{t, \theta_t^{\ast}}, \, \bar v_{\theta_t} \rangle - \frac{C_0 \gamma^{\ell_t}}{1 - \gamma},
\end{align}
where %$C_0 \in (0, \infty)$ is a constant that is defined as  
$ C_0 = 2B / \min_{\theta \in \Theta } \{ \|v_{\theta} - \mathds{1}_d\cdot\langle \mathds{1}_d, v_{\theta} \rangle / d \|_2 \} $. Similarly, if the principal takes an action $x$ and the agent in response takes an action $a$, then we must have 
\begin{align*}
    V(\theta_t, x, a) \geq V(\theta_t, x, a_{\theta_t, x}) - \frac{2B \gamma^{\ell_t}}{1 - \gamma}.
\end{align*}
\end{lemma}
\begin{proof}[Proof of \cref{lemma:best-response}]
    See Appendix \ref{sec:proof-lemma:best-response} for a detailed  proof.
\end{proof}

This lemma is a straightforward consequence of the premise that the agent aims to maximize his $\gamma$-discounted reward, which we recall is defined in \eqref{eq:agent_objective_at_t}. One implication of the lemma is that when the delay factor $\ell_t$ is sufficiently large, the principal can incentivize the agent to be approximately myopic at round $t$. 
Besides, since we assume that the number of types, $|\Theta|$, is finite, and that the reward vectors are not parallel to $\mathds{1}_d$, we then see that $C_0$ in \eqref{eq:approximately_myopic2} is a finite constant that depends only on the model instance $\mathscr{P}$.

We next introduce two delaying methods that we adopt in our algorithm. 
Suppose we want to make sure that 
the observations from the $t$-th round do not affect $\mathscr{A}(\cH_{t+i})$ for $i \in [\ell]$, 
we can choose either of the following two techniques:

% postpone the utilization of observations from the $t$-th round until round $t + \ell + 1$, then there are two straightforward approaches we can consider: 
%
\begin{enumerate}
	\item[(a)] (Run a dummy mechanism) The principal announces a dummy mechanism $\mathds{1}_{|\Theta| \times d} / d$ in rounds $t + 1, \cdots, t + \ell$, where $\mathds{1}_{|\Theta| \times d}  $ is an all-one matrix. 
	  
    \item [(b)] (Follow the previous mechanism) The principal can repeatedly announce the coordination mechanism $\Pi_t$ from round $t$ in rounds $t + 1, \cdots, t + \ell$. 
   In this case, for any $i\in [\ell]$, $\mathscr{A}(\cH_{t+i})$ only depends on $\cH_t$.

\end{enumerate}

\subsection{Reformulating the reward vectors as angles}\label{sec:representation}

%\yw{put assumption 2.2 here. }
One additional challenge to address is that the the feasible region of \eqref{eq:LP} is unknown. 
To overcome this obstacle, as we have discussed in Section \ref{sec:problem}, it is sufficient to estimate the normalized reward vectors from data.
We propose to represent these normalized reward vectors using spherical coordinates as sets of angles, which we refer to as the \emph{reward angles}. 
In this subsection, we introduce such a correspondence and discuss the associated assumptions.

To begin with, note that the normalized reward vectors are all orthogonal to $\mathds{1}_d$.  We first reparameterize these vectors to get rid of such a linear constraint. To this end, we define the following sets:
%Our algorithm for estimating the normalized reward vectors benefits from a useful reformulation that we describe in this section. 
%We state the reward estimation algorithm and the associated sample complexit in \cref{sec:complexity-est-angle}. 
%To proceed, we introduce the following sets: 
%
\begin{align}\label{eq:Delta-sets}
\begin{split}
	 \Delta = \left\{ x \in \RR^d: x \geq 0, \,\, \langle x, \mathds{1}_d \rangle = 1 \right\}, \qquad \Delta_0 = \left\{x \in \RR^d: \langle x, \mathds{1}_d \rangle = 0,\, \|x  \|_2 = 1 \right\}. 
\end{split}
\end{align} 
By the definition of the normalized reward vectors, 
 we see that $\{\bar{v}_{\theta}: \theta \in \Theta \} \subseteq \Delta_0$. 
We also define 
\begin{align*}
	&\bar\Delta_0  :=   \big\{x \in \RR^d:    \langle x, \mathds{1}_d \rangle = 0,\, \|x\|_2 \leq  1 \big\},   \\
   & \BB^{d - 1}    := \left\{ x \in \RR^{d - 1}: \|x\|_2 \leq 1 \right\}, \qquad 
    \SS^{d - 2}    := \left\{ x \in \RR^{d - 1}: \|x\|_2 = 1 \right\}. 
\end{align*}
That is, $\bar\Delta_0 $ is the intersection of the unit ball and the hyperplane orthogonal to $\mathds{1}_d$, $\BB^{d - 1}$ is the unit ball in $\RR^{d - 1}$, and $\SS^{d - 2}$ is the unit sphere in $\RR^{d - 1}$. 
We claim without a proof that there exists an isometry $\varphi_0: \bar \Delta_0 \mapsto \BB^{d - 1}$ that preserves the Euclidean distance\footnote{In particular, one can write $\varphi_0(x) = Mx$, where $M \in \RR^{(d - 1) \times d}$ satisfies (1) $M$ has orthonormal rows, and (2) $M \mathds{1}_d = {0}_{d - 1}$.}.  %{\color{red} Maybe add a foot note with reference to a math book?} 
Namely, for all $x, y \in \bar\Delta_0$, it holds that
\begin{align} \label{eq:varphi_preserve_distance}
	\langle x, y \rangle = \langle \varphi_0(x), \varphi_0(y) \rangle. 
\end{align} 
In addition, $\varphi_0$ is also an isometry between $\Delta_0$ and $\SS^{d - 2}$ when restricted on $\Delta_0$.  
%We emphasize that $\varphi_0$ is a mapping chosen by the principal before any interaction starts, hence the principal.
Therefore, in order to estimate $\{\bar{v}_{\theta}: \theta \in \Theta\}$, we can equivalently  estimate   $\{\zeta_{\theta}: \theta \in \Theta\}$, where $\zeta_{\theta} := \varphi_0(\bar{v}_{\theta}) \in \SS^{d - 2}$. Note that $\{ \zeta_{\theta}:  \theta \in \Theta \}$ are $(d-1)$-dimensional vectors.

As vectors on $\SS^{d-2}$ only have $d-2$ degrees of freedom, it is more convenient to represent each $\zeta_{\theta}$ using $d-2$ free variables. 
To this end,  we represent $\{\zeta_{\theta}: \theta \in \Theta\}$ using  the  spherical coordinate system. 
Specifically, for any  $\alpha_1, \alpha_2, \cdots, \alpha_{d - 3} \in [0, \pi]$ and $\alpha_{d - 2} \in [0, 2\pi)$, we define a mapping $\rho \colon [0, \pi]^{d - 3} \times [0, 2\pi) \rightarrow \SS^{d-2}$ as 
\begin{align}\label{eq:rho}
	\rho(\alpha_1, \alpha_2, \cdots, \alpha_{d - 2})_i := \left\{ 
\begin{array}{ll}
	\prod_{j = 1}^{i - 1} \sin \alpha_j \cos \alpha_i, &  i \in [d - 2], \\
	\prod_{j = 1}^{d - 2} \sin \alpha_j, &  i = d - 1. 
\end{array}  \right.
\end{align}
In the above display, $\rho(\alpha_1, \alpha_2, \cdots, \alpha_{d - 2})_i$ represents the $i$-th coordinate of $\rho(\alpha_1, \alpha_2, \cdots, \alpha_{d - 2}) \in \RR^{d - 1}$.
Note that $\rho$ is a bijection between  $[0, \pi]^{d - 3} \times [0, 2\pi)$ and $\SS^{d - 2}$. 
Using the spherical  coordinate system representation, each $\zeta_{\theta}$ can be reformulated as 
\begin{align*}
	\zeta_{\theta} = \rho(\alpha_1^{\theta}, \alpha_2^{\theta}, \cdots, \alpha_{d - 2}^{\theta}),
\end{align*}
where $\alpha^{\theta}_i \in [0, \pi]$ for $i \in [d - 3]$ and $\alpha^{\theta}_{d - 2} \in [0, 2\pi)$. As there is a one-to-one correspondence between $\{\alpha_i^{\theta}: {i \in [d - 2],\, \theta \in \Theta}\}$ and the normalized reward vectors, we refer to them as the \emph{reward angles}. 
For example, when $d = 3$, we have $  \rho(\alpha_1) = (\cos \alpha_1, \sin \alpha_1)^\top \in \RR^2$, which parameterizes a point on the unit circle. 
When $d = 4$, we have 
$$\rho(\alpha_1, \alpha_2) = \big(\cos \alpha_1, \, \sin \alpha_1 \cdot \cos \alpha_2 ,\, \sin \alpha_1\cdot \sin \alpha_2 \big)^\top \in \RR^3, 
$$ which parameterizes a point on the unit sphere in $\RR^3$.
Thus, in order to estimate $\bar{v}_{\theta}$, it is equivalent to estimate $\alpha^{\theta} := (\alpha_1^{\theta}, \alpha_2^{\theta}, \cdots, \alpha_{d - 2}^{\theta}) \in [0, \pi]^{d - 3} \times [0, 2\pi)$ instead. 

Next, we impose standard regularity conditions on the reward angles. We later show that these conditions are satisfied with probability one if we randomly select the isometry $\varphi_0$.
\begin{assumption}\label{assumption:angle}
	Assume that $\varphi_0$ is chosen appropriately such that the following conditions are satisfied: 
	\begin{enumerate}
		\item For all $\theta \in \Theta$ and $i \in [d - 3]$, it holds that $\alpha_i^{\theta} \in (0, \pi / 2) \cup ( \pi / 2,  \pi)$.% and $\alpha^{\theta}_{d - 2} \in (0, 2\pi)$.
		\item For any $\theta, \theta' \in \Theta$ with $\theta \neq \theta'$, we have  $\alpha_i^{\theta} \neq \alpha_i^{\theta'}$ for all  $i \in [d - 2]$. % or $|\alpha_i^{\theta} - \pi / 2| = |\alpha_i^{\theta'} - \pi / 2|$.
		\item For all $\theta, \theta' \in \Theta$ with $\theta \neq \theta'$, we assume $\cZ_{2\pi}(\alpha_{d - 2}^{\theta}) \neq \cZ_{2\pi}(\alpha_{d - 2}^{\theta'} + \pi)$, where we recall that $\cZ_{2\pi}(\cdot)$ is defined in \cref{sec:notation}. 
		\item For all $i \in [d - 3]$ and $\theta, \theta' \in \Theta$ with $\theta \neq \theta'$, we assume that $|\alpha_i^{\theta} - \pi / 2| \neq |\alpha_i^{\theta'} - \pi / 2|$.
		%\yw{we might not need the fourth assumption}
	\end{enumerate}
\end{assumption}
{
%\color{blue} 
Here, the first condition is imposed to avoid the degenerate cases of the reward vectors. When $d = 3$, this amounts to excluding cases where the reward vectors are on the equator or the poles of the unit sphere. Conditions 2-4 further mandate that the reward angles cannot be diametrically opposite on any coordinate. 
See Figure \ref{fig:angles} for an illustration of these four cases. 
We note that these conditions are necessary because the principal does not know the agent's private type $\theta$, yet seeks to estimate all $(d-2)\cdot |\Theta|$ reward angles from the available data. Achieving this goal is unattainable without specific identifiability conditions. 
As we will see in the next lemma, Assumption \ref{assumption:angle} can be satisfied with probability one through a random rotation, if we additionally assume that the reward vectors are not antipodal to each other.}
% {\color{red} Here, the first condition is imposed to avoid the degenerate cases where the reward vectors are on the equator of the unit sphere.
% The second condition is imposed to avoid the degenerate cases where the there are two reward vectors with the same angles. 
% The third and fourth condition is rules out the the corner cases where two reward vectors have reward angles that are equidistant to $\pi/2$ or differ by $\pi$. This condition essentially makes sure that the reward angles cannot coincide or be symmetric under the spherical coordinate system. See Figure \ref{fig:angles} for an illustration of these four cases. 
% These conditions are  required because the private type $\theta$ is unknown to the principal, while she aims to estimate all the $(d-2)\cdot |\Theta|$ reward angles from data. Such a task is impossible without certain identifiability conditions. As we will see in the next lemma, Assumption \ref{assumption:angle} is can be satisfied with probability one through a random rotation if we additionally assume that the reward vectors are not opposite to each other.}

%{\bf \color{red} Draw some figures to show what we rule out -- especially cases 3 and 4.}

\begin{figure}
    \centering
    \vspace{-3cm}
    \includegraphics[width=\textwidth]{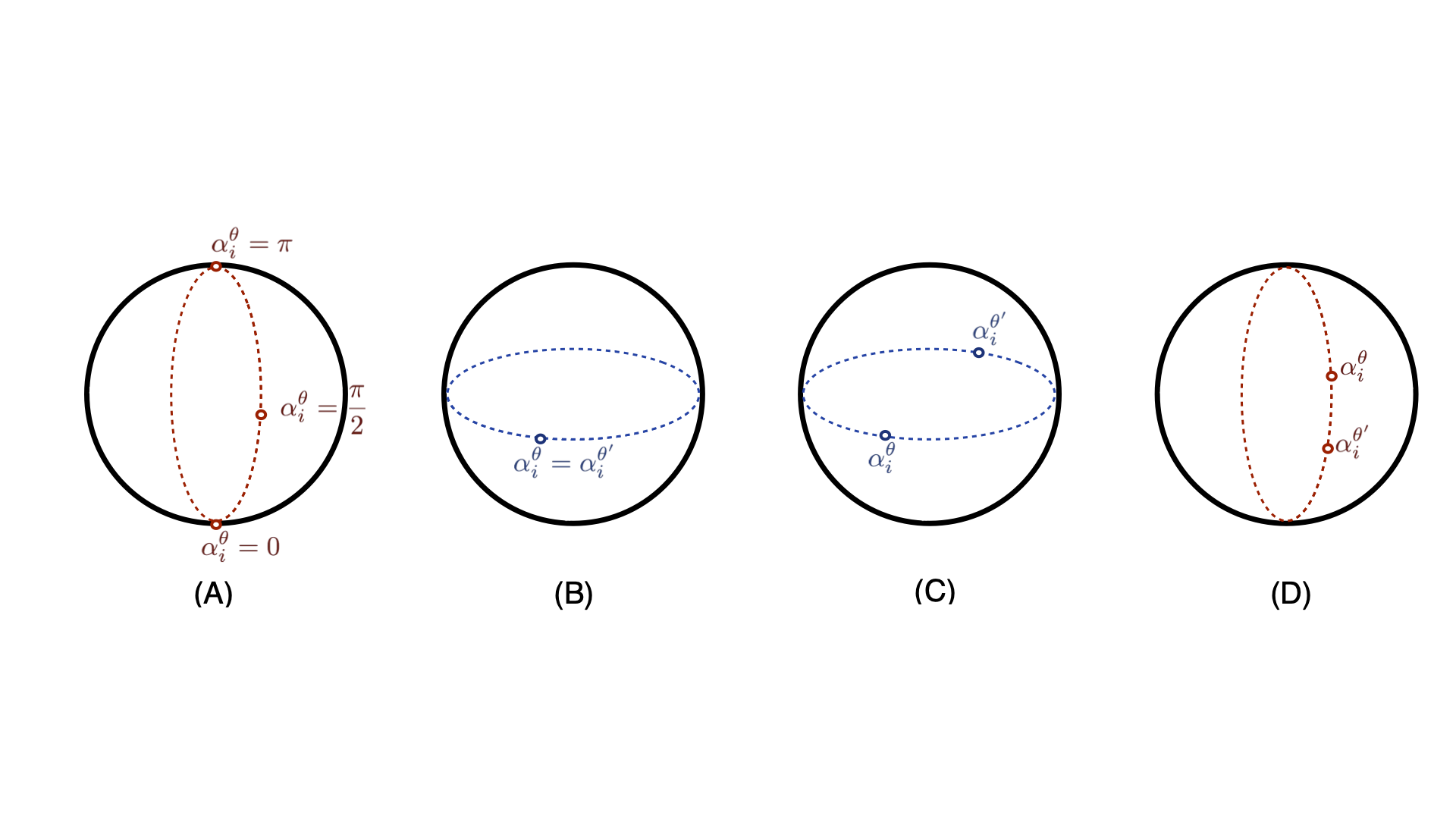}
    \vspace{-2.3cm}
    \caption{Illustrations of the angle configurations we exclude in Assumption \ref{assumption:angle}. Figure (A) corresponds to Condition 1: $\alpha_i^{\theta}$ cannot be one of $\{0, \pi / 2, \pi\}$. Figure (B) corresponds to Condition 2 and requires that the reward angles are distinct on every coordinate. Figure (C) corresponds to Condition 3, assuming the reward angles are not diametrically opposite on every coordinate. Condition 4 states that $\alpha_i^{\theta}$ and $\alpha_i^{\theta'}$ are not identical, and should not be posited as shown in Figure (D). As shown in Lemma \ref{lemma:pi}, these conditions are satisfied with probability one through a random rotation.}
    \label{fig:angles}
\end{figure}

%
%Through incorporating randomness into $\varphi_0$, we next show that the requirements listed in Assumption \ref{assumption:angle} will be satisfied with probability one if we additionally assume $\bar v_{\theta} \neq - \bar v_{\theta'}$ for $\theta \neq \theta'$. This claim is formally presented as \cref{lemma:pi} below. 
%
\begin{lemma}\label{lemma:pi}
	We assume    $\bar{v}_{\theta} \neq -\bar{v}_{\theta'}$ for all $\theta, \theta' \in \Theta$ with $\theta \neq \theta'$. Let $\varphi$ be a deterministic isometry that maps from $\bar\Delta_0$ to $\BB^{d - 1}$, and let $\Omega$ be an orthogonal matrix of size $(d - 1) \times (d - 1)$ that is uniformly sampled from the orthogonal group. We set $\varphi_0 = \Omega \, \circ \varphi$. Then under Assumptions \ref{assumption:feasible} and \ref{assumption:not-all-one}, the four conditions listed in  Assumption \ref{assumption:angle} are fulfilled with probability one.   
\end{lemma}
\begin{proof}[Proof of \cref{lemma:pi}]
    See Appendix \ref{sec:proof-lemma:pi} for a detailed proof.
\end{proof}

Here, the assumption that $\bar{v}_{\theta} \neq  - \bar{v}_{\theta'}$ for any $\theta, \theta' \in \Theta$ means that there do not exist two types of agents that have diametrically opposite normalized reward vectors. This assumption is required only for technical reasons.
It also rules out certain pathological cases. 

\subsection{Sample complexity for estimating the reward angles}
\label{sec:complexity-est-angle}

Following the discussions in \cref{sec:representation}, in order to estimate the feasible region of \eqref{eq:LP} (which we recall is characterized by the normalized reward vectors), it suffices to estimate the reward angles. 
We state in this section the algorithmic motivation for estimating the reward angles, leaving a detailed characterization of the algorithm to Section \ref{sec:reward-function}. 
We also outline the sample complexity achieved by this algorithm. 
At the end of this section, we use the estimated reward angles to construct an estimate of the feasible region using the pessimism principle. 

%To summarize the discussion in the previous subsections, a central challenge of the principal's learning problem is how to estimate the feasible region of \eqref{eq:LP}, which is characterized by the normalized reward vectors.
%By introducing the reward angles,  the problem boils down to estimating the reward angles from data and use them to construct an estimate of the feasible region.

% uniquely determines the feasible region of \eqref{eq:LP}.

% we propose to estimate the reward angles $\{\alpha^{\theta}: \theta \in \Theta\}$, which uniquely determine the feasible region of \eqref{eq:LP}.

% A central component of our algorithm is to estimate in a sample efficient manner the normalized reward vectors $\{\bar{v}_{\theta}: \theta \in \Theta\}$, which uniquely determines the feasible region of \eqref{eq:LP}. As per our discussions in Section \ref{sec:representation}, this task is equivalent to estimating the reward angles $\{\alpha^{\theta}: \theta \in \Theta\}$.  

A major contribution of this work is to design an algorithm that estimates the reward angles $\{\alpha^{\theta}: \theta \in \Theta\}$ in a sample-efficient manner.
In particular, we shall present an  algorithm that uses only $O(\mathsf{polylog}(n))$ samples and achieves $O(n^{-c})$ estimation error for some positive constant $c$, where $n$ we assume is a positive integer. 
Let us denote this algorithm by $\cA(n)$, whose details will be deferred to \cref{sec:reward-function}. 
We note that here $n$ is a parameter of this algorithm, which serves as a proxy for both the accuracy level and the sample complexity.
Thanks to the sample-efficiency of this algorithm, we are able to use a vanishingly small proportion of all samples to estimate the reward angles (i.e., the feasible region of \eqref{eq:LP}), and use the rest to solve a stochastic linear bandit problem derived from \eqref{eq:LP}. The $O(n^{-c})$ error rate is critical for achieving a $\tilde{O}(\sqrt{T})$ regret for the principal, as we will show in the proof of \cref{thm:main}. 

In the sequel, we first give an illustration of $\cA(n)$ in a simple case where $d = 3$ (i.e., $|\cX| = 3$) and the agent has only two possible types. 
Then we present the sample complexity guarantee for $\cA(n)$.

%Encouragingly, for the estimation of the reward angles, we are able to design a novel algorithm that uses only $O(\mathsf{polylog}(n))$ samples to achieve $O(n^{-c})$ estimation error for some positive constant $c$, where $n \in \NN_+$ serves as a proxy for both the accuracy level and the sample complexity. %Namely, for larger $n$, we require more samples and at the same time achieves higher accuracy. 
%We also propose an algorithm that attains such accuracy, and denote it by $\cA(n)$. %Our algorithm is indexed by $n$, and we denote it by $\cA_n$. It is worth emphasizing upfront that $\cA_n$ is only for estimating the normalized reward vectors.
%On a high level, the sample complexity achieved by $\cA(n)$ motivates us to use a vanishingly small proportion of samples to estimate the reward angles (i.e., the feasible region), and use the rest to solve a stochastic linear bandit problem. 
%An illustration of the algorithm pipeline is given as Figure \ref{fig:episodic}. 
%As we shall see in \cref{sec:regret-analysis}, this algorithm pipeline enables us to attain $\tilde{O}(\sqrt{T})$ regret in $T$ rounds, where $\tilde{O}$ hides both model-dependent constants and poly-logarithmic factors in $T$.

\vspace{5pt} 
{\noindent \bf Illustration of $\cA(n)$ in a simple case.} 
%To motivate, we present below an illustration of $\cA(n)$ when $d = 3$. 
When $d = 3$, there is only one reward angle $\alpha^{\theta} \in [0, 2\pi)$ for each type. 
Moreover, for the sake of simplicity  we assume $\Theta = \{1, 2\}$. We can visualize $\{\alpha^1, \alpha^2\}$ on a unit circle in $\RR^2$ using $(\cos \alpha^1, \sin \alpha^1)$ and $(\cos \alpha^2, \sin \alpha^2)$, as shown in panel (A) of Figure \ref{fig:example}.
%$\alpha^{\theta}$ reduces to a scalar in $[0, 2\pi)$ for all $\theta \in \Theta$. 
Per our discussions in \cref{sec:delayed-observations}, we can and will implement a delaying  algorithm to ensure that the agent approximately speaking takes the action that maximizes his reward in the current round, in the sense of  \eqref{eq:thetatp-at-best-response}.
For the clarity of exposition, in this part we assume that the agent is \emph{exactly myopic}, i.e., the agent strictly takes the action that maximizes his reward in the current round. 
We note that our algorithm $\cA(n)$ works without the myopic assumption by incorporating appropriate delays.
We make the myopic assumption here merely for the purpose of illustration.

Our algorithm proceeds by repeatedly performing binary searches, dubbed \emph{sector tests}. 
In particular, in each round, the principal picks two angles $\omega_1$ and $\omega_2$, 
and designs her mechanism by setting  $\Pi = (\Pi_{\theta} )_{\theta \in \Theta }$~as 
\begin{align}\label{eq:binary_Pi}
	 \Pi_1 = \mathds{1}_3 + \eps  \varphi_0^{-1} ((\cos \omega_1, \sin \omega_1)) \qquad \text{and}\qquad \Pi_2 = \mathds{1}_3 + \eps  \varphi_0^{-1} ((\cos \omega_2, \sin \omega_2)) ,
\end{align}
	 where  $\eps > 0$ is a small positive number. 
Note that we have 
\begin{align}\label{eq:inner_prod_pi_v}
\langle \Pi_j , \bar v_i  \rangle  & = \eps \cdot \langle \varphi_0^{-1} (( \cos \alpha ^i, \sin \alpha^i ) ),  \varphi_0^{-1} ((\cos \omega_j , \sin \omega_j )) \rangle \notag \\
 & = \eps \cdot ( \cos \alpha^i \cdot \cos \omega_j + \sin \alpha^i \cdot \sin \omega_j ) = \eps \cos (\alpha^i - \omega_j), \qquad i, j \in \{1,2\}. 
\end{align}
When the principal announces $\Pi = (\Pi_1, \Pi_2)$, the agent will report a type in $\{ 1, 2\}$  that maximizes his reward.
Therefore, suppose an agent with type $i$ reports type $j$, then by \eqref{eq:inner_prod_pi_v} we know that $\alpha^i$ is closer to $\omega_j$ than to $\omega_{|3 - j|}$ on the unit circle.
As a consequence, if the principal repeatedly employs such a mechanism that consists of rows $\Pi_1$ and $\Pi_2$, then after a sufficient number of rounds she will be able to deduce whether there exists an $\alpha^i$ that falls inside a neighborhood of $\omega_1$ and $\omega_2$.
In the next step, we further half the regions that contain the reward angles and apply sector test again. Through iterative repetition of this procedure, we can precisely pinpoint the reward angles with exponentially high accuracy.

% Hence, she is able to design a binary search style algorithm by carefully designing her choice of $(\omega_1, \omega_2)$. 

% As a consequence, if the principal repeatedly employs such  a mechanism that consists of rows $\Pi_1$ and $\Pi_2$, then after a sufficiently number of rounds she will be able to deduce whether there exists $\alpha^i$ that falls inside 
% $[\omega^j - \pi / 2, \omega^j + \pi / 2]$. Hence, she is able to design a binary search style algorithm by carefully designing her choice of $(\omega_1, \omega_2)$. 

To provide motivation, we illustrate the first two  iterations  of our algorithm as follows.  
A complete statement of the algorithm is deferred to \cref{sec:reward-function}. 
In the first step, the principal  chooses $\omega_1$ and $\omega_2$ that divide the circle into two semicircles, $\cI^{(1)}_1, \cI^{(1)}_2$.  
In  Figure \ref{fig:example}-(B), we take $\omega_2 = \omega_1 + \pi$ with $\omega_1 \sim \Unif[0, 2\pi)$, and the two blue arrows correspond to $\omega_1$ and $\omega_2$. 
When the principal announces such mechanism for a number of rounds, she observes that the agent always reports type 1. This is because $\omega_1$ is closer to both $\alpha^1$ and $\alpha^2$ than $\omega_2$.
Then the principal  confidently deduces that the semicircle containing $\omega_1$, namely $\cI^{(1)}_1$, contains both of the reward angles $\alpha^1$ and $\alpha^2$.
In the second step, she further divides $\cI^{(1)}_1$ into two pieces, denoted by  $\cI^{(2)}_1$ and $\cI^{(2)}_2$, and uses their centers as the new values of $\omega_1$ and $\omega_2$.  
The principal then announces such new $\Pi$ according to \eqref{eq:binary_Pi}, and observes the agent's responses. See Figure \ref{fig:example}-(C) for an illustration of this step, where $\omega_1$ and $\omega_2$ are represented using  green arrows. 
In this case, because $\alpha^1$ is closer to $\omega_2$ and $\alpha^2$ is closer to $\omega_1$, the agent will report type 2 when he is of type 1, and vice versa. 
The principal does not know the true types of the agent, but based on the reported types, she can deduce that there must contain one reward angle in both $\cI^{(2)}_1$ and $\cI^{(2)}_2$. 
Then, she can further divide both regions into two pieces, repeat this binary testing procedure for each region, and eventually accurately pin down the reward angles.
%{\color{red} MENTION INITIALIZATIOM OF $\omega_1$ and $\omega_2$? How do we choose $\omega_1$? Let's use "round" instead of "time" to be consistent with the rest of the paper.} \yuchen{randomly choose $\omega_1$, added}
%{\color{red} \bf Note that we should call the vector close to $\alpha^1$ as $\omega_2$} 

% To be precise, in panel (B) of Figure \ref{fig:example} the two blue arrows correspond to $\omega_1$ and $\omega_2$, which divide the circle into two semicircles. Implementing this mechanism for a number of times, the principal is able to confidently deduce that one semicircle contains at least one reward angle, while the other one does not. 
% Next, she divides the semicircle that contains reward angles into two quarter circles. She specifically designs her mechanism to achieve this, as shown by the  green arrows in panel (C). 
% Repeatedly employing this mechanism, she can conclude that both quarter circles contain one reward angle. Such  procedure of refining the region that contains reward angles continues. 

We comment that the above exposition is   a simplified illustration. 
In the general case where $d > 3$, each normalized reward vector contains $d - 2$ reward angles.
The sector test above is only able to estimate a single reward angle at one time. 
Thus, in the full-fledged version, the principal adopts a more sophisticated algorithm that estimates all the reward angles, with the sector test as a subroutine.
We postpone more details of $\cA(n)$ to \cref{sec:reward-function}. 
Thanks to the binary-search nature, this algorithm is sample-effcient in estimating the reward angles with exponentially high accuracy. 

%that the principal needs to carefully design the mechanism to attain an efficient binary search. In this example, we merely offer motivation. The algorithm statement is far from complete, and we defer a complete statement to Section \ref{sec:reward-function}. Thanks to the binary search nature, our algorithm offers an exponentially accurate estimate of the reward angles. 

\begin{figure}[ht]
    \centering
    \includegraphics[width=0.7\linewidth]{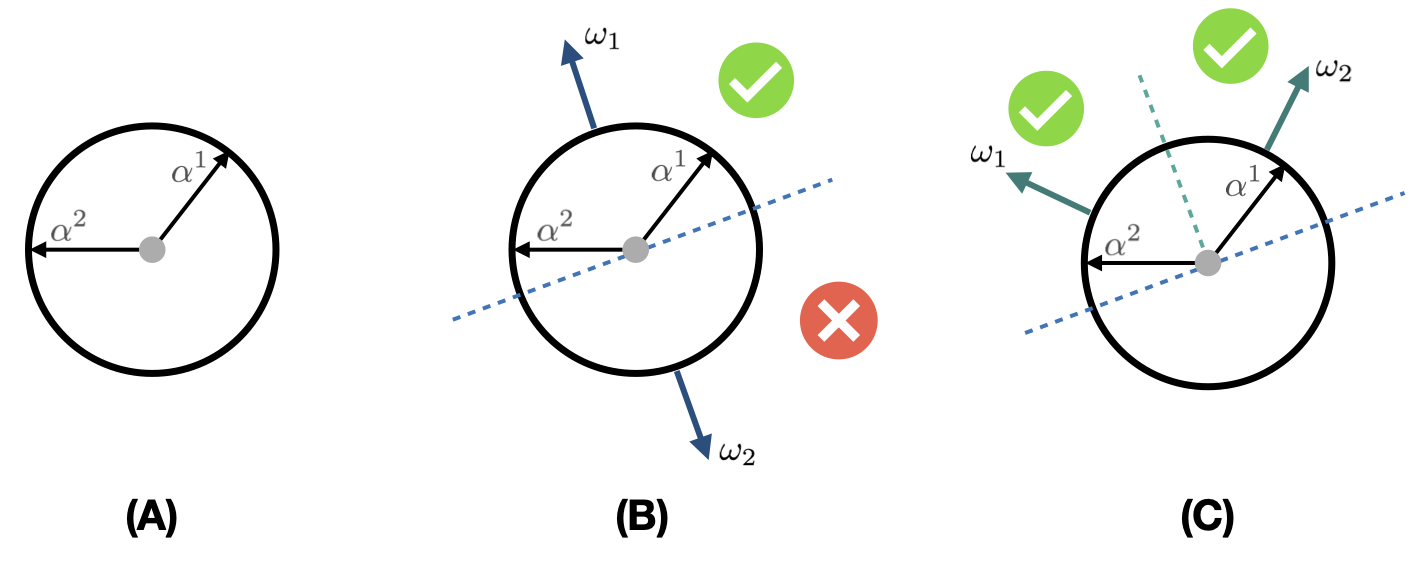}
    \caption{Illustration of $\cA(n)$ when $d = 3$ and $\Theta = \{1, 2\}$. 
    Note that when $d=3$ the reward angles are scalars, hence $\alpha^1, \alpha^2 \in [0, 2\pi)$. 
	Here, we visualize the reward angles $\alpha^1$ and $\alpha^2$ on a unit circle. 
    The principal chooses her mechanism (the blue and green arrows in panels (B) and (C)) to test whether there contains at least one angle in certain semicircle. }
    \label{fig:example}
\end{figure}

\vspace{5pt} 
{\noindent \bf Sample complexity of $\cA (n)$.} 
Next, we present  the sample complexity guarantee for $\cA(n)$. 
Let  $\balpha := \{\alpha^{ \theta}: \theta \in \Theta\}$ denote the set of \emph{true reward angles}. 
For any estimate $\hat\balpha := \{\hat\alpha^s: s  \in \Theta\}$ of $\balpha$, we measure the estimation accuracy via the following loss function:
%We define the set of angles $\balpha := \{\balpha^{\theta}: \theta \in \Theta\}$ and $\hat\balpha := \{\hat\balpha^s: s \in \Theta\}$, where $\hat\balpha^s := (\hat\alpha_i^s)_{i \in [d - 2]} \in \RR^{d - 2}$. Here, $\hat\balpha$ should be understood as an approximation to the true angle set $\balpha$.  
%We shall employ the following metric to measure the estimation error achieved by $\hat\balpha$:   
%
\begin{align}\label{eq:dist-D}
	\dist\left( \balpha, \hat\balpha \right) := \min_{\sfb \in \mathsf{Perm}(\Theta)} \sup_{s \in \Theta} \|\alpha^{\sfb(s )} - \hat\alpha^{s  }\|_1,
\end{align}
where $\mathsf{Perm}(\Theta)$ stands for the collection of all permutations over $\Theta$. 
Here, we utilize $s$ as the index for the estimated reward angles, highlighting that it should be viewed as a ``label'' rather than a type.
It is important to recognize that the principal lacks knowledge of the agent's true types, thus limiting her to estimating $\balpha$ collectively as a set. Namely, information-theoretically speaking, she cannot identify the type label of each component in $\hat \balpha$.
% Involving the minimization procedure over all permutations in \cref{eq:dist-D} is necessary, as information-theoretically speaking, we are only able to estimate the reward angles jointly as a set. 
% Namely, we are not able to tag the type labels. 
To see this, we can view the  
principal-agent interaction process that we investigate as follows: \emph{in each round, the principal presents $|\Theta|$ probability distributions to the agent (each distribution represented by one row of $\Pi$), and the agent selects the one that best aligns with his own reward function}. In this view, the types reported by the agent have no direct correlation with the true type, and serve merely as row indices of $\Pi$. 
%we simply apply an arbitrary permutation to the type labels. To see it from a different perspective, 
% we note that the principal-agent interaction process that we investigate can be viewed as follows: In each round, the principal presents $|\Theta|$ distributions to the agent (each distribution represented by one row of $\Pi$), and the agent selects  the one that best aligns with his own benefit. In this view, the types reported by the agent have no direct correlation with the true type, and serve merely as row labels of $\Pi$. %To see this point from a different perspective, we can simply permute the type labels, and observe that this does not lead to change in sample distribution. To be formal, we have the following lemma.  
In this view, the ordering of the $|\Theta|$ distributions presented by the principal does not matter.
Therefore, in \eqref{eq:dist-D}, we 
minimize over all permutations to find the best alignment between $\balpha$ and $\hat\balpha$.

We present the theoretical guarantee of $\cA(n)$ in the following theorem, which is an immediate consequence of \cref{lemma:matching} in \cref{sec:reward-function}.

%\yw{write a lemma about permutaion invariant. }

%\begin{lemma}
%\label{lemma:permutation-invariant}
%	We denote by $\sfb$ an arbitrary permutation over $\Theta$. Suppose we apply this permutation $\sfb$ to the type labels: We let $\bar{f}(\theta) := f(\sfb(\theta))$, $\bar U(\theta, \cdot, \cdot, \cdot) := U(\sfb(\theta), \cdot, \cdot, \cdot)$, $\bar V(\theta, \cdot, \cdot, \cdot) := V(\sfb(\theta), \cdot, \cdot, \cdot)$, $\bar F_{\theta, x, a} := F_{\sfb(\theta), x, a}$, and consider the problem $ \mathscr{P}' := (\Theta, \cX, \cA, \cO, \bar f, \bar U, \bar V, \{\bar F_{\theta, x, a}\}_{\theta \in \Theta, x \in \cX, a \in \cA})$. Then, for any principal's algorithm $(\mathscr{A}_{\cH_t})_{t \in \NN_+}$, the observations she gets have the same distribution under  $\mathscr{P}$ and $\mathscr{P}'$. 

%	We replace $f(\theta)$ with $f(\sfb(\theta))$, replace $U(\theta, \cdot, \cdot, \cdot)$ and $V(\theta, \cdot, \cdot, \cdot)$ with $U(\sfb(\theta), \cdot, \cdot, \cdot)$ and $V(\sfb(\theta), \cdot, \cdot, \cdot)$, and replace $F$
%\end{lemma}

%As we shall soon discover, accurate estimate of $\balpha$ is already sufficient for achieving square root regret.  

%We next present a theorem (\cref{thm:reward-function}) that upper bounds the sample complexity and error achieved by $\cA(n)$. For the compactness of presentation, in this section we only state the theoretical guarantee and provide motivation. We defer a detailed description of $\cA(n)$ to Section \ref{sec:reward-function}. %We  prove \cref{thm:reward-function} in Section \ref{sec:matching-alg}. 

%
\begin{theorem}\label{thm:reward-function}
	Let $\mathcal{A}(n)$ be the algorithm that estimates the reward angles $\balpha$ with parameter $n$. 
	Under  Assumptions \ref{assumption:feasible}, \ref{assumption:not-all-one},  \ref{assumption:model}, and  \ref{assumption:angle}, 
	there exist constants $n_0 \in \NN_+$ and $C_1, C_2, C_3 \in \RR_+$ that depend only on the problem instance $\mathscr{P}$ and the isometry $\varphi_0$, %normalized reward vectors and the isometry $(\{\bar{v}_{\theta}: \theta \in \Theta\}, \varphi_0)$, 
    such that for all $n \geq n_0$, with probability at least $1 - C_1 \cdot n^{-50}$, $\cA(n)$ outputs $\hat\balpha$ such that $ \dist(\balpha, \hat\balpha) \leq C_2 \cdot |\Theta|^{-1} \cdot n^{-50}$. Furthermore, $\cA(n)$ requires samples from no more than $C_3 \cdot (\log n)^{5}$ rounds. 
\end{theorem}

From \cref{thm:reward-function} we obtain that,  for all large enough $n \in \NN_+$, suppose the algorithm  $\cA(n)$ returns an estiamte  $\hat\balpha$ of $\balpha$, then  it holds with   probability at least $1 - C_1 n^{-50}$ that  %there exists a bijection $\sfb_n: \Theta \mapsto \Theta$ such that
\begin{align}\label{eq:bn}
	\sum_{s \in \Theta} \big\| \alpha^{\sfb_n^{\ast}(s)} - \hat{\alpha}^{s} \big\|_1\leq  C_2 n^{-50}, \qquad \text{where}\quad  \sfb_n^{\ast} = \argmin_{\sfb_n' \in \mathsf{Perm}(\Theta)} \sum_{s  \in \Theta} \big\| \alpha^{\sfb_n'(s)} - \hat{\alpha}^{s} \big\|_1.
\end{align}
Note that the bijection $\sfb_n^{\ast}\colon \Theta \rightarrow \Theta $ is a function of $\hat\balpha$ and hence is random. 
It is important to emphasize that  $\sfb_n^{\ast}$ is only used to define the error metric in \eqref{eq:bn}, and the algorithm does not require the knowledge of $\sfb_n^{\ast}$. 
Based on $\hat\balpha$, we can then construct a  confidence set of the normalized reward vectors as follows:
\begin{align} \label{eq:conf_set}
	\bar\cV_n := \biggl\{ \bar V \in \RR^{|\Theta| \times d}: \sum_{s \in \Theta}\|\bar V_{s}  - \varphi_0^{-1}(\rho(\hat{\alpha}^{s}))\|_2 \leq n^{-40},\, \bar V_{s } \in \Delta_0  \biggr\}, 
\end{align}
where $\bar V_s  \in \RR^d$ is the $s$-th row of matrix $\bar V$ and we recall that $\varphi_0, \rho$, and $ \Delta_0$ are introduced in \cref{sec:representation}. We also define the reward matrix that corresponds to the true reward function as $\bar{V}_n^{\ast} \in \RR^{|\Theta| \times d}$, where the $s$-th row of this matrix is set to $\bar{v}_{\sfb_n^{\ast}(s)}$.
We will establish in \cref{lemma:rho} in the appendix that $\rho$ is Lipschitz continuous, in the sense that $\|\rho(\alpha) - \rho(\alpha') \|_2 \leq   \|\alpha - \alpha'\|_1$ for all $\alpha, \alpha' \in  [0, \pi]^{d - 3} \times [0, 2\pi)$. 
Recall that $\varphi_0^{-1}$ preserves Euclidean distance, and $\varphi_0^{-1}(\rho (\alpha^{\sfb_n^{\ast}(s)})) = \bar v_{\sfb_n^{\ast}(s)}$.  
Combining \eqref{eq:bn} and \cref{lemma:rho}, we see that for $n$ large enough, with probability at least $1 - C_1 n^{-50}$ we have $\bar{V}_n^{\ast} \subseteq \bar{\cV}_n$. 
% By \eqref{eq:bn}, we have that $\bar{V}_n^{\ast} \in \bar\cV_n$ with probability at least $1 - C_1 n^{-50}$.
Note that  $\bar{V}_n^{\ast}$ is random because the permutation $\sfb_n^{\ast}$ is random. 
In other words, 
each element in $\bar\cV_n$ should be viewed as a set of $|\Theta|$ normalized reward vectors whose ordering does not matter.

\vspace{5pt}
{\bf \noindent Pessimistic estimate of the feasible region of \eqref{eq:LP}.} Based on the confidence set $\bar\cV_n$ in \eqref{eq:conf_set}, 
we can estimate the feasible region of \eqref{eq:LP}. 
We propose to employ the \emph{pessimism principle} to construct the feasible region. 
To this end, for 
any  $\bar V \in \RR^{|\Theta| \times d}$
that has rows belonging to $\Delta_0$, we let $\IC(\bar V)$ denote the set of incentive compatible mechanisms of the principal with respect to $\bar V$: 
\begin{align}\label{eq:IC_set}
	 \IC(\bar V) := &\left\{ \Pi \in \RR^{|\Theta| \times d}: \Pi \mbox{ satisfies the linear constraints listed below} \right\}. \\
	& \sum_{x \in \cX} \bar V_{s, x} \Pi_{s, x} \geq  \sum_{x \in \cX} \bar V_{s, x} \Pi_{s', x} + n^{-40}, \qquad \forall \,\, s, s' \in \Theta, \,\,s \neq s', \notag \\
	&  \sum_{x \in \cX} \Pi_{s, x} = 1, \qquad \Pi_{s, x} \geq 0, \qquad ~\qquad\qquad\forall \,\, s \in \Theta, \,\, x \in \cX .  \notag 
\end{align}
Here $\bar V_{s, x} \in \RR$ is an entry of $\bar V$ that corresponds to index $(s, x)$.
The definition of $\IC(\bar V)$ involves the parameter $n$, where $n^{-40}$ also appears in the definition of the confidence set $\bar \cV_n$ and characterizes the error for estimating $\bar V_n^*$. 
{ We include $n^{-40}$ in the constraint to ensure that truthful reporting is strictly optimal for the agents.} We let $\cI_n$ denote the intersection of incentive compatible mechanism sets for all $\bar V \in \bar{\cV}_n$, i.e., 
\begin{align}\label{eq:feasible-set}
	\cI_n := \bigcap\limits\limits_{\bar V \in \bar{\cV}_n} \IC (\bar V). 
\end{align}
Under the condition that $\bar V_n^* \in \bar \cV_n$, we observe that $\cI_n$ in \eqref{eq:feasible-set} is a \emph{pessimistic} estimate of the feasible region of \eqref{eq:LP}. 
Namely, $\cI_n$ is a subset of $\IC(\bar V_n^*)$. 
%To see this, for any $\Pi \in \cI_n$, we know that $\Pi$ satisfies \eqref{eq:IC_set} with respect to $\bar V_n^*$. 
Note that $\bar V_n^*$ is the same as the true normalized reward vectors $\{ \bar v_{\theta}:\theta \in \Theta\} $ up to a permutation $b_n$. Therefore, after a permutation, every $\Pi \in \IC(\bar V_n^*)$ is also feasible for  \eqref{eq:LP}. 
When $n$ is sufficiently large, we expect that every $\bar V$ in $\bar \cV_n$ is a good estimate of the true normalized reward vectors up to a permutation, and thus $\cI_n$ is a good estimate of the feasible region of \eqref{eq:LP}.

%we can construct a set of constraints for \eqref{eq:LP} based on it, and further derive the associated set of incentive compatible mechanisms. We denote this set by $\IC(\bar V)$, which we define below: 
%

% {\color{red} The implication of Theorem 3.2 is that we can construct confidence sets of the reward angles, and thus normalized reward vectors, and constraint of LP}

%To motivate the theorem, let us first consider $|\Theta| = 2$. Then in round $t$, the agent reporting one out of two types can be regarded as telling the principal which row of $\Pi_t$ his reward vector is more aligned with. This can be further regarded as dividing the reward vector space into two halves, and reveal to the principal that the true reward vector is in one of them. This then allows us to achieve polynomial accuracy for reward vector estimation with poly-logarithmic complexity. 

%\yw{Consider adding a discusssion / motivation? }
 
 %We denote the algorithm that achieves the above error bound by $\cA_n$. Namely, $\cA_n$ takes as input no more than $\mathsf{polylog}(n)$ samples and achieves $n^{-c}$ error. We describe $\cA_n$ later in \cref{sec:reward-function}. 

\subsection{Pessimistic-Optimistic LinUCB algorithm} \label{eq:pess_lin_ucb}

In the following, we present the details of the pessimistic-optimistic LinUCB algorithm for the principal's learning problem. 
We begin by considering the simpler scenario where the principal possesses prior knowledge of the number of rounds $T$. 
We further extend the algorithm to the case where $T$ is unknown using an additional doubling trick, which is deferred to Appendix \ref{sec:unknownT}. 
In particular, when $T$ is known, our algorithm is divided into two stages -- a {\color{PineGreen}\emph{constraint estimation}} stage and a {\color{PineGreen}\emph{pessimistic-optimistic planning}} stage.

To state our algorithm, we denote by $n \in \NN_+$ the largest positive integer such that $  (n + 2) \cdot (\lceil \log n \rceil^2 + 1) + \lceil\log n\rceil^6  \leq T$. 
%Here we use $n$ to drive the estimation accuracy.  
In the first stage, we apply the algorithm $\cA(n)$ to estimate the reward angles, and construct an estimated constraint set $\cI_n$ as in \eqref{eq:feasible-set}. 
As shown in Theorem \ref{thm:reward-function}, $\cA(n)$ uses samples from  no more than $C_3 \cdot (\log n)^6$ rounds and with 
 high probability outputs accurate estimates of the normalized reward vectors.
 Moreover, the resulting estimate $\cI_n$ is a subset of the feasible region of \eqref{eq:LP} up to a permutation. To simplify the notation, we let $T_1$ denote the number of rounds we spend in the first stage, which satisfies $T_1 \leq C_3 \cdot (\log n)^6$.  

 {In the second stage, we fix $\cI_n$ and use all samples from the remaining $T - T_1$ rounds to run a variant of the linear upper confidence bound (LinUCB) algorithm \citep{abbasi2011improved} tailored to the constrained optimization problem in \eqref{eq:LP}. 
We provide a succinct introduction to   LinUCB  in \cref{sec:classical-linUCB} in the appendix. 
The goal of the second stage is to 
 maximize the principal's cumulative reward in the online setting, based on the estimated constraint $\cI_n$.  
Recall that the principal should adopt a delaying algorithm as introduced in \cref{sec:delayed-observations} to mitigate the agent's strategic behavior. 
Moreover, to achieve a small regret, 
the principal should tradeoff exploration with exploitation, and properly leverage the estimated constraint.
 As a consequence, we propose a new variant of LinUCB that  features (i) the deliberate incorporation of {\color{PineGreen} delayed feedbacks} that force the agent to be approximately myopic, and  (ii) {\color{PineGreen} pessimistic-optimistic planning} that maximizes an upper confidence bound of the objective function under the pessimistic constraint $\cI_n$. }
We defer the details of the first stage, namely, the details of $\cA(n)$ to Section \ref{sec:reward-function}. 
This section delves into the details of the second stage.

\vspace{5pt}
{\noindent \bf 
Details of the second stage with delayed feedbacks.}  
Recall that we introduce two delaying methods in \cref{sec:delayed-observations} to deal with the non-myopic agent. 
We use both methods in the second stage. Specifically, 
the $T-T_1$ rounds in the second stage can be divided into  three groups as follows:
\begin{itemize}
	\item [(i)]
Starting from the $(T_1+1)$-th round, the first group consists of $\ell = \lceil \log n \rceil^2  $ rounds, and we run a dummy mechanism for rounds in this group to create delayed feedbacks. 
In this period, we control the agent's strategic behavior against the principal's mechanism used in the first $T_1$ rounds.   
% \item [(ii)] The second group consists of $n$ consecutive blocks of $(\ell + 1)$ rounds, and thus this group has $n \cdot (\ell + 1) $ rounds in total. In the first round of each block, we update the principal's coordination mechanism via pessimistic-optimistic planning, and then deploy this mechanism throughout this block. Thus, we use the second technique to create delayed feedback in each block. Moreover, only the first round of each block is used for the mechanism updates.  
\item [(ii)]
  The second group consists of $n$ consecutive blocks of $\ell+1$ rounds each, totaling $n \cdot (\ell + 1) $ rounds. In the first round of each block, we update the principal's coordination mechanism via pessimistic-optimistic planning, and then deploy this mechanism throughout this block. This ensures that the principal delays the feedback from the agent until the next mechanism update. Moreover, the principal only uses the feedback from the first round in each block to update her coordination mechanism. We denote the rounds where the principal collects data and updates the mechanism by $\cS_n = \{ T_1 + \ell + 1, T_1 + 2 (\ell + 1), \ldots , T_1 + n  \cdot  (\ell+1)\}$. Note that $|\cS_n| = n.$

\item [(iii)] The last group consists of the remaining $(T-T_2)$ rounds for $T_2 = T_1 + \ell + n(\ell + 1)$.  We can implement an arbitrary mechanism in these rounds. For example, we can repeat the current mechanism.
% e.g., a dummy mechanism. 
The size of this group is at most $\mathrm{polylog}(n)$ and thus does not affect the regret much.
\end{itemize}

% This construction means that we only perform $n$ mechanism updates in total. 
% The round indices of these mechanism updates are $\cS_n = \{ T_1 + \ell + 1, T_1 + 2 (\ell + 1), \ldots , T_1 + n  \cdot  (\ell+1)\}$. 
%{\color{orange} XZ deleted a sentence here.}
See Figure \ref{fig:pessimistic_optmistic_planning} for an illustration of the structure of the $T$ rounds.

\begin{figure}[ht]
    \centering
    \includegraphics[width=0.85\linewidth]{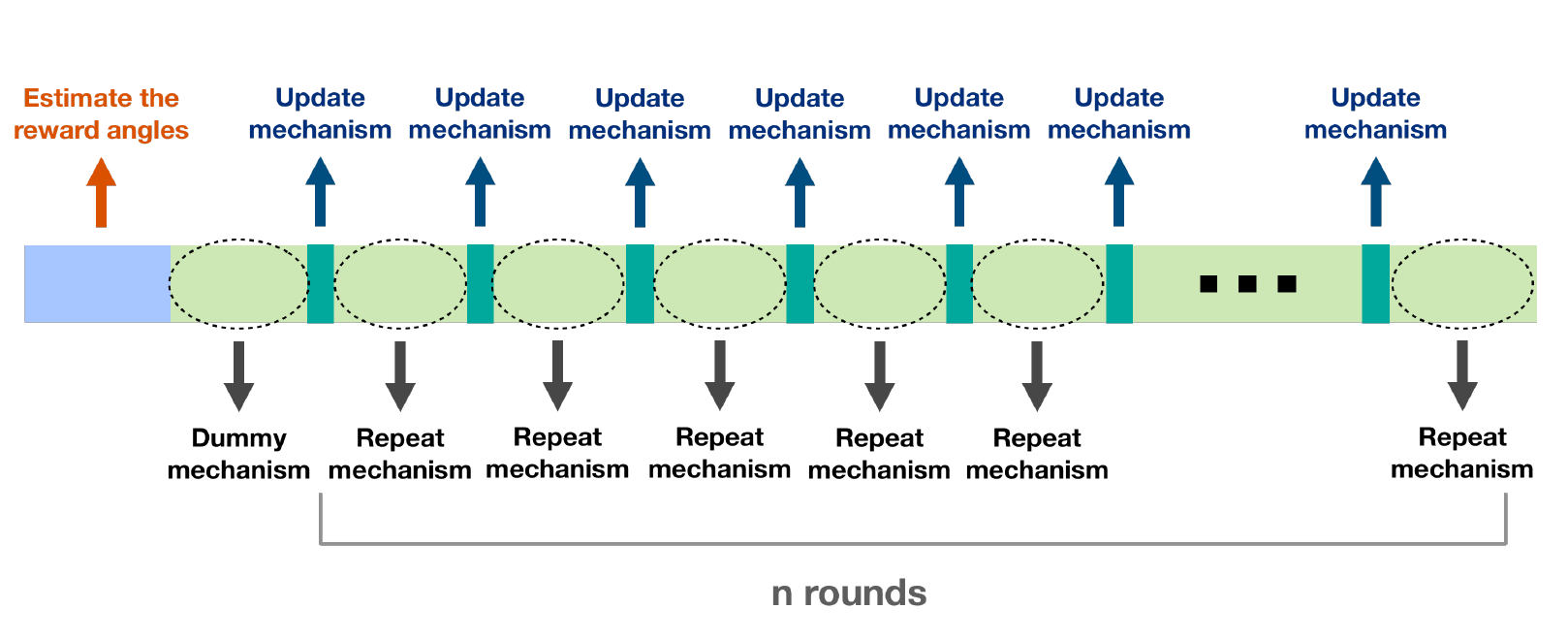}
    \caption{Illustration of the pessimistic-optimistic planning stage. The light blue rounds are for estimating the reward angles. The light green rounds implement the delaying algorithms, and in the dark green rounds the principal updates the implemented mechanism. }
    \label{fig:pessimistic_optmistic_planning}
\end{figure}

%We have already discussed the algorithm that we employ in the first stage at the end of \cref{sec:complexity-est-angle}. 
%In the sequel, we introduce the second stage of the methodology.   

%In the second stage, we use all the rest of the samples and run a variant of the LinUCB algorithm based on the feasible region estimate that we have derived in the first stage. As discussed in \cref{sec:delayed-observations}, throughout the process we shall incorporate a delay mechanism to deal with the non-myopic agent and robustify the interaction process. 

\vspace{5pt}
{\noindent \bf Estimating the principal's reward  via ridge regression.} Similar to vanilla LinUCB, we need to estimate the principal's reward function  using online data. 
Notice that by summing over all types, the objective function in \eqref{eq:LP} is invariant to the permutation of types. 
Thus, when fixing a particular permutation $\sfb_n^{\ast}$ induced by $\cA(n)$, we can equivalently reformulate \eqref{eq:LP} as
\begin{align}
	\mbox{maximize} \quad &  J(\Pi) := \sum_{s \in \Theta} \sum_{x \in \cX } f(\sfb_n^{\ast}(s)) \cdot \Pi_{s, x} \cdot U( \sfb_n^{\ast}(s), x ) = \langle \beta^* , \Vec(\Pi) \rangle, \label {eq:permute_obj} \\
\begin{split}
	\mbox{subject to} \quad & \langle \bar{v}_{\sfb_n^{\ast}(s)}, \Pi_{s} \rangle {\geq} \langle \bar{v}_{\sfb_n^{\ast}(s)}, \Pi_{s'} \rangle, \qquad \qquad\qquad ~~\forall \,\, s, s' \in \Theta, \,\,s \neq s', \\
		&      \langle \Pi_{s}, \mathds{1}_{d} \rangle  = 1, \qquad \pi(s, x) \geq 0, \qquad \forall \,\, s \in \Theta, \,\, x \in \cX.  
\end{split}
\end{align}
%
%\begin{align}\label {eq:permute_obj}
%J(\Pi) \colon = \sum_{s \in \Theta} \sum_{x \in \cX } f(\sfb_n(s)) \cdot \Pi_{\theta, x } \cdot U( \sfb_n(s), x ) = \langle \beta^* , \Vec(\Pi) \rangle ,
%\end{align} 
In the above display, $\beta ^* \in \RR^{d \cdot |\Theta|}$ is defined by setting  $\beta^*_{s, x } = f(\sfb_n^{\ast}(s))  \cdot  U( \sfb_n^{\ast}(s), x ) $ for all $(s, x ) \in \Theta \times \cX$, and $\Vec(\Pi)$ vectorizes $\Pi$ into a vector in $\RR^{d \cdot |\Theta|}$.
Thus, the coefficient $\beta^*$ of the linear objective function $J(\Pi)$ encodes both information about the unknown type distribution $f$ and the principal's unknown reward function $U$.  

%Thus, the unknown type distribution and the principal's reward function are both encoded in the parameter $\beta^*$ for the linear objective function $J(\Pi)$. 

Furthermore, if $\bar V_n^{\ast} \in \bar \cV_n$, then during the online learning process, 
when the principal announces a mechanism $\Pi  \in \cI_n$ and employs sufficient delay, by the definition in \eqref{eq:IC_set}, the agent with type $\theta$ will report a type that maximizes his reward, which, by \eqref{eq:bn},  is equal to 
$s = (\sfb_n^{\ast})^{-1} (\theta)$ with a sufficiently large $n$. 
Then, the principal takes an action $x \sim \Pi_{s}$.
Again thanks to sufficient delay, the agent in response will take a best-responding action $a_{\theta, x} $ defined in \eqref{eq:best-response-action}.
The action $a_{\theta, x} $ then triggers a random observation $o$,
and the principal receives  a reward $u = U(\theta, x, a_{\theta, x}, o)$.
Note that the principal does not observe $\theta$, which follows the type distribution $f$. 
Thus, conditioning on $\Pi \in \cI_n$, the principal's expected reward is equal to
\begin{align*}
	\E [ u \,|\, \Pi ] = \E [ U(\theta, x) \mid \Pi] = \sum_{\theta\in \Theta} \sum_{x \in \cX} f(\theta) \cdot U  (\theta, x)  \cdot \Pi_{(\sfb_n^{\ast})^{-1} (\theta), x}   = \langle \beta^* , \Vec(\Pi) \rangle,
\end{align*}
where the expectation is taken with respect to the randomness of $\theta \sim f$,  $x \sim \Pi_{s}$ and $o \sim F_{\theta, x, a_{\theta, x}}$ with $s = (\sfb_n^{\ast})^{-1} (\theta)$.
Therefore, when $\cI_n$ is fixed, we can estimate $\beta^*$ via performing linear regression using the observed pair $(\Pi, u)$. 

Specifically, for any $k \in [n]$, to perform the $k$-th update, i.e., computing $\Pi_{T_1 + k(\ell+1) }  $, 
we use a dataset $\{ (\Pi_{\ell}, u_{\ell} ) \}_{\ell \in \cS_{k-1} } $ to estimate the parameter $\beta^*$ by ridge regression: 
\begin{align}\label{eq:ridge}
\hat \beta_{k} = \biggl( \sum_{\ell \in \cS _{k-1} } \Vec(\Pi_{\ell})  \Vec(\Pi_{\ell})^\top + \lambda \cdot I_{d|\Theta|} \biggr)^{-1} \bigg( \sum_{\ell \in \cS _{k-1}  } \Vec(\Pi_{\ell}) \cdot  u_{\ell}
  \bigg), 
\end{align}
where $\cS_{k}$ contains the first $k$ elements of $\cS_n$. %\footnote{Namely, $\cS_k = \{ T_1 + \ell + 1, T_1 + 2 (\ell + 1), \ldots , T_1 + k  \cdot  (\ell+1)\}$. {\color{orange} We might remove this footnote as exact definition is already provided on the previous page.} }. 
Using standard analysis for vanilla LinUCB (e.g., Theorem 2 of \citet{abbasi2011improved}), we can establish a confidence set for $\beta^*$ based on $\hat \beta_k$. We introduce background on classical LinUCB algorithm in Appendix \ref{sec:classical-linUCB}. 
\begin{lemma}
\label{lemma:confidence-ellipsoid}
    Under Assumptions \ref{assumption:feasible}, \ref{assumption:not-all-one}, \ref{assumption:model}, and \ref{assumption:angle},  
    for any $\delta > 0$, with probability at least $1 - \delta$, for all $ k \in [n]$, $\beta_{\ast}$ lies in the set 
    \begin{align}\label{eq:Ct}
	\cC_k := \Bigl\{ \beta \in \RR^{d|\Theta|}: \langle \beta - \hat\beta_k, \Omega_k(\beta  - \hat\beta_k) \rangle^{1/2} \leq 4B \sqrt{\log \bigl (  \det (\Omega_k) \cdot  \det (\lambda I_{d|\Theta|} )^{-1} \big / \delta^2 \big)} + \lambda^{1/2} d^{1/2} B \Bigr\}, 
\end{align}
where $\Omega_k = \sum_{\ell \in \cS_{k-1}} \Vec(\Pi_{\ell})  \Vec(\Pi_{\ell})^\top + \lambda \cdot I_{d|\Theta|}$. 
\end{lemma}

{
%\color{blue} CONTINUE THE PAPER HERE. 
%Note that $u$ is bounded. Using standard analysis of vanilla LinUCB (e.g., Theorem 2 of \citep{abbasi2011improved}), we can establish a confidence set for $\beta^*$ that is centered at $\hat \beta_t$.}
%\begin{lemma}
%\label{lemma:confidence-ellipsoid}
%    Under Assumptions \ref{assumption:feasible}, \ref{assumption:not-all-one}, \ref{assumption:model}, and \ref{assumption:angle},  
%    for any $\delta > 0$, with probability at least $1 - \delta$, for all $t \geq T_1$, $\beta_{\ast}$ lies in the set 
%    %
    %\begin{align}\label{eq:Ct}
%	\cC_t := \left\{ \beta \in \RR^{d|\Theta|}: \langle \beta - \hat\beta_t, \Omega_t(\beta  - \hat\beta_t) \rangle^{1/2} \leq 2B \sqrt{2 \log \left( \frac{\det (\Omega_t)^{1/2} \det (\lambda I_{d|\Theta|})^{-1/2}}{\delta} \right)} + \lambda^{1/2} d^{1/2} B \right\}, 
%\end{align}
%
%where $\Omega_t = \sum_{\ell \in \cS \cap [t]} \Vec(\Pi_{\ell})  \Vec(\Pi_{\ell})^\top + \lambda \cdot I_{d|\Theta|}$. 
%\end{lemma}
%
{
%\color{red} Zhuoran: Remove the proof?  
\begin{proof}[Proof of \cref{lemma:confidence-ellipsoid}]
We prove the lemma in Appendix \ref{sec:proof-lemma:confidence-ellipsoid}. 
 
\end{proof}
}

\vspace{5pt}
{\noindent \bf Pessimistic-optimistic planning.} 
To summarize, so far we have established a pessimistic estimate of the feasible region of \eqref{eq:LP},  as outlined in \eqref{eq:feasible-set}. 
We have also obtained a confidence set for the linear coefficients in the objective function of \eqref{eq:LP} as in \eqref{eq:Ct}. 
Then we propose to put together these two components and  construct $\Pi_{T_1 + k (\ell + 1 ) } $ via a pessimistic-optimistic planning
\begin{align}\label{eq:pess_opt_plan}
	(\Pi_{T_1 + k (\ell + 1 )}, \tilde \beta_k) = \argmax_{(\Pi, \beta) \in \cI_n \times \cC_k} \langle \Pi, \beta \rangle, \qquad \forall k \in [n].  
	\end{align}
That is, we solve a constrained linear optimization problem similar to that in \eqref{eq:permute_obj}, but with a pessimistic constraint by restricting to a subset $\cI_n$ and an optimistic objective by maximizing over the confidence set $\cC_k$.
The new policy obtained in \eqref{eq:pess_opt_plan} is applied in the current round and the $\ell$ rounds that follow to create delayed feedback. Namely, $\Pi_{T_1 + k (\ell + 1 )}$ is applied from round $T_1 + k (\ell + 1 )$ to round $T_1 + k (\ell + 1 ) + \ell$.   

%{\color{red} REMOVE This confidence set then naturally leads to an upper confidence bound for the principal's expected reward following the optimism principle. 
%Next, we put together these two components and introduce the core concept of our algorithm based on the so-called \emph{pessimistic-optimistic planning}. 
%More precisely, to update the coordination mechanism in some round $t + 1$ using pessimistic-optimistic planning, the principal solves the following optimization problem: 
%
%}
%

\vspace{5pt}
{\noindent \bf Pessimistic-Optimistic LinUCB algorithm.}
Finally, by combining these ingredients together, we obtain the Pessimistic-Optimistic LinUCB algorithm, presented in  \cref{alg:delayed-UCB}.
For ease of presentation, we assume $T$ is known and thus $n$ can be determined before the beginning of the algorithm. Generalization to an unknown $T$ using a doubling trick is deferred to Appendix \ref{sec:unknownT}.

% Finally, we introduce the proposed pessimistic LinUCB algorithm, which combines the concepts of delayed feedbacks and pessimistic-optimistic planning.
% For brevity, we present the algorithm in the main text assuming a known value of $T$. 
% Generalization to unknown $T$ is provided in \cref{sec:unknownT} in the appendix utilizing a doubling trick. 

% Skipping the details of $\cA$, we present the pseudo-code of our algorithm below as \cref{alg:delayed-UCB}. 

\begin{algorithm}
	\caption{Pessimistic-Optimistic LinUCB (Known $T$)}\label{alg:delayed-UCB}
\textbf{Input:} Number of rounds $T$;	algorithm $\cA(\cdot)$ that estimates the reward angles.%; parameter  $\delta$.
\begin{algorithmic}[1]
%\For{each episode $k = 1,2, \cdots, $}
%		\State Set $n_k \gets 2^k$;
\State\texttt{\textcolor{blue}{// Split the $T$ rounds into two stages and determine $n$}}
\State Let $n$ be the largest integer  such that $(n + 2) \cdot (\lceil \log n \rceil^2 + 1) + \lceil\log n\rceil^6  \leq T$. 
        \State \texttt{\textcolor{blue}{// Stage I: estimating the feasible region using $\cA(\cdot)$}}
		\State Implement $\cA({n})$ using $T_1 = \lceil \log n \rceil^6$ samples, which outputs $\hat\balpha = \{\hat \alpha^s: s \in \Theta\}$ and further $\cI_{n}$ as in \eqref{eq:feasible-set}.
  
  \State \texttt{\textcolor{blue}{//Stage II: pessimistic-optimistic planning + delayed feedbacks}}
		\State Set $\ell \gets \lceil \log n \rceil^2  $, $\delta \gets n^{-10}$ and   $\lambda \gets 1$. 
  %$\delta \gets n^{-10}$, $\cS \gets \emptyset$; 
            \State \texttt{\textcolor{blue}{// Create delayed feedback via running a dummy mechanism}}
		\State  Run a dummy mechanism $\mathds{1}_{|\Theta| \times d} / d$ for $\ell$ rounds.   \qquad \qquad \texttt{\textcolor{blue}{// End of the $(T_1 + \ell)$-th round.}}
        \State \texttt{\textcolor{blue}{// Perform $n$ pessimistic-optimistic planning updates for rounds in $\cS_n$}}
		\State For $k \in \{1, 2, \ldots, n\} $
			\State  \qquad Compute mechanism  $\Pi_{ T_1 + k(\ell+1)}$ as in \eqref{eq:pess_opt_plan} via pessimistic-optimistic planning. 
             \State \qquad Deploy $\Pi = \Pi_{T_1 + k(\ell+1)}$ for $(\ell+1)$ rounds from round $T_1 + k (\ell+1)$ to round $T_1 + k(\ell+1) + \ell$.
			 % \For{$l \in [\ell]$}
    %                 \State \texttt{\textcolor{blue}{// Create delayed feedback by following the previous mechanism}}
			 % 	\State Start a new round, principal announces $\Pi$ and receives her reward; 
			 % \EndFor
        \State Implement an arbitrary mechanism in the remaining $T - T_1 - \ell - n(\ell + 1)$ rounds.
%		\State Run LinUCB with action space $\cI_{n_k}$ and $n_k$ samples;
%\EndFor
\end{algorithmic}
\end{algorithm}

We present the regret guarantees for  \cref{alg:delayed-UCB}  in the following theorem.

\begin{theorem}\label{thm:main}
	Under Assumptions  \ref{assumption:feasible}, \ref{assumption:not-all-one}, \ref{assumption:model} and \ref{assumption:angle}, there exists $T_0 \in \NN_+$ that depends only on $(\mathscr{P}, \varphi_0)$, such that for all $T \geq T_0$, % ({\color{red} make sure $\log n 
% > 1$}), 
\cref{alg:delayed-UCB} satisfies the following regret upper bounds:
 \begin{enumerate}
     \item (Expected regret) There exists a positive constant $C_{\ast}$ that depends only on $(\mathscr{P}, \varphi_0)$, such that  
     $$\E\left[\regret(T)\right] \leq C_{\ast} \sqrt{T} \cdot (\log T)^3. $$
     \item (Asymptotic regret) For any $g_0 > 3$, as $T \to \infty$ we have
$ \regret(T) / (\sqrt{T} \cdot  (\log T)^{g_0}  ) \overset{a.s.}{\to} 0.$
    \item (High-probability bound) There exists a positive constant $C_{\ast}'$ that depends only on $(\mathscr{P}, \varphi_0)$, such that with probability at least $1 - \eps$, we have 
    \begin{align*}
        \regret(T) \leq C_{\ast}' \cdot \bigl( \sqrt{T} \cdot (\log T)^3 + \eps^{-1/10} \cdot (\log \eps^{-1})^2 \bigr), 
    \end{align*}
    where $\eps$ is an arbitrary positive constant. 
 \end{enumerate}
 \end{theorem}
 \begin{proof}[Proof of \cref{thm:main}]
     We refer the readers to the proof of \cref{thm:main-episodic} in the appendix for a proof of \cref{thm:main}. 
\cref{thm:main-episodic} deals with a more challenging setting where the number of rounds $T$ is unknown to the principal.
 \end{proof}

\cref{thm:main} proves that  \cref{alg:delayed-UCB} attains  $\sqrt{T}$-regret, ignoring constant and logarithmic factors. 
We provide in-expectation, asymptotic, and high-probability bounds. The constants $C_{\ast}$ and $C_{\ast}'$, despite being independent of the number of rounds, exhibit an exponential dependency on the problem dimension. This exponential dependency is considered both sufficient and necessary for Stackelberg games. We refer readers to  \cite{zhu2022sample} for formal proof in a different setting. 

\section{Estimating the reward angles}
\label{sec:reward-function}

Recall that a valid coordination mechanism $\Pi $  has rows $\{ \Pi_{\theta} \}_{\theta \in \Theta} $ in the probability simplex  $\Delta$ in \eqref{eq:Delta-sets}. 
Also recall that we map the normalized reward vectors $\{ \bar v_{\theta} \}_{\theta \in \Theta} $ to $\SS^{d-2}$ through the isometry $\varphi_0$, 
which are parameterized by the spherical coordinates, dubbed as reward angles,  using the mapping $\rho$ defined in \eqref{eq:rho}.
In the following, we lay out the details of the algorithm $\cA$ that estimates the reward angles $\balpha = \{\alpha^{\theta}: \theta \in \Theta\}$
and satisfies the accuracy level given in \cref{thm:reward-function}.

% In this section we describe  $\cA$, the goal of which we recall is to estimate the reward angles $\balpha = \{\alpha^{\theta}: \theta \in \Theta\}$.
% If $\cA$ satisfies the properties stated in \cref{thm:reward-function}, then achieving $\eps$ estimation error for the reward angles requires only $O(\mathsf{polylog} (1 / \eps))$ samples.
% Namely, we can utilize $\cA$ to estimate the feasible region of \eqref{eq:LP} in a sample-efficient manner.

{Note that when $|\Theta| = 1$, the feasible region of \eqref{eq:LP} is solely determined by the last line in \eqref{eq:LP}, and we do not need to perform $\cA$ to estimate the feasible region. 
In this case, we can directly proceed to the pessimistic-optimistic planning stage given in \cref{alg:delayed-UCB}. 
In the following, we focus on the case where $|\Theta| \geq 2$ and establish $\cA$ within this regime. 
Specifically, we will separately discuss two scenarios: 
(1) $|\Theta| \geq 3$ and 
(2) $|\Theta| = 2$.
For simplicity, we address the first scenario in the main text. The idea that handles the second scenario is similar, and the associated algorithm is like that described in Appendix \ref{sec:appendix-other2}. We skip the details to avoid redundancy. }

%{\color{cyan} Modify later: Note that we only need to consider the case where $|\Theta| \geq 2$. 
%}
%As mentioned before, here we use the index $n$ to measure both the level of accuracy we aim to attain, as well as the number of samples we require for the estimation procedure.
% In words, we prove that in order to estimate $\balpha$ up to $n^{-c}$ precision under metric \eqref{eq:dist-D}, it suffices to use samples from $\mathsf{polylog}(n)$ rounds. 
  
Overall, Algorithm $\cA$ proceeds in a coordinatewise fashion. Namely, we first estimate the angles $\{\alpha_i^{\theta}: \theta \in \Theta\}$ following a reverse order from $i = d-2$ to $i = 1$. 
Then we ``glue'' the coordinate estimates together to approximate the reward angle vectors. 
See  Figure \ref{fig:glue} for an illustration of this procedure. 
{We note that the second step is nontrivial because we only estimate the reward angles as sets, and thus subject to unknown permutations. 
To combine the coordinate estimates into a single estimate of the reward angle vectors, we need to match the coordinates in a way such that all the permutations are aligned.} 
We state the estimation of the last coordinate in \cref{sec:est-last-coordinate}, and describe the gluing procedure in { \cref{sec:estimate-remaining-angles}.}
% \cref{sec:matching-alg}. 

\begin{figure}[ht]
    \centering
    \includegraphics[width=0.8\textwidth]{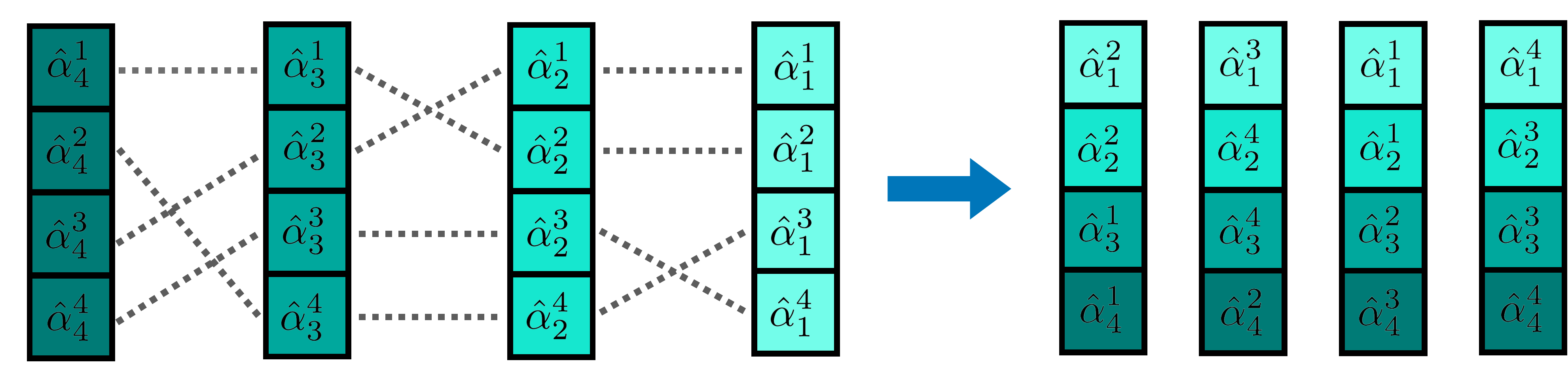}
    \caption{Illustration of the reward angle estimation algorithm $\cA$. In this example we take $d = 6$ and $\Theta = \{1,2,3,4\}$. 
    We sequentially estimate the coordinates of the reward angles from $i = 4$ to $i=1$. That is, we start with computing $\{\hat \alpha_4^i\}_{i \in [4]}$, then $\{\hat \alpha_3^i\}_{i \in [4]}$, $\{\hat \alpha_2^i\}_{i \in [4]}$,  and finally $\{\hat \alpha_1^i\}_{i \in [4]}$. Estimates of the reward angle vectors are obtained by gluing the coordinates together in an appropriate way. }
    \label{fig:glue}
\end{figure}

Let $r_d > 0$ be a small constant such that  $\{ x \in \RR^d: \|x - \mathds{1}_d / d\|_2 = r_d, \langle x, \mathds{1}_d \rangle = 1 \} \subseteq \Delta$. 
Throughout this section, we always construct coordination mechanisms whose rows are of the form 
\begin{align}\label{eq:mechanism_family}
 \mathds{1}_d / d + r_d \cdot \varphi_0^{-1}(\rho (\alpha_1, \ldots, \alpha_{d-2}) ), \qquad \alpha _{1}, \ldots, \alpha_{d-3} \in [0, \pi], \alpha_{d-2} \in [0, 2\pi), 
\end{align}
which maps a set of angles to a vector in $\Delta$. 
The angles in \eqref{eq:mechanism_family} are used to estimate the desired reward angles.

%\yw{todo: add a motivating example explaining intuitively why this can be done}

\subsection{Estimating the last coordinate of the reward angle vectors}
\label{sec:est-last-coordinate}

We first describe the algorithm that estimates $\balpha_{d - 2} := \{\alpha_{d - 2}^{\theta}: \theta \in \Theta\}$, the last coordinates of the reward angle vectors, and establish theoretical guarantees.
This serves as the first step of our reward angle estimation algorithm. 
% Specifically, we present an algorithm that estimates $\balpha_{d - 2} := \{\alpha_{d - 2}^{\theta}: \theta \in \Theta\}$ and give the associated theoretical guarantees. 
Note that we estimate the final coordinates collectively as a set.
For an estimate $\hat\balpha_{d - 2} = \{\hat{\alpha}_{d - 2}^{s}: s \in \Theta\}$ of $\balpha_{d - 2}$, we employ the following metric to evaluate the estimation accuracy: 
\begin{align}\label{eq:dist-D-d-2}
	\dist_{d-2}(\balpha_{d - 2}, \hat{\balpha}_{d - 2}) := \min_{\sfb \in \mathsf{Perm}(\Theta)} \sup_{s \in \Theta} |\alpha_{d - 2}^{\sfb(s)} - \hat\alpha_{d - 2}^{s}|,
\end{align}
where   $\mathsf{Perm}(\Theta)$ is the set of permutations over $\Theta$. 
When focusing on estimating $\balpha_{d-2}$, we set $\alpha_{1}, \ldots, \alpha_{d-3} = \pi/2$ in \eqref{eq:mechanism_family}. 
To simplify the notation, for any $\alpha \in [0, 2\pi)$, we define  
\begin{align} \label{eq:define_x_alpha}
   x_{\alpha} = \mathds{1}_d / d + r_d \cdot \varphi_0^{-1}(0, 0, \cdots, 0, \cos \alpha, \sin \alpha ). 
\end{align} 
We will construct mechanisms based on $\{ x_{\alpha} \}_{\alpha \in [0, 2\pi)}$ with some special choices of $\alpha$. 
Our algorithm implements the straightforward idea of {\color{PineGreen}binary search} -- suppose we know there exists an $\alpha_{d-2}^\theta$ in an interval $[\alpha, \alpha + \delta) $, we can split the interval in half and check whether $\alpha_{d-2}^\theta \in [\alpha, \alpha + \delta/ 2)$ or $\alpha_{d-2}^\theta \in [\alpha + \delta / 2, \alpha + \delta)$. 
The building block of this method is a subroutine called \emph{sector test}, which determines whether there exists an $\alpha_{d-2}^\theta$ that falls inside a specific sector.

\paragraph{Sector test.}
%\yw{What a sector test is doing? Describe. }
A sector test with parameters $(\alpha, \delta)$, denoted by $\mathtt{SecTest}(\alpha, \delta)$ checks whether there exists an angle $\alpha_{d - 2}^{\theta}$ from the set $\balpha_{d-2}$ that falls in a sector $(\alpha - \delta/ 2 , \alpha + \delta / 2)$. 
Here, $\alpha$ characterizes the position of the circular sector and $\delta$ is the sector length. 
Given $(\alpha, \delta)$, we define a coordination mechanism $\Pi^{\alpha, \delta} \in \RR^{|\Theta| \times d} $ by letting 
\begin{align}\label{eq:simple-mechanism}
	& \Pi_{ 1}^{\alpha, \delta} = x_{\alpha - \delta}, \qquad  \Pi_{ 2}^{\alpha, \delta}  = x_{\alpha}, \qquad \Pi_{ 3}^{\alpha, \delta} = x_{\alpha + \delta},  \qquad  
	  \Pi_{ s}^{\alpha, \delta} = \Pi_{ 3}^{\alpha, \delta}, \qquad \forall  s \in \{ 4, \ldots,  |\Theta| \} .
\end{align}
Recall that we define $\{ x _{\alpha} \}_{\alpha \in [0, 2\pi)}$ in \eqref{eq:define_x_alpha}. Since $\varphi_0$ preserves the Euclidean distance (see \eqref{eq:varphi_preserve_distance}), for any $\beta \in [0, 2\pi)$ and any $\theta \in \Theta$, we have 
\begin{align}\label{eq:simple-product}
	\langle x_{\beta}, \bar{v}_{\theta} \rangle  = r_d \cdot  \cos(\alpha^{\theta}_{d - 2} - \beta) \cdot  \prod_{j = 1}^{d - 3} \sin \alpha^{\theta}_j.  
\end{align}
Therefore, when the principal announces   $\Pi^{\alpha, \delta}$ and the agent is myopic, if we further suppose the agent has type  $\theta$ and reports type $2$, then we know that $2$ maximizes $\langle \Pi^{\alpha, \delta}_{\vartheta}, \bar {v}_{\theta} \rangle $ over $\vartheta \in \Theta$.  
By \eqref{eq:simple-product} 
this means that $\alpha $ is closer to $\alpha_{d-2}^\theta$ than $\alpha - \delta$ and $\alpha + \delta$. 
Thus, we can conclude that 
$\alpha _{d-2}^\theta \in  \cap (\alpha - \delta/ 2 , \alpha + \delta / 2)$. 
However, in the setting we are considering, the agent is non-myopic with strategic response behaviors given in \eqref{eq:agent_objective_at_t}. Again, we leverage the delaying methods introduced in \cref{sec:delayed-observations}, i.e., sticking to the same mechanism $\Pi^{\alpha, \delta}$ or implementing a dummy mechanism for a certain number of rounds,  to ensure the agent is approximately myopic. 
Specifically, after deploying $\Pi^{\alpha, \delta} $ for $T_{\mathrm{sec}} $ rounds to conduct the sector test, we additionally run the dummy mechanism for $\ell$ rounds to ensure that the mechanisms used in the sector test will not interfere with the immediate future steps after the sector test. 
We present the details of the sector test in \cref{alg:sector_dector} and its theoretical guarantee in Appendix \ref{sec:theory-alg:sector_dector}.

\begin{algorithm}
\caption{Sector test   $\mathtt{SecTest}(\alpha, \delta )$}
\label{alg:sector_dector}
\textbf{Input:} parameters $\alpha,\,\delta$, and accuracy level $n$.
\begin{algorithmic}[1]
\State Set $\ell = \lceil  \log n \rceil^2 $ and $T_{\mathrm{sec}}= \lceil  \log n \rceil^4$;
\State Deploy mechanism     $\Pi^{\alpha, \delta}$  defined in \eqref{eq:simple-mechanism} for $T_{\mathrm{sec}}$ rounds and observe the agent's reported types; 
\State Let $N_{\mathrm{sec}} $ denote the total number of rounds in which the agent reports type $2$;
\State \texttt{\textcolor{blue}{// Create delayed feedback}}
\State Announce the dummy mechanism $\mathbf{1}_{|\Theta| \times d} / d$ for $\ell$ rounds;

\State Return $\mathtt{True}$ if $N_{\mathrm{sec}} \geq 1$ and return $\mathtt{False}$ if $N_{\mathrm{sec}} = 0$.

% \State Set $\mathsf{count}, \mathsf{count}_+\gets 0$;
% \While{$\mathsf{count} \leq \lceil  \log n \rceil^4 + 1$}
% 	\State $\mathsf{count} \gets \mathsf{count} + 1$; 
% 	\State Start a new round, announce the mechanism $\Pi^{\alpha, \delta}$ \cref{eq:simple-mechanism}, and observe the agent's reported type $r$;
% 	\If{$r = 2$}
% 		\State $\mathsf{count}_+ \gets \mathsf{count}_+ + 1$; 
% 	\EndIf
% \EndWhile

% \State $\ell \gets \lceil \log n \rceil^2 + 1$;
% \State \texttt{\textcolor{blue}{// Create delayed feedback}}
% \For{$i \in [\ell]$}
% 	\State Start a new round and announce a dummy mechanism $\mathbf{1}_{|\Theta| \times d} / d$; 
% \EndFor
% \If{$\mathsf{count}_+ \geq 1$}
% 	\State \textbf{return }True; 
% \Else
% 	\State \textbf{return }False; 
% \EndIf
\end{algorithmic}
\end{algorithm}
Note that \cref{alg:sector_dector} lasts for  $\lceil \log n \rceil^2 + \lceil \log n \rceil^4$ rounds.
As we will show in Appendix \ref{sec:theory-alg:sector_dector}, for a large enough $n$, with high probability \cref{alg:sector_dector} returns \texttt{True} when there exists $\alpha_{d - 2}^{\theta} \in (\alpha - \delta / 2, \alpha + \delta / 2)$, and returns \texttt{False} otherwise. 
In this case, we say \cref{alg:sector_dector} succeeds. 

\paragraph{Binary search algorithm.}
\cref{alg:sector_dector} allows the principal to determine whether there exists $\alpha_{d - 2}^{\theta}$ that falls inside the sector $(\alpha - \delta / 2, \alpha + \delta / 2)$. If yes, then we can further divide this sector into two smaller sectors and conduct a sector test on each one of them. Repeating such a procedure for sufficiently many times, we can pin down the reward angles 
with sufficiently high accuracy. 
See Figure \ref{fig:example} for an illustration of this procedure.

Specifically, let $\alpha \sim \mathrm{Unif}[0, 2\pi)$. 
We initialize the binary search algorithm by running sector tests on $(\alpha , \alpha + \pi)$ and $(\alpha + \pi, \alpha + 2\pi)$, i.e., $\mathtt{SecTest}(\alpha + \pi / 2, \pi)$ and $\mathtt{SecTest}(\alpha + 3\pi / 2, \pi)$.  
If any of these tests returns \texttt{True}, we then halve the corresponding sector into two sectors and run sector tests on each one of them. 
For example, if $\mathtt{SecTest}(\alpha + \pi / 2, \pi)$ returns  true, we further consider sectors $(\alpha, \alpha + \pi/2)$ and $(\alpha + \pi/2, \alpha + \pi)$ and thus run $\mathtt{SecTest}(\alpha + \pi/4, \pi/2)$ and $\mathtt{SecTest}(\alpha + 3\pi/4, \pi/2)$. 
This process is continued until the desired accuracy level is reached. 

Specifically, let $K$ denote the number of binary search iterations. We let $\cQ_k$ denote the starting points of the surviving sectors\footnote{We call a sector surviving sector if sector test indicates there is at least one reward angle in this sector.} in the $k$-th iteration. 
All the surviving sectors in the $k$-th iteration have lengths equal to $\pi /2^{k-1} $.  
For example, $\cQ_1 = \{ \alpha, \alpha + \pi\}$ and the two sectors are $(\alpha, \alpha + \pi) $ and $(\alpha + \pi, \alpha + 2 \pi)$. 
{Here, we include the random angle $\alpha $ to ensure the desired reward angles are not too close to the boundaries of the sectors with high probability.} 
In the $k$-th iteration, for all $q \in \cQ_{k}$,  we run sector test $\mathtt{SecTest}(q + \pi / 2^k, \pi / 2^{k-1})$.  
If the test returns true, we add $q$ and $q+ \pi/2^k$ to $\cQ_{k+1}$ and halve the sector $(q, q + \pi / 2^{k-1})$. 
As a result, after $K$ iterations of binary search, we obtain $|\cQ_{K+1}|$ sectors each of size $\pi / 2^{K}$. 
Finally, we run sector tests on each of these sectors and only keep those with positive results. 
The details of the binary search algorithm are stated in \cref{alg:binary_search}.
The sectors returned by this algorithm are $\{ (q, q + \pi / 2^{K})\}_{q\in \cQ_*}$.

\begin{algorithm}
\caption{Binary search algorithm that constructs a set of sectors containing $\{ \alpha^\theta_{d-2}\}_{\theta \in \Theta} $}
\label{alg:binary_search}
\textbf{Input:} random angle $\alpha$ and iteration number $K$ 
\begin{algorithmic}[1]
\State Initialize by setting  $\cQ_1 = \{\alpha , \alpha + \pi\}$, $\cQ_{k} = \emptyset$ for all $k \in [K+1] \backslash \{1\}$, and $\cQ_* = \emptyset$. 
\State For $k = 1,2,\cdots, K $ 
	\State \qquad  Run $\mathtt{SecTest} (q + \pi / 2^k, \pi / 2^{k-1})$   (\cref{alg:sector_dector}) for all $q\in \cQ_k$.

 \State \qquad If $\mathtt{SecTest} (q + \pi / 2^k, \pi / 2^{k-1}) = \mathtt{True}$, set $\cQ_{k + 1} \gets \cQ_{k + 1} \cup \{q, q + \pi / 2^k\}$.

  \State \texttt{\textcolor{blue}{// Perform sector tests for all sectors of size $\pi / 2^{K}$}}

 \State Run $\mathtt{SecTest} (q + \pi / 2^{K+1}, \pi / 2^{K})$ for all $q \in \cQ_{K+1}$.  
 \State If $\mathtt{SecTest} (q + \pi / 2^{K+1}, \pi / 2^{K}) = \mathtt{True}$, set $\cQ_{\ast} \gets \cQ_{\ast} \cup \{q \}$
\State Return $\cQ_{\ast}$ and sectors $\{ (q, q + \pi / 2^{K} )\}_{q \in \cQ_*}$.
\end{algorithmic}
\end{algorithm}

When $K$ is sufficiently large such that no two angles are within arc length $\pi/ 2^K$,
we can prove that each sector returned by \cref{alg:binary_search} contains a unique reward angle in $\balpha$ with high probability.  
We present theoretical guarantee for \cref{alg:binary_search} in Lemma \ref{lemma:binary-search} below. 

\begin{lemma}\label{lemma:binary-search}
%	Recall that $\bar\chi_{d - 2} := \min_{\theta, \theta' \in \Theta, \theta \neq \theta'} \arc(\alpha_{d - 2}^{\theta}, \alpha_{d - 2}^{\theta'}) > 0$.
Under the same assumptions made by  \cref{thm:reward-function}, further assume that $\alpha \sim \Unif[0, 2\pi)$. 
In addition, we set $K = \lceil 4^{d + 1} \log_2 n \rceil + 1$ in \cref{alg:binary_search}. 
Then there exist $n_0 \in \NN_+$ and $C_{\sfP} \in \RR_+$, both depending only on $(\mathscr{P}, \varphi_0)$, such that for $n \geq n_0$, with probability at least $1 - C_{\sfP}n^{-4^{d + 1}}$,
%	We assume
	% 
%	$n \geq \sqrt[4^{d + 1}]{\pi \bar\chi_{d - 2}^{-1}} \vee \exp\big( {4^{(d + 3) / 3}f_{\min}^{-1 / 3}} \big). $
	%
%	Then with probability at least $1 - 24|\Theta| (r_d \delta_{\sin})^{-1/2} n^{-4^{d + 1}} - 8n^{-4^{d + 1}}$, 
 the following statements are true:  %{\color{red}why do you write $\log _2 n$? Can we just write $\log n$? $\log $ has base $e$ by default. It's okay if you write $ K = C \cdot 4^d \cdot \log n $ for some constant. \yuchen{Yes, but we need to specify $C$ as it is input to our algorithm. I recommend sticking to the current choice though, as changing $C$ means changing the entire proof (although in a kind of trivial way)}}
 \begin{enumerate}
     \item $|\cQ_k| \leq 2|\Theta|$ for all $k \in [K + 1]$,
     \item $|\cQ_{\ast}| = |\Theta |$, and for all $q \in \cQ_{\ast}$ there exists a unique $\alpha_{d - 2}^{\theta}$ such that $|q - \alpha_{d - 2}^{\theta}| \leq 4 \cdot n^{-4^{d + 1}}$.
 \end{enumerate}

% (1) $|\cQ_k| \leq 2|\Theta|$ for all $k \in [K]$, (2) $|\cQ_{\ast}| = |\Theta |$ and for all $q \in \cQ_{\ast}$, there exists a unique $\alpha_{d - 2}^{\theta}$ such that $|q - \alpha_{d - 2}^{\theta}| \leq \pi n^{-4^{d + 1}}$. 
\end{lemma}

\begin{proof} 
    See  Appendix \ref{sec:proof-lemma:binary-search} for a detailed proof. 
\end{proof}

A direct consequence of \cref{lemma:binary-search} is that \cref{alg:binary_search} estimates $\balpha$ with exponentially high accuracy based on   $\mathrm{polylog}(n)$ samples. 
It outputs $\hat\balpha_{d - 2} = \{\hat\alpha_{d - 2}^{s}: s \in \Theta\}$ that satisfies 
$\dist_{d-2}(\balpha_{d - 2}, \hat\balpha_{d - 2}) \leq 4 n^{-4^{d + 1}}$,
where $\dist_{d-2}$ is defined in \eqref{eq:dist-D-d-2}. 
Moreover, Algorithm \ref{alg:binary_search} uses samples from no more than $4^{d + 2}|\Theta| \cdot (\lceil \log_2 n \rceil + 1) \cdot (\lceil \log n\rceil^4 + \lceil \log n\rceil^2)$ rounds to output a $\cQ_{\ast}$ with high accuracy. 
 To see this, note that according to Lemma \ref{lemma:binary-search} Algorithm \ref{alg:binary_search} implements  $\sum_{k = 1}^{K + 1} |\cQ_k| \leq 2(K + 1) |\Theta|$ sector tests and each sector test takes  $\lceil \log n\rceil^4 + \lceil \log n\rceil^2$ rounds.

Furthermore, since we estimate  $\balpha_{d-2}$ as a set of reward angles, we can enumerate angles in  $\hat \balpha_{d - 2} $ in an arbitrary order. By \Cref{lemma:binary-search} and the definition of $\dist_{d - 2}$, we conclude that, with high probability, there exists a permutation over $\Theta$, denoted by $\sfb _n \in \mathsf{Perm}(\Theta)$   such that  
\begin{align}\label{eq:def-bn-perm}
   \sup_{s \in \Theta} |\alpha_{d - 2}^{\sfb_n(s)} - \hat\alpha_{d - 2}^{s}| \leq 4 n^{-4^{d + 1}}, \qquad \textrm{where}\quad  \sfb_n(\cdot) := \argmin_{\sfb \in \mathsf{Perm}(\Theta)} \sup_{s \in \Theta} \big| \alpha_{d - 2}^{\sfb(s)} - \hat\alpha_{d - 2}^{s} \big|. 
\end{align}
If a tie occurs, then we randomly pick any permutation that minimizes the objective above. 
Note that $\sfb_n(\cdot)$ is a function of $\hat\balpha_{d - 2}$ and hence is random. 
Here we add a subscript $n$ to indicate that $\sfb_n(\cdot)$ depends on the parameter $n$ through 
\cref{alg:binary_search}. 
In the sequel, as we enumerate items in $\hat \balpha_{d-2}$, we always follow the order $1, \ldots, |\Theta|$ unless mentioned otherwise.

\subsection{Estimating coordinates of the reward angle vectors in a reverse order}
\label{sec:estimate-remaining-angles}

In the following, we propose an algorithm that estimates the coordinates of the reward angle vectors $\{\alpha^{\theta}: \theta \in \Theta \}$ in reverse order. 
That is, for   $i = d - 3, d - 4, \cdots, 1$, given estimates $\hat\balpha_{i + 1}, \hat\balpha_{i + 2}, \cdots, \hat\balpha_{d - 2}$, we aim to construct  
$\hat{\balpha}_i := \{\hat\alpha_i^s: s \in \Theta\}$ as an estimate of $\balpha_i := \{\alpha_i^{\theta}: \theta \in \Theta\}\subseteq [0, \pi)$.
In addition to estimating $\balpha_i$ as a set of angles, we need to {\color{PineGreen}match the indices} of $\hat \balpha_i$ to those of $\{\hat \balpha_{j}\}_{j > i} $ to ensure that, for any $s\in \Theta$, $(\hat \alpha_i^s, \ldots, \hat \alpha_{d-2}^s)$ is close to the reward angles of type $\sfb_n(s) $, where $\sfb_n$  is defined in \eqref{eq:def-bn-perm}.

\vspace{5pt} 
{\bf \noindent Algorithm structure.} 
To summarize, our algorithm iteratively constructs the reward angles in $\hat \balpha_i$ using a two-step procedure.
When constructing a specific $\hat \alpha_i^s$ for some $s \in \Theta$, we first divide $[0, \pi)$ into intervals of size $\pi/ (2N) $
and identify the interval that contains $\alpha_i^{\sfb_n(s)}$,
where $N =\lceil \log n \rceil $. 
Recall that $\sfb_n$ is defined in \eqref{eq:def-bn-perm} and depends on the outcome of Algorithm \ref{alg:binary_search}. 
This step is achieved by combining a novel grid search algorithm with a conditional version of the sector test (see \cref{alg:sector_dector2} for more details). 
In the second step, 
we raise the accuracy level by performing binary searches similar to \cref{alg:binary_search} for each identified interval. 
Here, we design a conditional sector test because when searching for $\alpha_i^{\sfb_n(s)}$, we need to fix the angles $\{ \hat \balpha_{j} \}_{j > i}$ and ensure that the indices of the angles are matched.

\vspace{5pt} 
{\bf \noindent Conditional sector test.} 
We then describe the conditional sector test. 
In the sequel, we use shorthand notation $\alpha_{j:k}$ to denote the angle vector $(\alpha_j, \ldots, \alpha_k)$ for any positive integers $j < k$. 

Let $\cM\subseteq \Theta$ denote the set containing all the indices $s$ such that $\hat \alpha _{i}^{s}$ has already been constructed, starting from $\cM = \emptyset$.
Given parameters $(\alpha, \delta, s)$ with index $s \in \Theta \setminus \cM$, similar to the sector test in \cref{alg:sector_dector}, the goal of a conditional sector test with parameters $(\alpha, \delta, s)$ is to determine whether the interval $   (\alpha - \delta / 2 , \alpha + \delta / 2)$ contains $\alpha_{i}^{\sfb_n(s)}$. Note that by Lemma \ref{lemma:binary-search}  and induction hypothesis we have $\hat \alpha _{(i+1):(d-2)} ^{s} \approx \alpha _{(i+1):(d-2)} ^{\sfb_n(s) }$ with high probability.

Specifically, for any $\alpha$, $\delta$, and $s \in \Theta \setminus \cM$, 
to implement the conditional sector test, we employ the coordination mechanism 
$\Pi^{\alpha, \delta, s } $ defined as follows:
\begin{align}\label{eq:simple-mechanism2}
\begin{split}
	 & \Pi_{1}^{\alpha, \delta, s } = \mathds{1}_d / d + r_d \cdot  \varphi_0^{-1} (\xi_i(\alpha - \delta,  \hat \alpha^{s} _{(i + 1): (d - 2)})), \\
  & \Pi_{2 }^{\alpha, \delta, s} = \mathds{1}_d / d + r_d \cdot \varphi_0^{-1}(\xi_i(\alpha, \hat \alpha^{s} _{(i + 1): (d - 2)})), \\
	&  \Pi_{3 }^{\alpha, \delta, s}  = \mathds{1}_d / d + r_d \cdot \varphi_0^{-1}(\xi_i(\alpha + \delta, \hat \alpha^{s} _{(i + 1): (d - 2)})),   \\
& \{  \Pi_{h}^{\alpha, \delta, s }    \colon  4 \leq h \le |\cM | + 3 \}  = 
 \{ \mathds{1}_d / d + r_d \cdot \varphi_0^{-1}(\xi_i ( \hat \alpha_i^{s'}, \hat \alpha_{ (i + 1): (d - 2)}^{s'} )) \colon  s'\in \cM \big\} ,    \\
	&  \Pi_{h }^{\alpha, \delta, s } = \Pi_{3 }^{\alpha, \delta, s}, \qquad \qquad\qquad\qquad \qquad ~~  h \in \{ |\cM| +4 , \ldots,  |\Theta|\} .
\end{split}
\end{align}
%\yuchen{what is $\sharp$?}
In the above display, for $\alpha_{i:(d-2)} \in [0, \pi]^{d - 2 - i} \times [0, 2\pi)$, we define 
\begin{align}
    \label{eq:define_mapping_xi_i}
    \xi_i( \alpha _{i:(d-2)} ) =  \rho(\pi / 2, \cdots, \pi / 2, \alpha_i, \alpha_{i+1}, \ldots, \alpha_{d - 2}),
\end{align}
where we recall that the mapping $\rho$ is defined in \Cref{eq:rho}.

In \cref{eq:simple-mechanism2}, the first three rows of $\Pi^{\alpha, \delta, s}$ are similar to those of $\Pi^{\alpha, \delta}$ defined in \eqref{eq:simple-mechanism} and are the key components of the conditional sector test. 
When $\hat \alpha_{(i + 1):(d - 2)} ^s$ is close to $\alpha_{(i + 1):(d - 2)}^{\sfb_n(s) }$, for any $\beta \in [0, \pi]$ 
we have 
\begin{align*}
	\langle \mathds{1}_d / d + r_d \varphi_0^{-1}(\xi_i(\beta, \hat \alpha_{(i + 1): (d - 2)}^s )), \bar{v}_{\theta} \rangle \approx r_d \cdot \cos(\beta - \alpha_i^{\theta}) \cdot \prod_{j = 1}^{i - 1} \sin \alpha_j^{\theta}, \qquad \textrm{where}~~ \theta  = \sfb_n(s). 
\end{align*}
We design the first three rows of the coordination mechanism  such that when $\theta = \sfb_n(s)$, among $\{\alpha, \alpha + \delta, \alpha - \delta\}$, with high probability the agent will choose the angle that is closest to $\alpha_i^{\theta}$ and can only report the corresponding type. 
Hence, when $\theta = \sfb_n(s)$, with high probability the agent reports type 2 if and only if $\alpha_i^{\theta} \in (\alpha - \delta / 2, \alpha + \delta / 2)$. 

The fourth line of \eqref{eq:simple-mechanism2} states that the fourth to the $(|\cM| + 3)$-th rows of the coordination mechanism are constructed based on the matched angles $\{ \hat \alpha_{i:(d-2)}^h \}_{ h\in \cM}$. 
The equality in the fourth line of \cref{eq:simple-mechanism2} means that the sets on both sides are the same. 
The motivation for such a construction is that, when the agent has type $\theta = \sfb_n(h)$ for some $h \in \cM$, under mechanism \eqref{eq:simple-mechanism2}, with high probability he will always report a type in $\{4, \ldots, |\cM| + 3\}$. 
%However, this might not be the case 
Without these rows, the agent might report type 2 even if $\alpha_{i}^{\sfb_n(s)} \notin (\alpha - \delta /2 , \alpha + \delta /2)$. 
Therefore,  
 the construction in the fourth line of \eqref{eq:simple-mechanism2} ensures that if the agent reports type $  2$, then we must have $\alpha_{i}^{\sfb_n(s)} \in (\alpha - \delta /2 , \alpha + \delta /2)$. 
See \cref{alg:sector_dector2} for the details of the conditional sector test. 
Similar to \cref{alg:sector_dector}, the conditional sector test lasts for $\lceil \log n\rceil^2 + \lceil \log n\rceil^4$ rounds and with high probability correctly tests whether $\alpha_{i}^{\sfb_n(s)}$ is in $(\alpha - \delta /2 , \alpha + \delta /2)$. 
The remaining rows are set to $\Pi_3^{\alpha, \delta, s}$.

\begin{algorithm}
\caption{Conditional sector test $\mathtt{ConSecTest}(\alpha, \delta, s)$}
\label{alg:sector_dector2}
\textbf{Input:} parameters $\alpha,\,\delta, \,s$, $\hat \alpha_{(i+1):(d-2)}^s$, $\{\hat  \alpha_{ i: (d - 2)}^h \}_{h \in \cM}$, and accuracy level $n$;
\begin{algorithmic}[1]

\State Set $\ell = \lceil  \log n \rceil^2 $ and $T_{\mathrm{sec}}= \lceil  \log n \rceil^4$; 
\State Deploy mechanism     $\Pi^{\alpha, \delta, s }$  defined in \eqref{eq:simple-mechanism2} for $T_{\mathrm{sec}}$ rounds and observe the agent's reported types;
\State Let $N_{\mathrm{sec}} $ denote the total number of rounds where the agent reports type $2$;
\State \texttt{\textcolor{blue}{// Create delayed feedback}}
\State Announce the dummy mechanism $\mathbf{1}_{|\Theta| \times d} / d$ for $\ell$ rounds; 

\State Return $\mathtt{True}$ if $N_{\mathrm{sec}} \geq 1$ and return $\mathtt{False}$ if $N_{\mathrm{sec}} = 0$; 

\end{algorithmic}
\end{algorithm}

We establish theoretical guarantee for \cref{alg:sector_dector2} in Appendix \ref{sec:alg4}. 
Using the conditional sector test as a subroutine, we introduce the algorithm for estimating $\balpha_i$ as follows.

\vspace{5pt}

{\bf \noindent Step 1: Crude estimate of $\balpha_i$ via adaptive grid search.} 
Notice that $\balpha_i = \{ \alpha _{i}^{\sfb_n(s) } \}_{s \in \Theta }$.
%Essentially, we search for these angles in $|\Theta | $ different $[0, \pi)$ intervals. 
Let $\iota = \pi / (2N)$, $u = \sqrt{\iota} + u_0$, where $u_0 \sim \mathtt{Unif}[0, \iota ) $ is a random variable that is independent of everything else and $N = \lceil \log n \rceil$.  
Starting from $\pi/2 + u$ and $\pi /2 - u$, we construct $N_0$ intervals of size $\iota = \pi / (2N)$:
\begin{align}
    \label{eq:define_intervals}
  \mathtt{I}_{k} &= \big (\pi/2  + u + (k-1)\cdot \iota,~ \pi/2 + u + k \cdot \iota \big ),   \quad   \mathtt{I}_{k}' = \big (\pi/2 - u - k \cdot\iota,~ \pi/2 - u - (k-1)    \cdot \iota \big )
\end{align}
for all $k \in [N_0]$,
where $N_0$ is the largest integer $k$ such that $\sqrt{\iota} + (k+2) \cdot \iota \leq \pi /2  $. 
Starting from $k = 1$,
%$\mathtt{I}_{1}$ and $\mathtt{I}_{1}'$, 
we check if $\alpha_i^{\sfb_n(s)} \in \mathtt{I}_{k}$ or $\alpha_i^{\sfb_n(s)} \in \mathtt{I}_{k}'$ for all $s \in \Theta \setminus \cM$ using conditional sector tests. 
For example, to check whether 
$\alpha_i^{\sfb_n(s)} \in \mathtt{I}_{k}$, we run $\mathtt{ConSecTest}(\alpha, \delta, s) $ with $\alpha = \pi/2  + u + (k-1/2)\cdot \iota   $ and $\delta = \iota.$
{Notice that the union of the intervals in \eqref{eq:define_intervals} is a strict subset of $(0, \pi)$. Thanks to the randomness in $u$ and the  spherical coordinate system, we prove in Appendix \ref{sec:proof-lemma_matching} that these intervals contain the desired reward angles with high probability.}
In other words, for the set of intervals listed in \cref{eq:define_intervals}, we start from the middle ones and iterate through all intervals following an increasing $k$ order. 
For each interval, we look at all the remaining indices in $\Theta \setminus \cM$ and conduct conditional sector test for that index. Whenever a conditional sector test returns $\mathtt{True}$, we conclude that the interval contains $\alpha^{\sfb_n(s)}_i$ for some $s\in \Theta \setminus \cM$. 
Then we will zoom in on this interval and perform a binary search procedure, which enables us to estimate $\alpha^{\sfb_n(s)}_i$ with the desired accuracy. 
Then we add $s$ to $\cM$ and continue the searching procedure over the other intervals. 
See Figure \ref{fig:adaptive_grid_search} for an illustration of this method. 
\begin{figure}[ht]
    \centering
    \includegraphics[width=0.95\linewidth]{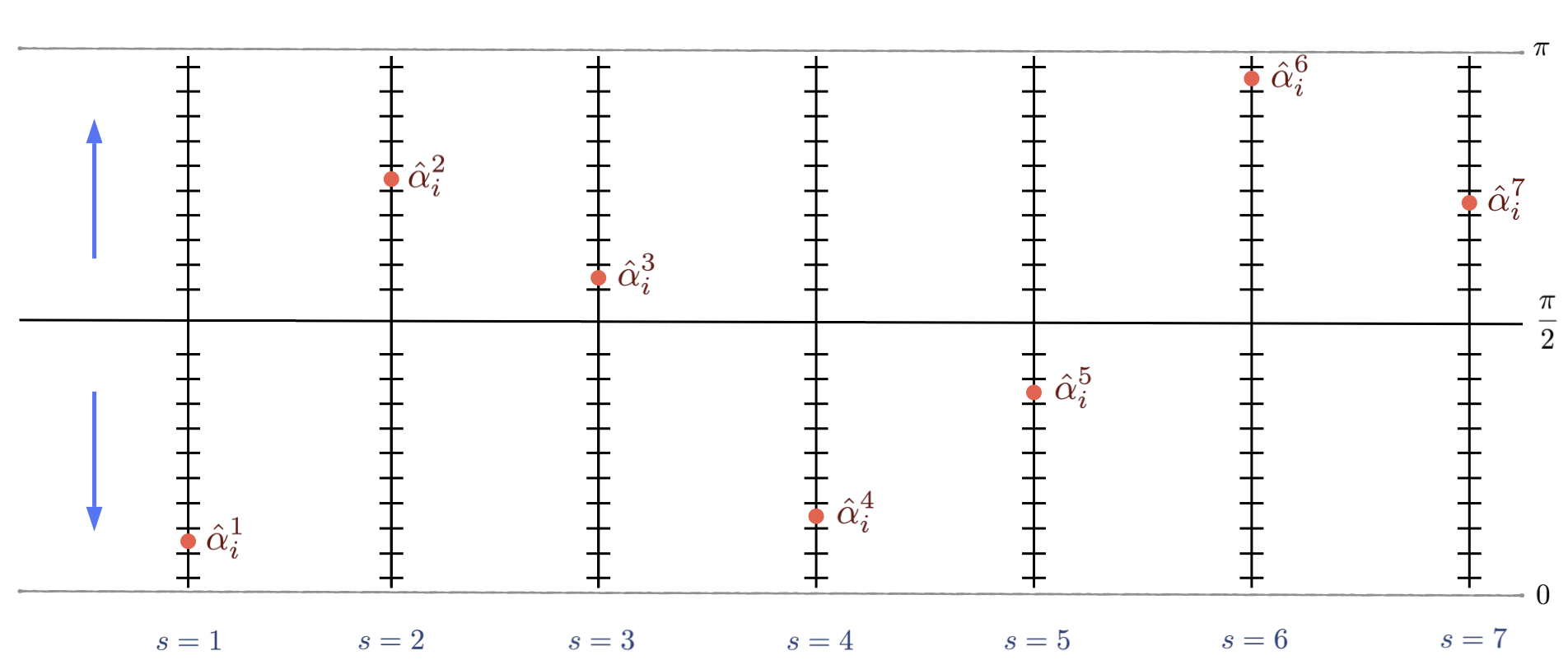}
    \caption{Illustration of the adaptive grid search procedure. In this example, we take $\Theta = \{1, 2, \cdots, 7\}$. The principal sequentially constructs $\hat\alpha_i^3$,  $\hat\alpha_i^5$, $\hat\alpha_i^7$, $\hat\alpha_i^2$, $\hat\alpha_i^4$, $\hat\alpha_i^1$, $\hat\alpha_i^6$, as in this example we have $|\alpha_i^{\sfb_n(3)} - \pi / 2| < |\alpha_i^{\sfb_n(5)} - \pi / 2| < |\alpha_i^{\sfb_n(7)} - \pi / 2| < |\alpha_i^{\sfb_n(2)} - \pi / 2| < |\alpha_i^{\sfb_n(4)} - \pi / 2| < |\alpha_i^{\sfb_n(1)} - \pi / 2| < |\alpha_i^{\sfb_n(6)} - \pi / 2|$. 
    Here, a binary search algorithm is adopted to attain high estimation accuracy (see \cref{alg:binary_search2} for more details). 
    As discussed in the main text, at initialization we have $\cM = \emptyset$. As we implement this algorithm, we sequentially add $3, 5, 7, 2, 4, 1, 6$ to $\cM$. }
    \label{fig:adaptive_grid_search}
\end{figure}

\vspace{5pt}
{\bf \noindent Step 2: Refined estimates of $\balpha_i$ via binary search.} 
Suppose for some $s \in \Theta \setminus \cM$ and some interval $\mathtt{I} \in \{\mathtt{I} _k, \mathtt{I} _k'\} _{k\in [N_0]} $, the conditional sector test returns $\mathtt{True}$. 
Let $\alpha_{\mathtt{I} }  $ denote the center of $\mathtt{I} $, and thus $\mathtt{I}  = (\alpha _{\mathtt{I} } - \iota /2 , \alpha _{\mathtt{I} } + \iota /2)$. 
Then we further perform a binary search procedure by splitting $\mathtt{I}$ in half and performing conditional sector tests on both sub-intervals, i.e., 
$\mathtt{ConSecTest}(\alpha_{\mathtt{I} } - \iota /4  ,\iota /2, s) $ and $\mathtt{ConSecTest}(\alpha_{\mathtt{I} } +  \iota /4  ,\iota /2, s) $.
With high probability one of these tests returns 
$\mathtt{True}$, and we can  further halve that interval and repeat this procedure. 
By performing such binary search algorithm for sufficiently many iterations, we obtain an estimator $\hat \alpha_{i}^{s}$ of $\alpha_{i}^{\sfb_n(s)} $ with high accuracy. 
See \cref{alg:binary_search2} for more details on the binary search procedure.

The complete algorithm is detailed in \cref{alg:grid_search}. 
Note that the two-step procedure stated here can only be applied until $|\cM| = |\Theta | -3$. 
For $|\cM| \geq |\Theta| - 2$, the implementation of mechanism \cref{eq:simple-mechanism2} requires more than $|\Theta|$ rows (the first three rows of the coordination mechanism $\Pi^{\alpha, \delta, s}$ in \eqref{eq:simple-mechanism2} should be free). 
%The reason is that the first three rows of the coordination mechanism $\Pi^{\alpha, \delta, s}$ in \eqref{eq:simple-mechanism2} should be free. 
To estimate the last two reward angles, %$\{ \alpha _i^{\sfb_n(s)} \}_{s\in \Theta \setminus \cM}$,
{we perform a modified conditional sector test that requires fewer free rows (note that the conditional sector test based on mechanism \cref{eq:simple-mechanism2} requires at least three free rows).
We refer the readers to Appendices \ref{sec:appendix-other1} and \ref{sec:appendix-other2} for more details on the modified conditional sector test. }

Finally, we discuss the sample complexity of this two-step procedure. 
For each $s \in \Theta$, in step 1 we search for at most $2N_0$ intervals to determine the rough location of $\alpha_i^{\sfb_n(s)}$.
A single conditional sector test is conducted for each of these intervals. 
Hence, step 1 requires samples from  at most $2N_0 |\Theta| (\lceil \log n \rceil^2 + \lceil \log n \rceil^4)$ rounds. 
In step 2, for each $s \in \Theta$ we need $O(\log n)$ rounds of binary search to accurately estimate $\alpha_i^{\sfb_n(s)}$, with each binary search consisting of a single conditional sector test. 
Namely, the number of rounds required in this step is $O(|\Theta|\lceil \log n \rceil^5)$. 

%\yuchen{comment on the sample complexity. }

%\yuchen{distance subscript changed}

%{\color{red} Comment on the number of rounds required to implement this algorithm. Give a crude estimation. AT most $2N_0$ intervals, thus $2N_0 \times $ number of samples for conditional sector test. In binary search stage, we need $\log n $ iterations, so also $\log n \times T_{\mathrm{sec}} $ number of samples for conditional sector test.}

To evaluate the performance of the proposed method, 
we consider a metric defined as follows. 
Recall that $\{ \hat \alpha_{i:(d-2)}^s \}_{s\in \Theta} $ is an estimate of $\{ \alpha_{i:(d-2)}^\theta  \}_{\theta \in \Theta} $. 
We define $\hat \balpha_{i:(d-2)} = \{ \hat \alpha_{i:(d-2)}^s \}_{s\in \Theta} $ and $\balpha_{i:(d-2)} = \{ \alpha_{i:(d-2)}^\theta  \}_{\theta \in \Theta}$.
%For any $i$, let $\hat \balpha_{i:(d-2)}$ and $ \balpha_{i:(d-2)}$ denote $\{ \hat \alpha_{i:(d-2)}^s \}_{s\in \Theta} $ and $\{ \alpha_{i:(d-2)}^\theta  \}_{\theta \in \Theta} $, respectively. 
We define a metric $ \tilde \dist_i $ 
by letting 
\begin{align}\label{eq:dist-tilde-D}
	\tilde\dist_i \bigl( \balpha_{i:(d-2)}, \hat\balpha_{i:(d-2)} \bigr) :=  \sup_{s \in \Theta}  \sum_{j=i} ^{d-2} | \alpha_j^{\sfb_n(s)} - \hat \alpha_j^{s}| ,
\end{align} 
where we recall that $\sfb_n $ is the permutation defined in \eqref{eq:def-bn-perm}.
%Notice that $\tilde \dist_{d-2}$ coincides with $\dist_{d-2}$ defined in \eqref{eq:dist-D-d-2}.
Moreover, by definition, $\tilde \dist_1$ upper bounds the loss function $\dist$ defined in \eqref{eq:dist-D}. 
In the following lemma, we prove that if $\hat \balpha_{(i+1):(d-2)}$ is sufficiently accurate for $\balpha_{(i + 1):(d - 2)}$,  then the two-step algorithm can use it as the input and obtain a sufficiently accurate $\hat \balpha_{i:(d-2)}$.
Moreover, the total number of samples required is $\mathrm{polylog}(n)$. 
By induction hypothesis, we see that the reward angles can be accurately estimated in a sample-efficient manner. 

Formally speaking, we present the theoretical guarantee for the two-stage procedure as \cref{lemma:matching} below. 
Combining this lemma and \cref{lemma:binary-search}, we obtain \cref{thm:reward-function}.

%Therefore, combining this lemma with \cref{lemma:binary-search} and using recursion, we obtain \cref{thm:reward-function}. 

%

%Note that in the base case $i = d - 2$, this is already true by \cref{lemma:binary-search} and \Cref{eq:bijection-b}. 

%Formally speaking, in this section we seek to establish the following lemma. We comment that if \cref{lemma:matching} is true, then \cref{thm:reward-function} immediately follows from this lemma. 
%
\begin{lemma}\label{lemma:matching}
Under the same assumptions made in \cref{thm:reward-function}, 
further suppose we have obtained a sufficiently accurate  $\hat \balpha_{(i+1):(d-2)}$ such that $\tilde \dist_{i+1}(\balpha_{(i+1):(d-2)}, \hat \balpha_{(i+1):(d-2)}) \leq 4(d - 2 - i) \cdot n^{-4^{i + 4}}$, 
%whose error is at most $4(d - 2 - i) \cdot n^{-4^{i + 4}} $ measured by $\tilde\dist_{i+1}$ in \eqref{eq:dist-tilde-D}. 
then there exists $n_0 \in \NN_+$, such that for all $n \geq n_0$, 
with probability at least $1 - C_2 \cdot n^{-50}$, \cref{alg:grid_search}
outputs $\hat \balpha_{i:(d-2)}$ such that $\tilde \dist_i(\balpha_{i:(d-2)}, \hat \balpha_{i:(d-2)}) \leq 4 (d  - 1 - i)  \cdot n^{-4^{i + 3}}$. 
%with error at most $4(d - 2 - i) \cdot n^{-4^{i + 4}}$, measured by distance $\tilde \dist_i$ defined in \eqref{eq:dist-tilde-D}. 
Moreover,  the total number of rounds (samples) required to implement \cref{alg:grid_search} is no more than $C_1 \cdot (\log n)^6$. 
Here, $C_1, C_2 \in \RR_+$ and $n_0 \in \NN_+$ are instance-dependent constants that depend only on $(\mathscr{P}, \varphi_0)$. 
 
%with probability at least $1 - C_{\sfP}n^{-4^{d + 2}}$,
%Assume $n \geq n_0$. For all $i \in \{d - 3, d - 4, \cdots, 1\}$, given $\cZ_{i + 1}$ satisfying $\tilde \dist(\cZ_{i + 1}, \bar{\cZ}_{i + 1}) \leq \pi(d - 2 - i) \cdot n^{-4^{i + 4}}$, there exists an algorithm that uses no more than $C_1 \cdot (\log n)^5$ samples and generates $\cZ_i$, such that with probability at least $1 - C_2 \cdot n^{-50}$, this algorithm returns $\cZ_i$ such that  
	 
\end{lemma}
%\yw{replace pi with a constant}
\begin{proof}[Proof of \cref{lemma:matching}]
    We prove \cref{lemma:matching} in Appendix \ref{sec:proof-lemma_matching}.
\end{proof}

%{\color{red} STOP HERE. Finish. }

\begin{algorithm}
\caption{Binary search $\mathtt{BinSearch}(\mathtt{I}, s)$}
\label{alg:binary_search2}
\textbf{Input:} index $ s$, interval $\mathtt{I} $,  $\hat \alpha_{(i+1):(d-2)}^s$, $\{\hat  \alpha_{ i: (d - 2)}^h \}_{h \in \cM}$, and the desired accuracy level $\varepsilon$; 
\begin{algorithmic}[1]

\State Initialization: Let $L $ be the length of $\mathtt{I}$ and $u_1$, $u_2$ be the end points of $\mathtt{I}$  with $u_1 < u_2$;
\State While $L >  \varepsilon$ do:
\State \qquad Perform $\mathtt{ConSecTest}(u_1 + L/4, L/2, s)$;
\State \qquad If $\mathtt{True}$, set $u_2 \leftarrow  u_1 + L/2$ and $L\leftarrow  L/2 $, otherwise set $u_1 \leftarrow u_1 + L/2$ and $L\leftarrow  L/2 $;
\State Return $\hat \alpha_i^{s} = (u_1 + u_2) / 2$.  
\end{algorithmic}
\end{algorithm}

\begin{algorithm}
\caption{Algorithm for estimating $\balpha_{i:(d - 2)} = \{ \alpha _{i:(d - 2)}^{\theta} \}_{\theta \in \Theta} $}
\label{alg:grid_search}
\textbf{Input:}  parameter $n$ that reflects a desired accuracy level, estimated angles $\{ \hat \alpha_{(i+1):(d-2)}^s \}_{s\in \Theta} $.

 %, \{\hat\alpha_i^{s_j}: j \in [k]\}$, $k$;
\begin{algorithmic}[1]
\State Define $N = \lceil \log n \rceil$, $\iota = \pi / (2N)$, $u = \sqrt{\iota} + u_0$, and $\{ \mathtt{I}_j, \mathtt{I}_{j}'\}_{j \in [N_0]}$ as in \eqref{eq:define_intervals}, where $u_0 \sim \Unif[0, \iota)$;
\item Set $\varepsilon = 4 n^{-4^{i+3}}$ in $\mathtt{BinSearch}$ (\cref{alg:binary_search2}) and initialize with $\cM \leftarrow \emptyset$. 
\State For $j = 1, 2, \ldots, N_0$, if we also have $|\cM| \leq |\Theta | -3 $ do:
\State \qquad Set $\mathcal{K}, \mathcal{K}' \gets \emptyset$. 
\State \qquad \texttt{\textcolor{blue}{// Perform conditional sector test for all $s \in \Theta \backslash \cM$ on the intervals $\mathtt{I}_j$}}
\State \qquad For all $s \in \Theta \setminus \cM$, perform $\mathtt{ConSecTest}( \alpha, \iota, s) $ with $\alpha = \pi /2 + u + (j -1/2) \cdot \iota  $, 
\State \qquad if the test returns $\mathtt{True}$, update $\mathcal{K} \gets \mathcal{K} \cup \{s\}$. 
\State \qquad \texttt{\textcolor{blue}{// Perform conditional sector test for all $s \in \Theta \backslash \cM$ on the intervals $\mathtt{I}_j'$}}
\State \qquad For all $s \in \Theta \setminus \cM$, perform $\mathtt{ConSecTest}( \alpha, \iota, s) $ with $\alpha = \pi /2 - u - (j -1/2) \cdot \iota  $, 
\State \qquad if the test returns $\mathtt{True}$, update $\mathcal{K}' \gets \mathcal{K}' \cup \{s\}$.
\State \qquad \texttt{\textcolor{blue}{// Perform binary search for all intervals that return $\mathtt{True}$}}
\State \qquad For all $s \in \mathcal{K}$, perform $\mathtt{BinSearch}(\mathtt{I}_j, s)$ and obtain $\hat \alpha_{i}^s$ with $\eps = 4n^{-4^{i + 3}}$.
\State \qquad For all $s \in \mathcal{K}'$, perform $\mathtt{BinSearch}(\mathtt{I}_j', s)$ and obtain $\hat \alpha_{i}^s$ with $\eps = 4n^{-4^{i + 3}}$.

%\State \qquad If the test returns $\mathtt{True}$, perform $\mathtt{BinSearch}(\mathtt{I}_j, s)$ and obtain $\hat \alpha_{i}^s$. Set $\cM \leftarrow \cM \cup \{ s\}$.

%\State \qquad \texttt{\textcolor{blue}{// Perform conditional sector test for all $s \in \Theta \backslash \cM$ on the interval $\mathtt{I}_k'$}}
%\State \qquad For all $s \in \Theta \setminus \cM$, perform $\mathtt{ConSecTest}( \alpha, \iota, s) $ with $\alpha = \pi /2 - u - (j -1/2) \cdot \iota  $.
%\State \qquad If the test returns $\mathtt{True}$, perform $\mathtt{BinSearch}(\mathtt{I}_j', s)$ with $\eps = 4n^{-4^{i + 3}}$, and obtain $\hat \alpha_{i}^s$. 
\State \qquad Set $\cM \leftarrow \cM \cup \mathcal{K} \cup \mathcal{K}'$.
 
			%\State \qquad Terminate the outer for loop if $k \geq |\Theta| - 3$; 

        \State Implement {the algorithms stated in Appendices \ref{sec:appendix-other1} and \ref{sec:appendix-other2} to estimate the rest of the two angles $\{ \hat \alpha _i^{s}\}_{s \in \Theta \setminus \cM} $.}
	\State  Return $ \{ \hat\alpha_ {i:(d-2)}^s \}_{s \in \Theta }    $; 
\end{algorithmic}
\end{algorithm}

\section{Conclusion}
\label{sec:conclusion}

In this paper, we consider online learning under a generalized principal agent model, where a principal and an agent interact repeatedly over multiple rounds. 
We consider a challenging setting for the principal, 
where the agent behaves strategically, 
while the principal has access to very limited information and does not observe the agent's preferences, actions or types.
We propose a novel online learning algorithm for the principal that provably achieves $\tilde{O}(\sqrt{T})$ regret, which is nearly optimal for problems of this kind. 
Our method constitutes of several novel algorithmic components: optimistic-pessimistic planning, conditional and unconditional sector tests, and a matching procedure, 
which may be of independent interest for a wide range of online learning problems. 

While our work makes significant progress, it also has several limitations that open up promising directions for future research, which we discuss below. 

\begin{itemize}
    \item[--] \textbf{Generalization to continuous action and type spaces: } 
    Our algorithm assumes discrete action and type spaces. However, many important applications, such as dynamic pricing \citep{gallego1994optimal} and contract design \citep{10.1093/jleo/7.special_issue.24} involve continuous action or type spaces, posing additional modeling and algorithmic challenges. 
    Extending our approach to handle continuous domains is an interesting direction for future work.
    \item[--] \textbf{Incorporating contextual information:}
    Another interesting direction  is to incorporate contextual information into the learning process. 
    In many real-world applications, the interaction between the principal and the agent depends not only on the agent's types and actions, but also on external covariates like market conditions, user attributes, and historical behaviors. 
    Integrating such contextual signals into the model could enable more fine-grained decision-making and improve the overall performance of our algorithm. 
    \item[--] \textbf{Interaction with multiple agents: }
    It is also interesting to see if we can extend the current framework to accommodate the existence of multiple agents. 
    In many practical scenarios such as crowdsourcing and online marketplaces, the principal needs to interact with multiple heterogeneous agents at the same time. 
    This introduces new complexities, including competition and cooperation among agents, which could significantly affect the principal's learning dynamics. 
    Modeling and analyzing these interactions would require more sophisticated algorithmic designs, and remain an interesting open problem.
\end{itemize}

\newpage 
\bibliographystyle{ims}
\bibliography{bib}

\begin{appendices}
	\newpage 
\section{Background on classical LinUCB algorithm} \label{sec:classical-linUCB}
 
%\subsubsection{LinUCB without delay}

In this section, we introduce the classical LinUCB algorithm without any delay mechanism to provide the readers with background knowledge. 

\vspace{5pt}
{\bf \noindent Setting of linear bandit.} We focus on the linear bandit model considered in \cite{abbasi2011improved}.  Suppose the learner (in our setting this role is taken by the principal) is given a decision set $\cI \subseteq \RR^q$. In every round $t$, the learner chooses an action $M_t$ from this decision set, and observes a reward $r_t = \langle M_t, \beta_{\ast} \rangle + \eta_t$, where $\beta_{\ast} \in \RR^q$ is an unknown parameter vector, and $\eta_t$ is a random noise satisfying 
\begin{align*}
	\E\big[\eta_t \mid M_{1:t}, \eta_{1:(t - 1)} \big] = 0, \qquad \E\left[ e^{\lambda \eta_t} \mid M_{1:t}, \eta_{1:(t - 1)} \right] \leq \exp \left( \frac{\lambda^2 G^2}{2} \right)
\end{align*}
for some positive constant $G$.
Let $M_{\ast} := \argmax_{M \in \cI} \langle M, \beta_{\ast} \rangle$. We define the pseudo-regret for the first $n$ rounds as
\begin{align*}
	R_n := n \cdot  \langle M_{\ast} , \beta_{\ast} \rangle - \sum_{t = 1}^n \langle M_{t} , \beta_{\ast} \rangle.
\end{align*}
In addition, we assume that $\|\beta_{\ast}\|_2 \leq S$ and $\cI$ belongs to an $\ell_2$-ball that has radius $L$ centered at the origin.

\vspace{5pt}
{\bf \noindent LinUCB algorithm.}
To minimize the regret, any successful algorithm needs to balance the exploration and exploitation trade-off. 
That is, the algorithm should explore the action space to collect a informative dataset that leads to a good estimate of the unknown parameter $\beta_{\ast}$. Simultaneously, the algorithm should exploit its estimate of $\beta_{\ast}$ to take actions that maximize the reward.
LinUCB algorithm \citep{abbasi2011improved} 
adopts the {optimism in the face of uncertainty} principle to balance the exploration and exploitation trade-off. 
In specific, the algorithm consists of two major steps: (i) constructing a confidence set containing the true parameter $\beta_{\ast}$ using the current data, and (ii) selecting an action via optimistic planning using the confidence set.

 \vspace{5pt}
{\bf \noindent Step 1: Constructing the confidence ellipsoid.} Suppose the learner is about to start the $(t+1)$-th round for some $t \geq 0$. 
Then she 
has collected data $M_{1:t}$ and $r_{1:t}$ up to round $t$. 
Let $\hat\beta_t$ be the $\ell^2$-regularized least-squares estimate of $\beta_{\ast}$ with regularization parameter $\lambda > 0$: 
\begin{align*}
	\hat\beta_t := \big( M_{1:t}^{\top} M_{1:t} + \lambda 
 I_q \big)^{-1} M_{1:t}^{\top} r_{1:t}, 
\end{align*}
where $M_{1:t} \in \RR^{t \times q}$ has rows $M_1^{\top}, \cdots, M_t^{\top}$, and $r_{1:t} = (r_1, \cdots, r_t)^{\top}$. 
We define the confidence ellipsoid as 
\begin{align}\label{eq:Ct-appendix}
	\cC_t := \left\{ \beta \in \RR^q: \langle \beta - \hat\beta_t, \Omega_t(\beta  - \hat\beta_t) \rangle^{1/2} \leq G \cdot \sqrt{2 \log \left( \frac{\det (\Omega_t)^{1/2} \det (\lambda I_q)^{-1/2}}{\delta} \right)} + \lambda^{1/2} S \right\}, 
\end{align}
where $\Omega_t = \lambda I_q + M_{1:t}^{\top} M_{1:t} \in \R^{q \times q}$. 
It can be shown that $\beta_{\ast} \in \cC_t$ with probability at least $1-\delta$. 

\vspace{5pt}

{\noindent \bf Step 2: Optimism in the face of uncertainty.}
Next, the learner chooses an optimistic parameter $\tilde \beta_t$ from the confidence ellipsoid as a surrogate reward parameter, and an action $M_{t+1}$ from the action space that maximizes the reward with respect to $\tilde \beta_t$.
That is, she solves the following optistic planning problem:  
\begin{align}\label{eq:Mt}
	(M_{t + 1}, \tilde \beta_t) = \argmax\limits_{(M, \beta) \in \cI \times \cC_{t }} \langle M, \beta \rangle. 
\end{align}
She then takes action $M_{t + 1}$,receives a reward $r_{t+1}$, and continue.  

\vspace{5pt}

\vspace{5pt}

{\bf \noindent Theoretical guarantee of LinUCB.}
It can be shown that LinUCB successfully balances the exploration-exploitation trade-off by achieving a sublinear regret.
In specific, we have the following theorem. 
%
%\yw{use this only in the proof.}
\begin{lemma}[Regret upper bound for LinUCB as given in \cite{abbasi2011improved}]\label{thm:linUCB}
	With probability at least $1 - \delta$, for all $n \in \NN_+$, the regret incurred by LinUCB after $n$ rounds satisfies 
	\begin{align*}
		R_n \leq 4LS \sqrt{nq \log (\lambda + nL / q)} \cdot \big ( \lambda^{1/2} S + G\sqrt{2 \log (1 / \delta) + q \log (1 + nL / (\lambda q))} \big). 
	\end{align*}
	
\end{lemma}

Here recall that  $L$ is an upper bound on the $\ell_2$-norm of the actions in $\cI$, $S$ is an upper bound on the $\ell_2$-norm of the true parameter $\beta_{\ast}$, and $G$ stands for the noise scale.
If we regard $L$, $G$, and $S$ as constants, then \cref{thm:linUCB} implies that the regret of LinUCB is $\tilde \cO(q\cdot \sqrt{n})$, where $\tilde \cO(\cdot )$ omits logarithmic factors.

\newpage 
\section{Unknown $T$: Pessimistic-Optimistic LinUCB with a doubling trick} \label{sec:unknownT}

If $T$ is unknown, then we propose to adopt an adaptive procedure with a doubling trick.
To be precise, 
we design the algorithm based on an episodic structure: we divide the time horizon into ``episodes'' that have increasing lengths.
Episode $k$ starts with a short phase in which we implement $\cA({n_k})$ for some $n_k \in \NN_+$. 
In this stage, an estimate of the feasible region of \eqref{eq:LP} is generated based on the estimates of the reward angles produced by $\cA({n_k})$.
We denote the resulting feasible region estimate by $\cI_{n_k}$, whose definition of is given in \eqref{eq:feasible-set}. 
The algorithm $\cA(n_k)$ uses only a vanishingly small proportion of samples in episode $k$. 
Given an approximation to the feasible region (this should be interpreted as the action space in the bandit setting), 
as discussed in the main text,
we can then implement the pessimistic-optimistic planning to achieve a small regret. 
%
%minimizing the principal's regret roughly speaking amounts to solving a stochastic linear bandit problem. 
%We then implement a simple variant of the prominent  LinUCB algorithm in the literature, using the rest of the samples in the $k$-th episode. 
%LinUCB enjoys sample-efficient properties and is known to achieve $\tilde O(\sqrt{T})$ regret \citep{auer2002using,abbasi2011improved}. Here, we employ the delayed feedback LinUCB (see \cref{sec:delayed-linUCB} for details). As discussed in \cref{sec:delayed-observations}, the idea is to delay the use of samples in order to prevent the agent from acting strategically. 
%
Specifically, for each episode $k = 1, 2, \cdots$, we simply do the following: set $n_k = 2^k$, and run \cref{alg:delayed-UCB} with $n = n_k$. 
%\begin{enumerate}
%	\item  Set $n_k = 2^k$. Run $\cA({n_k})$ and get an update of the reward angle set $\hat\balpha$. This further leads to an estimate of the feasible region $\cI_{n_k}$ as of \eqref{eq:feasible-set}. This phase uses samples from $m_k = \lceil \log n_k \rceil^6$ rounds. 
 %From \cref{thm:reward-function} we know that $m_k \leq C_3 (\log n_k)^5$. 
% We then run an all-one dummy mechanism for $\ell_k = \lceil \log n_k \rceil^2$ rounds. 
%	\item Perform $n_k$ pessimistic-optimistic planning updates as of \cref{alg:delayed-UCB}, with $\delta = n_k^{-10}$ and $\lambda = 1$. We
% Run LinUCB with delayed feedbacks with action space $\cI_{n_k}$ and $(n_k + 1) \cdot (\lceil \log n_k \rceil^2 + 1)$ samples.  
%\end{enumerate}
A formal description of our algorithm pipeline is given as \cref{alg:delayed-UCB-episodic}. We also provide a schematic representation of this algorithm in Figure \ref{fig:episodic}. 

%\yw{If we know $T$, then we can choose not to use the doubling trick. 
%In the beginning of this section, write a few paragraphs to motivate tha algorithm. We can state a simplified version of our algorithm, i.e., only have two stages. Robustify the constraint, all candidate set, how to deal with non-myopic. After the angle part, we state out algorithm. Then we talk about how to estimate the angle and how to do LinUCB. If $T$ unkonwn, we use doubling trick}
\begin{figure}[ht]
	\centering \includegraphics[width=0.8\textwidth]{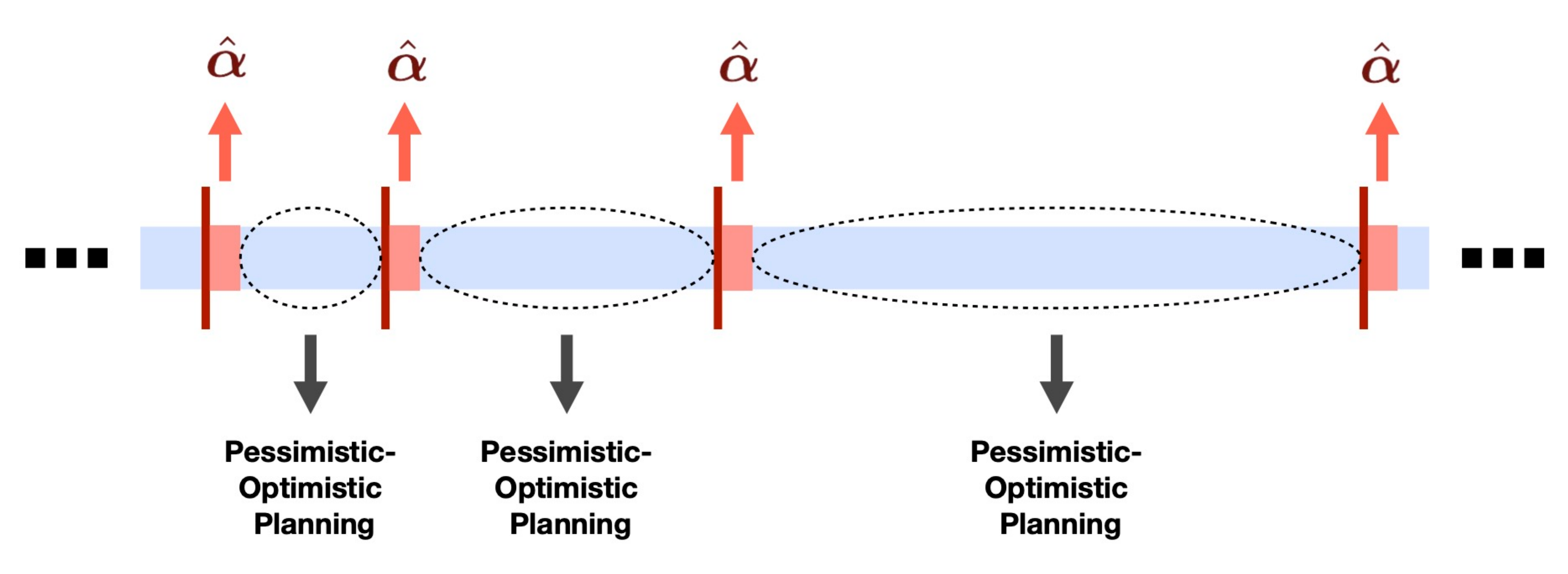}
	\caption{Schematic representation of the episodic structure. In the above figure, the dark red vertical lines separate episodes that have increasing lengths. In the beginning of episode $k$, an estimate $\hat\balpha$ is derived using $\cA(n_k)$. This phase uses only a vanishingly small proportion of all samples from episode $k$, and is  indicated by the light red regions in the figure. A  LinUCB algorithm that features delayed feedback and pessimistic-optimistic planning is implemented using the rest of the samples from the current episode. Samples used in this step are colored in light blue. }
    \label{fig:episodic}
\end{figure}

\subsection{Algorithm pipeline based on the Doubling Trick}
\label{sec:delayed-linUCB}

We summarize in this section the pipeline of our algorithm as well as some implementation details. 
Specifically, we employ the \emph{Doubling Trick} \citep{auer1995gambling} to turn a fixed-$T$ algorithm (\cref{alg:delayed-UCB}) into an anytime one. 
In the $k$-th episode, by the Doubling Trick we set $n_k = 2^k$. 
Recall that \cref{alg:delayed-UCB} delays the usage of observations for $\ell = \lceil \log n_k \rceil^2$ rounds.
To emphasize the dependency on $k$, we set $\ell_k = \lceil \log n_k \rceil^2$. 

The algorithm pipeline with an episodic structure is then presented as \cref{alg:delayed-UCB-episodic}. 
We implement dummy mechanisms between episodes to ensure sufficient delay. 
%To this end, we adopt a simple while effective trick: Between two rounds $t$ and $t + 1$ in the classical LinUCB algorithm, we copy the mechanism used in round $t$ and repeatedly apply it for another $\ell$ rounds (when $t = 0$, we simply run a dummy mechanism $\Pi = 1_{|\Theta| \times d} / d$). The observations from these $\ell$ rounds are not used for policy learning at any stage. In this way, we automatically achieve $\ell$ rounds of delay in the sense stated in \cref{eq:ell-t-sigma-alg}. 
%In addition, we shall see that inserting these duplicate rounds have small effect on the order of the regret. 
%Now we are ready to state our entire algorithm (skipping the details of $\cA$), which we present as \cref{alg:delayed-UCB}.

%Our strategy for incorporating delayed feedbacks is stated as \cref{alg:delayed-UCB}. 

\begin{algorithm}
	\caption{Algorithm pipeline}\label{alg:delayed-UCB-episodic}
%\textbf{Input:} $n$;	
\begin{algorithmic}[1]
\For{each episode $k = 1,2, \cdots, $}
    \State Run \cref{alg:delayed-UCB} with $n = n_k$ and $T = (n_k + 2) (\lceil \log n_k \rceil^2 + 1) + \lceil \log n_k \rceil^6$;
    \State Run a dummy mechanism $\mathds{1}_{|\Theta| \times d} / d$ for $\ell_k = \lceil \log n_k \rceil^2$ rounds;
%		\State Set $n_k \gets 2^k$;
%         \State \texttt{\textcolor{blue}{// Estimating the  feasible region}}
% 		\State Implement $\cA({n_k})$ using $m_k$ samples, which gives $\hat\balpha = \{\hat \alpha^s: s \in \Theta\}$ and hence $\cI_{n_k}$ as of \eqref{eq:feasible-set}; 
% 		\State $\ell \gets \lceil \log n_k \rceil^2 + 1$,   $\lambda \gets 1$, $\delta \gets n_k^{-10}, \cS \gets \emptyset$; 
%             \State \texttt{\textcolor{blue}{// Create delayed feedback via running a dummy mechanism}}
% 		\State Start $\ell$ new rounds and run a dummy mechanism $\mathds{1}_{|\Theta| \times d} / d$;  
%         \State \texttt{\textcolor{blue}{// Pessimistic-optimistic planning}}
% 		\For{$j \in [n_k]$}
% 			\State Compute $\cC_{(j - 1)(\ell + 1) + m_k}$ as of \cref{eq:Ct};
% 			\State Compute $\Pi_{(j - 1)(\ell + 1) + m_k + 1}$ as of \cref{eq:Mt}; 
% 			 \State Start a new round, principal announces $\Pi = \Pi_{(j - 1)(\ell + 1) + m_k + 1}$ and receives her reward;
%     \State $\cS \gets \cS \cup \{(j - 1)(\ell + 1) + m_k + 1\}$;
% 			 \For{$l \in [\ell]$}
%                     \State \texttt{\textcolor{blue}{// Create delayed feedback by following the previous mechanism}}
% 			 	\State Start a new round, principal announces $\Pi$ and receives her reward; 
% 			 \EndFor
% 		\EndFor
% %		\State Run LinUCB with action space $\cI_{n_k}$ and $n_k$ samples;
\EndFor
\end{algorithmic}
\end{algorithm}

% \begin{remark}
%     For $k$ that is larger than a constant threshold depending only on $(\{\bar{v}_{\theta}: \theta \in \Theta\}, \varphi_0, B)$, we always have $\lceil \log n_k\rceil^2 \geq \bar{\ell}^{(k)}$. Hence, in this case \eqref{eq:response} holds for all samples within episode $k$. 
% \end{remark}

%It might be helpful to provide some intuitions for \cref{alg:delayed-UCB}: Episodes are indexed by $k$, and every episode starts with a short phase in which we perform reward function estimation. According to \cref{thm:reward-function}, this phase uses $O(\mathsf{poly}(k))$ samples. The reward estimating phase is followed by $\ell$ rounds of interaction with a dummy mechanism. The goal of this part is not for minimizing the regret, but solely for delaying the use of estimated reward functions for at least $\ell_L^{(n_k)}$ rounds. Next, the algorithm enters the LinUCB phase, with the action space in this part estimated by $\cA_{n_k}$ in the reward estimating phase. This phase consumes $O(2^k)$ samples and roughly speaking achieves $O(2^{k / 2})$ regret. Regret incurred in the LinUCB phase is the only major contribution for the total regret. Sufficient delay is also required in order to apply \cref{lemma:best-response2}. 

 %we compute the confidence ellipsoids based on previous observations in the current episode, and pick our mechanism based on joint reward maximization. Observations from this round is again delayed until use for another $\ell$ rounds to ensure incentive-compatible reportings. 

\subsection{Regret upper bound}
\label{sec:regret-analysis-unknownT}

In this section, we introduce how to derive a regret upper bound for \cref{alg:delayed-UCB-episodic}. Our analysis is conducted conditioning on the event that the reward angle estimates are sufficiently accurate, which by \cref{thm:reward-function} occurs with high probability. 
To be specific, for $k \in \NN_+$ we define the set
\begin{align}\label{eq:event-cEk}
	\cE_k := \left\{ \mbox{The output of $\cA({n_k})$ satisfies $\dist(\balpha, \hat\balpha) \leq C_2 |\Theta|^{-1} n_k^{-50}$} \right\}.
\end{align}
According to \cref{thm:reward-function}, we know that $\P(\cE_k) \geq 1 - C_1 n_k^{-50}$. Since reward angles uniquely determine the feasible region, \eqref{eq:event-cEk} implies that for a  sufficiently large $k$, the difference between the estimated feasible region $\cI_{n_k}$ and the ``true'' feasible region $\IC(\bar{V}_{n_k}^{\ast})$ (recall this is defined in \cref{sec:complexity-est-angle}) should also be small.  

We make this claim precise in \cref{lemma:LP-close} below. 
Recall that $u_{\ast}$ is the  value of \eqref{eq:LP}. For $k \in \NN_+$ we denote by ${u}_k$ the value of the following linear programming problem: 
\begin{align}\label{eq:LPk}
\begin{split}
	& \mbox{maximize }\sum_{s \in \Theta} f(\sfb_{n_k}^{\ast}(s)) \sum_{x\in \cX} \pi (s, x) U(\sfb_{n_k}(s), x), \\
	& \mbox{s.t. } \left( \pi(s, x) \right)_{s \in \Theta, x \in \cX} \in \cI_{n_k},
\end{split}\tag{$\mbox{LP}_k$}
\end{align}
where we recall $\sfb^{\ast}_{n_k}$ is defined in \eqref{eq:bn}. We then upper bound the difference between $u_k$ and $u_{\ast}$. 
\begin{lemma}\label{lemma:LP-close}
	There exist constants $k_{\mathsf{LP}} \in \NN_+$ and $C_{\mathsf{LP}} \in \RR_+$ that depend only on $(\{\bar{v}_{\theta}: \theta \in \Theta\}, \varphi_0, B)$, such that for all $k \geq k_{\mathsf{LP}}$, on the event $\cE_k$, it holds that $u_{\ast} \geq u_k \geq u_{\ast} - 8^{-k} \cdot C_{\mathsf{LP}}$. Recall that $B$ is from Assumption \ref{assumption:model}. 
\end{lemma}
\begin{proof}[Proof of \cref{lemma:LP-close}]
    We postpone the proof of \cref{lemma:LP-close} to Appendix \ref{sec:proof-lemma:LP-close}.
\end{proof}
\cref{lemma:LP-close} says that replacing the true feasible region of \eqref{eq:LP} with an estimated one has small effect on the optimal value of \eqref{eq:LP}, provided that the reward angle estimation is sufficiently accurate. 
Finally, with \cref{lemma:LP-close} we are ready to establish a regret upper bound for \cref{alg:delayed-UCB-episodic}. We formally present the upper bound in \cref{thm:main-episodic} below.

\begin{theorem}\label{thm:main-episodic}
	We assume Assumptions  \ref{assumption:feasible}, \ref{assumption:not-all-one}, \ref{assumption:model} and \ref{assumption:angle}. Then \cref{alg:delayed-UCB-episodic} satisfies the following regret upper bounds:
 \begin{enumerate}
     \item (Expected reward) There exists a positive constant $C_{\ast}$ that depends only on $(\mathscr{P}, \varphi_0)$, such that for all $T \in \NN_+$ and $T \geq 2$, we have
     \begin{align*}
         \E\left[\regret(T)\right] \leq C_{\ast} \sqrt{T} \cdot (\log T)^3. 
     \end{align*}
     \item (Asymptotic regret) For any $g_0 > 3$, as $T \to \infty$ 
     \begin{align*}
	\frac{\regret(T)}{\sqrt{T} \cdot  (\log T)^{g_0} } \overset{a.s.}{\to} 0. 
    \end{align*}
    \item (High-probability bound) There exists a positive constant $C_{\ast}'$ that depends only on $(\mathscr{P}, \varphi_0)$, such that for all $T \in \NN_+$, with probability at least $1 - \eps$, 
    \begin{align*}
        \regret(T) \leq C_{\ast}' \cdot \left( \sqrt{T} \cdot (\log T)^3 + \eps^{-1/10} \cdot (\log \eps^{-1})^2 \right). 
    \end{align*}
 \end{enumerate}

% Then almost surely as $T \to \infty$ we have
	%
%\begin{align*}
%	\frac{\regret(T)}{\sqrt{T} \cdot  (\log T)^{4} } \overset{a.s.}{\to} 0. 
%\end{align*} 
	%
	%where $g_0 \in \NN_+$ is independent of $T$. 
\end{theorem}

\begin{proof}[Proof of \cref{thm:main-episodic}]
    We defer the proof of the  theorem to Appendix \ref{sec:proof-thm:main}.
\end{proof}

\newpage

\section{Proofs for regret upper bound}

We present in this section proofs related to upper bounding the pseudo-regret. 

\subsection{The revelation principle}
\label{sec:proof-lemma:revelation}

We prove in this section \eqref{eq:LP1} is feasible, and that the value of \eqref{eq:OPT} is the same as that of \eqref{eq:LP1}.
Namely, we establish the revelation principle \cref{lemma:revelation}.

\begin{proof}[Proof of \cref{lemma:revelation}]

We first prove the feasibility result. 
One can verify that there exists $\eps > 0$, such that 
\begin{align}\label{eq:ball_in_simplex}
	\left\{ x\in \RR^d: \|x - \mathds{1}_d / d\|_2 \leq \eps, \, \langle x, \mathds{1}_d \rangle = 1 \right\} \subseteq \left\{ x\in \RR^d: x \geq 0, \, \langle x, \mathds{1}_d \rangle = 1 \right\}. 
\end{align} 
By Assumption \ref{assumption:not-all-one} the normalized reward vectors $\bar{v}_{\theta}$ are distinct, we can simply set $\Pi_{\theta} := (\pi (\theta, x))_{x \in \cX}$ to be $\mathds{1}_d / d + \eps \bar{v}_{\theta}$ and let $\Pi = \{\Pi_{\theta}\}_{\theta \in \Theta} $. By \eqref{eq:ball_in_simplex}, $\Pi_{\theta} $ is a probability distribution over $\cX$. 
Moreover, since the reward vectors are disjoint, 
for any distinct $\theta$ and $\theta'$, we have 
$\langle v_{\theta} , \Pi_{\theta} \rangle > \langle v_{\theta} , \Pi_{\theta'}\rangle$. 
Thus, such a choice of $\Pi$ 
 falls inside the feasible region of \eqref{eq:LP1}. 
 Since the inequality is strict, a positive constant $\delta$ that satisfies the desired property also exists.
 This proves the feasibility result.

We then prove that the two optimization problems, \eqref{eq:OPT} and \eqref{eq:LP1}, have the same value. 
We denote by ${V}_{\mathsf{OPT}}$ the value of \eqref{eq:OPT} and ${V}_{\mathsf{LP}}$ the value of \eqref{eq:LP1}. 
Note that when the constraints of \eqref{eq:LP1} are satisfied, \eqref{eq:OPT} and \eqref{eq:LP1} have the same objective function while the feasible region of \eqref{eq:LP1} is a subset of that of \eqref{eq:OPT}. 
As a result, ${V}_{\mathsf{LP}} \leq {V}_{\mathsf{OPT}}$. 

%{\color{red} STOP}
It remains to prove that ${V}_{\mathsf{LP}} \geq {V}_{\mathsf{OPT}}$. 
We notice that by definition, %under Assumption \ref{assumption:not-all-one}, 
for any $w \in (0,1)$, there exists $\Pi^w \in \RR^{|\Theta| \times |\cX|}$ that is inside the feasible region of \eqref{eq:OPT}, such that $\Pi^w$ achieves value no smaller than $\mathsf{V}_{\mathsf{OPT}} - w$ for \eqref{eq:OPT}. 
In addition, it is not hard to see that there exists $\lambda_w > 0$, such that the following containment relationship is satisfied: 
\begin{align*}
	& \left\{ \mathds{1}_d / d + (1 - w) (x - \mathds{1}_d / d) + \lambda_w y: x \geq 0, \, \langle x, \mathds{1}_d \rangle = 1, \, \|y\|_2 = 1, \, \langle y, \mathds{1}_d \rangle = 0 \right\} \\
	& \qquad \subseteq \left\{ x\in \RR^d: x \geq 0, \, \langle x, \mathds{1}_d \rangle = 1 \right\}. 
\end{align*}
Furthermore, $\lambda_w \to 0^+$ as $w \to 0^+$. %Note that $\Pi_w$ is also inside the feasible region of \eqref{eq:OPT}. 

Recall that $r_{\theta, \Pi^w}$ comes from \eqref{eq:OPT}, and stands for a myopic agent's reported type when the principal implements mechanism $\Pi^w$ and the agent has true type $\theta$.
We denote by $\Pi^{w}_{ \theta}$ the row of $\Pi^w$ corresponds to index $\theta$. 
Let 
\begin{align*}
	\tilde\Pi^{w}_{ \theta} := \mathds{1}_d / d + (1 - w) \cdot (\Pi^w_{r_{\theta, \Pi_w}} - \mathds{1}_d / d) + \lambda_w \bar{v}_{\theta}. 
\end{align*}
Then as $\Pi^{w}_{ r_{\theta, \Pi_w}} \subseteq \{ x\in \RR^d: x \geq 0, \, \langle x, \mathds{1}_d \rangle = 1 \}$, one also has $\tilde \Pi^{w}_{ \theta} \subseteq \{ x\in \RR^d: x \geq 0, \, \langle x, \mathds{1}_d \rangle = 1 \}$ by the definition of $\lambda_w$. Furthermore, for any distinct $\theta, \theta' \in \Theta$, 
\begin{align*}
	\langle \bar{v}_{\theta}, \tilde{\Pi}^w_{\theta} \rangle = & (1 - w) \cdot \langle \bar{v}_{\theta}, \Pi^w_{r_{\theta, \Pi_w}} \rangle + \lambda_w   \geq (1 - w)\cdot \langle \bar{v}_{\theta}, \Pi^w_{r_{\theta', \Pi_w}} \rangle + \lambda_w \\
	> & (1 - w)\cdot \langle \bar{v}_{\theta}, \Pi^w_{r_{\theta', \Pi_w}} \rangle + \lambda_w  \cdot \langle \bar{v}_{\theta}, \bar{v}_{\theta'} \rangle = \langle \bar{v}_{\theta}, \tilde{\Pi}^w_{ \theta'} \rangle, 
\end{align*}
where the strict inequality is because by assumption $\bar{v}_{\theta}$'s are distinct and all have Euclidean norm 1. Therefore, we can conclude that $\tilde{\Pi}_w$ is feasible for \eqref{eq:LP}. In addition, standard application of the triangle inequality shows that $\tilde{\Pi}_w$ achieves value no smaller than
\begin{align*}
	(1 - w) \cdot ({V}_{\mathsf{OPT}} - w) - \sup_{\|y\|_2 = 1} B \cdot \| w \mathds{1}_d / d + \lambda_w y \|_1 \geq (1 - w) \cdot ({V}_{\mathsf{OPT}} - w) - B(w + \sqrt{d} \cdot \lambda_w).
\end{align*} 
Since $w$ is an arbitrary positive constant and $\lambda_w \to 0^+$ as $w \to 0^+$, we can then conclude that ${V}_{\mathsf{LP}}$ is no smaller than ${V}_{\mathsf{OPT}}$. This completes the proof of the lemma. 
% \subsection{Proof of \cref{lemma:best-response2}}
% \label{sec:proof-lemma:best-response2}
% \subsection{An impossibility result}
% \label{sec:impossible}
% We present in this section an impossibility result when Assumption \ref{assumption:distinct-v-bar} is not satisfied. Suppose the agent is myopic ($\gamma = 0$), there are only two types $\Theta = \{1,2\}$ and $d = 2$. 
\end{proof}

\subsection{Proof of \cref{lemma:best-response}}
\label{sec:proof-lemma:best-response}

In this section, we show that an agent aiming to maximize his $\gamma$-discounted cumulative reward faces a tradeoff between the immediate and future rewards. 
In particular, we prove \cref{lemma:best-response} by comparing the difference between (i) the gain of reporting a different type in terms of future rewards and (ii) the gain of immediate rewards the agent would get by choosing the myopic best response $\theta_t^*$.

% With the principal's delaying algorithm, the agent's decision to report their type involves balancing immediate gains (myopic best response) with the potential for higher future rewards (discounted utility). Consequently, at each step we can characterize the difference between the reward the agent receives based on their reported type and the reward they would get by choosing the myopic best response.

\begin{proof}[Proof of Lemma~\ref{lemma:best-response}]
At round $t$, the agent reports a value of $\theta'_t$ and then selects $a_t$ after the realization of $x_t$ to maximize~\eqref{eq:agent_objective_at_t}. 
By the construction of the delaying algorithm, 
$\theta'_t$ and $a_t$ can only directly influence its realized reward in the $t$-th round round and indirectly influence its reward after round $t+\ell_t$. Thus,  different choices of $\theta'_t$ and $a_t$ only impact the summand in~\eqref{eq:agent_objective_at_t} with $\ell=0$ and $\ell\geq \ell_t$.

By reporting $\theta'_t$ that is different from $\theta_t^*$, the agent could cumulatively gain at most ${2B\gamma^{\ell_t}}/{(1-\gamma)}$ for all future rounds with $\ell\geq \ell_t$ but, in the current round, it will incur a loss of at least $\langle \Pi_{t,\theta^*_t}, v_{\theta_t} \rangle - \langle \Pi_{t,\theta'_t}, v_{\theta_t} \rangle$ in expectation. 
Therefore, the agent will deviate from the best response only if the potential gain is greater than the immediate loss.
That is, only if 
$$
{2B\gamma^{\ell_t}}/{(1-\gamma)} \geq \langle \Pi_{t,\theta^*_t}, v_{\theta_t} \rangle - \langle \Pi_{t,\theta'_t}, v_{\theta_t} \rangle. 
$$
This proves the first claim. 

We next prove the second claim. At the $t$-th round  with the realized type $\theta_t$, by the definition of the normalized reward vectors, we have 
\begin{align*}
\langle \Pi_{t,\theta^{\ast}_t} - \Pi_{t,\theta'_t}, \bar{v}_{\theta_t} \rangle 
& = \|v_{\theta_t} - \mathds{1}_d\cdot\langle \mathds{1}_d, v_{\theta_t} \rangle / d \|_2^{-1}
\cdot 
\langle \Pi_{t,\theta^*_t} - \Pi_{t,\theta'_t}, v_{\theta_t} - \mathds{1}_d\cdot\langle \mathds{1}_d, v_{\theta_t} \rangle / d \rangle \\
& = \|v_{\theta_t} - \mathds{1}_d\cdot\langle \mathds{1}_d, v_{\theta_t} \rangle / d \|_2^{-1}
\cdot 
\langle \Pi_{t,\theta^*_t} - \Pi_{t,\theta'_t}, v_{\theta_t} \rangle,
\end{align*}
where the second equality follows from the fact that $\Pi_t \mathds{1}_d = \mathds{1}_{|\Theta|}$. 
By the first claim, 
the agent will report $\theta_t'$ only if 
\begin{align*}
\langle \Pi_{t,\theta^*_t} - \Pi_{t,\theta'_t}, \bar v_{\theta_t} \rangle
\leq 
\frac{1}{\|v_{\theta_t} - \mathds{1}_d\cdot\langle \mathds{1}_d, v_{\theta_t} \rangle / d \|_2}\cdot \frac{2B\gamma^{\ell_t}}{1 - \gamma}.
\end{align*}
By taking $C_0 = 2B / \min_{\theta \in \Theta } \{ \|v_{\theta} - \mathds{1}_d\cdot\langle \mathds{1}_d, v_{\theta} \rangle / d \|_2 \}$, we conclude the proof of the second claim. The proof of the third claim follows similarly, and we skip it for the sake of simplicity. 
\end{proof}

\subsection{Proof of \cref{lemma:pi}}
\label{sec:proof-lemma:pi}

In this section, 
we show that if we assume Assumption~\ref{assumption:not-all-one} and that $\bar{v}_\theta \neq - \bar{v}_{\theta'}$ for all $\theta, \theta' \in \Theta$, 
by applying a random rotation on the axis, 
with probability one the normalized reward angles are in `generic positions' in the sense of Assumption~\ref{assumption:angle}.

We prove the first claim by examining the marginal distribution of $\zeta_{\theta} = \varphi_0(\bar{v}_\theta)$ for all $\theta \in \Theta$, and prove the remaining three claims by examining the joint distribution of $(\zeta_{\theta}, \zeta_{\theta'})$ for all $\theta, \theta' \in \Theta$.

\begin{proof}[Proof of Claim 1]
Assumption~\ref{assumption:not-all-one} ensures $\bar{v}_\theta$ are distinct points on ${\Delta}_0$, and therefore $\zeta_{\theta} = \varphi_0(\bar{v}_\theta)$ are distinct points on $\mathbb{S}^{d-2}$.

Denote the $i$-th entry of $\zeta_\theta$ in Cartesian coordinates by $\zeta^\theta_i$. We prove a slightly stronger result than the first claim: We shall prove $\zeta^\theta_{i} \neq 0$ with probability one for all $\theta \in \Theta$ and $i \in [d - 1]$.
It is not hard to see that $\zeta_{\theta} \overset{d}{=} \Unif(\SS^{d - 2})$, hence
\begin{align*}
	\zeta_i^{\theta} \overset{d}{=} \frac{z_i}{(\sum_{i = 1}^{d - 1} z_i^2)^{1/2}}, \qquad \{z_i\}_{i \in [d - 1]} \iidsim \normal(0,1). 
\end{align*}
The above distribution is continuous and obviously places no mass at zero. 
\end{proof}

\begin{proof}[Proof of Claims 2, 3, and 4]

Note that for all $i \in [d - 2]$, we have 
\begin{align*}
	\big(\cot \alpha_i^{\theta} \big)^2 = \frac{\big(\zeta_i^{\theta} \big)^2}{\sum_{j = i + 1}^{d - 1} |\zeta_j^{\theta}|^2}.
\end{align*}
By our proof for the first claim the right hand side above is well-defined with probability 1. 
%
%Suppose claim 2 is not true, then there exists $i \in [d - 2]$ and $\theta \neq \theta'$ such that $\cot \alpha_i^{\theta} = \cot \alpha_i^{\theta'}$ with probability 1 over uniform $\Omega$.   
To prove claims 2-4, it suffices to show for all $i \in [d-2]$ and $\theta \neq \theta'$, with probability 1 we have $(\cot \alpha_i^{\theta})^2 \neq (\cot \alpha_i^{\theta'})^2$. Namely, we want to show
\begin{align}\label{eq:pi-goal}
{ (\zeta_i^{\theta'})^2} {{\sum_{j=i+1}^{d-1} (\zeta_j^\theta)^2 }} \neq { (\zeta_i^\theta)^2} {{\sum_{j=i+1}^{d-1} (\zeta_j^{\theta'})^2 }} .
\end{align}
Since $d$ and $|\Theta|$ are both finite, it suffices to check for each fixed $(i,\theta,\theta')$ the above statement holds.

% Suppose this is not true for a specific $(i, \theta, \theta')$. Let $\cO_{i} := \{\mbox{orthogonal }\Omega \mbox{ such that \cref{eq:goal}} \}$ It is not hard to see that there exists an orthogonal transformation $O$ such that both $O\varphi(\bar{v}_{\theta})$ and $O \varphi(\bar{v}_{\theta'})$ belong to $\cS_i := \{t \in \SS^{d - 1}: t_{1:(i - 1)} = 0\}$. 

Fix $\theta\neq\theta'$, there exists an orthogonal transformation $O$ such that $O \zeta_\theta = e_{d-1}$ and $O \zeta_{\theta'} = \sqrt{1 - a^2} e_{d-2} + a e_{d-1}$, where $ a = \langle \bar{v}_\theta, \bar{v}_{\theta'} \rangle  \in (-1,1)$ (by the two assumptions). Then we can conclude that $(\zeta_\theta, \zeta_{\theta'}) \overset{d}{=} (\omega_{1}, a\omega_{1} + \sqrt{1-a^2}\omega_2)$, where $\omega_1$ and $\omega_2$ are the first and second columns of $O^{\top}$. By rotational invariance, we conclude that $O^{\top}$ is uniformly distributed over the orthogonal group. 

We then prove \cref{eq:pi-goal} for $i = 1$. As $\|\zeta_\theta\|_2 \equiv 1$ for all $\theta\in\Theta$, then proving \cref{eq:pi-goal} is equivalent to proving $|\zeta^\theta_1|\neq |\zeta^{\theta'}_1|$. Therefore, our task reduces to showing
\begin{align*}
\mathbb{P}(|\zeta^\theta_1|\neq |\zeta^{\theta'}_1|) = \mathbb{P}(|\omega_{11}|\neq |a\omega_{11} + \sqrt{1-a^2}\omega_{21}|) = 0, 
\end{align*}
where $w_{11}$ and $w_{21}$ are the first coordinates of $w_1$ and $w_2$, respectively. 
This is obviously true if once again we express $(w_{k1})_{k \in [d - 1]}$ as a Gaussian vector normalized by its Euclidean norm. 

The proof for $i > 1$ is more involved. Let $x,y\in\mathbb{R}^{d-1}$ be two independent random variables sampled from $\mathcal{N}(0,I_{d-1})$. One can verify that 
\begin{align*}
(\omega_1, \omega_2) \overset{d}{=} 
\left( \tilde x, \, \frac{y - \langle \tilde x,y \rangle \tilde x}{\|y - \langle \tilde x,y \rangle \tilde x\|_2} \right), \qquad \tilde x := \frac{x}{\| x\|_2}, 
\end{align*}
and therefore
\begin{align*}
(\zeta_\theta, \zeta_{\theta'}) \overset{d}{=} 
\left(\tilde x, \, a \tilde x + \frac{\sqrt{1 - a^2} \cdot (y - \langle \tilde x,y \rangle \tilde x)}{\|y - \langle \tilde x,y \rangle \tilde x\|_2} \right).%(\frac{a}{\|x\|} - \frac{\sqrt{1-a^2}\langle x,y \rangle }{\|y - \langle x,y \rangle x\|}) x + \frac{\sqrt{1-a^2}}{\|y - \langle x,y \rangle x\|} y \right),
\end{align*}
The above equation further implies that
\begin{align}
&\mathbb{P}\bigg(
{ (\zeta_i^{\theta'})^2} {{\sum_{j=i+1}^{d-1} (\zeta_j^\theta)^2 }} = { (\zeta_i^\theta)^2} {{\sum_{j=i+1}^{d-1} (\zeta_j^{\theta'})^2 }} 
\bigg) = \mathbb{P}\Bigg(
(2cx_i y_i + y_i^2) \cdot \sum_{j > i} x_j^2 
= x_i^2 \cdot  \bigg(\sum_{j > i} 2cx_j y_j + y_j^2 \bigg) 
\Bigg) \notag\\
 & \qquad  = \mathbb{P}\Bigg(
2c \cdot \bigg(x_iy_i \sum_{j>i}x_j^2 - x_i^2 \sum_{j>i}{x_j y_j}\bigg)
= -y_i^2 \sum_{j>i}x_j^2 + x_i^2 \sum_{j>i}y_j^2
\Bigg)\label{eq:pi-goal-equal}
\end{align}
where we define $c$ as 
\begin{align*}
c = \frac{a}{\sqrt{1-a^2}} \frac{\|y - \langle \tilde x,y \rangle \tilde x\|_2}{\|x\|_2} - \frac{\langle \tilde x,y \rangle}{\|x\|_2}.
%\label{eq:c-y1-expression}
\end{align*}
We remark $c$ is the component dependent on $y_1$, and $\|y - \langle \tilde x,y \rangle \tilde x\|_2^2$ is quadratic in $y_1$ (recall $\|\tilde{x}\|_2 = 1$). Conditioning on any realization of $\{x_1,\ldots,x_{d-1}, y_2,\dots, y_{d-1}\}$ satisfying $\| x \|_2 \neq 0$ and $\langle x, y\rangle \neq 0$, the left hand side of \cref{eq:pi-goal-equal} is a rational function of $y_1$ while the right hand side is not, and the two functions agree for only finitely many values for $y_1$. Recall all $x_i,y_i$ are i.i.d.\ $\mathcal{N}(0,1)$, we conclude that the probability in \cref{eq:pi-goal-equal} is zero. This completes the proof of the lemma. 
\end{proof}

\subsection{Proof of \cref{lemma:confidence-ellipsoid}}
\label{sec:proof-lemma:confidence-ellipsoid}

\begin{proof}
{Observe that $u_{\ell} = \langle \beta^* , \Vec(\Pi_{\ell}) \rangle + \eta_{\ell}$, where $\eta_{\ell} := u_{\ell} - \langle \beta^* , \Vec(\Pi_{\ell}) \rangle$ is random and satisfies 
\begin{align}
\label{eq:eta-property}
    \E[\eta_{\ell} \mid \cH_{\ell}, \Vec(\Pi_{\ell})] = 0, \qquad \E[e^{\xi \eta_{\ell}} \mid \cH_{\ell}, \Vec(\Pi_{\ell})] \leq \exp \left( 8\xi^2 B^2  \right). 
\end{align}
for all $\xi \in \RR$. In the above display, to get the second upper bound we note that 
\begin{align*}
    \eta_{\ell} = u_{\ell} - \sum_{\theta\in \Theta} \sum_{x \in \cX} f(\theta) \cdot U  (\theta, x)  \cdot (\Pi_{\ell})_{(\sfb_n^{\ast})^{-1} (\theta), x} \in [-2B, 2B] 
\end{align*}
has zero conditional expectation and is bounded. 
It is a standard fact that any zero-mean random variable $X \in [a, b]$ satisfies $\E[e^{\xi X}] \leq \exp(\xi^2 (b - a)^2 / 2)$ for all $\xi \in \RR$ \citep{vershynin2018high}. 
Invoking the properties listed in \eqref{eq:eta-property}, we are able to construct a confidence ellipsoid such that $\beta^{\ast}$ falls within such an ellipsoid with high probability. 
} 

The remaining proof follows from that of Theorem 2 in  \cite{abbasi2011improved}. 
We skip it here to avoid redundancy.
\end{proof}

%$c = \frac{a}{\sqrt{1-a^2}} \frac{\|y - \langle x,y \rangle x\|}{\|x\|} - \langle x,y \rangle$.

% One quickly observe
%\begin{align}\label{eq:pi-goal-zero-coef-event}
%\mathbb{P}\left(x_iy_i \sum_{j>i}x_j^2 - x_i^2 \sum_{j>i}{x_j y_j} = 0 \right) = 0. 
%\end{align}We now study the distribution of $c$ conditioned on any arbitrary realization of $\{x_1,\ldots,x_{d-1}, y_2, \dots, y_{d-1}\}$. Since $y_1$ is $N(0,1)$ and the second term in $c$ is linear in $y_1$, one can conclude that a necessary condition for $c$ to have non-zero probability being a constant is $\sum_{j-2}^{d-1} x_jy_j = 0$.

%We are now able to conclude the proof. Denote the event in Eq.~\eqref{eq:pi-goal-equal} by $B$, the event $\{\sum_{j-2}^{d-1} x_jy_j = 0\}$ by $B$, the event in Eq.~\eqref{eq:pi-goal-zero-coef-event} by C, and the complement of a event by $\bar{\cdot}$. Then 
%\begin{align*}
%P(A) = P(A|B B\cup C C)P(B\cup C) + P(A|\bar{B}\cap \bar{C})P(\bar{B}\cap \bar{C})
%\end{align*}
%Since $P( B\cup C ) = 0$, it suffices to check $P(A|\bar{B}\cap \bar{C}) = 0$. Conditioned on any realization of $\{x_1,\ldots,x_{d-1}, y_2, \dots, y_{d-1}\}$ in $\bar{B}\cap \bar{C}$, $c$ is this only random variable in Eq.~\eqref{eq:pi-goal-equal} which is continuous (ensured by $\bar{B}$) with non-zero coefficient (ensured by $\bar{C}$), the probability of the LHS in Eq.~\eqref{eq:pi-goal-equal} equals to the constant on the RHS is zero. We hereby conclude the proof.

\subsection{Proof of \cref{lemma:LP-close}}
\label{sec:proof-lemma:LP-close}

We prove \cref{lemma:LP-close} in this section. We divide the proof into several major components. In \cref{sec:C31} we reformulate \eqref{eq:LP} to obtain a more manageable format for analysis. 
In \cref{sec:C32} we derive a lower bound for $u_k$, which we recall is the maximum value associated with \eqref{eq:LPk}. As shall become clear soon, this lower bound, which we denote by $u_k'$, is the maximum value of a separate LP problem. 
Finally, we show $u_k'$ and $u_{\ast}$ are close for a large $k$ in \cref{sec:C33} with the aid of LP duality theory.

\subsubsection{Reformulating \eqref{eq:LP}}
\label{sec:C31}

Note that for all $k \in \NN_+$, \eqref{eq:LP} can be equivalently reformulated as follows: 
\begin{align}\label{eq:LP-star}
\begin{split}
	& \mbox{maximize }\sum_{s \in \Theta} f(\sfb_{n_k}^{\ast}(s)) \sum_{x \in \cX} \pi (s, x) U(\sfb^{\ast}_{n_k}(s), x), \\
	& \mbox{s.t. }\langle \varphi_0^{-1}(\rho(\alpha^{\sfb^{\ast}_{n_k}(s)})),  \Pi_{s}\rangle \geq \langle \varphi_0^{-1}(\rho(\alpha^{\sfb^{\ast}_{n_k}(s)})),  \Pi_{s'}\rangle , \qquad \forall \,\, s, s' \in \Theta, \,\,s \neq s', \\
	& \,\,\,\,\,\,\,\,\,\,\, \langle \Pi_{s}, \mathds{1}_d \rangle = 1, \qquad \Pi_{s} \geq 0, \qquad  \forall \,\, s \in \Theta,  
\end{split}\tag{$\mbox{LP}^{\ast}$-reformulated}
\end{align}
where $\Pi_{s} = (\pi(s, x))_{x \in \cX} \in \RR^d$, $\alpha^{\sfb^{\ast}_{n_k}(s)} = (\alpha^{\sfb^{\ast}_{n_k}(s)}_1, \alpha^{\sfb^{\ast}_{n_k}(s)}_2, \cdots, \alpha^{\sfb^{\ast}
_{n_k}(s)}_{d - 2})$, and we recall that $\rho$ is defined in \cref{eq:rho}. Note that \eqref{eq:LP-star} is obtained by simply permuting the type labels in \eqref{eq:LP}, hence also has maximum value $u_{\ast}$.

\subsubsection{A lower bound for $u_k$}
\label{sec:C32}

Recall that $\cE_k$ is defined in \cref{eq:event-cEk}. 
On the event $\cE_k$ we have $\bar{V}_{n_k}^{\ast} \subseteq \bar{\cV}_{n_k}$ (this follows from the discussion after Theorem \ref{thm:reward-function}), therefore the feasible region of \ref{eq:LPk} is a subset of that of \ref{eq:LP-star}. As a consequence, we immediately get $u_{\ast} \geq {u}_k$.
We next establish a lower bound for $u_k$. 
Recall that $n_k = 2^{-k}$. 
By definition in \cref{eq:feasible-set}, $\cI_{n_k}$ can be expressed as the intersection of the following sets for all $\bar V \in \bar{\cV}_{n_k}$:
\begin{align}%\label{eq:LPk2}
\begin{split}
	%& \mbox{maximize }\sum_{\theta \in \Theta} f(\psi_{2^k}(\theta)) \sum_{x \in \cX} \pi (\theta, x) U(\psi_{2^k}(\theta), x), \\
	&  \langle \bar V_{s}, \Pi_{s} \rangle \geq \langle \bar V_{s}, \Pi_{s'}\rangle + 2^{-40k}, \qquad \forall \,\, s, s' \in \Theta, \,\,s \neq s', \\
	&  \langle \Pi_{s}, \mathds{1}_d \rangle = 1, \qquad \Pi_{s} \geq 0, \qquad \forall \,\, s \in \Theta. 
\end{split}%\tag{${\mbox{LP}}_k(\bX)$}
\end{align}
By definition in \eqref{eq:conf_set}, for all $\bar V \in \bar\cV_{n_k}$ we have 
\begin{align*}
	\sum_{s \in \Theta} \| \bar V_s - \varphi_0^{-1}(\rho(\hat\alpha^s))  \|_2 \leq  2^{-40k}. 
\end{align*}
Recall that on $\cE_k$
Invoking \cref{lemma:rho}, on $\cE_k$ we get 
\begin{align*}
	\sum_{s \in \Theta} \big\| \varphi_0^{-1}(\rho(\alpha^{\sfb^{\ast}_{n_k}(s)})) - \varphi_0^{-1}( \rho(\hat\alpha^s)) \big\|_2 \leq 2^{-40k}.
\end{align*}
Putting together the above upper bounds and applying the triangle inequality, we conclude that
\begin{align*}
	\sum_{s \in \Theta } \big\| \bar V_s - \varphi_0^{-1}(\rho(\alpha^{\sfb^{\ast}_{n_k}(s)})) \big\|_2 \leq 2^{-39k}
\end{align*}
for all $\bar V \in \bar\cV_{n_k}$. Therefore, using the Cauchy–Schwarz inequality we deduce that $\cI_{n_k}$ contains the following set:
\begin{align*}
	&  \langle  \varphi_0^{-1}(\rho(\alpha^{\sfb^{\ast}_{n_k}(s)})),  \Pi_{s} \rangle \geq \langle  \varphi_0^{-1}(\rho(\alpha^{\sfb^{\ast}_{n_k}(s)})),  \Pi_{s'} \rangle + 2^{-40k} + 2^{-38k}, \qquad \forall \,\, s, s' \in \Theta, \,\,s \neq s', \\
	&  \langle \Pi_{s}, \mathds{1} \rangle = 1, \qquad \Pi_{s} \geq 0, \qquad \forall \,\, s \in \Theta.
\end{align*}
%
%By definition of $\bar{\cV}_{2^k}$, we know that $\sum_{\theta \in \Theta}\|X_{\theta} - \rho(\hat{\alpha}_{\theta}^{(2^k)})\|_1 \leq 2\pi d |\Theta| \cdot 2^{-3k}$. Furthermore, on $\cE_k$ it holds that $\sum_{\theta \in \Theta}\|\alpha_{\psi_{2^k}(\theta)} - \hat{\alpha}_{\theta}^{(2^k)}\|_1 \leq 2\pi d|\Theta| \cdot 2^{-3k}$, which together with \cref{lemma:rho} gives $\sum_{\theta \in \Theta}\|\rho(\alpha_{\psi_{2^k}(\theta)} ) - \rho( \hat{\alpha}_{\theta}^{(2^k)})\|_2 \leq 2\pi d|\Theta| \cdot 2^{-3k}$. Putting together these upper bounds and employ the triangle inequality, we derive that $\sum_{\theta \in \Theta}\|X_{\theta} - \rho(\alpha_{\psi_{2^k}(\theta)} )\|_2 \leq 4\pi d|\Theta| \cdot 2^{-3k}$. Plugging this into the definition of $\cI_{2^k}$, we can conclude that $\cI_{2^k}$ contains the following set:
%
%\begin{align*}
%	&  \langle  \rho(\alpha_{\psi_{2^k}(\theta)} ),  \Pi_{\theta} \rangle > \langle  \rho(\alpha_{\psi_{2^k}(\theta)} ),  \Pi_{\theta'} \rangle + 2^{-10k} + 4\pi d |\Theta| \cdot 2^{-3k}, \qquad \forall \,\, \theta, \theta' \in \Theta, \,\,\theta \neq \theta', \\
%	&  \langle \Pi_{\theta}, \mathds{1} \rangle = 1, \qquad \Pi_{\theta} \geq 0, \qquad \forall \,\, \theta \in \Theta.
%\end{align*}
%
Hence, $u_k$ is lower bounded by the optimal value of the following linear programming problem:
\begin{align}\label{eq:LPk3}
\begin{split}
	& \mbox{maximize }\sum_{s \in \Theta} f(\sfb^{\ast}_{n_k}(s)) \sum_{x \in \cX} \pi (s, x) U(\sfb^{\ast}_{n_k}(s), x), \\
	& \mbox{s.t. }  \langle  \varphi_0^{-1}(\rho(\alpha^{\sfb^{\ast}_{n_k}(s)})),  \Pi_{s} \rangle \geq \langle  \varphi_0^{-1}(\rho(\alpha^{\sfb^{\ast}_{n_k}(s)})),  \Pi_{s'} \rangle + 2^{-40k} + 2^{-38k}, \qquad \forall \,\, s, s' \in \Theta, \,\,s \neq s', \\
	& \qquad \langle \Pi_{s}, \mathds{1} \rangle = 1, \qquad \Pi_{s} \geq 0, \qquad \forall \,\, s \in \Theta. 
\end{split}\tag{${\mbox{LP}}_k'$}
\end{align}
We denote the maximum value of \eqref{eq:LPk3} by $u_k'$. Based on our previous discussions, we have $u_k' \leq u_k \leq u_{\ast}$. 

Before proceeding, we make a few comments on the feasibility of \eqref{eq:LPk3}: According to  \cref{lemma:revelation}, under the current set of assumptions the feasible region of \eqref{eq:LP} is non-empty. 
Furthermore, there exists a positive constant $\delta$, such that $\{\pi: \|\pi - \bar \pi\|_2 \leq \delta, \pi \geq 0, \langle\pi(\theta,\cdot ), \mathds{1}  \rangle = 1 \mbox{ for all }\theta \in \Theta \}$ is a subset of the feasible region of \eqref{eq:LP} for some feasible $\bar\pi$ (recall \eqref{eq:LP} and \eqref{eq:LP1} are equivalent).
Hence, we can conclude that for $k$ large enough the feasible region of \eqref{eq:LPk3} is also non-empty.

\subsubsection{Dual linear program}
\label{sec:C33}

\begin{proof}[Proof of \cref{lemma:LP-close}]

In order to show $u_k$ and $u_{\ast}$ are close, it suffices to instead show $u_k'$ and $u_{\ast}$ are close. 
To this end, we shall resort to the LP duality theory. To enhance clarity of presentation, 
note that there exists $c \in \RR^{|\Theta| \cdot |\cX|}$, $b \in \RR^{|\Theta|(|\Theta| + 1)}$ and $A \in \RR^{|\Theta|(|\Theta| + 1) \times |\Theta| \cdot |\cX|}$, such that \eqref{eq:LP-star} can be expressed as
%if we replace all ``$>$'' in \eqref{eq:LP-star} with ``$\geq$'', then the resulting linear programming problem can be expressed as 
%
\begin{align}\label{eq:P1}
\begin{split}
	& \mbox{maximize }\langle c, \Pi \rangle, \\
	& \mbox{s.t. } A \Pi \leq  b, \qquad \Pi \geq 0. 
\end{split}\tag{$\mbox{P}_1$}
\end{align}
The maximum value of \eqref{eq:P1} is $u_{\ast}$. 
We can also express \eqref{eq:LPk3} using $(A, b, c)$:
\begin{align}\label{eq:P2}
\begin{split}
	& \mbox{maximize }\langle c, \Pi \rangle, \\
	& \mbox{s.t. } A \Pi \leq b - \eps v, \qquad \Pi \geq 0, 
\end{split}\tag{$\mbox{P}_2$}
\end{align}
where $\eps = 2^{-40k} + 2^{-38k}$ and $v \in \{0,1\}^{|\Theta|(|\Theta| + 1)}$. We denote the above two problems by \eqref{eq:P1} and \eqref{eq:P2}. The dual problems for \eqref{eq:P1} and \eqref{eq:P2} respectively are
\begin{align}\label{eq:D1}
\begin{split}
	& \mbox{minimize }\langle b, y \rangle, \\
	& \mbox{s.t. } A^{\top} y \geq  c, \qquad y \geq 0
\end{split}\tag{$\mbox{D}_1$}
\end{align}
and 
\begin{align}\label{eq:D2}
\begin{split}
	& \mbox{minimize }\langle b, y \rangle - \eps \langle v, y \rangle, \\
	& \mbox{s.t. } A^{\top} y \geq  c, \qquad y \geq 0. 
\end{split}\tag{$\mbox{D}_2$}
\end{align}
%
%For large enough $k$ both \eqref{eq:P1} and \eqref{eq:P2} are feasible and have finite optimal value. Hence, by the strong duality theorem, we know that both \eqref{eq:D1} and \eqref{eq:D2} are feasible, have finite optimal values. For \eqref{eq:D2}, that value is equal to $u_k'$, and for \eqref{eq:D1} the optimal value is no smaller than $u_{\ast}$.
%
Note that according to Assumption \ref{assumption:feasible} problem \eqref{eq:P1} is always feasible and obviously has a finite value. Also by Assumption \ref{assumption:feasible}, we can conclude that there exists $k_0 \in \NN_+$ depending only on $(\{\bar{v}_{\theta}: \theta \in \Theta\}, \varphi_0)$, such that for all $k \geq k_0$ problem \eqref{eq:P2} is feasible and has a finite value. Using the strong duality theorem, we know that problem \eqref{eq:D1} is always feasible and has finite value. On the other hand, for all $k \geq k_0$ problem \eqref{eq:D2} is feasible and has finite value. For \eqref{eq:D2}, the optimal value is equal to $u_k'$, and for \eqref{eq:D1} the optimal value is $u_{\ast}$. 

Observe that \eqref{eq:D1} and \eqref{eq:D2} share a common feasible region that does not depend on $k$. Also note that for $k \geq k_0$ we have
\begin{align*}
    & \langle b, y \rangle - (2^{-40k} + 2^{-38k}) \langle v, y \rangle \\
     & \qquad = \frac{2^{-40k} + 2^{-38k}}{2^{-40k_0} + 2^{-38k_0}} \cdot \left( \langle b, y \rangle - (2^{-40k_0} + 2^{-38k_0}) \langle v, y \rangle \right) + \left(1 - \frac{2^{-40k} + 2^{-38k}}{2^{-40k_0} + 2^{-38k_0}} \right) \cdot \langle b, y \rangle. 
\end{align*}
Therefore, we obtain that 
\begin{align*}
    u_k' \geq \frac{2^{-40k} + 2^{-38k}}{2^{-40k_0} + 2^{-38k_0}} \cdot u_{k_0}' + \left(1 - \frac{2^{-40k} + 2^{-38k}}{2^{-40k_0} + 2^{-38k_0}} \right) \cdot u_{\ast}.
\end{align*}
As $k_0$ is a function of $(\{\bar{v}_{\theta}: \theta \in \Theta\}, \varphi_0)$ and $|u_k'|, |u_{\ast}| \leq B |\Theta|$, we deduce that there exists $k_{\mathsf{LP}} \in \NN_+$ depending only on $(\{\bar{v}_{\theta}: \theta \in \Theta\}, \varphi_0, B)$, such that for all $k \geq k_{\mathsf{LP}}$, 
\begin{align*}
	u_{\ast} - 8^{-k} C_{\mathsf{LP}} \leq u_k' \leq u_k \leq u_{\ast},
\end{align*}
where $C_{\mathsf{LP}}$ is a constant that depends only on $(\{\bar{v}_{\theta}: \theta \in \Theta\}, \varphi_0, B)$ as well. This completes the proof of the lemma. 
\end{proof}

%We denote by $y_{\ast}$ a point that achieves the optimal value of \eqref{eq:D2}.
%As \eqref{eq:D1} and \eqref{eq:D2} share a common feasible region, we can then conclude that $y_{\ast}$ is also feasible for \eqref{eq:D1}, hence the optimal value of \eqref{eq:D1} is no larger than $u_k' - \eps \langle v, y_{\ast} \rangle$ 

%Straightforwardly, the optimal value of \eqref{eq:D2} exists and is no larger than $u_{\ast} - \eps \langle \bv, \by_{\ast} \rangle$. In summary, we conclude that
%
%\begin{align*}
%	u_{\ast} - \eps \langle v, y_{\ast} \rangle \leq u_k' \leq u_k \leq u_{\ast},
%\end{align*}
%
%thus completing the proof of the lemma for large enough $k$. 

\subsection{Proof of \cref{thm:main-episodic}}
\label{sec:proof-thm:main}

\begin{proof}[Proof of \cref{thm:main-episodic}]

By design, episodes are independent of each other, thus we analyze each one of them separately.  

Let us zoom in on the $k$-th episode. 
Checking Algorithm \ref{alg:delayed-UCB}, we see that in the $k$-th episode, $m_k$ samples are used to estimate the feasible region with $m_k = \lceil \log n_k \rceil^6$, $\lceil \log n_k \rceil^2$ samples are used in the dummy rounds that follow the feasible region estimation stage, and $n_k \cdot (\lceil \log n_k \rceil^2 + 1)$ samples are used to implement the  LinUCB algorithm, where $n_k = 2^k$. 
We first look at the feasible region estimation stage and the dummy stage that follows. 
These two stages come with no regret guarantee, and we simply provide a vacuous upper bound for the total regret $\sf{R}_{1,k}$ that arises from this stage:
\begin{align}\label{eq:regret-bound1}
    {\sf{R}_{1,k}} \lesssim B \cdot \left( \lceil \log n_k \rceil^6 + \lceil \log n_k \rceil^2 \right). 
\end{align}
%
%We let $K = \lceil \log_2 T \rceil$. It is not hard to verify that $\regret(T) \leq \regret (2^K)$. The later can be further upper bounded by the sum of regrets up to episode $K$. 
%We then upper bound the regret in each episode. In the $k$-th episode, the reward function estimation stage together with the dummy stage that follows uses samples from no more than $C_3 (k \log 2)^5 + \lceil k \log 2 \rceil^2 + 1$ rounds. Recall that the utility function is bounded between $-B$ and $B$. Hence, the regret caused by these two stages together is upper bounded by
%
%\begin{align}\label{eq:regret-bound1}
%	2B  \cdot \left( C_3 (k \log 2)^5 + \lceil k \log 2 \rceil^2 + 1 \right). 
%\end{align}
%
In the above display, we recall that $B$ is from Assumption \ref{assumption:model}.

Next, we upper bound the regret that arises from the pessimistic-optimistic planning stage. 
Note that on $\cE_k$ (which we recall is defined in \cref{eq:event-cEk}), by the discussion after Theorem \ref{thm:reward-function} we always have $\bar{V}_{n_k}^{\ast} \subseteq \bar\cV_{n_k}$. In this case, all $\Pi \in \cI_{n_k} $  satisfies the following constraints:
\begin{align*}%\label{eq:LPk2}
\begin{split}
	%& \mbox{maximize }\sum_{\theta \in \Theta} f(\psi_{2^k}(\theta)) \sum_{x \in \cX} \pi (\theta, x) U(\psi_{2^k}(\theta), x), \\
	&  \langle \bar{v}_{\sfb_{n_k}^{\ast}(s)}, \Pi_{s} \rangle > \langle \bar{v}_{\sfb^{\ast}_{n_k}(s)}, \Pi_{s'}\rangle + 2^{-40k}, \qquad \forall \,\, s, s' \in \Theta, \,\,s \neq s', \\
	&  \langle \Pi_{s}, \mathds{1}_d \rangle = 1, \qquad \Pi_{s} \geq 0, \qquad \forall \,\, s \in \Theta. 
\end{split}%\tag{${\mbox{LP}}_k(\bX)$}
\end{align*}
Let $k_0 \in \NN_+$ such that for all $k \geq k_0$, $\frac{2B \gamma^{\lceil \log n_k \rceil^2}}{1 - \gamma} > 2^{-40k}$. 
Such $k_0$ exists and depends only on $(\mathscr{P}, \varphi_0)$.
In the remainder of the proof, we restrict our attention to $k \geq k_0$. This is without loss of generality, as the first $k_0 - 1$ rounds contribute only a constant amount to the regret. 
By Lemma \ref{lemma:best-response}, for all $k \geq k_0$, when $\cE_k$ occurs and the principal picks her mechanism from $\cI_{n_k}$, the following statement is true: When the agent has type $\sfb_{n_k}^{\ast}(s)$ he will always report $s$. In this case, the conditional expectation of the principal's reward in one round of the pessimistic-optimistic planning stage is
$$\sum_{s \in \Theta} f(\sfb_{n_k}^{\ast}(s)) \sum_{x\in \cX} \pi (s, x) U(\sfb_{n_k}^{\ast}(s), x).$$
In this case, minimizing the regret reduces to solving a stochastic linear bandit problem with decision space $\cI_{n_k}$ and true coefficients $(f(\sfb^{\ast}_{n_k}(s)) U(\sfb^{\ast}_{n_k}(s), x))_{s \in \Theta, x \in \cX}$.%, and $n_k \cdot (\lceil \log n_k \rceil^2 + 1)$ samples. 

Recall this stage involves at most $n_k \cdot (\lceil \log n_k \rceil^2 + 1)$ rounds of samples, as in the for loop in Algorithm \ref{alg:delayed-UCB}. 
By the nature of the pessimistic-optimistic planning stage, in order to upper bound the regret, it suffices to consider the regret that arises from rounds $j (\lceil \log n_k \rceil^2 + 1) + 1$, with $j = 0, 1, \cdots, n_k - 1$.  
For $n \in [n_k \cdot (\lceil \log n_k \rceil^2 + 1)]$, we denote by ${\sf{R}}_{2,k,n}$ the total regret that arises from the first $n$ rounds of the pessimistic-optimistic planning stage.
Let $j_n = \max\{j: j (\lceil \log n_k \rceil^2 + 1) + 1 \leq n\}$, and denote by $\bar {\sf{R}}_{2,k,n}$ the regret that arises from rounds $j (\lceil \log n_k \rceil^2 + 1) + 1$ for $0 \leq j \leq j_n$. 
One can verify that ${\sf{R}}_{2,k,n} \leq (\lceil \log n_k \rceil^2 + 1) \cdot \bar {\sf{R}}_{2,k,n}$.
Applying \cref{thm:linUCB}, we conclude that with probability at least $1 - n_k^{-10}$ the following upper bound holds for all $n \in [n_k \cdot (\lceil \log n_k \rceil^2 + 1)]$: 
\begin{align*}
    \bar {\sf{R}}_{2,k,n} \leq 4 B |\Theta|^{1/2} d^{1/2} \sqrt{j_n d |\Theta| \log (1 + j_n d^{-1} |\Theta|^{-1/2} )} \cdot (B d^{1/2} + B \sqrt{20 \log n_k + d |\Theta| \log (1 + j_n d^{-1} |\Theta|^{-1/2})} ). 
\end{align*}
Using the above formula, we see that 
\begin{align}
\label{eq:regret-bound2}
   \begin{split}
	{\sf{R}}_{2,k,n} & \leq  4 B |\Theta|^{1/2} d^{1/2} \cdot (\lceil  \log n_k \rceil^2 + 1) \sqrt{j_n d |\Theta| \log (1 + j_n d^{-1} |\Theta|^{-1/2} )} \times \\
	& \left( B d^{1/2} + B \sqrt{20 \log n_k + d |\Theta| \log (1 + j_n d^{-1} |\Theta|^{-1/2})} \right), 
\end{split} 
\end{align}
%
% \begin{align}\label{eq:regret-bound2}
% \begin{split}
% 	{\sf{R}}_{2,k,n} & \leq  4Bd^{1/2}|\Theta|^{1/2} \cdot (\lceil  \log n_k \rceil^2 + 1) \sqrt{\tilde n d |\Theta| \log (1 + \tilde n d^{-1} |\Theta|^{-1/2})} \times \\
% 	& \left( B d^{1/2} + B \sqrt{2 0 \log \tilde n + d|\Theta| \log (1 + \tilde n d^{-1} |\Theta|^{-1/2})} \right), 
% \end{split}
% \end{align}
% %
% where $\tilde n = \lceil n / (\lceil  \log n_k \rceil^2 + 1) \rceil$.
We denote by $\bar\cE_k$ the intersection of $\cE_k$ and the above event. 
Then it holds that $\P(\bar\cE_k^c) \leq n_k^{-10} + C_1 n_k^{-50}$. 
In what follows, we consider $k$ large enough such that $\P(\bar\cE_k^c) \leq 2n_k^{-10}$. 

We then upper bound the regret using \cref{eq:regret-bound1,eq:regret-bound2}. 
We first control the expected regret. 
Let $K := \lceil \log_2 T \rceil$. Putting together \cref{eq:regret-bound1,eq:regret-bound2}, we get
\begin{align*}
    \E[\regret(T)] \leq & \sum_{k = 1}^K \left({\sf{R}}_{1,k} + {\sf{R}}_{2, k, n_k(\lceil \log n_k \rceil^2 + 1)}  + 2B\P(\bar\cE_k^c) \cdot (m_k + (n_k + 2)(\lceil \log n_k \rceil^2 + 1)) \right) \\
    \lesssim & \, C_{\ast} \sqrt{T} \cdot (\log T)^3, 
\end{align*}
where $C_{\ast}$ is a positive constant that depends only on $(\mathscr{P}, \varphi_0)$.   

Next, we characterize the asymptotic regret. 
Observe that $\{\bar\cE_k\}_{k = 1}^{\infty}$ are independent events, and $\sum_{k = 1}^{\infty}\P(\bar{\cE}_k^c) < \infty$.
By the first Borel–Cantelli lemma, 
we see that almost surely only finitely many $\bar{\cE_k}$ does not occur.
Combining this result and the regret upper bounds in \cref{eq:regret-bound1,eq:regret-bound2}, we have for any $g_0 > 3$, as $T \to \infty$ we have
\begin{align*}
	\frac{\regret(T)}{\sqrt{T} \cdot  (\log T)^{g_0} } \overset{a.s.}{\to} 0. 
\end{align*} 
Finally, we provide a finite-sample regret upper bound. 
We choose a $K_{\eps} \in \NN_+$ such that $2^{-10(K_{\eps} + 1)} \leq \eps \leq 2^{-10 K_{\eps} }$. 
In this case, $\PP(\cup_{k = K_{\eps} + 2}^{\infty} \bar{\cE_k}) \leq \sum_{k = K_{\eps} + 2}^{\infty} \PP(\bar{\cE_k}) \leq 2^{-10(K_{\eps} + 1)} \leq \eps$. 
Therefore, with probability at least $1 - \delta$,
\begin{align*}
    \regret(T) \lesssim &\,  B \sum_{k = 1}^{K_{\eps} + 1} (m_k + (n_k + 2) \cdot (\lceil \log n_k \rceil^2 + 1)) + \sum_{k = K_{\eps} + 2}^{K} \left({\sf{R}}_{1,k} + {\sf{R}}_{2, k, n_k(\lceil \log n_k \rceil^2 + 1)} \right) \\
    \lesssim &\, C_{\ast}' \cdot \left( \sqrt{T} \cdot (\log T)^3 + \eps^{-1/10} \cdot (\log \eps^{-1})^2 \right)
\end{align*}
% Note that there exists $K_0 \in \NN_+$ with $2^{-10K_0} \asymp \delta$, such that $\P(\cap_{k = K_0}^{\infty}\bar{\cE_k}) \geq 1 - \eps$. Therefore, with probability at least $1 - \eps$, 
% %
% \begin{align*}
%     \regret(T) \lesssim &\,  B \sum_{k = 1}^{K_0 - 1} (m_k + (n_k + 1) \cdot (\lceil \log n_k \rceil^2 + 1)) + \sum_{k = K_0}^{K} \left({\sf{R}}_{1,k} + {\sf{R}}_{2, k, n_k(\lceil \log n_k \rceil^2 + 1)} \right) \\
%     \lesssim & C_{\ast}' \cdot \left( \sqrt{T} \cdot (\log T)^3 + \eps^{-1/10} \cdot (\log \eps^{-1})^2 \right),
% \end{align*}
%
where once again $C_{\ast}'$ is a positive constant that depends only on $(\mathscr{P}, \varphi_0)$. The proof is done. 
\end{proof}

\newpage 
\section{Proofs for reward angle estimation}

In this section, we prove the supporting lemmas related to reward angle estimation. 

\subsection{Theoretical guarantee for \cref{alg:sector_dector}}
\label{sec:theory-alg:sector_dector}

%\color{cyan} 
In this section, we present theoretical guarantees for \cref{alg:sector_dector}. 
To summarize, in \cref{lemma:sector-test} below we show that for a large enough $n$ and a randomly selected $\alpha$, 
with high probability \cref{alg:sector_dector} returns \texttt{True} when there exists $\alpha_{d - 2}^{\theta} \in (\alpha - \delta / 2, \alpha + \delta / 2)$ and returns \texttt{False} otherwise. 
%This is the case we say \cref{alg:sector_dector} \emph{succeeds}. 

% Below we define $f_{\min} := \min_{\theta \in \Theta} f(\theta)$, which under Assumption \ref{assumption:model} is strictly positive. Here, the randomness is solely over $\alpha$. 

\begin{lemma}\label{lemma:sector-test}
 Let $\delta  \in (0, \pi]$ and $\alpha \sim \Unif[0, 2\pi)$. 
 %In addition, we assume that $\delta > $
 %We set $\ell = \lceil  \log n \rceil^2 $ and $T_{\mathrm{sec}} = \lceil  \log n \rceil^4$ in \cref{alg:sector_dector}. 
Under the assumptions of \cref{thm:reward-function}, 
there exist $n_0 \in \NN_+$ and $C_{\sfP} \in \RR_+$ that  depend only on $(\mathscr{P}, \varphi_0)$, 
such that for $n \geq n_0$, with probability at least $1 - C_{\sfP} \cdot n^{-4^{d + 2}}$, 
%Then for $n  \geq  \exp\big( {4^{(d + 3) / 3}f_{\min}^{-1 / 3}} \big)$,  
%with probability at least $1 - 3|\Theta| (r_d \delta_{\sin})^{-1/2} n^{-4^{d + 2}} - n^{-4^{d + 3}}$, 
the following statement is true: If there exists $\theta \in \Theta$ such that $\alpha_{d - 2}^{\theta} \in (\alpha - \delta / 2, \alpha + \delta / 2)$, then \cref{alg:sector_dector} returns \texttt{True}, otherwise it returns \texttt{False}.  
\end{lemma}

% \begin{proof}[Proof of \cref{lemma:sector-test}]
%     We delay the proof to Appendix \ref{sec:proof-lemma:sector-test}.
% \end{proof}}

% \subsection{Proof of \cref{lemma:sector-test}}
% \label{sec:proof-lemma:sector-test}

\begin{proof}[Proof of \cref{lemma:sector-test}]

Recall that $f_{\min}$ is defined in Assumption \ref{assumption:model}. In addition, we define 
\begin{align}
    %& \bar{\ell}^{(k)}:= \lceil - 4^{d + 3} \log_{\gamma} n_k + \log_{\gamma} \delta_{\sin} + \log_{\gamma} \frac{1 - \gamma}{C_0}\rceil + 1, \\
     \delta_{\sin} := \min_{\theta \in \Theta} \prod_{j = 1}^{d - 3} \sin \alpha^{\theta}_j, \label{eq:def_delta_sin}
\end{align}
which is in $(0,1)$ by Assumption \ref{assumption:angle}. We select $n_0$ large enough, such that for all $n \geq n_0$, it holds that 
$$
\lceil \log n \rceil^2   \geq \bigg\lceil - 4^{d + 3} \log_{\gamma} n + \log_{\gamma} \delta_{\sin} + \log_{\gamma} \frac{1 - \gamma}{C_0}\bigg\rceil + 1 , \qquad \textrm{and} \qquad n  \geq  \exp\big( {4^{(d + 3) / 3}f_{\min}^{-1 / 3}} \big), $$
where $C_0$ is from \cref{lemma:best-response}. 
Recall that we set $\ell = \lceil \log n \rceil^2$ in \cref{alg:sector_dector}. 
As a consequence of \cref{lemma:best-response}, we see that for such an $n$, during the implementation of \cref{alg:sector_dector}, the agent in round $t$ will report $\theta_t'$ if  
\begin{align}
\label{eq:response-D1}
    \langle \Pi_{t, \theta_t'}, \bar{v}_{\theta_t} \rangle \geq   \langle \Pi_{t, \theta}, \, \bar v_{\theta_t} \rangle  - \frac{C _0 \cdot \gamma^{\ell} }{1- \gamma } > \langle \Pi_{t, \theta}, \, \bar v_{\theta_t} \rangle  - \frac{\delta_{\sin}}{n^{4^{d + 3}}}, \qquad \forall \, \theta \in \Theta. 
\end{align}
Since $\alpha \sim \Unif[0,2\pi)$, we conclude that with probability at least $1 - 3|\Theta| \cdot (r_d \delta_{\sin})^{-1/2}  \cdot n^{-4^{d + 2}}$ the following event occurs: 
\begin{align}\label{eq:4.1-event1}
	\bigcap_{\theta \in \Theta}\left\{ \alpha + \delta / 2, \alpha - \delta / 2, \alpha + \pi \notin \Big(\alpha_{d - 2}^{\theta}- 3 (r_d\delta_{\sin})^{-1/2} n^{-4^{d + 2}}, \alpha_{d - 2}^{\theta} + 3 (r_d\delta_{\sin})^{-1/2} n^{-4^{d + 2}} \Big) \right\}. 
\end{align}
We emphasize that in the equation above, angles are treated in the sense of modulus $2\pi$. That is to say, we make the convention that $\alpha = \cZ_{2\pi}(\alpha)$, with $\cZ_{2\pi}$ defined in Section \ref{sec:notation}. 
In what follows, we always assume event \eqref{eq:4.1-event1} occurs.  
We verify the validity of \cref{alg:sector_dector} under the true and false cases separately. 

\paragraph{Case I: there exists $\alpha_{d - 2}^{\theta}$ inside $(\alpha - \delta / 2, \alpha + \delta / 2)$.}
If there exists some $\theta \in \Theta$, such that $\alpha_{d - 2}^{\theta} \in (\alpha - \delta / 2, \alpha + \delta / 2)$,
by \eqref{eq:4.1-event1}, 
we know that 
$$
\arc(\alpha_{d - 2}^{\theta}, \alpha) + 3 (r_d\delta_{\sin})^{-1/2} n^{-4^{d + 2}} \leq \delta / 2.
$$
Notice that $\arc(\alpha_{d - 2}^{\theta} , \alpha + \delta ) \geq \delta - \arc(\alpha_{d - 2}^{\theta}, \alpha) $. 
Thus we have 
\begin{align}\label{eq:arc-diff}
	\arc(\alpha_{d - 2}^{\theta}, \alpha) \leq \arc(\alpha_{d - 2}^{\theta}, \alpha + \delta) - 6  (r_d \delta_{\sin})^{-1/2} n^{-4^{d + 2}}.
\end{align}
Furthermore, implementing the  mechanism in \eqref{eq:simple-mechanism} and applying \cref{eq:simple-product}, we get (recall $\delta_{\sin}$ is defined in \cref{eq:def_delta_sin})
\begin{align*}
	\langle \bar{v}_{\theta}, \Pi_{2, t} \rangle - \langle \bar{v}_{\theta}, \Pi_{3, t} \rangle \geq r_d \delta_{\sin} \cdot \Big( \cos(\alpha - \alpha_{d - 2}^{\theta}) - \cos (\alpha + \delta - \alpha_{d - 2}^{\theta})  \Big),
\end{align*}
which, by \cref{lemma:cos}, is no smaller than 
\begin{align*}
	\frac{r_d \delta_{\sin}}{30} \cdot \left(\arc(\alpha_{d - 2}^{\theta}, \alpha) -  \arc(\alpha_{d - 2}^{\theta}, \alpha + \delta) \right)^2. 
\end{align*}
Using \cref{eq:arc-diff}, we see that the above quantity is further strictly larger than $n^{-4^{d + 3}}$. 
In summary, we conclude that $\langle \bar{v}_{\theta}, \Pi_{2, t} \rangle - \langle \bar{v}_{\theta}, \Pi_{3, t} \rangle \geq n^{-4^{d + 3}}$. 
Similarly, we also have $\langle \bar{v}_{\theta}, \Pi_{2, t} \rangle - \langle \bar{v}_{\theta}, \Pi_{1, t} \rangle \geq n^{-4^{d + 3}}$. 
Putting these upper bounds and the responding condition \cref{eq:response-D1} together, we conclude that if $\theta_t = \theta$ in some round $t$, then the agent in this round will always report type 2. That is to say, if such $\theta$ exists, then in every round, with probability at least $f_{\min}$ the agent will report $2$. Hence the probability that the principal observes at least one reported type $2$, i.e., $N_{\mathrm{sec} } \geq 1 $ in a total number of $T_{\mathrm{sec}}= \lceil  \log n \rceil^4 $ rounds is at least 
\begin{align*}
	1 - (1 - f_{\min})^{\lceil  \log n \rceil^4  }.
\end{align*}
When $n \geq \exp\big( {4^{(d + 3) / 3}f_{\min}^{-1 / 3}} \big)$, we further have 
\begin{align*}
    1 - (1 - f_{\min})^{\lceil  \log n \rceil^4  } \geq 1 - (1 - f_{\min})^{4^{d + 3} f_{\min}^{-1} \lceil \log n \rceil} \geq 1 - e^{-4^{d + 3} \lceil \log n \rceil} \geq 1 - n^{-4^{d + 3}},
\end{align*}
where the second inequality follows as $(1 - x)^{1 / x} \leq e^{-1}$ for $x \in (0, 1)$.

%which is no smaller than $1 - n^{-4^{d + 3}}$ when $n \geq  \exp\big( {4^{(d + 3) / 3}f_{\min}^{-1 / 3}} \big) $. 
%\yuchen{explain this computation}

\paragraph{Case II: there does not exist $\alpha_{d - 2}^{\theta}$ inside $(\alpha - \delta / 2, \alpha + \delta / 2)$.}
In this case, for any $\theta \in \Theta$, $\alpha_{d - 2}^{\theta}$ either falls inside the sector $(\alpha + \delta / 2, \alpha + \pi)$, or falls inside the sector $(\alpha + \pi, \alpha -\delta / 2 + 2\pi )$. Without loss, we assume $\alpha_{d - 2}^{\theta} \in (\alpha + \delta / 2, \alpha + \pi)$. We comment that the other situation can be covered analogously. In this case we have
\begin{align*}
	& \arc(\alpha_{d - 2}^{\theta}, \alpha + \delta) \leq \arc(\alpha_{d - 2}^{\theta}, \alpha) - 6  (r_d\delta_{\sin})^{-1/2} n^{-4^{d + 2}}, \\
	& \arc(\alpha_{d - 2}^{\theta}, \alpha + \delta) \leq \arc(\alpha_{d - 2}^{\theta}, \alpha - \delta )- 6  (r_d \delta_{\sin})^{-1/2} n^{-4^{d + 2}}. 
\end{align*} 
%\yuchen{check this}
%To see the first inequality, note that by  \eqref{eq:4.1-event1} and the fact that $\alpha_{d - 2}^{\theta} \in (\alpha + \delta / 2, \alpha + \pi)$ we have 
%$$
%\arc(\alpha_{d - 2}^{\theta}, \alpha) \geq \delta /2 + 3   (r_d\delta_{\sin})^{-1/2} n^{-4^{d + 2}}, \qquad \arc(\alpha_{d - 2}^{\theta}, \alpha + \delta)  \leq \delta /2 - 3 (r_d\delta_{\sin})^{-1/2} n^{-4^{d + 2}}.
%$$
%
Similar to the analysis for Case I, we obtain
\begin{align*}
	\langle \bar{v}_{\theta}, \Pi_{3, t} \rangle - \langle \bar{v}_{\theta}, \Pi_{1, t} \rangle \geq n^{-4^{d + 3}}, \qquad \langle \bar{v}_{\theta}, \Pi_{3, t} \rangle - \langle \bar{v}_{\theta}, \Pi_{2, t} \rangle \geq n^{-4^{d + 3}}.
\end{align*}
Therefore, if $\theta_t = \theta$ in some round $t$, then the agent in this round will always report type 3. Following exactly the same route, we obtain that if $\alpha_{d - 2}^{\theta}$ falls insider $(\alpha + \pi, \alpha - \delta / 2 + 2\pi)$, then when $\theta_t = \theta$, the agent will always report type 1. In summary, the agent will never report type $2$.

\vspace{5pt}
{\noindent \bf Conclusion.} Putting together the above discussions, we conclude for $n \geq n_0$, with probability at least $1 - 3|\Theta| (r_d \delta_{\sin})^{-1/2} n^{-4^{d + 2}} - n^{-4^{d + 3}}$, the following statement is true: If there exists $\alpha_{d - 2}^{\theta}$ that falls inside the sector $(\alpha + \delta / 2, \alpha - \delta / 2)$, then \cref{alg:sector_dector} will return \texttt{True}, otherwise it will return \texttt{False}. This concludes the proof of the lemma. 
\end{proof}

\subsection{Proof of \cref{lemma:binary-search}}
\label{sec:proof-lemma:binary-search}

We establish theoretical guarantees for \cref{alg:binary_search} in this section by proving  \cref{lemma:binary-search}. 
We introduce the following notations before stating the proof. 

\begin{definition}[Angle gaps]\label{def:chi}
We define here the angle gaps, which measure the deviation from the equalities that appear in Assumption \ref{assumption:angle}.  
	\begin{enumerate}
		\item For $i \in [d - 3]$, we define
		\begin{align*}
			& \chi_i := \min_{\theta \neq \theta', \theta, \theta' \in \Theta}\Big||\alpha_i^{\theta} - \pi / 2| - |\alpha_i^{\theta'} - \pi / 2|\Big|, \\
			&  \tilde \chi_i := \min_{\theta \in \Theta} \min \Big\{ |\alpha_i^{\theta}|, |\alpha_i^{\theta} - \pi|, |\alpha_i^{\theta} - \pi / 2|, 0.1\Big \} .
		\end{align*}
  We make the convention that $\chi_{d - 2} = \tilde\chi_{d - 2} = \infty$. 
		\item For $i \in [d - 2]$, we let
		\begin{align*}
			\bar{\chi}_{i} := \min_{\theta \neq \theta', \theta, \theta' \in \Theta} \min \Big\{\arc(\alpha_{i}^{\theta}, \,\alpha_{i}^{\theta'}),\, \arc(\alpha_i^{\theta}, \, \alpha_i^{\theta'} + \pi ) \Big\}, 
		\end{align*}
  where we recall that $\arc$ is defined in \cref{sec:notation}.
	\end{enumerate}
 
%These quantities characterize the separation among the reward angles. 
We note that Assumption \ref{assumption:angle} implies that $\chi_i, \tilde \chi_i, \bar\chi_i$ are all strictly positive. 
	%For $i \in [d - 3]$ we define $\chi_i := \min_{\theta \neq \theta', \theta, \theta' \in \Theta}||\alpha_i^{\theta} - \pi / 2| - |\alpha_i^{\theta'} - \pi / 2|| $, $\tilde \chi_i := \min \{|\alpha_i^{\theta}|, |\alpha_i^{\theta} - \pi|, |\alpha_i^{\theta} - \pi / 2|: \theta \in \Theta\} \wedge 10^{-1}$ and $\bar\chi_i := \min_{\theta \neq \theta', \theta, \theta' \in \Theta}\arc_0(\alpha_i^{\theta}, \alpha_i^{\theta'})$. We define $\bar{\chi}_{d - 2} := \min_{\theta \neq \theta', \theta, \theta' \in \Theta} \arc(\alpha_{d - 2}^{\theta}, \alpha_{d - 2}^{\theta'})$. Assumption \ref{assumption:angle} implies that  $\chi_i, \tilde \chi_i, \bar\chi_i$ are all strictly positive. 	
\end{definition}

Now we are ready to prove \cref{lemma:binary-search}.

\begin{proof}[Proof of \cref{lemma:binary-search}]

We choose $n_0$ large enough such that it satisfies the requirement of \cref{lemma:sector-test}. 
In addition, we require $n_0$ is sufficiently large such that $\pi  n^{-4^{d + 1}} < \bar{\chi}_{d - 2}$ for all $n \geq n_0$. 
Note that $\pi / 2^K \leq \pi n^{-4^{d + 1}}$, hence for such $n$ we have $\arc(\alpha_{d - 2}^{\theta}, \alpha_{d - 2}^{\theta'}) > \pi / 2^K$ for all $\theta \neq \theta'$.

%
%Inspecting the proof of \cref{lemma:sector-test}, we see that conditioning on the occurence of event \eqref{eq:4.1-event1}, whether \cref{alg:sector_dector} can be implemented successfully during the course of \cref{alg:binary_search} are independent events that each has failure probability no larger than $ n^{-4^{d + 3}}$. 
For any $k \geq 1$, we denote by $\tilde\cE_k$ the event that \cref{alg:sector_dector} fails at least once in the $k$-th iteration of \cref{alg:binary_search}. 
Observe that 
\begin{align*}
    \PP(\tilde\cE_k) \leq \sum_{j = 1}^{2^{k}} \PP\big( \mathtt{SecTest} (\alpha + (j - 1) \pi / 2^{k - 1} + \pi / 2^k, \pi / 2^{k - 1})  \mbox{ fails} \big) \overset{(i)}{\leq} 2^k \cdot C_{\sfP} \cdot  n^{-4^{d + 2}}, 
\end{align*}
where $(i)$ follows from \cref{lemma:sector-test}.

%Then by \cref{lemma:sector-test} it holds that $\P(\tilde\cE_k) \leq 2^k \cdot C_{\sfP} \cdot  n^{-4^{d + 2}}$. 
{Here, we add up $2^k$ instead of $2|\Theta|$ probabilities, as failure guarantee for sector test is marginal, while implementing a specific sector test is a conditional event. }
%{\color{red} Why $2 ^k$? Should it be $2 |\Theta | $ because we run $2|\Theta| $ tests in each iteration. } \yuchen{explain $2^k$ }
%2^k \cdot \big( 3|\Theta| (r_d \delta_{\sin})^{-1/2} n^{-4^{d + 2}} + n^{-4^{d + 3}} \big) $.
%We define 
%
%\begin{align*}
%	\tilde\cE_k := \left\{ \mbox{\cref{alg:sector_dector} succeeds with inputs $(\alpha + q + \pi / 2^k, \pi / 2^{k - 1}, n)$ }: q \in \cQ_k\right\}. 
%\end{align*}
%
%From \cref{lemma:sector-test} we see that $\P(\tilde \cE_k) \leq 2^k \cdot \big( 3|\Theta| (r_d \delta_{\sin})^{-1/2} n^{-4^{d + 2}} + n^{-4^{d + 3}} \big) $. 
Note that when $\arc(\alpha_{d - 2}^{\theta}, \alpha_{d - 2}^{\theta'}) > \pi / 2^K$, the desiderata of the lemma might not be satisfied only when \cref{alg:sector_dector} fails at least once while implementing \cref{alg:binary_search}. According to our discussions above, this probability is upper bounded by 
	\begin{align*}
		\sum_{k = 1}^K \P(\tilde \cE_k) \leq \sum_{k = 1}^K 2^k \cdot C_{\sfP} n^{-4^{d + 2}} \leq   8 C_{\sfP} n^{-4^{d + 1}}, 
	\end{align*}
	thus completing the proof of the lemma. %\textcolor{cyan}{to be checked}
 \end{proof}

\subsection{Theoretical guarantee for \cref{alg:sector_dector2}}
\label{sec:alg4}

In this section, we provide theoretical guarantee for the conditional sector test  (\cref{alg:sector_dector2}).
We first state several conditions required for establishing \cref{lemma:conditional-sector-test}. 
Note that to prove Theorem \ref{thm:reward-function} we do not assume these conditions. 
Instead, later we shall prove that these conditions are satisfied with high probability. 

In \hyperlink{A1}{$\mathsf{(A1)}$}, we assume  $\hat \alpha^s_{(i + 1):(d - 2)}$, the input of \cref{alg:sector_dector2}, accurately recovers some true reward angles that do not belong to $\cM$. 

\begin{itemize}
\item[\hypertarget{A1}{$\mathsf{(A1)}$}]
There exists $\theta_{p} \in \Theta$, such that $\|\alpha_{(i + 1): (d - 2)}^{\theta_p} - \hat \alpha_{(i + 1): (d - 2)}^s\|_1 \leq 4 (d - 2 - i) \cdot n^{-4^{i + 4}}$. In addition, $\theta_p \notin \{\sfb_n(j): j \in \cM\}$.
\end{itemize}
 
In \hyperlink{A2}{$\mathsf{(A2)}$}, we assume the estimates in $\{\hat \alpha_{i:(d - 2)}^h\}_{h \in \cM}$ are accurate.

\begin{itemize}
\item[\hypertarget{A2}{$\mathsf{(A2)}$}]
For all $h \in \cM$, it holds that  $\| \hat \alpha_{i:(d - 2)}^h-  \alpha^{\sfb_n(h)}_{i:(d - 2)}\|_1 \leq 4 (d - 1 - i) \cdot  n^{-4^{i + 3}}$.
\end{itemize}

\cref{alg:grid_search} attempts to reconstruct the reward angles following the order $\alpha_i^{\sfb_n(s_1)}, \alpha_i^{\sfb_n(s_2)}, \cdots, \alpha_i^{\sfb_n(s_{|\Theta|})}$, as given in \cref{eq:s1-s2-permutation}. 
In \hyperlink{A3}{$\mathsf{(A3)}$} we assume the ordering is correct. 

\begin{itemize}
\item[\hypertarget{A3}{$\mathsf{(A3)}$}]
Recall $\alpha$ and $\delta$ are inputs to \cref{alg:sector_dector2}. 
Assume for all $\theta \in \Theta\, \backslash \{\sfb_n(j): j \in \cM\}$ and $\theta' \in \{\sfb_n(j): j \in \cM\}$, it holds that
		\begin{align*}
			|\alpha_i^{\theta} - \pi / 2| > \min \{|\alpha - \delta / 2 - \pi / 2|, |\alpha + \delta / 2 - \pi / 2|\} > |\alpha_i^{\theta'} - \pi / 2|.
		\end{align*}
\end{itemize}

In \hyperlink{A4}{$\mathsf{(A4)}$} we assume the true reward angle is not too close to the boundary of the testing interval $(\alpha - \delta / 2, \alpha + \delta / 2)$. 

\begin{itemize}
\item[\hypertarget{A4}{$\mathsf{(A4)}$}]
We assume that
\begin{align}
\label{eq:sevent}
\begin{split}
    & \big|\alpha + \delta / 2 - \alpha_i^{\theta_p} \big|  \geq 6(r_d \delta_{\sin})^{-1/2} n^{-4^{d + 2}} + 20 d^{1/2}( \delta_{\sin} )^{-1/2} n^{-4^{i + 3.5}}, \\
    & \big|\alpha - \delta / 2 - \alpha_i^{\theta_p} \big|  \geq 6(r_d \delta_{\sin})^{-1/2} n^{-4^{d + 2}} + 20 d^{1/2}( \delta_{\sin} )^{-1/2} n^{-4^{i + 3.5}}.
\end{split}
\end{align}
Recall that $\delta_{\sin}$ is defined in \cref{eq:def_delta_sin}, and $r_d$ is from \cref{sec:reward-function}. 
		%
		% \begin{align}
		% 	\alpha \pm \delta / 2 \in \left\{ x: \big|x - \alpha_i^{\theta_p} \big|  \geq 6(r_d \delta_{\sin})^{-1/2} n^{-4^{d + 2}} + 20 d^{1/2}( \delta_{\sin} )^{-1/2} n^{-4^{i + 3.5}}\right\}. 
		% \end{align}
\end{itemize}

In \hyperlink{A5}{$\mathsf{(A5)}$}, we assume the testing interval belongs to $[0, \pi]$. 

\begin{itemize}
\item[\hypertarget{A5}{$\mathsf{(A5)}$}]
We assume $\alpha \pm \delta \in [0, \pi]$.
\end{itemize}

In \hyperlink{A6}{$\mathsf{(A6)}$} we assume the testing interval length $\delta$ is sufficiently small. 

\begin{itemize}
\item[\hypertarget{A6}{$\mathsf{(A6)}$}]
We assume $\delta \leq \frac{\min\{\chi_i,   \bar{\chi}_i,  \tilde \chi_i\}}{4}  \wedge \frac{\min\{\chi_{i + 1}, \tilde\chi_{i + 1}, \bar\chi_{i + 1}\}^2 \tilde\chi_i}{180}$, where we recall that the $\chi$'s are from Definition \ref{def:chi}. 
\end{itemize}

Finally, in \hyperlink{A7}{$\mathsf{(A7)}$} we assume $n$ is sufficiently large. 

\begin{itemize}
\item[\hypertarget{A7}{$\mathsf{(A7)}$}]
We assume that $n$ is larger than a threshold that depends only on $(\mathscr{P}, \varphi_0)$, such that 
%sufficiently large, such that $n \delta^{1 / 4^{i + 4}}$ is larger than a threshold that depends only on $(\mathscr{P}, \varphi_0)$.
%
\begin{align*}
    & \frac{\delta_{\sin} r_d \bar\chi_i^2}{120} \geq   n^{-4^{d + 3}} + 8 r_d d n^{-4^{i + 3}}, \\
    & \lceil \log n \rceil^2   \geq \bigg\lceil - 4^{d + 3} \log_{\gamma} n + \log_{\gamma} \delta_{\sin} + \log_{\gamma} \frac{1 - \gamma}{C_0}\bigg\rceil + 1, \qquad \mbox{recall $C_0$ is from \cref{lemma:best-response},} \\
    & n  \geq  \exp\big( {4^{(d + 3) / 3}f_{\min}^{-1 / 3}} \big), \\
    & \frac{r_d \delta_{\sin}^2 \tilde \chi_i \bar\chi_{i + 1}^2}{120} \geq n^{-4^{d + 3}} + 8 r_d d n^{-4^{i + 3}}, \\
    & \frac{r_d \delta  \delta_{\sin} \tilde \chi_i \min \{\chi_{i + 1},  \tilde\chi_{i + 1}, \bar\chi_{i + 1}\}^2}{120} \geq n^{-4^{d + 3}} + 8 r_d d n^{-4^{i + 4}}.
\end{align*}
\end{itemize}

We then establish \cref{lemma:conditional-sector-test} that gives theoretical guarantee for \cref{alg:sector_dector2}. 

\begin{lemma}\label{lemma:conditional-sector-test}
%	This lemma studies the theoretical properties of \cref{alg:sector_dector2}. In this part, we 
Consider \cref{alg:sector_dector2} with $|\cM| + 3 \leq |\Theta|$, assume the conditions of \cref{thm:reward-function}, and additionally assume \hyperlink{A1}{$\mathsf{(A1)}$}-\hyperlink{A7}{$\mathsf{(A7)}$}. 
%Under the conditions of \cref{lemma:alpha-k-next}, 
% if in addition:
%
Then with probability at least $1 - n^{-4^{d + 3}}$ the following statement is true: If $\alpha_i^{\theta_p} \in (\alpha - \delta / 2, \alpha + \delta / 2)$, then \cref{alg:sector_dector2} returns $\mathtt{True}$, otherwise it returns $\mathtt{False}$. 

\end{lemma}

\begin{proof}[Proof of \cref{lemma:conditional-sector-test}]
    We prove \cref{lemma:conditional-sector-test} in Appendix \ref{sec:proof-lemma:conditional-sector-test}. 
\end{proof}

If the event stated in \cref{lemma:conditional-sector-test} occurs, then we say \cref{alg:sector_dector2} \emph{succeeds}. \cref{lemma:conditional-sector-test} tells us that if the conditions listed in the lemma  are all satisfied, then with probability at least $1 - n^{-4^{d + 3}}$ \cref{alg:sector_dector2} will succeed.

 \subsection{Proof of \cref{lemma:matching}}
 \label{sec:proof-lemma_matching}

To prove \cref{lemma:matching}, we instead prove \cref{lemma:alpha-k-next} that is strictly stronger than \cref{lemma:matching}.
In this proof, we denote by $(s_1, \cdots, s_{|\Theta|})$ a permutation of $\Theta$ that satisfies
\begin{align}
\label{eq:s1-s2-permutation}    |\alpha_i^{\sfb_n (s_1)} - \pi / 2| < |\alpha_i^{\sfb_n (s_2)} - \pi / 2| < \cdots < |\alpha_i^{\sfb_n (s_{|\Theta|})} - \pi / 2|. 
\end{align}
Recall that $\sfb_n$ defined in \cref{eq:def-bn-perm} is random, hence $(s_1, \cdots, s_{|\Theta|})$ is also random. We do not know $s_j$, and shall estimate it from data. 

\begin{lemma}\label{lemma:alpha-k-next}

Under the same assumptions made in \cref{thm:reward-function}, 
further suppose we have obtained a sufficiently accurate  $\hat \balpha_{(i+1):(d-2)}$ such that $\tilde \dist_{i+1}(\balpha_{(i+1):(d-2)}, \hat \balpha_{(i+1):(d-2)}) \leq 4(d - 2 - i) \cdot n^{-4^{i + 4}}$.
%We assume the assumptions of \cref{thm:reward-function}. 
%We also assume that we have obtained $\cZ_{i + 1}$ that satisfies $\tilde \dist(\cZ_{i + 1}, \bar{\cZ}_{i + 1}) \leq 4(d - 2 - i) \cdot n^{-4^{i + 4}}$.
In addition, we assume that for some $k \in \{0, 1, \cdots, |\Theta| - 1\}$ we have access to $(\hat \alpha_i^{s_j})_{ j \in [k] }$  such that 
%$\hat\alpha_i^{\hat s_j}$ for all $j \in \{1,\cdots, k\}$ that satisfies 
%
%Let $i \in \{d - 3, d - 4, \cdots, 1\}$. Assume we are given $\cZ_{i + 1}$ such that $\tilde \dist(\cZ_{i + 1}, \bar{\cZ}_{i + 1}) \leq \pi(d - 2 - i)  \cdot n^{-4^{i + 4}}$. 
%In addition, for some $0 \leq k \leq |\Theta| - 1$, assume that we know $\hat\alpha_i^{s_j}$ for all $j \in \{1,\cdots, k\}$, such that 
	 %
	 \begin{align*}
	 	\sup_{j \in [k]}|\hat \alpha_i^{ s_j} - \alpha_i^{\sfb_n(s_j)}| \leq 4 n^{-4^{i + 3}}. 
	 \end{align*}
	 %
%Here, $\hat s_j$ is an estimate of $s_j$. 
Then, there exist $n_0 \in \NN_+$ and $C_1, C_2 \in \RR_+$ that depend only on $(\mathscr{P}, \varphi_0)$, such that for all $n \geq n_0$, there exists an algorithm that uses no more than $C_1 \cdot (\log n)^6$ samples, and with probability at least $1 - C_2 \cdot n^{-50}$, this algorithm constructs $\hat\alpha_i^{ s_{k + 1}}$ that satisfies $|\hat{\alpha}_i^{s_{k + 1}} - \alpha_i^{\sfb_n(s_{k + 1})}| \leq 4 n^{-4^{i + 3}}$.
  
%	 Then there exists $n_0 \in \NN_+$ and $C_1, C_2 \in \RR_+$ that depends only on $(\{\bar{v}_{\theta}: \theta \in \Theta\}, \varphi_0)$, and an algorithm that uses samples from no more than $C_1\cdot (\log n)^5$ rounds, such that for all $n \geq n_0$, with probability at least $1 - C_2 \cdot n^{-50}$ this algorithm generates $\hat\alpha_i^{s_{k + 1}}$ satisfying $|\hat{\alpha}_i^{s_{k + 1}} - \alpha_i^{\sfb(s_{k + 1})}| \leq \pi n^{-4^{i + 3}}$. 
\end{lemma}

\begin{remark}
    \cref{lemma:alpha-k-next} implies that we can construct an accurate $\hat \balpha_{i:(d-  2)}$ given that we have access to an accurate $\hat\balpha_{(i + 1):(d - 2)}$. 
\end{remark}

\begin{proof}[Proof of \cref{lemma:alpha-k-next}]
     We divide the proof of \cref{lemma:alpha-k-next} into three parts depending on the value of $k$. The proof for $0 \leq k \leq |\Theta| - 3$ is in \cref{sec:appendix-D3}, the proof for $k = |\Theta| - 2$ is in \cref{sec:appendix-other1}, and the proof for $k = |\Theta| - 1$ is in \cref{sec:appendix-other2}. 

\end{proof}

%For the sake of compactness, in the main text we will only state the algorithm that achieves the desiderata of \cref{lemma:alpha-k-next} for  $k \leq |\Theta| - 3$. Theoretical guarantee for this part is presented in Appendix \ref{sec:appendix-D3}.  We delay the statement of the algorithm when $k \in \{|\Theta| - 1, |\Theta| - 2\}$ to Appendices \ref{sec:appendix-other1} and \ref{sec:appendix-other2}. 

%Indeed, when $k \leq |\Theta | - 3$, we shall prove that \cref{alg:grid_search} presented below actually achieves a stronger guarantee than that required by \cref{lemma:alpha-k-next}. More precisely, we shall prove that as long as the input $\cZ_{i + 1}$ satisfies $\tilde \dist(\cZ_{i + 1}, \bar{\cZ}_{i + 1}) \leq \pi(d - 2 - i) \cdot n^{-4^{i + 4}}$ (which is required as a condition of \cref{lemma:alpha-k-next}), \cref{alg:grid_search} with probability at least $1 - C_2 \cdot n^{-60}$ uses no more than $C_1 (\log n)^5$ samples, and is able to output $(\hat\alpha_i^{s_j}: j \in [|\Theta| - 2])$ such that $|\hat\alpha_i^{s_j} - \alpha_i^{\sfb(s_j)}| \leq \pi n^{-4^{i + 3}}$ for all $j \in [|\Theta| - 2]$. This automatically proves \cref{lemma:alpha-k-next} for all $0 \leq k \leq |\Theta| - 3$.

\subsubsection{Proof of \cref{lemma:alpha-k-next} when $0 \leq k \leq |\Theta| - 3$}
\label{sec:appendix-D3}

In this section, we prove \cref{lemma:alpha-k-next} for $0 \leq k \leq |\Theta| - 3$. 
To this end, we prove a stronger version of the lemma that we state below. 
\begin{lemma}\label{lemma:stronger-grid}

%We assume the assumptions of \cref{thm:reward-function}. 
%We also assume that we have obtained $\cZ_{i + 1}$ that satisfies $\tilde \dist(\cZ_{i + 1}, \bar{\cZ}_{i + 1}) \leq 4(d - 2 - i) \cdot n^{-4^{i + 4}}$.
Under the same assumptions made in \cref{thm:reward-function}, 
further suppose we have obtained a sufficiently accurate  $\hat \balpha_{(i+1):(d-2)}$ such that $\tilde \dist_{i+1}(\balpha_{(i+1):(d-2)}, \hat \balpha_{(i+1):(d-2)}) \leq 4(d - 2 - i) \cdot n^{-4^{i + 4}}$.
Then, there exist $n_0 \in \NN_+$ and $C_1, C_2 \in \RR_+$ that depend only on $(\mathscr{P}, \varphi_0)$, such that for all $n \geq n_0$, the for loop of \cref{alg:grid_search} with inputs $(n, \hat \balpha_{(i+1):(d-2)})$ uses no more than $C_1 \cdot (\log n)^6$ samples, and with probability at least $1 - C_2 \cdot n^{-50}$ generates $(\hat\alpha_i^{s_{k}})_{1 \leq k \leq |\Theta| - 2}$  that satisfies $|\hat{\alpha}_i^{s_{k}} - \alpha_i^{\sfb_n(s_{k})}| \leq 4 n^{-4^{i + 3}}$ for all $k \in  [|\Theta| - 2]$.

	%Let $i \in \{d - 3, d - 4, \cdots, 1\}$. Assume we are given $\cZ_{i + 1}$ such that $\tilde \dist(\cZ_{i + 1}, \bar{\cZ}_{i + 1}) \leq \pi(d - 2 - i) \cdot |\Theta| \cdot n^{-4^{i + 4}}$. 
%	Then there exists $n_0 \in \NN_+$ that depends only on $(\cP, \varphi_0)$, 
%	and an algorithm that uses samples from no more than 
%	$$|\Theta| \cdot \big( \lceil 4^{i + 3} \log _2 n\rceil + 1 + 2\lceil \log n \rceil \big) \cdot \big( \lceil \log n \rceil^4 + \lceil \log n \rceil^2 + 2 \big)$$ rounds, such that for all $n \geq n_0$, with probability at least 
%	$$1 - 16\lceil \log n \rceil^2 \cdot \frac{24(r_d \delta_{\sin})^{-1/2} n^{-60} + 80 d^{1/2}( \delta_{\sin} )^{-1/2} n^{-60}}{\pi }$$ \cref{alg:grid_search} generates $(\hat\alpha_i^{s_{k}})_{1 \leq k \leq |\Theta| - 3}$ such that $|\hat{\alpha}_i^{s_{k}} - \alpha_i^{\sfb(s_{k})}| \leq \pi n^{-4^{i + 3}}$ for all $k \in  [|\Theta| - 3]$.
\end{lemma}

%To prove \cref{lemma:alpha-k-next} when $1 \leq k \leq |\Theta| - 3$, it suffices to prove \cref{lemma:stronger-grid}.  

%\subsubsection{Theoretical guarantee for \cref{alg:sector_dector2}} \label{sec:appendix-conditional-sector-test}

%Prior to proving \cref{lemma:stronger-grid}, 

%\subsubsection{Proof of \cref{lemma:stronger-grid}} \label{sec:proof:lemma:stronger-grid}
\begin{proof}[Proof of \cref{lemma:stronger-grid}]
With \cref{lemma:conditional-sector-test}, we are ready to prove \cref{lemma:stronger-grid}. In what follows, we shall choose a sufficiently large $n$ such that whenever $\pi / (2 \lceil \log n \rceil) \geq \delta \geq 2 n^{-4^{i + 3}}$, \hyperlink{A6}{$\mathsf{(A6)}$} and \hyperlink{A7}{$\mathsf{(A7)}$} hold. 
In addition, we require $\min\{\tilde \chi_i, \chi_i\} \geq \pi /  \log n  + \sqrt{\pi /   \log n }$ and $\tilde \chi_i > 4\iota$ (recall $\tilde \chi_i$ and $\chi_i$ are from Definition \ref{def:chi}). 
Moreover, observe that throughout \cref{alg:grid_search}, all conditional sector tests $\mathtt{ConSecTest}$ are performed over intervals contained in $[0, \pi]$. As a result, assumption \hyperlink{A5}{$\mathsf{(A5)}$} holds for every conditional sector test used in \cref{alg:grid_search}.
In addition, under the assumptions of \cref{lemma:stronger-grid}, assumption \hyperlink{A1}{$\mathsf{(A1)}$} holds for every conditional sector test used in \cref{alg:grid_search}. 

Next, we prove that with high probability \hyperlink{A4}{$\mathsf{(A4)}$} holds for all conditional sector tests that appear in \cref{alg:grid_search}. 
Recall that $\iota = \pi / (2N)$ and $N = \lceil \log n \rceil$. 
Hence to prove the desired claim, it suffices to upper bound the probability of  
\begin{align*}
	\alpha \pm \delta / 2 \in \left\{ x: \big|x - \alpha_i^{\theta_p} \big|  \geq 6(r_d \delta_{\sin})^{-1/2} n^{-4^{d + 2}} + 20 d^{1/2}( \delta_{\sin} )^{-1/2} n^{-4^{i + 3.5}}\right\}
\end{align*}
for all 
\begin{align*}
	\alpha = \pm \big( u + \frac{\pi (s + 1/2)}{2^{K + 1} N} \big), \qquad \delta = \frac{\pi }{2^{K + 1} N}, 
\end{align*}
where $s \in \{0, 1, \cdots, 2^{K + 1}N - 1\}$ and $K \in \{0, 1, \cdots \lceil 4^{i + 3} \log_2 n \rceil\}$. Since $u \sim \sqrt{\pi / 2N} + \Unif[0, \pi / 2N)$, we can then conclude that the probability of the above event is no smaller than  
\begin{align*}
	& 1 - \sum_{K = 0}^{\lceil 4^{i + 3} \log_2 n \rceil} 2^{K + 2} N^2  \cdot \frac{24(r_d \delta_{\sin})^{-1/2} n^{-4^{d + 2}} + 80 d^{1/2}( \delta_{\sin} )^{-1/2} n^{-4^{i + 3.5}}}{\pi } \\
	& \qquad \geq 1 - 16\lceil \log n \rceil^2 n^{4^{i + 3}} \cdot \frac{24(r_d \delta_{\sin})^{-1/2} n^{-4^{d + 2}} + 80 d^{1/2}( \delta_{\sin} )^{-1/2} n^{-4^{i + 3.5}}}{\pi } \\
	 & \qquad \geq 1 - 16\lceil \log n \rceil^2 \cdot \frac{24(r_d \delta_{\sin})^{-1/2} n^{-60} + 80 d^{1/2}( \delta_{\sin} )^{-1/2} n^{-60}}{\pi }. 
\end{align*} 
In summary, with probability at least $1 - 16\lceil \log n \rceil^2 \cdot \frac{24(r_d \delta_{\sin})^{-1/2} n^{-60} + 80 d^{1/2}( \delta_{\sin} )^{-1/2} n^{-60}}{\pi }$, \hyperlink{A4}{$\mathsf{(A4)}$} holds for all conditional sector tests that appear in \cref{alg:grid_search}. 

Next, we prove that with high probability \hyperlink{A2}{$\mathsf{(A2)}$} and \hyperlink{A3}{$\mathsf{(A3)}$} hold for all conditional sector tests that appear in \cref{alg:grid_search}. 
To this end, we proceed by a induction over $j \in [N_0]$, 
and establish the following lemma. 
\begin{lemma}
\label{lemma:vtu}
Under the conditions of \cref{lemma:stronger-grid}, for every $j \in [N_0]$, %throughout the $j$-th round of the for loop in \cref{alg:grid_search},
with probability at least $1 - j \rho_n$ the following statements are true for the first $j$ rounds of the for loop in \cref{alg:grid_search}: 
\begin{enumerate}
    \item For current $\cM$ and all $h \in \cM$, it holds that  $\| \hat \alpha_{i:(d - 2)}^h-  \alpha^{\sfb_n(h)}_{i:(d - 2)}\|_1 \leq 4 (d - 1 - i) \cdot  n^{-4^{i + 3}}$. Namely, \hyperlink{A2}{$\mathsf{(A2)}$} holds throughout the first $j$ rounds. 
    \item For current $\cM$, all $\theta \in \Theta\, \backslash \{\sfb_n(j): j \in \cM\}$ and $\theta' \in \{\sfb_n(j): j \in \cM\}$, it holds that
		\begin{align*}
			|\alpha_i^{\theta} - \pi / 2| > \min \{|\alpha - \delta / 2 - \pi / 2|, |\alpha + \delta / 2 - \pi / 2|\} > |\alpha_i^{\theta'} - \pi / 2|.
		\end{align*}
        Namely, \hyperlink{A3}{$\mathsf{(A3)}$} holds throughout the first $j$ rounds. Here, recall $\alpha$ is the center of the testing interval and $\delta$ is the interval length. In this lemma, we define 
        \begin{align*}
            \rho_n = 16\lceil \log n \rceil^2 \cdot \frac{24(r_d \delta_{\sin})^{-1/2} n^{-60} + 80 d^{1/2}( \delta_{\sin} )^{-1/2} n^{-60}}{\pi } + 4^{d + 2}  |\Theta| \lceil \log_2 n \rceil^2 n^{-4^{d + 3}}. 
        \end{align*}
\end{enumerate}
\end{lemma}
\begin{proof}[Proof of \cref{lemma:vtu}]
    Throughout the proof, we assume that \hyperlink{A1}{$\mathsf{(A1)}$}, \hyperlink{A4}{$\mathsf{(A4)}$}, \hyperlink{A5}{$\mathsf{(A5)}$}, \hyperlink{A6}{$\mathsf{(A6)}$}, and \hyperlink{A7}{$\mathsf{(A7)}$} hold for all conditional sector tests in \cref{alg:grid_search}, which by the discussions above occurs with probability at least $1 - 16\lceil \log n \rceil^2 \cdot \frac{24(r_d \delta_{\sin})^{-1/2} n^{-60} + 80 d^{1/2}( \delta_{\sin} )^{-1/2} n^{-60}}{\pi }$.      
    Also observe that in \cref{alg:grid_search}, the set $\cM$ is updated only at the end of each iteration of the for loop.

    We prove this lemma by induction. 
    In the first round of the for loop, we start with an empty $\cM$, and we shall prove that with high probability $\cM$ remains empty after completing this round. 
    By assumption, for all $\theta \in \Theta$ we have $|\alpha_i^{\theta} - \pi / 2| \geq \tilde \chi_i \geq \pi /  \lceil \log n \rceil + \sqrt{\pi /  \lceil \log n \rceil} > 2 \iota + \sqrt{\iota}$ (recall $\iota = \pi / (2 \lceil \log n \rceil)$).  
    Observe that for all $\alpha$'s that appear in the first for loop of \cref{alg:grid_search}, 
    it holds that $\max \{|\alpha - \pi / 2 - \delta / 2|, |\alpha - \pi / 2 + \delta / 2| \} \leq u + \iota < 2\iota + \sqrt{\iota} < |\alpha_i^{\theta} - \pi / 2|$ (recall $\delta$ is the length of the testing interval used in $\mathtt{ConSecTest}$ and $\alpha$ is the center of the testing interval).   
    Therefore, during the first round of the for loop, \hyperlink{A2}{$\mathsf{(A2)}$} and \hyperlink{A3}{$\mathsf{(A3)}$} hold as long as $\cM$ is still empty, which is naturally guaranteed before the final step of the first round.
    In addition, no reward angle $\alpha_i^{\theta}$ lies within any testing interval considered during the first round of the for loop in \cref{alg:grid_search}.
    
    Putting together the above discussions, we see that \hyperlink{A1}{$\mathsf{(A1)}$}-\hyperlink{A7}{$\mathsf{(A7)}$} are satisfied for all conditional sector tests that appear in the first round of the for loop. 
    Using \cref{lemma:conditional-sector-test}, we see that any conditional sector test that appears in the first round of the for loop in \cref{alg:grid_search} will return False with probability at least $1 - n^{-4^{d + 3}}$. 
    %Note that $\cM$ remains empty throughout the first round as long as no conditional sector test from the first round returns True. 
    Therefore, with probability at least $$1 - 2|\Theta| n^{-4^{d + 3}} - 16\lceil \log n \rceil^2 \cdot \frac{24(r_d \delta_{\sin})^{-1/2} n^{-60} + 80 d^{1/2}( \delta_{\sin} )^{-1/2} n^{-60}}{\pi } \geq 1 - \rho_n,$$ all conditional sector tests that appear in the first round of the for loop in \cref{alg:grid_search} will return False. 
    In this case, $\cM$ remains empty throughout the first round of the for loop, hence \hyperlink{A2}{$\mathsf{(A2)}$} and \hyperlink{A3}{$\mathsf{(A3)}$} hold throughout this round. 
    To summarize, with probability at least $1 - \rho_n$,
    %$$1 - 2|\Theta| n^{-4^{d + 3}} - 16\lceil \log n \rceil^2 \cdot \frac{24(r_d \delta_{\sin})^{-1/2} n^{-60} + 80 d^{1/2}( \delta_{\sin} )^{-1/2} n^{-60}}{\pi },$$ 
    claims 1 and 2 hold throughout the first round of the for loop. 
    We have completed the proof for $j = 1$. 

    Now suppose the lemma holds for all $j \in [j_0]$, we then prove it also holds for $j = j_0 + 1$. 
    In this proof, we assume that claims 1 and 2 hold throughout the first $j_0$ rounds of the for loop in \cref{alg:grid_search}, which by induction happens with probability at least $1 - j_0 \rho_n$. 
    In this case, with probability at least $1 - j_0 \rho_n - 16\lceil \log n \rceil^2 \cdot \frac{24(r_d \delta_{\sin})^{-1/2} n^{-60} + 80 d^{1/2}( \delta_{\sin} )^{-1/2} n^{-60}}{\pi }$, \hyperlink{A1}{$\mathsf{(A1)}$}-\hyperlink{A7}{$\mathsf{(A7)}$} hold throughout the first $j_0$ rounds of the for loop.
    Using \cref{lemma:conditional-sector-test}, we conclude that with probability at least $$1 - j_0 \rho_n - 16\lceil \log n \rceil^2 \cdot \frac{24(r_d \delta_{\sin})^{-1/2} n^{-60} + 80 d^{1/2}( \delta_{\sin} )^{-1/2} n^{-60}}{\pi } - 4^{d + 2} j_0 |\Theta| \lceil \log_2 n \rceil n^{-4^{d + 3}},$$ \cref{alg:sector_dector2} always succeeds throughout the first $j_0$ rounds. 
    In this case, if there does not exist $\alpha_i^{\theta}$ that lies within $(\pi / 2 + u + (j_0-1) \iota,\, \pi / 2 + u + j_0 \iota ) \cup (\pi / 2 - u - j_0 \iota,\, \pi / 2 - u - (j_0 - 1) \iota)$, then $\cM$ does not get updated at the end of the $j_0$-th round. Otherwise, if there exists $\alpha_i^{\theta} \in (\pi / 2 + u + (j_0-1) \iota,\, \pi / 2 + u + j_0 \iota ) \cup (\pi / 2 - u - j_0 \iota,\, \pi / 2 - u - (j_0 - 1) \iota)$. Since $\chi_i > \iota$, then such $\theta$ is unique, and at the end of the $j_0$-th round, \cref{alg:grid_search} updates $\cM \gets \cM \cup \{\sfb_n^{-1}(\theta)\}$. 
    In both cases, after completing the $j_0$-th round of the for loop, for all $\theta \in \Theta \backslash \{\sfb_n (j): j \in \cM\}$ and $\theta' \in \{\sfb_n(j): j \in \cM\}$, we have $|\alpha_i^{\theta} - \pi / 2| > u + j_0 \iota > |\alpha_i^{\theta'} - \pi / 2|$. 
    As a consequence, claims 1 and 2  hold in the first $(j_0 + 1)$ rounds of the for loop before \cref{alg:grid_search} potentially updates $\cM$ in the final step of the $(j_0 + 1)$-th round. 
    Once again using \cref{lemma:conditional-sector-test}, we see that with probability at least 
$$1 - j_0 \rho_n - 16\lceil \log n \rceil^2 \cdot \frac{24(r_d \delta_{\sin})^{-1/2} n^{-60} + 80 d^{1/2}( \delta_{\sin} )^{-1/2} n^{-60}}{\pi } - 4^{d + 2} (j_0 + 1) |\Theta| \lceil \log_2 n \rceil n^{-4^{d + 3}} \geq 1 - (j_0 + 1) \rho_n,$$
\cref{alg:sector_dector2} always succeeds throughout the first $(j_0 + 1)$ rounds.  
In the above equation, the second lower bound is because $j_0 + 1 \leq N_0 \leq \lceil \log n \rceil$. 
In this case, if there does not exist $\alpha_i^{\theta}$ that lies within $(\pi / 2 + u + j_0 \iota,\, \pi / 2 + u + (j_0 + 1) \iota ) \cup (\pi / 2 - u - (j_0 + 1) \iota,\, \pi / 2 - u - j_0 \iota)$, 
then identical to the previous proof, we can show that the set $\cM$ does not get updated at the end of the $(j_0 + 1)$-th round, and claims 1 and 2 hold throughout the first $(j_0 + 1)$ rounds. 
On the other hand, if there exists $\alpha_i^{\theta}$ that belongs to $(\pi / 2 + u + j_0 \iota,\, \pi / 2 + u + (j_0 + 1) \iota) \cup (\pi / 2 - u - (j_0 + 1) \iota,\, \pi / 2 - u - j_0 \iota)$. Since $\chi_i > \iota$, we conclude that there is a unique $\alpha_i^{\theta}$ that lies within $(\pi / 2 + u + j_0 \iota,\, \pi / 2 + u + (j_0 + 1) \iota) \cup (\pi / 2 - u - (j_0 + 1) \iota,\, \pi / 2 - u - j_0 \iota)$, and at the end of the $(j_0 + 1)$-th round \cref{alg:grid_search} updates $\cM$ as $\cM \cup \{ \sfb_n^{-1}(\theta) \}$. 
In this case one can verify that  $|\hat\alpha_{i}^{\sfb_n^{-1}(\theta)} - \alpha_i^{\theta}| \leq 4 n^{-4^{i + 3}}$, hence
\begin{align*}
    \|\hat \alpha_{i:(d - 2)}^{\sfb_n^{-1}(\theta)} - \alpha_{i:(d - 2)}^{\theta}\|_1 \leq & \|\hat \alpha_{(i + 1):(d - 2)}^{\sfb_n^{-1}(\theta)} - \alpha_{(i + 1):(d - 2)}^{\theta}\|_1 + |\hat \alpha_i^{\sfb_n^{-1} (\theta)} - \alpha_i^{\theta}| \\
    \leq & 4(d - 2 - i) n^{-4^{i + 4}} + 4n^{-4^{i + 3}} \leq  4(d - 1 - i) n^{-4^{i + 3}}. 
\end{align*}
In this case, claims 1 and 2 hold throughout the first $(j_0 + 1)$ rounds of the for loop of \cref{alg:grid_search}. 

In summary, with probability at least $1 - (j_0 + 1) \rho_n$, claims 1 and 2 hold throughout the first $(j_0 + 1)$ rounds of the for loop of \cref{alg:grid_search}. 
The proof is done by induction.

\end{proof}

Finally, we prove \cref{lemma:stronger-grid}.
Combining \cref{lemma:vtu} and the previous discussions, we see that with probability at least $1 - (\lceil \log n \rceil + 1) \rho_n \geq 1 - C_2 \cdot n^{-50}$, conditions \hyperlink{A1}{$\mathsf{(A1)}$} to \hyperlink{A7}{$\mathsf{(A7)}$} are satisfied throughout the for loop of \cref{alg:grid_search}. 
In addition, the for loop takes samples from no more than $C_1 \cdot (\log n)^6$ rounds.

Note that when the for loop of \cref{alg:grid_search} is completed, if $|\cM| = |\Theta| - 2$ then the desired claim follows from \hyperlink{A2}{$\mathsf{(A2)}$} and \hyperlink{A3}{$\mathsf{(A3)}$}. 
On the other hand, if $|\cM| < |\Theta| - 2$ after the for loop concludes, 
then after the for loop concludes, for all $\theta \in \Theta \backslash \{\sfb_n(j): j \in \cM\}$, by \hyperlink{A3}{$\mathsf{(A3)}$} it holds that $|\alpha_i^{\theta} - \pi / 2| > u + (N - 1) \iota \geq \sqrt{\iota} + (N - 1) \iota \geq \pi / 2 - 4 \iota$. 
Therefore, if $|\cM| = |\Theta| - 2$ we have $\tilde \chi_i \leq 4 \iota$, and there is a contradiction.
So it must be the case that $|\cM| = |\Theta| - 2$ after the for loop is completed. 
In addition, by \hyperlink{A3}{$\mathsf{(A3)}$} at this point we have $\cM = \{s_1, s_2, \cdots, s_{|\Theta| - 2} \}$ (recall that $s_i$ is defined in \cref{eq:s1-s2-permutation}), and \cref{alg:sector_dector3} constructs $\hat\alpha_i^j$ sequentially for $j = s_1, s_2, \cdots, s_{|\Theta| - 2}$.  
The proof is done.
\end{proof}

\begin{comment}

Note that if we can prove with high probability \cref{alg:sector_dector2} succeeds for all inputs that take the following form 
%
\begin{align*}
	\left( \pm \big( u + \frac{\pi (s + 1/2)}{2^{K + 1} N} \big),\, \frac{\pi }{2^{K + 1} N}, \, n, \,\hat\balpha^{h}_{(i + 1) : (d - 2)}, \, \{s_r: 1 \leq r \leq k\},\, \cZ_{i + 1} \right),
\end{align*}
%
where $s \in \{0, 1, \cdots, 2^{K + 1}N - 1\}$, $K \in \{0, 1, \cdots \lceil 4^{i + 3} \log_2 n \rceil\}$, $h \in [|\Theta|]$, and $1 \leq k \leq |\Theta| - 3$. Then it is not hard to check that the output of \cref{alg:grid_search} also will  satisfy the desiderata of \cref{lemma:stronger-grid}. 

Next, we shall verify that with high probability the conditions of \cref{lemma:conditional-sector-test} are satisfied for all inputs defined as above. Then we can employ \cref{lemma:conditional-sector-test} to prove our desired claim.
The rest of this proof is dedicated to checking the conditions of the lemma. 
%
	 \begin{align}\label{eq:n-lower}
	 	 n \geq  \exp\big( {4^{(d + 3) / 4}f_{\min}^{-1 / 4}} \big) \vee \sqrt[4^{d + 3} - 3]{ \frac{240}{\delta_{\sin} r_d \delta^2}} \vee \sqrt[4^{i + 3} - 3]{\frac{480 \pi d}{\delta_{\sin} \delta^2}}. 
	 \end{align}
	 %  

\subsubsection*{Condition 1, 6}

These conditions follow immediately by assumption og the lemma. 

%\subsubsection*{Condition 5}
\end{comment}

\subsubsection{Proof of \cref{lemma:alpha-k-next} when $k = |\Theta| - 2$}\label{sec:appendix-other1}

We present in this section an algorithm that achieves the desiderata of \cref{lemma:alpha-k-next} when $k = |\Theta| - 2$. Our algorithm for this part is based on a \emph{modified conditional sector test}, which is constructed using the following mechanism: 
\begin{align}\label{eq:modified-mechanism}
\begin{split}
	& \Pi_{1}^{\alpha, \delta, s} = \mathds{1}_d / d + r_d \varphi_0^{-1} (\xi_i(\alpha - \delta, \hat \alpha^s_{(i + 1): (d - 2)})), \\
	& \Pi_{2}^{\alpha, \delta, s} = \mathds{1}_d / d + r_d \varphi_0^{-1} (\xi_i(\alpha +\delta, \hat \alpha^s_{(i + 1): (d - 2)})), \\
	& \Pi_{\ell}^{\alpha, \delta, s} = \mathds{1}_d / d + r_d \varphi_0^{-1}(\xi_i(\hat \alpha_{i:(d - 2)}^{j_{\ell - 2}})), \qquad \mbox{for all }3 \leq \ell \leq |\Theta|. 
\end{split}
\end{align}
In the above display, we let $\cM = \{j_1, j_2, \cdots, j_{|\Theta| - 2}\} \subseteq \Theta$. 
We assume that $s \notin \cM$. 
We state the modified conditional sector test as \cref{alg:sector_dector3} below.  

\begin{algorithm}
\caption{Modified conditional sector test}
\label{alg:sector_dector3}
\textbf{Input:} parameters $\alpha,\,\delta, \,s$, $\hat \alpha_{(i+1):(d-2)}^s$, $(\hat  \alpha_{ i: (d - 2)}^h )_{h \in \cM}$, and accuracy level $n$; % \cM_{i, |\Theta| - 2}, \,  \cZ_{i + 1}, \, \{\hat\alpha_i^{s}: s \in \cM_{i, |\Theta| - 2}\}$;
\begin{algorithmic}[1]
\State Set $\ell = \lceil  \log n \rceil^2 $ and $T_{\mathrm{sec}}= \lceil  \log n \rceil^4$;
\State $\mathsf{count}, \mathsf{count}_+, \mathsf{count}_-\gets 0$;
\While{$\mathsf{count} \leq T_{\mathrm{sec}}$}
	\State $\mathsf{count} \gets \mathsf{count} + 1$; 
	\State Start a new round, announce the mechanism as in \cref{eq:modified-mechanism}, and observe the agent's reported type $r$;
	\If{($\alpha > \pi / 2$ \& $r = 2$) \textbf{or} ($\alpha < \pi / 2$ \& $r = 1$)}
		\State $\mathsf{count}_+ \gets \mathsf{count}_+ + 1$; 
	\Else \,\,\textbf{if }{($\alpha > \pi / 2$ \& $r = 1$) \textbf{or} ($\alpha < \pi / 2$ \& $r = 2$)} 
		\State $\mathsf{count}_- \gets \mathsf{count}_- + 1$;
	\EndIf
\EndWhile
\State \texttt{\textcolor{blue}{// Create delayed feedback}}
%\State $\ell \gets \lceil \log n \rceil^2 + 1$;
\For{$i' \in [\ell]$}
	\State Start a new round and announce a dummy mechanism $\mathds{1}_{|\Theta| \times d} / d$; 
\EndFor
\State \Return $(\mathsf{count}_+,   \mathsf{count}_-)$;  
\end{algorithmic}
\end{algorithm}

The output of \cref{alg:sector_dector3} indicates the locations of the reward angles. 
The following lemma gives theoretical justification for \cref{alg:sector_dector3}. 
Similar to \cref{lemma:conditional-sector-test}, we prove the lemma under some conditions, and later show that these conditions will be satisfied with high probability.

In \hyperlink{A1p}{$\mathsf{(A1')}$}, we assume  $\hat \alpha^s_{(i + 1):(d - 2)}$ as an input of \cref{alg:sector_dector3} offers an accurate estimation of some true reward angles.

\begin{itemize}
\item[\hypertarget{A1p}{$\mathsf{(A1')}$}]
There exists $\theta_p \in \Theta$, 
  such that $\|\alpha_{(i + 1): (d - 2)}^{\theta_p} - \hat \alpha^s_{(i + 1): (d - 2)}\|_1 \leq 4 (d - 2 - i) \cdot n^{-4^{i + 4}}$.
  In addition, $\theta_p \notin \{\sfb_n(j): j \in \cM\}$, where we recall $\sfb_n(\cdot)$ is defined in \cref{eq:def-bn-perm}. 
\end{itemize}

In \hyperlink{A2p}{$\mathsf{(A2')}$}, we assume $\{\hat \alpha_{i:(d - 2)}^h\}_{h \in \cM}$ is accurate. 

\begin{itemize}
\item[\hypertarget{A2p}{$\mathsf{(A2')}$}]
$\cM = \{s_j: j \in [|\Theta| - 2]\}$, where we recall that $(s_j)_{j \in [|\Theta|]}$ is from \cref{eq:s1-s2-permutation}. 
  In addition, for all $j \in \cM$, it holds that  $\| \hat \alpha_{i:(d - 2)}^j  -  \alpha^{\sfb_n(j)}_{i:(d - 2)}\|_1 \leq 4 (d - 1 - i) \cdot  n^{-4^{i + 3}}$. 
\end{itemize}

\hyperlink{A3p}{$\mathsf{(A3')}$} places constraints on the testing interval.

\begin{itemize}
\item[\hypertarget{A3p}{$\mathsf{(A3')}$}]

We assume $\alpha \pm \delta \in [0, \pi]$, and
		\begin{align*}
			|\alpha_i^{\sfb_n(s_{|\Theta| - 1})} - \pi / 2| + \delta \geq |\alpha - \pi / 2| \geq  |\alpha_i^{\sfb_n(s_{|\Theta| - 2})} - \pi / 2| + \delta^{1/4} + 2\delta, 
		\end{align*}
%where $\delta$ is a positive constant that satisfies: 
 
\end{itemize}

\hyperlink{A4p}{$\mathsf{(A4')}$} assumes $n$ is sufficiently large.

\begin{itemize}
\item[\hypertarget{A4p}{$\mathsf{(A4')}$}]

Assume that $n$ is sufficiently large, such that 
\begin{align*}
    & \lceil \log n \rceil^2   \geq \bigg\lceil - 4^{d + 3} \log_{\gamma} n + \log_{\gamma} \delta_{\sin} + \log_{\gamma} \frac{1 - \gamma}{C_0}\bigg\rceil + 1, \qquad \mbox{recall $C_0$ is from \cref{lemma:best-response},} \\
    & \frac{r_d \delta \delta_{\sin} \tilde\chi_i \min \{\chi_{i + 1}, \bar\chi_{i + 1}, \tilde\chi_{i + 1}\}^2}{60} \geq n^{-4^{d + 3}} + 8 r_d d n^{-4^{i + 4}}, \\
    & \delta \geq 6(r_d \delta_{\sin})^{-1/2} n^{-4^{d + 2}} + 20 d^{1/2}( \delta_{\sin} )^{-1/2} n^{-4^{i + 3.5}}, \\
    & \frac{r_d \delta_{\sin} \min \{\delta^{1/2}, \bar\chi_i^2\} }{30} \geq n^{-4^{d + 3}} + 8 r_d d n^{-4^{i + 3}}, \\
    & \chi_i \geq 4 n^{-4^{i + 3}} + 2(\pi / (2 \lceil \log n \rceil))^{1/4} + 4\pi / (2 \lceil \log n \rceil), \\
    & \lceil \log n \rceil^2   \geq \bigg\lceil - 4^{d + 3} \log_{\gamma} n + \log_{\gamma} \delta_{\sin} + \log_{\gamma} \frac{1 - \gamma}{C_0}\bigg\rceil + 1 , \qquad \textrm{and} \qquad n  \geq  \exp\big( {4^{(d + 3) / 3}f_{\min}^{-1 / 3}} \big), \\
    & \chi_i \geq \frac{\pi}{\lceil \log n \rceil}, \\
    & \frac{r_d \delta_{\sin} \tilde\chi_i \min \{\chi_{i + 1}, \tilde \chi_{i + 1},  \bar\chi_{i + 1}\}^2 }{240\lceil \log n \rceil} > \frac{2\delta_{\sin}}{n^{4^{d + 3}}} + 8r_d(d - 1 - i) n^{-4^{i + 3}}. 
\end{align*}

%Observe that the lower bound depends only on $(\mathscr{P}, \varphi_0)$. 
\end{itemize}

\hyperlink{A5p}{$\mathsf{(A5')}$} assumes $\delta$ is sufficiently small.

\begin{itemize}
\item[\hypertarget{A5p}{$\mathsf{(A5')}$}]
Assume that 
$$\delta \leq { \frac{\tilde \chi_i\min\{\chi_{i + 1}, \bar\chi_{i + 1}, \tilde\chi_{i + 1}\}^2}{240}},$$
observe that the upper bound above depends only on $(\mathscr{P}, \varphi_0)$. 
\end{itemize}

In \hyperlink{A6p}{$\mathsf{(A6')}$} we assume that the target angle $\alpha_i^{\theta_p}$ is bounded away from the edge of the testing interval.

\begin{itemize}
\item[\hypertarget{A6p}{$\mathsf{(A6')}$}]
	The center of the testing interval $\alpha$ satisfies	\begin{align}\label{eq:sevent2}
			\alpha  \in \left\{ x: \big|x - \alpha_i^{\theta_p} \big|  \geq 6(r_d \delta_{\sin})^{-1/2} n^{-4^{d + 2}} + 20 d^{1/2}( \delta_{\sin} )^{-1/2} n^{-4^{i + 3.5}}\right\}. 
		\end{align}
\end{itemize}

\cref{lemma:sector_dector3} offers theoretical guarantee for \cref{alg:sector_dector3}. 

\begin{lemma}\label{lemma:sector_dector3}
We consider \cref{alg:sector_dector3}. 
Under the conditions of \cref{lemma:alpha-k-next} with $k = |\Theta| - 2$, if in addition  \hyperlink{A1p}{$\mathsf{(A1')}$}-\hyperlink{A6p}{$\mathsf{(A6')}$} are satisfied, then the following statements hold:
		\begin{enumerate}
		\item Suppose in round $t$ we have $\theta_t \neq \theta_p$. Then 
  \begin{enumerate}
      \item[(a)] If $\alpha > \pi / 2$ and $\alpha_i^{\theta_t} \in [\alpha - \delta, \pi]$, then the agent will never report type $r = 1$ in round $t$.
      \item[(b)] If $\alpha  < \pi / 2$ and  $\alpha_i^{\theta_t} \in [0, \alpha + \delta]$, then the agent will never report type $r = 2$ in round $t$.
      \item[(c)] If $\alpha > \pi / 2$ and $\alpha_i^{\theta_t} \in [0, \pi - \alpha + \delta]$, then the agent will never report type $r = 2$ in round $t$. 
      \item[(d)] If $\alpha < \pi / 2$ and $\alpha_i^{\theta_t} \in [\pi - \alpha - \delta, \pi]$, then the agent will never report type $r = 1$ in round $t$. 
  \end{enumerate}

		%\item If $\alpha > \pi / 2$ and $\alpha_i^{\theta_t} \in (\alpha, \pi]$, then the agent will never report type $r = 1$ in round $t$. If $\alpha < \pi / 2$ and $\alpha_i^{\theta_t} \in [0, \alpha)$, then the agent will never report $r = 2$ in round $t$. 
		%\item %We assume $\theta_t = \theta_p$.
  \item Suppose in round $t$ we have $\theta_t = \theta_p$.
  Then 
  \begin{enumerate}
      \item[(a)] If $\alpha > \pi / 2$ and $\alpha_i^{\theta_t} \in (\alpha, \pi]$, then the agent will report type $r = 2$ in round $t$.
      \item[(b)] If $\alpha < \pi / 2$ and $\alpha_i^{\theta_t} \in [0, \alpha)$, then the agent will report type $r = 1$ in round $t$.
  \end{enumerate}
		\item   Suppose in round $t$ we have $\theta_t = \theta_p$. 
  Then 
  \begin{enumerate}
      \item If $\alpha > \pi / 2$ and $\alpha_i^{\theta_t} \in (\alpha - \delta, \alpha]$, then the agent will report type $r = 1$ in round $t$.
      \item If $\alpha < \pi / 2$ and $\alpha_i^{\theta_t} \in [\alpha, \alpha + \delta)$, then the agent will report type $r = 2$ in round $t$.
  \end{enumerate}
     
\item If $\theta_t = \sfb_n(j_{\ell})$ for some $\ell \in [|\Theta| - 2]$, then the agent will always report type $\ell + 2$ in round $t$.

\item Suppose in round $t$ we have $\theta_t = \theta_p$. Then 
\begin{enumerate}
    \item[(a)] If $\alpha_i^{\theta_t} \in [0, \alpha)$, then the agent will never report $r = 2$ in round $t$. 
    \item[(b)] If $\alpha_i^{\theta_t} \in (\alpha, \pi]$, then the agent will never report $r = 1$ in round $t$. 
\end{enumerate}
		     
	\end{enumerate}
\end{lemma}

\begin{proof}[Proof of \cref{lemma:sector_dector3}]
    We prove \cref{lemma:sector_dector3} in Section \ref{sec:proof-lemma:sector_dector3}. 
\end{proof}

When \hyperlink{A1p}{$\mathsf{(A1')}$}-\hyperlink{A6p}{$\mathsf{(A6')}$} are satisfied, 
\cref{alg:sector_dector3} will respond according to \cref{lemma:sector_dector3}. 
%In this case, we say \cref{alg:sector_dector3} is implemented successfully. 
We next state an algorithm that achieves the desiderata of \cref{lemma:alpha-k-next} when $k = |\Theta| - 2$. 
This algorithm is constructed based on the modified conditional sector test \cref{alg:sector_dector3}. 

\begin{algorithm}
\caption{Algorithm for estimating $\hat \alpha_i^{s_{|\Theta| - 1}}$}
\label{alg:sector_dector4}
\textbf{Input:} $n$, $\hat\balpha_{(i+1):(d-2)}$, $(\hat\alpha_i^{s_j})_{1 \leq j \leq |\Theta| - 2}$, $\cM = [|\Theta|] \backslash \{{s}_j: 1 \leq j \leq |\Theta| - 2\} = \{q_1, q_2\}$; % \cM_{i, |\Theta| - 2}, \,  \cZ_{i + 1}, \, \{\hat\alpha_i^{s}: s \in \cM_{i, |\Theta| - 2}\}$;
\begin{algorithmic}[1]
\State Set $N = \lceil \log n \rceil$ and $\iota \gets \pi / 2N$; 
\State Sample $u_0 \sim \Unif[0, \iota)$; 
\State Define $u \gets \max\{|\hat\alpha_i^{s_j} - \pi / 2|: 1 \leq j \leq |\Theta| - 2\} + 2\iota^{1/4} + 3\iota + u_0$; 
\State Define $N_1 \gets \max\{k: \max\{|\hat\alpha_i^{s_j} - \pi / 2|: 1 \leq j \leq |\Theta| - 2\} + 2\iota^{1/4} + 4 \iota + (k+1) \iota \leq \pi / 2 \}$; 
\State Run \cref{alg:sector_dector3} with inputs $(\pi / 2 + u, \iota, q_w,  \hat\alpha_{(i + 1):(d - 2)}^{q_w}$, $(\hat\alpha_{i:(d - 2)}^{s_j})_{1 \leq j \leq |\Theta| - 2}, n)$, which returns $(c_w^+, c_w^-)$ for $w \in \{1,2\}$; 
\State  Run \cref{alg:sector_dector3} with inputs $(\pi / 2 - u, \iota, q_w,  \hat\alpha_{(i + 1):(d - 2)}^{q_w}$, $(\hat\alpha_{i:(d - 2)}^{s_j})_{1 \leq j \leq |\Theta| - 2}, n)$, which returns $(\bar c_w^+, \bar c_w^-)$ for $w \in \{1,2\}$;
\If{$((c_1^+ > 0) \vee (c_2^+ > 0)) \wedge ((\bar c_1^+ > 0) \vee (\bar c_2^+ > 0)) = $ True}
\State Run \cref{alg:7}, which returns $\hat\alpha_i^{s_{|\Theta| - 1}}$;
\Else
%\State If $((\bar c_1^+ > 0) \vee (\bar c_2^+ > 0)) = \mbox{False}$, we run \cref{alg:8}; 
\If{$((\bar c_1^+ > 0) \vee (\bar c_2^+ > 0)) = \mbox{False}$}
\State Run \cref{alg:8}, which returns $\hat \alpha_i^{s_{|\Theta| - 1}}$;
\Else
\State Run \cref{alg:88}, which returns $\hat \alpha_i^{s_{|\Theta| - 1}}$;
\EndIf
\EndIf
\end{algorithmic}
\end{algorithm}

\begin{algorithm}
\caption{Algorithm for estimating $\hat \alpha_i^{s_{|\Theta| - 1}}$ when $((c_1^+ > 0) \vee (c_2^+ > 0)) \wedge ((\bar c_1^+ > 0) \vee (\bar c_2^+ > 0)) = $ True}
\label{alg:7}
\textbf{Input:} $n$, $\hat \balpha_{(i+1):(d-2)}$, $(\hat\alpha_i^{s_j})_{1 \leq j \leq |\Theta| - 2}$, $[|\Theta|] \backslash \{{s}_j: 1 \leq j \leq |\Theta| - 2\} = \{q_1, q_2\}$, $(c_w^+, c_w^-, \bar c_w^+, \bar c_w^-)_{w \in \{1,2\}}$;
\begin{algorithmic}[1]
	\State $\cS \gets \emptyset$, $\iota \gets \pi / (2 \lceil \log n \rceil)$; 
	\For{$w \in \{1,2\}$}
		\State If $c_w^+ > 0$, then $\cS \gets \cS \cup \{(+, w)\}$;
        \State If $\bar c_w^+ > 0$, then $\cS \gets \cS \cup \{(-, w)\}$;
	\EndFor
	\For{$\Delta = u + \iota,  u + 2\iota, \cdots, u + N_1 \iota$}
		\For{$(s, w) \in \cS$}
			\State Run Algorithm \ref{alg:sector_dector3} with inputs $(\pi / 2 + s \Delta, \iota, \, q_w, \,   \hat\alpha^{q_w}_{(i + 1):(d - 2)}, \, (\hat\alpha_{i:(d - 2)}^{s_j})_{1 \leq j \leq |\Theta| - 2},\,  n)$,  get $(c^+, c^-)$;
				\If{$c^+ = 0 \mbox{ and } c^- > 0$}\State $L \gets \iota$, $u_1 \gets \pi / 2 + s (\Delta - \iota)$, $u_2 \gets \pi / 2 + s\Delta$; 
				\While{$L > 4 n^{-4^{i + 3}}$}
					\State Run \cref{alg:sector_dector3} with inputs $((u_1 + u_2) / 2, L / 2, q_w, \,   \hat\alpha^{q_w}_{(i + 1):(d - 2)}, \,(\hat\alpha_{i:(d - 2)}^{ s_j})_{1 \leq j \leq |\Theta| - 2},\,  n)$, \State and get $(c^+, c^-)$; 
					\If{$c^+ = 0$ and $c^- > 0$}
						\State $u_2 \gets (u_1 + u_2) / 2$; 
					\Else
						\State $u_1 \gets (u_1 + u_2) / 2$; 
					\EndIf
					\State $L \gets L / 2$; 
				\EndWhile
				\State \textbf{break};
				\EndIf
		\EndFor
	\EndFor
	\State \texttt{\textcolor{blue}{// Create delayed feedback}}
\State $\ell \gets \lceil \log n \rceil^2$;
\For{$e \in [\ell]$}
	\State Start a new round and announce a dummy mechanism $\mathds{1}_{|\Theta| \times d} / d$; 
\EndFor
	\State \Return $\hat\alpha_i^{{s}_{|\Theta| - 1}} = (u_1 + u_2) / 2$;  
\end{algorithmic}
\end{algorithm}

\begin{algorithm}
\caption{Algorithm for estimating $\hat \alpha_i^{s_{|\Theta| - 1}}$ when $((\bar c_1^+ > 0) \vee (\bar c_2^+ > 0)) = \mbox{False}$}
\label{alg:8}
\textbf{Input:} $n$, $\hat\balpha_{(i+1):(d-2)}$, $(\hat\alpha_i^{s_j})_{1 \leq j \leq |\Theta| - 2}$, $[|\Theta|] \backslash \{{s}_j: 1 \leq j \leq |\Theta| - 2\} = \{q_1, q_2\}$, $(c_w^+, c_w^-, \bar c_w^+, \bar c_w^-)_{w \in \{1,2\}}$;
\begin{algorithmic}[1]
	\For{$\alpha - \pi / 2 = u, u + \iota, \cdots, u + N_1 \iota$}
		\For{$q \in \{q_1, q_2\}$}
		\State Run Algorithm \ref{alg:sector_dector3} with $(\alpha, \delta, q,   \hat\alpha_{(i + 1):(d - 2)}^{q}, \, (\hat\alpha_{i:(d - 2)}^{s_j})_{1 \leq j \leq |\Theta| - 2},\,  n)$, 
		 and observe $(c^+, c^-)$;
		\If{$c^- > 0$}
			\State $L \gets \iota$, $u_1 \gets \alpha - \iota$, $u_2 \gets \alpha$; 
			\While{$L > 4 n^{-4^{i + 3}}$}
				\State Run \cref{alg:sector_dector3} with inputs $((u_1 + u_2) / 2, L / 2,\, q, \,   \hat\alpha^{q}_{(i + 1):(d - 2)}, \,(\hat\alpha_{i:(d - 2)}^{s_j})_{1 \leq j \leq |\Theta| - 2},\,  n)$ 
					\State and observe $(c^+, c^-)$;
				\If{$c^- > 0$}
					\State $u_2 \gets (u_1 + u_2) / 2$;
				\Else
					\State $u_1 \gets (u_1 + u_2) / 2$; 
				\EndIf
				\State $L \gets L / 2$; 
			\EndWhile
            \State \textbf{break};
		\EndIf
		\EndFor
	\EndFor
	\State \texttt{\textcolor{blue}{// Create delayed feedback}}
\State $\ell \gets \lceil \log n \rceil^2$;
\For{$e \in [\ell]$}
	\State Start a new round and announce a dummy mechanism $\mathds{1}_{|\Theta| \times d} / d$; 
\EndFor
	\State \Return $\hat\alpha_i^{{s}_{|\Theta| - 1}} = (u_1 + u_2) / 2$; 
\end{algorithmic}
\end{algorithm}

\begin{algorithm}
\caption{Algorithm for estimating $\hat \alpha_i^{s_{|\Theta| - 1}}$ when $((c_1^+ > 0) \vee (c_2^+ > 0)) = \mbox{False}$}
\label{alg:88}
\textbf{Input:} $n$, $\hat\balpha_{(i+1):(d-2)}$, $(\hat\alpha_i^{s_j})_{1 \leq j \leq |\Theta| - 2}$, $[|\Theta|] \backslash \{{s}_j: 1 \leq j \leq |\Theta| - 2\} = \{q_1, q_2\}$, $(c_w^+, c_w^-, \bar c_w^+, \bar c_w^-)_{w \in \{1,2\}}$;
\begin{algorithmic}[1]
	\For{$\alpha - \pi / 2 = u, u - \iota, \cdots, u - N_1 \iota$}
		\For{$q \in \{q_1, q_2\}$}
		\State Run Algorithm \ref{alg:sector_dector3} with $(\alpha, \delta, q,   \hat\alpha_{(i + 1):(d - 2)}^{q}, \, (\hat\alpha_{i:(d - 2)}^{s_j})_{1 \leq j \leq |\Theta| - 2},\,  n)$, 
		 and observe $(c^+, c^-)$;
		\If{$c^- > 0$}
			\State $L \gets \iota$, $u_1 \gets \alpha + \iota$, $u_2 \gets \alpha$; 
			\While{$L > 4 n^{-4^{i + 3}}$}
				\State Run \cref{alg:sector_dector3} with inputs $((u_1 + u_2) / 2, L / 2,\, q, \,   \hat\alpha^{q}_{(i + 1):(d - 2)}, \,(\hat\alpha_{i:(d - 2)}^{s_j})_{1 \leq j \leq |\Theta| - 2},\,  n)$ 
					\State and observe $(c^+, c^-)$;
				\If{$c^- > 0$}
					\State $u_2 \gets (u_1 + u_2) / 2$;
				\Else
					\State $u_1 \gets (u_1 + u_2) / 2$; 
				\EndIf
				\State $L \gets L / 2$; 
			\EndWhile
		\EndIf
		\EndFor
	\EndFor
	\State \texttt{\textcolor{blue}{// Create delayed feedback}}
\State $\ell \gets \lceil \log n \rceil^2$;
\For{$e \in [\ell]$}
	\State Start a new round and announce a dummy mechanism $\mathds{1}_{|\Theta| \times d} / d$; 
\EndFor
	\State \Return $\hat\alpha_i^{{s}_{|\Theta| - 1}} = (u_1 + u_2) / 2$; 
\end{algorithmic}
\end{algorithm}

% \begin{lemma}\label{lemma:T-2}
% Under the assumptions of \cref{lemma:alpha-k-next} with $k = |\Theta| - 2$. 
% Then there exist $n_0 \in \NN_+$ and $C > 0$ that depend only on $(\mathscr{P}, \varphi_0)$, such that for all $n \geq n_0$, with probability at least $1 - C n^{-60}$, the output of the above procedure satisfies $|\alpha_i^{\sfb_n(s_{|\Theta| - 1})} - \hat\alpha_i^{\hat{s}_{|\Theta| - 1}}| \leq 4n^{-4^{i + 3}}$.
%Exactly the same statement also holds for \cref{alg:8}.

%Conditioning on some $(c_w^+, c_w^-, \bar{c}_w^+, \bar{c}_w^-)_{w \in \{1,2\}}$ for which \cref{alg:7} is implemented, there exist $n_0 \in \NN_+$ and $C > 0$ that depend only on $(\cP, \varphi_0)$, such that for all $n \geq n_0$, with probability at least $1 - C n^{-60}$, the output of \cref{alg:7} satisfies $|\alpha_i^{\sfb_n(s_{|\Theta| - 1})} - \hat\alpha_i^{\hat{s}_{|\Theta| - 1}}| \leq 4n^{-4^{i + 3}}$. The same statement also holds for \cref{alg:8}.
%	The above algorithm satisfies the desiderata of \cref{lemma:alpha-k-next} when $k = |\Theta| - 2$. 
% \end{lemma}
%

\begin{proof}[Proof of \cref{lemma:alpha-k-next} when $k = |\Theta| - 2$]
%	We shall choose $n_0$ large enough, such that for all $n \geq n_0$, $n$ is no smaller than 
%	\begin{align*}
%		 	& \exp\big( {4^{(d + 3) / 3}f_{\min}^{-1 / 3}} \big) \vee \sqrt[4^{d + 3}]{\frac{120 n^{4^{i + 3}}}{r_d \delta_{\sin} \tilde\chi_i \min \{\chi_{i + 1}, \bar\chi_{i + 1}, \tilde\chi_{i + 1}\}^2}} \\
%		 	& \vee \sqrt[4^{i + 4}]{\frac{240\pi d n^{4^{i + 3}} }{ \delta_{\sin} \tilde \chi_i \min \{\chi_{i + 1}, \bar\chi_{i + 1}, \tilde\chi_{i + 1}\}^2}} \vee \sqrt[4^{i + 3}]{\frac{120 \pi d  \cdot n^{4^{(i + 3) / 2}}}{\delta_{\sin} }} \\
%		 	& \vee \sqrt[4^{d + 3}]{\frac{60}{r_d \delta_{\sin} \min \{\chi_i, \bar\chi_i, \tilde\chi_i\}^2}} \vee \sqrt[4^{i + 3}]{\frac{120 \pi d}{\delta_{\sin }\min \{\chi_i, \bar\chi_i, \tilde\chi_i\}^2}}. 
%		 \end{align*}
		 %
We choose $n_0$ large enough such that for all $n \geq n_0$, 
\hyperlink{A4p}{$\mathsf{(A4')}$} and \hyperlink{A5p}{$\mathsf{(A5')}$} hold for all $2 n^{-4^{i + 3}} \leq \delta \leq \pi / (2 \lceil \log n \rceil)$. 
Note that such $n_0$ exists and depends only on $(\mathscr{P}, \varphi_0)$. 
For $n \geq n_0$, \hyperlink{A4p}{$\mathsf{(A4')}$} and \hyperlink{A5p}{$\mathsf{(A5')}$} hold for all modified conditional sector tests (\cref{alg:sector_dector3}) implemented during \cref{alg:sector_dector4}. 
In addition, under the assumptions of \cref{lemma:alpha-k-next}, \hyperlink{A1p}{$\mathsf{(A1')}$} and \hyperlink{A2p}{$\mathsf{(A2')}$} hold throughout \cref{alg:sector_dector4}. 

We next show that with high probability \hyperlink{A6p}{$\mathsf{(A6')}$} holds for all modified conditional sector tests that appear in \cref{alg:sector_dector4}. 
To this end, it suffices to lower bound the probability of the following event: 
%	 We then provide a lower bounded for the probability that the following event occurs: 
	%
	\begin{align*}
		\alpha \in \left\{x: |x - \alpha_i^{\theta_p}| \geq 6(r_d \delta_{\sin})^{-1/2} n^{-4^{d + 2}} + 20 d^{1/2}( \delta_{\sin} )^{-1/2} n^{-4^{i + 3.5}} \right\}
	\end{align*}
	for all 
\begin{align*}
	\alpha = \pm \big( u + \frac{\pi s}{2^{K + 1} \lceil \log n \rceil} \big), \qquad \delta = \frac{\pi }{2^{K + 1} \lceil \log n \rceil}, 
\end{align*}
where $s \in \{0, 1, \cdots, N_K\}$, $N_K = \max\{m \in \NN_+: \, u + \pi (m + 1) / (2^{K + 1} \lceil \log n \rceil) \leq \pi / 2 \}$ and $K \in \{0, 1, \cdots \lceil 4^{i + 3} \log_2 n \rceil\}$. 
Recall that $u = \max\{|\hat\alpha_i^{s_j} - \pi / 2|: 1 \leq j \leq |\Theta| - 2\} + 2\iota^{1/4} + 3\iota + u_0$ with $u_0 \sim \Unif[0, \iota)$. 
We denote this event by $\cE$. %Recall $u = \max\{|\hat\alpha_i^{s_j} - \pi / 2|: 1 \leq j \leq |\Theta| - 2\} + 2{\iota^{1/4}} + 3\iota + u_0$, 
Straightforward computation gives  %for large enough $n$, $\P(\cE)$ is no smaller than 
\begin{align*}
    \PP(\cE^c) \leq & \sum_{K = 0}^{\lceil 4^{i + 3} \log_2 n \rceil} 2^{K + 3} \lceil \log n\rceil^2 \times \frac{12(r_d \delta_{\sin})^{-1/2} n^{-4^{d + 2}} + 40 d^{1/2}( \delta_{\sin} )^{-1/2} n^{-4^{i + 3.5}}}{\pi} \\
    \leq &\, 16 n^{4^{i + 3}} \lceil \log n \rceil^2 \times \frac{12(r_d \delta_{\sin})^{-1/2} n^{-4^{d + 2}} + 40 d^{1/2}( \delta_{\sin} )^{-1/2} n^{-4^{i + 3.5}}}{\pi} \\
    \leq & 16\lceil \log n \rceil^2 \cdot \frac{24(r_d \delta_{\sin})^{-1/2} n^{-60} + 80 d^{1/2}( \delta_{\sin} )^{-1/2} n^{-60}}{\pi }.
\end{align*}
% \begin{align*}
% 	1 - 16\lceil \log n \rceil^2 \cdot \frac{24(r_d \delta_{\sin})^{-1/2} n^{-60} + 80 d^{1/2}( \delta_{\sin} )^{-1/2} n^{-60}}{\pi }.
% \end{align*}
%
%We denote this event by $\cE$. 
%In what follows, we always assume $\cE$ occurs.
Therefore, with probability at least $1 - 16\lceil \log n \rceil^2 \cdot \frac{24(r_d \delta_{\sin})^{-1/2} n^{-60} + 80 d^{1/2}( \delta_{\sin} )^{-1/2} n^{-60}}{\pi }$ condition \hyperlink{A6p}{$\mathsf{(A6')}$} holds for all modified conditional sector tests that appear in \cref{alg:sector_dector4}. 
In the following proof we assume this high-probability event happens. 

Next, we prove that with high probability \hyperlink{A3p}{$\mathsf{(A3')}$} holds for all modified conditional sector tests that appear in \cref{alg:sector_dector4}.
With that, we further show that the output of \cref{alg:sector_dector4} satisfies the desiderata of \cref{lemma:alpha-k-next} when $k = |\Theta| - 2$. 

We first show that \hyperlink{A3p}{$\mathsf{(A3')}$} is satisfied by the calls to \cref{alg:sector_dector3} in lines 5 and 6 of \cref{alg:sector_dector4}. 
To this end, it suffices to show 
\begin{align}
\label{eq:target555}
	|\alpha_i^{\sfb_n(s_{|\Theta| - 1})} - \pi / 2| + \iota \geq u \geq |\alpha_i^{\sfb_n(s_{|\Theta| - 2})} - \pi / 2| + \iota^{1/4} + 2\iota. 
\end{align}
Recall that $u = \max\{|\hat\alpha_i^{s_j} - \pi / 2|: 1 \leq j \leq |\Theta| - 2\} + 2\iota^{1/4} + 3\iota + u_0$. 
By assumption, $\sup_{j \in [k]}|\hat \alpha_i^{ s_j} - \alpha_i^{\sfb_n(s_j)}| \leq 4 n^{-4^{i + 3}}$. 
Therefore, 
\begin{align*}
    |\alpha_i^{\sfb_n(s_{|\Theta| - 2})} - \pi / 2| - 4 n^{-4^{i + 3}} + 2\iota^{1/4} + 3\iota + u_0 \leq u \leq |\alpha_i^{\sfb_n(s_{|\Theta| - 2})} - \pi / 2| + 4 n^{-4^{i + 3}} + 2\iota^{1/4} + 3\iota + u_0. 
\end{align*}
By \hyperlink{A4p}{$\mathsf{(A4')}$} we know that $\chi_i \geq 4 n^{-4^{i + 3}} + 2\iota^{1/4} + 3\iota$, hence 
\begin{align*}
    u \leq &\, |\alpha_i^{\sfb_n(s_{|\Theta| - 2})} - \pi / 2| + 4 n^{-4^{i + 3}} + 2\iota^{1/4} + 3\iota + u_0 \\
    \leq &\, |\alpha_i^{\sfb_n(s_{|\Theta| - 2})} - \pi / 2| + 4 n^{-4^{i + 3}} + 2\iota^{1/4} + 4\iota \\
    \leq &\, |\alpha_i^{\sfb_n(s_{|\Theta| - 1})} - \pi / 2| - \chi_i + 4 n^{-4^{i + 3}} + 2\iota^{1/4} + 4\iota \\
    \leq &\, |\alpha_i^{\sfb_n(s_{|\Theta| - 1})} - \pi / 2| + \iota. 
\end{align*}
By \hyperlink{A4p}{$\mathsf{(A4')}$} we also have $\iota^{1/4} + \iota \geq 4 n^{-4^{i + 3}}$. Therefore,  
\begin{align*}
    u \geq |\alpha_i^{\sfb_n(s_{|\Theta| - 2})} - \pi / 2| - 4 n^{-4^{i + 3}} + 2\iota^{1/4} + 3\iota \geq |\alpha_i^{\sfb_n(s_{|\Theta| - 2})} - \pi / 2| + \iota^{1/4} + 2\iota. 
\end{align*}
Combining the above two inequalities, 
we have completed the proof of \cref{eq:target555}. 
As a consequence, \hyperlink{A1p}{$\mathsf{(A1')}$}-\hyperlink{A6p}{$\mathsf{(A6')}$} hold for the calls to \cref{alg:sector_dector3} in lines 5 and 6 of \cref{alg:sector_dector4}. 
With that, we establish the following lemma.
\begin{lemma}
\label{lemma:line5line6}
    Under the conditions of \cref{lemma:alpha-k-next} with $k = |\Theta| - 2$, the following statements are true: 
    \begin{enumerate}
        \item If $\alpha_i^{\sfb_n(s_{|\Theta| - 1})}, \alpha_i^{\sfb_n(s_{|\Theta|})} > \pi / 2$, with probability at least $1 - 16\lceil \log n \rceil^2 \cdot \frac{24(r_d \delta_{\sin})^{-1/2} n^{-60} + 80 d^{1/2}( \delta_{\sin} )^{-1/2} n^{-60}}{\pi } - 2 n^{-4^{d + 3}} $ we have $c_1^+ > 0$, $c_2^+ > 0$, $\bar c_1^+ = 0$ and $\bar c_2^+ = 0$.  
        \item If $\alpha_i^{\sfb_n(s_{|\Theta| - 1})}, \alpha_i^{\sfb_n(s_{|\Theta|})} < \pi / 2$, with probability at least $1 - 16\lceil \log n \rceil^2 \cdot \frac{24(r_d \delta_{\sin})^{-1/2} n^{-60} + 80 d^{1/2}( \delta_{\sin} )^{-1/2} n^{-60}}{\pi } - 2 n^{-4^{d + 3}}$ we have $\bar c_1^+ > 0$, $\bar c_2^+ > 0$, $c_1^+ = 0$ and $c_2^+ = 0$. 
        \item If $\alpha_i^{\sfb_n(s_{|\Theta| - 1})} > \pi / 2$ and $\alpha_i^{\sfb_n(s_{|\Theta|})} < \pi / 2$, then with probability at least $1- 2 n^{-4^{d + 3}} - 16\lceil \log n \rceil^2 \cdot \frac{24(r_d \delta_{\sin})^{-1/2} n^{-60} + 80 d^{1/2}( \delta_{\sin} )^{-1/2} n^{-60}}{\pi } $ we have $(c_1^+ > 0) \vee (c_2^+ > 0) = \mbox{True}$ and $(\bar c_1^+ > 0) \vee (\bar c_2^+ > 0) = \mbox{True}$. 
        \item If $\alpha_i^{\sfb_n(s_{|\Theta| - 1})} < \pi / 2$ and $\alpha_i^{\sfb_n(s_{|\Theta|})} > \pi / 2$, then with probability at least $1 - 2 n^{-4^{d + 3}} - 16\lceil \log n \rceil^2 \cdot \frac{24(r_d \delta_{\sin})^{-1/2} n^{-60} + 80 d^{1/2}( \delta_{\sin} )^{-1/2} n^{-60}}{\pi }$ we have $(c_1^+ > 0) \vee (c_2^+ > 0) = \mbox{True}$ and $(\bar c_1^+ > 0) \vee (\bar c_2^+ > 0) = \mbox{True}$. 
    \end{enumerate}
\end{lemma}
\begin{proof}[Proof of \cref{lemma:line5line6}]

In this proof, we assume \hyperlink{A6p}{$\mathsf{(A6')}$} holds for all modified conditional sector tests, which occurs with probability at least $1 - 16\lceil \log n \rceil^2 \cdot \frac{24(r_d \delta_{\sin})^{-1/2} n^{-60} + 80 d^{1/2}( \delta_{\sin} )^{-1/2} n^{-60}}{\pi }$.
By symmetry, it suffices to prove claims 1 and 3. 
We start with claim 1. 

We consider the modified conditional sector test that returns $(c_1^+, c_1^-)$. 
By \cref{lemma:sector_dector3}, the following statements are true: 
(1) If $\theta_t = \sfb_n(j_{\ell})$ for some $\ell \in [|\Theta| - 2]$ (recall that $\cM = \{j_1, j_2, \cdots, j_{\ell}\}$), then the agent in round $t$ will report type $\ell + 2$; 
(2) If $\theta_t = \sfb_n(q_1)$, since $\pi / 2 + u > \pi / 2$ and $\alpha_i^{\sfb_n(q_1)} \in (\pi / 2 + u, \pi]$, by the second claim of \cref{lemma:sector_dector3}, in round $t$ the agent will report type $2$; 
(3) If $\theta_t = \sfb_n(q_2)$, since $\pi / 2 + u > \pi / 2$ and $\alpha_i^{\sfb_n(q_2)} \geq u + \pi / 2 - \iota$, 
then by the first claim of \cref{lemma:sector_dector3},
in round $t$ the agent will never report type $1$. 
Therefore, to have $c_1^+ > 0$, it suffices to show that type $\sfb_n(q_1)$ appears at least once in a total number of $\lceil \log n \rceil^4$ rounds. 
The probability that type $\sfb_n(q_1)$ does not appear in a total number of $\lceil \log n \rceil^4$ rounds is upper bounded by $(1 - f_{\min})^{\lceil \log n \rceil^4}$. 
By \hyperlink{A4p}{$\mathsf{(A4')}$}, we further have 
\begin{align*}
    1 - (1 - f_{\min})^{\lceil  \log n \rceil^4  } \geq 1 - (1 - f_{\min})^{4^{d + 3} f_{\min}^{-1} \lceil \log n \rceil} \geq 1 - e^{-4^{d + 3} \lceil \log n \rceil} \geq 1 - n^{-4^{d + 3}}.
\end{align*}
To summarize, in the setting of claim 1, with probability at least $1 - n^{-4^{d + 3}}$ we have $c_1^+ > 0$. 
Similarly, with probability at least $1 - n^{-4^{d + 3}}$ we have $c_2^+ > 0$. 

We then consider the modified conditional sector test that returns $(\bar c_1^+, \bar c_1^-)$.
Once again by \cref{lemma:sector_dector3}, we obtain the following statements: 
(1) If $\theta_t = \sfb_n(j_{\ell})$ for some $\ell \in [|\Theta| - 2]$, then the agent in round $t$ will report type $\ell + 2$; 
(2) If $\theta_t = \sfb_n(q_1)$, since $\pi / 2 - u < \pi / 2$ and $\alpha_i^{\sfb_n(q_1)} \in (\pi / 2 - u, \pi]$,   
by the second claim of \cref{lemma:sector_dector3}, 
the agent in round $t$ will report type $2$; 
(3) If $\theta_t = \sfb_n(q_2)$, since $\pi / 2 - u < \pi / 2$ and $\alpha_i^{\sfb_n(q_2)} \in [\pi / 2 + u - \iota, \pi]$, then by the first claim of \cref{lemma:sector_dector3}, in round $t$ the agent will never report type $1$. 
As a consequence, we have $\bar c_1^+ = 0$. 
Similarly, we have $\bar c_2^+ = 0$. 

Next, we prove claim 3.
Without loss, we assume $s_{|\Theta| - 1} = q_1$ and $s_{|\Theta|} = q_2$. 
We consider the modified conditional sector test that outputs $(c_1^+, c_1^-)$. 
By the second claim of \cref{lemma:sector_dector3}, 
we see that for this test, when $\theta_t = \alpha_i^{\sfb_n(q_1)}$, the agent will return type 2. 
Hence, with probability at least $1 - n^{-4^{d + 3}}$ we have $c_1^+ > 0$. 
Similarly, with probability at least $1 - n^{-4^{d + 3}}$ we have $\bar c_2^+ > 0$. 
The proof is done.     
\end{proof}

With \cref{lemma:line5line6}, we establish the following lemma. With \cref{lemma:call-algs}, we complete the proof of \cref{lemma:alpha-k-next} when $k = |\Theta| - 2$. 
\begin{lemma}
\label{lemma:call-algs}

Under the conditions of \cref{lemma:alpha-k-next} with $k = |\Theta| - 2$, the following statements are true: 
    \begin{enumerate}
        \item If $\alpha_i^{\sfb_n(s_{|\Theta| - 1})}, \alpha_i^{\sfb_n(s_{|\Theta|})} > \pi / 2$, with probability at least $1 - 16\lceil \log n \rceil^2 \cdot \frac{24(r_d \delta_{\sin})^{-1/2} n^{-60} + 80 d^{1/2}( \delta_{\sin} )^{-1/2} n^{-60}}{\pi } - (4^{i + 3}\lceil \log_2 n \rceil + 5 + \lceil \log n \rceil) \cdot n^{-4^{d + 3}}$, \cref{alg:sector_dector4} implements \cref{alg:8}, and outputs $\hat\alpha_i^{s_{|\Theta| - 1}}$ that satisfies $|\hat\alpha_i^{s_{|\Theta| - 1}} - \alpha_i^{\sfb_n(s_{|\Theta| - 1})}| \leq 4n^{-4^{i + 3}}$. 
        \item If $\alpha_i^{\sfb_n(s_{|\Theta| - 1})}, \alpha_i^{\sfb_n(s_{|\Theta|})} < \pi / 2$, with probability at least $1 - 16\lceil \log n \rceil^2 \cdot \frac{24(r_d \delta_{\sin})^{-1/2} n^{-60} + 80 d^{1/2}( \delta_{\sin} )^{-1/2} n^{-60}}{\pi } - (4^{i + 3}\lceil \log_2 n \rceil + 5 + \lceil \log n \rceil) \cdot n^{-4^{d + 3}}$, \cref{alg:sector_dector4} implements \cref{alg:88}, and outputs $\hat\alpha_i^{s_{|\Theta| - 1}}$ that satisfies $|\hat\alpha_i^{s_{|\Theta| - 1}} - \alpha_i^{\sfb_n(s_{|\Theta| - 1})}| \leq 4n^{-4^{i + 3}}$.
        \item If $\alpha_i^{\sfb_n(s_{|\Theta| - 1})} > \pi / 2$ and $\alpha_i^{\sfb_n(s_{|\Theta|})} < \pi / 2$, then with probability at least $1  - (4^{i + 3}\lceil \log_2 n \rceil + 5 + \lceil \log n \rceil) \cdot n^{-4^{d + 3}}- 16\lceil \log n \rceil^2 \cdot \frac{24(r_d \delta_{\sin})^{-1/2} n^{-60} + 80 d^{1/2}( \delta_{\sin} )^{-1/2} n^{-60}}{\pi }$, \cref{alg:sector_dector4} implements \cref{alg:7}, and outputs $\hat\alpha_i^{s_{|\Theta| - 1}}$ that satisfies $|\hat\alpha_i^{s_{|\Theta| - 1}} - \alpha_i^{\sfb_n(s_{|\Theta| - 1})}| \leq 4n^{-4^{i + 3}}$.
        \item If $\alpha_i^{\sfb_n(s_{|\Theta| - 1})} < \pi / 2$ and $\alpha_i^{\sfb_n(s_{|\Theta|})} > \pi / 2$, then with probability at least $1 - (4^{i + 3}\lceil \log_2 n \rceil + 5 + \lceil \log n \rceil) \cdot n^{-4^{d + 3}} - 16\lceil \log n \rceil^2 \cdot \frac{24(r_d \delta_{\sin})^{-1/2} n^{-60} + 80 d^{1/2}( \delta_{\sin} )^{-1/2} n^{-60}}{\pi } $, \cref{alg:sector_dector4} implements \cref{alg:7}, and outputs $\hat\alpha_i^{s_{|\Theta| - 1}}$ that satisfies $|\hat\alpha_i^{s_{|\Theta| - 1}} - \alpha_i^{\sfb_n(s_{|\Theta| - 1})}| \leq 4n^{-4^{i + 3}}$.
    \end{enumerate}

\end{lemma}
\begin{proof}[Proof of \cref{lemma:call-algs}]

In this proof, we assume \hyperlink{A6p}{$\mathsf{(A6')}$} holds for all modified conditional sector tests, which occurs with probability at least $1 - 16\lceil \log n \rceil^2 \cdot \frac{24(r_d \delta_{\sin})^{-1/2} n^{-60} + 80 d^{1/2}( \delta_{\sin} )^{-1/2} n^{-60}}{\pi }$.
By symmetry, we only need to prove claims 1 and 3. 
Without loss, in this proof we assume $q_1 = s_{|\Theta| - 1}$ and $q_2 = s_{|\Theta|}$. 
The case where $q_1 = s_{|\Theta|}$ and $q_2 = s_{|\Theta| - 1}$ can be handled symmetrically. 

\subsubsection*{Proof of claim 1}

In this case, $\alpha_i^{\sfb_n(q_2)} > \alpha_i^{\sfb_n(q_1)} > \pi / 2$. 
By \cref{lemma:line5line6}, with probability at least $1 - 2 n^{-4^{d + 3}} - 16\lceil \log n \rceil^2 \cdot \frac{24(r_d \delta_{\sin})^{-1/2} n^{-60} + 80 d^{1/2}( \delta_{\sin} )^{-1/2} n^{-60}}{\pi }$, 
we have $c_1^+ > 0$, $c_2^+ > 0$, $\bar c_1^+ = 0$ and $\bar c_2^+ = 0$. 
Therefore, with probability at least $1 - 2 n^{-4^{d + 3}} - 16\lceil \log n \rceil^2 \cdot \frac{24(r_d \delta_{\sin})^{-1/2} n^{-60} + 80 d^{1/2}( \delta_{\sin} )^{-1/2} n^{-60}}{\pi }$ \cref{alg:sector_dector4} implements \cref{alg:8}.

Let $m = \sup\{ a \in \NN_+: u + a \iota < |\alpha_i^{\sfb_n(s_{|\Theta| - 1})} - \pi / 2| \}$. 
For all $\alpha - \pi / 2 = u, u + \iota, \cdots, u + m \iota$, we consider \cref{alg:sector_dector3} with inputs $(\alpha, \iota, q, \hat \alpha_{(i + 1):(d - 2)}^q, (\hat \alpha_{i:(d - 2)}^{s_j})_{1 \leq j \leq |\Theta| - 2}, n)$ for $q \in \{q_1, q_2\}$.
By the definition of $u$ and \hyperlink{A4p}{$\mathsf{(A4')}$}, we have 
\begin{align}
    |\alpha - \pi / 2| \geq & u \geq |\hat \alpha_i^{s_{|\Theta| - 2}} - \pi / 2| + 3\iota + 2\iota^{1/4} \geq |\alpha_i^{\sfb_n(s_{|\Theta| - 2})} - \pi / 2| + 3\iota + 2\iota^{1/4} - 4n^{-4^{i + 3}} \nonumber  \\
    \geq & |\alpha_i^{\sfb_n(s_{|\Theta| - 2})} - \pi / 2| + \iota^{1/4} + 2\iota, \label{eq:u-lower-bound} \\
    |\alpha - \pi / 2| \leq & u + m \iota < |\alpha_i^{\sfb_n(s_{|\Theta| - 1})} - \pi / 2|. \label{eq:u-upper-bound}
\end{align}
Hence, \hyperlink{A3p}{$\mathsf{(A3')}$} holds for all \cref{alg:sector_dector3} with inputs $(\alpha, \iota, q, \hat \alpha_{(i + 1):(d - 2)}^q, (\hat \alpha_{i:(d - 2)}^{s_j})_{1 \leq j \leq |\Theta| - 2}, n)$ and $q \in \{q_1, q_2\}$. 

We then consider the response of the agent. 
In round $t$: (1) If $\theta_t = \sfb_n(j_{\ell})$ for some $\ell \in [|\Theta| - 2]$, then by the fourth claim of \cref{lemma:sector_dector3} the agent will report type $\ell$; 
(2) If $\theta_t \in \{ \sfb_n(q_1), \sfb_n(q_2)\}$, since $|\alpha_i^{\sfb_n(s_{|\Theta|})} - \pi / 2| > |\alpha_i^{\sfb_n(s_{|\Theta| - 1})} - \pi / 2| > u + m \iota$, by \cref{lemma:sector_dector3}, the agent will never report type $1$.  
Therefore, for all $\alpha - \pi / 2 = u, u + \iota, \cdots, u + m \iota$ and $q \in \{q_1, q_2\}$, if we run \cref{alg:sector_dector3} with inputs $(\alpha, \iota, q, \hat \alpha_{(i + 1):(d - 2)}^q, (\hat \alpha_{i:(d - 2)}^{s_j})_{1 \leq j \leq |\Theta| - 2}, n)$ and observe $(c^+, c^-)$, then $c^- = 0$. 
This implies that the ``if'' condition in line 4 of \cref{alg:8} does not hold for all $\alpha - \pi / 2 = u, u + \iota, \cdots, u + m \iota$ and $q \in \{q_1, q_2\}$. 

We next consider \cref{alg:sector_dector3} with inputs $(\pi / 2 + u + (m + 1)\iota, \iota, q_w, \hat \alpha_{(i + 1):(d - 2)}^{q_w}, (\hat \alpha_{i:(d - 2)}^{s_j})_{1 \leq j \leq |\Theta| - 2}, n)$ for $w \in \{1, 2\}$. 
Since $u + m \iota < |\alpha_i^{\sfb_n(q_1)} - \pi / 2|$ and $u \geq |\alpha_i^{\sfb_n(s_{|\Theta| - 2})} - \pi / 2| + \iota^{1/4} + 2\iota$ (by \cref{eq:u-lower-bound}), we see that \hyperlink{A3p}{$\mathsf{(A3')}$} holds for \cref{alg:sector_dector3} with inputs $(\pi / 2 + u + (m + 1)\iota, \iota, q_w, \hat \alpha_{(i + 1):(d - 2)}^{q_w}, (\hat \alpha_{i:(d - 2)}^{s_j})_{1 \leq j \leq |\Theta| - 2}, n)$.  

By the definition of $m$ we have $u + (m + 1) \iota \geq |\alpha_i^{\sfb_n(q_1)} - \pi / 2|$, which implies that $\alpha_i^{\sfb_n(q_1)} \in (\pi / 2 + u + m \iota, \pi / 2 + u + (m + 1) \iota]$.
If we run \cref{alg:sector_dector3} with $(\pi / 2 + u + (m + 1)\iota, \iota, q_1, \hat \alpha_{(i + 1):(d - 2)}^{q_1}, (\hat \alpha_{i:(d - 2)}^{s_j})_{1 \leq j \leq |\Theta| - 2}, n)$ and get $(c^+, c^-)$, 
then by the third claim of \cref{lemma:sector_dector3}, when $\theta_t = \sfb_n(q_1)$, the agent in round $t$ will report type $1$. 
Therefore, with probability at least $1 - n^{-4^{d + 3}}$, we have $c^- > 0$. 

Suppose we run \cref{alg:sector_dector3} with inputs $(\pi / 2 + u + (m + 1)\iota, \iota, q_2, \hat \alpha_{(i + 1):(d - 2)}^{q_2}, (\hat \alpha_{i:(d - 2)}^{s_j})_{1 \leq j \leq |\Theta| - 2}, n)$ and get $(c^+, c^-)$. 
By \cref{lemma:sector_dector3}, the following statements are true: 
(1) If $\theta_t = \sfb_n(j_{\ell})$ for some $\ell \in [|\Theta| - 2]$, then the agent in round $t$ will report type $\ell + 2$; 
(2) If $\theta_t = \sfb_n(q_1)$, as $ \alpha_i^{\sfb_n(q_1)} \geq \pi / 2 + u + m \iota$, then by the first claim of \cref{lemma:sector_dector3}, in this case the agent will never report type $1$; 
(3) If $\theta_t = \sfb_n(q_2)$, by Definition \ref{def:chi}, $|\alpha_i^{\sfb_n(q_2)} - \pi / 2| \geq |\alpha_i^{\sfb_n(q_1)} - \pi / 2| + \chi_i$. 
By \hyperlink{A4p}{$\mathsf{(A4')}$} we further get $\alpha_i^{\sfb_n(q_2)} \geq \alpha_i^{\sfb_n(q_1)} + \chi_i \geq \pi / 2 + u + m\iota + \chi_i > \pi / 2 + u + (m + 1) \iota$. 
By the second claim of \cref{lemma:sector_dector3}, in round $t$ the agent will report type $2$. 
To summarize, $c^- = 0$ if we run \cref{alg:sector_dector3} with inputs $(\pi / 2 + u + (m + 1)\iota, \iota, q_2, \hat \alpha_{(i + 1):(d - 2)}^{q_2}, (\hat \alpha_{i:(d - 2)}^{s_j})_{1 \leq j \leq |\Theta| - 2}, n)$. 
As a consequence, in line 5 of \cref{alg:8}, we have $\alpha = \pi / 2 + u + (m + 1) \iota$ and $q = q_1$.

In the ``while'' loop of \cref{alg:8}, \cref{alg:sector_dector3} is implemented for $J$ times with $J \leq  4^{i + 3}\lceil \log_2 n \rceil$. 
We denote by $(u_1^j, u_2^j, L^j)$ the value of $(u_1, u_2, L)$ at the beginning of the $j$-th round of the while loop. 
Observe that $L^j = u_2^j - u_1^j$. 
We then prove by induction that if $\alpha = \pi / 2 + u + (m + 1) \iota$ and $q = q_1$ in line 5, then for $j = 1, 2, \cdots, J$, with probability at least $1 - 4j n^{-4^{d + 3}}$, in the $j$-th round of the while loop $\alpha_i^{\sfb_n(s_{|\Theta| - 1})} \in ((u_1^j + u_2^j) / 2 - L^j / 2, (u_1^j + u_2^j) / 2 + L^j / 2)$. 
In addition, if $\alpha_i^{\sfb_n(s_{|\Theta| - 1})} \in ((u_1^j + u_2^j) / 2 - L^j / 2, (u_1^j + u_2^j) / 2)$ then $c^- > 0$, otherwise $c^- = 0$.

For the base case $j = 1$, we have $u_1^1 = \pi / 2 + u + m \iota$ and $u_2^1 = \pi / 2 + u + (m + 1) \iota$. 
By the definition of $m$ we see that  $\alpha_i^{\sfb_n(s_{|\Theta| - 1})} \in (\pi / 2 + u + m\iota, \pi / 2 + u + (m + 1) \iota)$ (by \hyperlink{A6p}{$\mathsf{(A6')}$} $\alpha_i^{\sfb_n(s_{|\Theta| - 1})} \neq \pi / 2 + u + (m + 1) \iota$). 
This inequality also implies  $|\alpha_i^{\sfb_n(s_{|\Theta| - 1})} - \pi / 2| + \iota / 2 \geq |(u_1^1 + u_2^1) / 2 - \pi / 2|$. 
In the first round of the while loop, we run \cref{alg:sector_dector3} with inputs $((u_1^1 + u_2^1) / 2, \iota / 2,\, q_1, \,   \hat\alpha^{q_1}_{(i + 1):(d - 2)}, \,(\hat\alpha_{i:(d - 2)}^{s_j})_{1 \leq j \leq |\Theta| - 2},\,  n)$ and get $(c^+, c^-)$. 
Per our discussions above, we see that \hyperlink{A1p}{$\mathsf{(A1')}$}-\hyperlink{A6p}{$\mathsf{(A6')}$} hold for this modified conditional sector test. 
Therefore, if $\alpha_i^{\sfb_n(s_{|\Theta| - 1})} \in ((u_1^1 + u_2^1) / 2 - \iota / 2, (u_1^1 + u_2^1) / 2]$, by the second claim of \cref{lemma:sector_dector3}, when $\theta_t = \sfb_n(s_{|\Theta| - 1})$ the agent will report type $1$. 
As a consequence, with probability at least $1 - 4 n^{-4^{d + 3}}$ we have $c^- > 0$. 
Otherwise if $\alpha_i^{\sfb_n(s_{|\Theta| - 1})} \in ((u_1^1 + u_2^1) / 2, (u_1^1 + u_2^1) / 2 + \iota / 2)$, then by \cref{lemma:alpha-k-next} the agent in this case will never report type $1$, and $c^- = 0$. 
We have completed the proof for the base case. 

Now suppose the desired claim holds for the first $j$ rounds of the while loop, we then prove it also holds for the $(j + 1)$-th round of the while loop. 
If $c^- > 0$ in the $j$-th round of the while loop, then $u_2^{j + 1} = (u_1^j + u_2^j) / 2$ and $u_1^{j + 1} = u_1^j$. By induction,  $\alpha_i^{\sfb_n(s_{|\Theta| - 1})} \in ((u_1^j + u_2^j) / 2 - L^j / 2, (u_1^j + u_2^j) / 2) = (u_2^{j + 1} - L^j / 2, u_2^{j + 1}) = ((u_1^{j + 1} + u_2^{j + 1}) / 2 - L^{j + 1} / 2, (u_1^{j + 1} + u_2^{j + 1}) / 2 + L^{j + 1} / 2) $.  
If $c^- = 0$ in the $j$-th round of the while loop, then $u_1^{j + 1} = (u_1^j + u_2^j) / 2$ and $u_2^{j + 1} = u_2^j$. 
By induction, $\alpha_i^{\sfb_n(s_{|\Theta| - 1})} \in ((u_1^j + u_2^j) / 2, (u_1^j + u_2^j) / 2 + L^j / 2) = ((u_1^{j + 1} + u_2^{j + 1}) / 2 - L^{j + 1} / 2, (u_1^{j + 1} + u_2^{j + 1}) / 2 + L^{j + 1} / 2)$.
Similar to the proof for the base case, we can show that  \hyperlink{A1p}{$\mathsf{(A1')}$}-\hyperlink{A6p}{$\mathsf{(A6')}$} hold for the modified conditional sector test that appears in the $(j + 1)$-th round of the while loop. 
Therefore, if $\alpha_i^{\sfb_n(s_{|\Theta| - 1})} \in ((u_1^{j + 1} + u_2^{j + 1}) / 2 - L^{j + 1} / 2,  (u_1^{j + 1} + u_2^{j + 1}) / 2)$, then with probability at least $1 - 4n^{-4^{d + 3}}$ we have $c^- > 0$ in the $(j + 1)$-th round of the while loop. 
If $\alpha_i^{\sfb_n(s_{|\Theta| - 1})} \in ((u_1^{j + 1} + u_2^{j + 1}) / 2, (u_1^{j + 1} + u_2^{j + 1}) / 2 + L^{j + 1} / 2)$, 
then $c^- = 0$ in the $(j + 1)$-th round. 
We have completed the proof by induction. 

In the last round of the while loop, by induction we have $\alpha_i^{\sfb_n(s_{|\Theta| - 1})} \in ((u_1^J + u_2^J) / 2 - L^J / 2, (u_1^J + u_2^J) / 2 + L^J / 2)$. 
Note that $L^J < 8n^{-4^{i + 3}}$,
hence, $|(u_1^J + u_2^J) / 2 - \alpha_i^{\sfb_n(s_{|\Theta| - 1})}| < 4n^{-4^{i + 3}}$. 
The proof for claim 1 is done.

\subsubsection*{Proof of claim 3}

Without loss, once again we assume $s_{|\Theta| - 1} = q_1$ and $s_{|\Theta|} = q_2$. 
Under the conditions of claim 3 we have $\alpha_i^{\sfb_n(s_{|\Theta| - 1})} > \pi / 2$ and $\alpha_i^{\sfb_n(s_{|\Theta|})} < \pi / 2$. 
By \cref{lemma:line5line6} we know that \cref{alg:sector_dector4} will call \cref{alg:7}. 
Inspecting the proof of \cref{lemma:line5line6}, we further see that with probability at least $1 - 2 n^{-4^{d + 3}} - 16\lceil \log n \rceil^2 \cdot \frac{24(r_d \delta_{\sin})^{-1/2} n^{-60} + 80 d^{1/2}( \delta_{\sin} )^{-1/2} n^{-60}}{\pi }$ we have $c_1^+ > 0$ and $\bar c_2^+ > 0$.

In \cref{lemma:c2+c1+}, we discuss the implications of $c_2^+$ and $\bar c_1^+$. 
\cref{lemma:c2+c1+} implies that the ``if'' condition in line 8 of \cref{alg:7} never holds if we run \cref{alg:sector_dector3} with inputs $(\pi / 2 + \Delta, \iota, q_2,  \hat\alpha_{(i + 1):(d - 2)}^{q_2}$, $(\hat\alpha_{i:(d - 2)}^{s_j})_{1 \leq j \leq |\Theta| - 2}, n)$ or $(\pi / 2 - \Delta, \iota, q_1, \hat\alpha_{(i + 1):(d - 2)}^{q_1}$, $(\hat\alpha_{i:(d - 2)}^{s_j})_{1 \leq j \leq |\Theta| - 2}, n)$, 
where $\Delta = u + \iota, u + 2\iota, \cdots, u + N_1 \iota$ with $\Delta - \iota < |\alpha_i^{\sfb_n(s_{|\Theta| - 1})} - \pi / 2|$. 

We let $m = \sup\{m' \in \NN: u + m'\iota < |\alpha_i^{\sfb_n(s_{|\Theta| - 1})} - \pi / 2| \}$.
For $m' \in \{0, 1, \cdots, m\}$,
If we run \cref{alg:sector_dector3} with inputs $(\pi / 2 + u + m'\iota, \iota, q_1,  \hat\alpha_{(i + 1):(d - 2)}^{q_1}$, $(\hat\alpha_{i:(d - 2)}^{s_j})_{1 \leq j \leq |\Theta| - 2}, n)$ and get $(c^+, c^-)$, 
by the second claim of \cref{lemma:sector_dector3}, 
during the implementation of \cref{alg:sector_dector3}, 
when $\theta_t = \sfb_n(q_1)$, the agent in round $t$ will report type 2. 
As a consequence, with probability at least $1 - 4n^{-4^{d + 3}}$ we have $c^+ > 0$. 
Similarly, 
if we run \cref{alg:sector_dector3} with inputs $(\pi / 2 - u - m'\iota, \iota, q_2,  \hat\alpha_{(i + 1):(d - 2)}^{q_2}$, $(\hat\alpha_{i:(d - 2)}^{s_j})_{1 \leq j \leq |\Theta| - 2}, n)$ and get $(c^+, c^-)$, 
then with probability at least $1 - 4n^{-4^{d + 3}}$ we have $c^+ > 0$. 

By the definition of $m$ we know that $\pi / 2 + u + m\iota < \alpha_i^{\sfb_n(s_{|\Theta| - 1})} \leq \pi / 2 + u + (m + 1)\iota$. 
Now we run \cref{alg:sector_dector3} with inputs $(\pi / 2 + u + (m + 1)\iota, \iota, q_1,  \hat\alpha_{(i + 1):(d - 2)}^{q_1}$, $(\hat\alpha_{i:(d - 2)}^{s_j})_{1 \leq j \leq |\Theta| - 2}, n)$ and get $(c^+, c^-)$. 
During the implementation of this modified conditional sector test: (1) if $\theta_t = \sfb_n(j_{\ell})$ for some $\ell \in [|\Theta| - 2]$, then by the fourth claim of \cref{lemma:sector_dector3} the agent reports type $\ell + 2$; 
(2) if $\theta_t = \sfb_n(q_1)$, 
then by the third claim of \cref{lemma:sector_dector3} the agent will report type $1$; 
(3) if $\theta_t = \sfb_n(q_2)$, since $\alpha_i^{\sfb_n(q_2)} < \pi / 2$ and $|\alpha_i^{\sfb_n(q_2)} - \pi / 2| \geq |\alpha_i^{\sfb_n(q_1)} - \pi / 2| + \chi_i$, then $\alpha_i^{\sfb_n(q_2)} \in [0, \pi / 2 - u - m\iota]$.  
By the first claim of \cref{lemma:sector_dector3} the agent will never report type $2$ in this case.
In summary, with probability at least $1 - 4n^{-4^{d + 3}}$ we have $c^+ = 0$ and $c^- > 0$. 

Suppose we run \cref{alg:sector_dector3} with inputs $(\pi / 2 - u - (m + 1)\iota, \iota, q_2,  \hat\alpha_{(i + 1):(d - 2)}^{q_2},(\hat\alpha_{i:(d - 2)}^{s_j})_{1 \leq j \leq |\Theta| - 2}, n)$ and get $(c^+, c^-)$.
During the implementation of this modified conditional sector test: 
(1) if $\theta_t = \sfb_n(j_{\ell})$ for some $\ell \in [|\Theta| - 2]$, then by the fourth claim of \cref{lemma:sector_dector3} the agent reports type $\ell + 2$; 
(2) if $\theta_t = \sfb_n(q_2)$, 
then by the second claim of \cref{lemma:sector_dector3} the agent will report type $1$; 
(3) if $\theta_t = \sfb_n(q_1)$, then by the first claim of \cref{lemma:sector_dector3} the agent will never report type $1$. 
In summary, with probability at least $1 - 4n^{-4^{d + 3}}$ we have $c^+ > 0$. 

Putting together the above analysis, 
we conclude that in line 9 of \cref{alg:7}, we have $u_1 = \pi / 2 + u + m\iota$, $u_2 = \pi / 2 + u + (m + 1)\iota$ and $w = 1$. 

In the ``while'' loop of \cref{alg:7}, 
\cref{alg:sector_dector3} is implemented for $J$ times with $J \leq 4^{i + 3} \lceil \log_2 n \rceil$. 
We denote by $(u_1^j, u_2^j, L^j)$ the value of $(u_1, u_2, L)$ at the beginning of the $j$-th round of the while loop. 
Observe that $L^j = u_2^j - u_1^j$. 
We then prove by induction that if $u_1 = \pi / 2 + u + m\iota$, $u_2 = \pi / 2 + u + (m + 1)\iota$ and $w = 1$ in line 9 of \cref{alg:7}, 
then for all $j = 1, 2, \cdots, J$, with probability at least $1 - 4j n^{-4^{d + 3}}$, 
in the $j$-th round of the while loop we have $\alpha_i^{\sfb_n(s_{|\Theta| - 1})} \in ((u_1^j + u_2^j) / 2 - L^j / 2, (u_1^j + u_2^j) / 2 + L^j / 2)$. 
In addition, if $\alpha_i^{\sfb_n(s_{|\Theta| - 1})} \in ((u_1^j + u_2^j) / 2 - L^j / 2, (u_1^j + u_2^j) / 2)$ then $c^+ = 0$ and $c^- > 0$, otherwise $c^+ > 0$. 

The proof of this inductive claim is similar to that for claim 1, and we skip it for the sake of simplicity. 

In the last round of the while loop, by induction we have $\alpha_i^{\sfb_n(s_{|\Theta| - 1})} \in ((u_1^J + u_2^J) / 2 - L^J / 2, (u_1^J + u_2^J) / 2 + L^J / 2)$. 
Note that $L^J < 8n^{-4^{i + 3}}$,
hence, $|(u_1^J + u_2^J) / 2 - \alpha_i^{\sfb_n(s_{|\Theta| - 1})}| < 4n^{-4^{i + 3}}$. 
The proof for claim 3 is done. 

\begin{lemma}
\label{lemma:c2+c1+}

Under the conditions of \cref{lemma:call-algs} claim 3, the followins statements are true: 
\begin{enumerate}
    \item If $c_2^+ > 0$, then for all $\Delta = u + \iota, u + 2\iota, \cdots, u + N_1 \iota$ with $\Delta - \iota < |\alpha_i^{\sfb_n(s_{|\Theta| - 1})} - \pi / 2|$, suppose we run \cref{alg:sector_dector3} with inputs $(\pi / 2 + \Delta, \iota, q_2,  \hat\alpha_{(i + 1):(d - 2)}^{q_2}$, $(\hat\alpha_{i:(d - 2)}^{s_j})_{1 \leq j \leq |\Theta| - 2}, n)$ and outputs $(c^+, c^-)$, then with probability at least $1 - 4n^{-4^{d + 3}}$ we have $c^+ > 0$.  
    \item If $\bar c_1^+ > 0$, then for all $\Delta = u + \iota, u + 2\iota, \cdots, u + N_1 \iota$ with $\Delta - \iota < |\alpha_i^{\sfb_n(s_{|\Theta| - 1})} - \pi / 2|$, suppose we run \cref{alg:sector_dector3} with inputs $(\pi / 2 - \Delta, \iota, q_1, \hat\alpha_{(i + 1):(d - 2)}^{q_1}$, $(\hat\alpha_{i:(d - 2)}^{s_j})_{1 \leq j \leq |\Theta| - 2}, n)$ and outputs $(c^+, c^-)$, then with probability at least $1 - 4n^{-4^{d + 3}}$ we have $c^+ > 0$.
\end{enumerate}

\end{lemma}
\begin{proof}[Proof of \cref{lemma:c2+c1+}]
    By symmetry, it suffices to prove the first claim. 
    In \cref{alg:sector_dector4}, we run \cref{alg:sector_dector3} with inputs $(\pi / 2 + u, \iota, q_2,  \hat\alpha_{(i + 1):(d - 2)}^{q_2}$, $(\hat\alpha_{i:(d - 2)}^{s_j})_{1 \leq j \leq |\Theta| - 2}, n)$ and get $(c_2^+, c_2^-)$. 
    For this modified conditional sector test: (1) if $\theta_t = \sfb_n(j_{\ell})$ for some $\ell \in [|\Theta| - 2]$, 
    then by the fourth claim of \cref{lemma:sector_dector3} the agent in round $t$ will report type $\ell + 2$; 
    (2) if $\theta_t = \sfb_n(q_2)$, since $0 \leq \alpha_i^{\sfb_n(q_2)} < \pi / 2 < \pi / 2 + u$, then by the fifth claim of \cref{lemma:sector_dector3}, the agent in round $t$ will never report type $2$.
    In what follows, we consider $\theta_t = \sfb_n(q_1)$. 
    Therefore, if $c_2^+ > 0$, then the agent must report type $2$ when $\theta_t = \sfb_n(s_{|\Theta| - 1})$. By \cref{eq:new48}, we have
    \begin{align}
    \label{eq:c2+implication}
    \begin{split}
        & \langle \Pi_{2}^{\alpha, \iota, q_2}(\pi / 2 + u), \,  \bar v_{\sfb_n(s_{|\Theta| - 1})} \rangle \geq \langle  \Pi_{\ell}^{\alpha, \iota, q_2}, \, \bar v_{\sfb_n(s_{|\Theta| - 1})}  \rangle - \frac{\delta_{\sin}}{n^{4^{d + 3}}} \qquad \mbox{for all $\ell \in [|\Theta|] \, \backslash \, \{1, 2\}$.} \\
        & \langle \Pi_{2}^{\alpha, \iota, q_2}(\pi / 2 + u), \,  \bar v_{\sfb_n(s_{|\Theta| - 1})} \rangle \geq \langle  \Pi_{1}^{\alpha, \iota, q_2}(\pi / 2 + u), \, \bar v_{\sfb_n(s_{|\Theta| - 1})}  \rangle - \frac{\delta_{\sin}}{n^{4^{d + 3}}}.
    \end{split}
    \end{align}
    In the above equation, 
\begin{align*}
\begin{split}
	& \Pi_{1}^{\alpha, \iota, q_2}(\pi / 2 + u) = \mathds{1}_d / d + r_d \varphi_0^{-1} (\xi_i(\pi / 2 + u - \iota, \hat \alpha^{q_2}_{(i + 1): (d - 2)})), \\
	& \Pi_{2}^{\alpha, \iota, q_2}(\pi / 2 + u) = \mathds{1}_d / d + r_d \varphi_0^{-1} (\xi_i(\pi / 2 + u + \iota, \hat \alpha^{q_2}_{(i + 1): (d - 2)})), \\
	& \Pi_{\ell}^{\alpha, \iota, q_2} = \mathds{1}_d / d + r_d \varphi_0^{-1}(\xi_i(\hat \alpha_{i:(d - 2)}^{j_{\ell - 2}})), \qquad \mbox{for all }3 \leq \ell \leq |\Theta|. 
\end{split}
\end{align*}
We define 
\begin{align*}
\begin{split}
	& \Pi_{1}^{\ast}(\pi / 2 + u) = \mathds{1}_d / d + r_d \varphi_0^{-1} (\xi_i(\pi / 2 + u - \iota, \alpha^{\sfb_n(q_2)}_{(i + 1): (d - 2)})), \\
	& \Pi_{2}^{\ast}(\pi / 2 + u) = \mathds{1}_d / d + r_d \varphi_0^{-1} (\xi_i(\pi / 2 + u + \iota, \alpha^{\sfb_n(q_2)}_{(i + 1): (d - 2)})), \\
	& \Pi_{\ell}^{\ast} = \mathds{1}_d / d + r_d \varphi_0^{-1}(\xi_i(\alpha_{i:(d - 2)}^{\sfb_n(j_{\ell - 2})})), \qquad \mbox{for all }3 \leq \ell \leq |\Theta|. 
\end{split}
\end{align*}
We then consider \cref{alg:sector_dector3} with inputs $(\pi / 2 + \Delta, \iota, q_2,  \hat\alpha_{(i + 1):(d - 2)}^{q_2}$, $(\hat\alpha_{i:(d - 2)}^{s_j})_{1 \leq j \leq |\Theta| - 2}, n)$, 
where $\Delta \in \{ u + \iota, u + 2\iota, \cdots, u + N_1 \iota \}$ and $\Delta - \iota < |\alpha_i^{\sfb_n(s_{|\Theta| - 1})} - \pi / 2|$. 
The agent in this setting announces the following mechanism: 
\begin{align*}
    \begin{split}
	& \Pi_{1}^{\alpha, \iota, q_2}(\pi / 2 + \Delta) = \mathds{1}_d / d + r_d \varphi_0^{-1} (\xi_i(\pi / 2 + \Delta - \iota, \hat \alpha^{q_2}_{(i + 1): (d - 2)})), \\
	& \Pi_{2}^{\alpha, \iota, q_2}(\pi / 2 + \Delta) = \mathds{1}_d / d + r_d \varphi_0^{-1} (\xi_i(\pi / 2 + \Delta + \iota, \hat \alpha^{q_2}_{(i + 1): (d - 2)})), \\
	& \Pi_{\ell}^{\alpha, \iota, q_2} = \mathds{1}_d / d + r_d \varphi_0^{-1}(\xi_i(\hat \alpha_{i:(d - 2)}^{j_{\ell - 2}})), \qquad \mbox{for all }3 \leq \ell \leq |\Theta|. 
\end{split}
\end{align*}
Similar to the proof of \cref{lemma:mechanism-close}, we can show that 
\begin{align}
\label{eq:mec-close}
\begin{split}
    & \big| \langle \bar v_{\sfb_n(s_{|\Theta| - 1})},\, \Pi_{\ell}^{\alpha, \iota, q_2}(\pi / 2 + u) \rangle - \langle \bar v_{\sfb_n(s_{|\Theta| - 1})},\, \Pi_{\ell}^{\ast}(\pi / 2 + u) \rangle \big| \leq 4r_d \cdot (d - 2 - i) \cdot n^{-4^{i + 4}}, \qquad \mbox{for all }\ell \in \{1, 2\}, \\
    & \big| \langle \bar v_{\sfb_n(s_{|\Theta| - 1})},\, \Pi_{\ell}^{\alpha, \iota, q_2}(\pi / 2 + \Delta) \rangle - \langle \bar v_{\sfb_n(s_{|\Theta| - 1})},\, \Pi_{\ell}^{\ast}(\pi / 2 + \Delta) \rangle \big| \leq 4r_d \cdot (d - 2 - i) \cdot n^{-4^{i + 4}}, \qquad \mbox{for all }\ell \in \{1, 2\}, \\
    & \big| \langle \bar v_{\sfb_n(s_{|\Theta| - 1})},\, \Pi_{\ell}^{\alpha, \iota, q_2} \rangle - \langle \bar v_{\sfb_n(s_{|\Theta| - 1})},\, \Pi_{\ell}^{\ast} \rangle \big| \leq 4r_d \cdot (d - 1 - i) \cdot n^{-4^{i + 3}}, \qquad \mbox{for all }\ell \in \{3, 4, \cdots, |\Theta|\}. 
\end{split}
\end{align}
The proof of \cref{eq:mec-close} is similar to that of \cref{lemma:mechanism-close}, and we skip it for the sake of compactness. 
%Combining \cref{eq:c2+implication} and \cref{eq:mec-close}, we conclude that 
%
% \begin{align}
%     & \langle \Pi_{2}^{\ast}(\pi / 2 + u), \,  \bar v_{\sfb_n(s_{|\Theta| - 1})} \rangle \geq \langle  \Pi_{\ell}^{\ast}, \, \bar v_{\sfb_n(s_{|\Theta| - 1})}  \rangle - \frac{\delta_{\sin}}{n^{4^{d + 3}}} - 8r_d \cdot (d - 1 - i) \cdot n^{-4^{i + 3}} \qquad \mbox{for all $\ell \in [|\Theta|] \, \backslash \, \{1, 2\}$, } \nonumber \\
%     & \langle \Pi_{2}^{\ast}(\pi / 2 + u), \,  \bar v_{\sfb_n(s_{|\Theta| - 1})} \rangle \geq \langle  \Pi_{1}^{\ast}(\pi / 2 + u), \, \bar v_{\sfb_n(s_{|\Theta| - 1})}  \rangle - \frac{\delta_{\sin}}{n^{4^{d + 3}}} - 8r_d \cdot (d - 1 - i) \cdot n^{-4^{i + 3}}. \label{eq:Pi1Pi2-star}
% \end{align}
%
We define 
\begin{align*}
    F(a) = & \,\langle \Pi_2^{\ast}(a), \, \bar v_{\sfb_n(s_{|\Theta| - 1})} \rangle = \langle \bar v_{\sfb_n(s_{|\Theta| - 1})}, \, \mathds{1}_d / d + r_d \varphi_0^{-1} (\xi_i(a + \iota, \alpha^{\sfb_n(q_2)}_{(i + 1): (d - 2)})) \rangle \\
    = & \, r_d \prod_{j = i}^{i - 1} \sin \alpha_j^{\sfb_n(s_{|\Theta| - 1})} \times \Big( \cos (a + \iota) \cos (\alpha_i^{\sfb_n(s_{|\Theta| - 1})}) \\
    & + \sin (a + \iota) \sin (\alpha_i^{\sfb_n(s_{|\Theta| - 1})}) \langle \xi_{i + 1}(\alpha_{(i + 1):(d - 2)}^{\sfb_n(s_{|\Theta| - 1})}), \, \xi_{i + 1}(\alpha_{(i + 1):(d - 2)}^{\sfb_n(q_2)})  \rangle \ \Big). 
\end{align*}
Taking the derivative of $F$, we get 
\begin{align*}
    F'(a) = r_d \Big( -\sin(a + \iota) \cot \alpha_i^{\sfb_n(s_{|\Theta| - 1})} + \cos(a + \iota) \langle \xi_{i + 1}(\alpha_{(i + 1):(d - 2)}^{\sfb_n(s_{|\Theta| - 1})}), \, \xi_{i + 1}(\alpha_{(i + 1):(d - 2)}^{\sfb_n(q_2)})  \rangle \Big) \prod_{j = 1}^i \sin \alpha_j^{\sfb_n(s_{|\Theta| - 1})}. 
\end{align*}
Recall that $q_2 = s_{|\Theta|} \neq s_{|\Theta| - 1}$. 
Note that when $\alpha_i^{\sfb_n(s_{|\Theta| - 1})} \in [a - \iota, \pi]$, 
following the proof of \cref{lemma:D6} case I, 
we see that 
\begin{align*}
    F'(a) \geq \frac{r_d \delta_{\sin} \tilde\chi_i \min \{\chi_{i + 1}, \tilde \chi_{i + 1},  \bar\chi_{i + 1}\}^2}{120}.
\end{align*}
for all $\pi / 2 < a \leq \alpha_i^{\sfb_n(s_{|\Theta| - 1})} + \iota$.

For all $x \in [u, \pi / 2]$ with $x - \iota < |\alpha_i^{\sfb_n(s_{|\Theta| - 1})} - \pi / 2|$, we have $\pi / 2 < \pi / 2 + x \leq \alpha_i^{\sfb_n(s_{|\Theta| - 1})} + \iota$, hence $F'(x) \geq \frac{r_d \delta_{\sin} \tilde\chi_i \min \{\chi_{i + 1}, \tilde \chi_{i + 1},  \bar\chi_{i + 1}\}^2}{120}$.
Therefore, for all $\Delta = u + \iota, u + 2\iota, \cdots, u + N_1 \iota$ with $\Delta - \iota < |\alpha_i^{\sfb_n(s_{|\Theta| - 1})} - \pi / 2|$, we have 
\begin{align}
\label{eq:two-F-func}
    F(\pi / 2 + \Delta) - F(\pi / 2 + u) \geq \frac{r_d \delta_{\sin} \tilde\chi_i \min \{\chi_{i + 1}, \tilde \chi_{i + 1},  \bar\chi_{i + 1}\}^2 \iota }{120}. 
\end{align}
Combining \cref{eq:two-F-func,eq:mec-close,eq:c2+implication}, we see that for all $\Delta = u + \iota, u + 2\iota, \cdots, u + N_1 \iota$ with $\Delta - \iota < |\alpha_i^{\sfb_n(s_{|\Theta| - 1})} - \pi / 2|$ and all $\ell \in [|\Theta|] \, \backslash \, \{1, 2\}$, 
\begin{align*}
    & \langle \Pi_2^{\alpha, \iota, q_2}(\pi / 2 + \Delta), \bar v_{\sfb_n(s_{|\Theta| - 1})} \rangle \\
    & \geq \langle  \Pi_{\ell}^{\alpha, \iota, q_2}, \, \bar v_{\sfb_n(s_{|\Theta| - 1})}  \rangle - \frac{\delta_{\sin}}{n^{4^{d + 3}}} + \frac{r_d \delta_{\sin} \tilde\chi_i \min \{\chi_{i + 1}, \tilde \chi_{i + 1},  \bar\chi_{i + 1}\}^2 \iota }{120} - 8r_d(d - 1 - i) n^{-4^{i + 3}} \\
    & > \langle  \Pi_{\ell}^{\alpha, \iota, q_2}, \, \bar v_{\sfb_n(s_{|\Theta| - 1})}  \rangle + \frac{\delta_{\sin}}{n^{4^{d + 3}}}, 
\end{align*}
where the last line above follows from \hyperlink{A4p}{$\mathsf{(A4')}$}. 
Therefore, when we run \cref{alg:sector_dector3} with inputs $(\pi / 2 + \Delta, \iota, q_2,  \hat\alpha_{(i + 1):(d - 2)}^{q_2}$, $(\hat\alpha_{i:(d - 2)}^{s_j})_{1 \leq j \leq |\Theta| - 2}, n)$ and $\theta_t = \sfb_n(s_{|\Theta| - 1})$, 
the agent in round $t$ will never report type $\ell$ for all $\ell \in [|\Theta|] \backslash \{1, 2\}$. 
Since $\pi / 2 + \Delta > \pi / 2$ and $\alpha_i^{\sfb_n(s_{|\Theta| - 1})} \geq \pi / 2 + \Delta - \iota$, 
by the first claim of \cref{lemma:sector_dector3}, the agent in round $t$ will never report type $1$. 
Therefore, when $\theta_t = \sfb_n(s_{|\Theta| - 1})$ the agent reports type $2$ in round $t$. As a consequence, $c^+ > 0$ if type $\sfb_n(s_{|\Theta| - 1})$ appears at least once in a total number of $\lceil \log n \rceil^4$ rounds. 
Under \hyperlink{A4p}{$\mathsf{(A4')}$}, with probability at least $1 - 4n^{-4^{d + 3}}$ we have $c^+ > 0$. 
The proof is done.

\end{proof}

\end{proof}

\end{proof}

%Our next lemma describes under what settings will \cref{alg:7} and \cref{alg:8} be implemented. We comment that \cref{lemma:78} is a direct consequence of \cref{lemma:sector_dector3}. We only need to check that the conditions of \cref{lemma:sector_dector3} are satisfied on $\cE$, which is again straightforward. 
%
%\begin{lemma}\label{lemma:78}
%	Assume $\cE$ occurs. Then with probability at least  $1 - 4n^{-4^{d+3}}$ the following statement is true: If $\sign(\alpha_i^{s_{|\Theta|}} - \pi / 2) = \sign (\alpha_i^{s_{|\Theta| - 1}} - \pi / 2)$, then \cref{alg:8} will be implemented. Otherwise, if  $\sign(\alpha_i^{s_{|\Theta|}} - \pi / 2) \neq \sign (\alpha_i^{s_{|\Theta| - 1}} - \pi / 2)$ then \cref{alg:7} will be implemented. 
%\end{lemma}

%The remaining proof is completed by checking the conditions of \cref{lemma:sector_dector3}. 

\subsubsection{Proof of \cref{lemma:alpha-k-next} when $k = |\Theta| - 1$}\label{sec:appendix-other2}

Finally, we state the algorithm that achieves the desiderata of \cref{lemma:alpha-k-next} when $k = |\Theta| - 1$. 
Recall that $\varphi_0: \bar \Delta_0 \mapsto \mathbb{B}^{d - 1}$ is defined in  \cref{sec:representation}.  
For $R > 0$, we define 
\begin{align*}
    & \bar \Delta_0(R) = \big\{ x \in \RR^d: \langle x, \mathds{1}_d \rangle = 0, \|x\|_2 \leq R\big\}, \\
    & \BB^{d - 1}(R) = \big\{ x \in \RR^{d - 1}: \|x\|_2 \leq R \big\}. 
\end{align*}
Note that $\bar\Delta_0(1) = \bar \Delta_0$ and $\BB^{d - 1}(1) = \BB^{d - 1}$.  
We can extend $\varphi_0$ to $\bar \Delta_0(R)$ for any $R > 1$. 
Specifically, for $x \in \bar\Delta_0(R)$ with $\|x\|_2 > 1$, 
we define 
\begin{align*}
    \varphi_0(x) = \|x\|_2 \cdot \varphi_0\big( x \, / \, \| x\|_2 \big). 
\end{align*}
Such extension of $\varphi_0$ also preserves the Euclidean distance.

We set $\bar r_d$ to be a sufficiently small positive constant, 
such that for all $x \in \BB^{d - 1}(\,3  (\min_{\theta \in \Theta} |\cos (\alpha_i^\theta)|)^{-1}\,)$, 
\begin{align*}
    \mathds{1}_d / d + \bar r_d \varphi_0^{-1}(x) \in \left\{ x \in \RR^d: x \geq 0, \langle x,  \mathds{1}_d \rangle = 1 \right\}. 
\end{align*}
Recall that $s_j$ is defined in \cref{eq:s1-s2-permutation}. 
By its definition, we see that
\begin{align}
\label{eq:tan-ranking}
    |\tan \alpha_i^{\sfb_n(s_1)}| > |\tan \alpha_i^{\sfb_n(s_2)}| > \cdots > |\tan \alpha_i^{\sfb_n(s_{|\Theta|})}|. 
\end{align}
The algorithm is constructed based on the following mechanism:
%In this part we choose $n_0$ large enough, such that for all $n \geq n_0$, $n$ is no smaller than 
%	\begin{align*}
%		 	& \exp\big( {4^{(d + 3) / 3}f_{\min}^{-1 / 3}} \big) \vee \sqrt[4^{d + 3}]{\frac{240 \log n}{r_d \delta_{\sin}\tilde \chi_i \min\{\chi_{i + 1}, \tilde\chi_{i + 1}, \bar\chi_{i + 1}\}^2}} \vee \sqrt[4^{i + 3}]{\frac{480 \pi d \log n}{\delta_{\sin} \tilde\chi_i \min\{\chi_{i + 1}, \tilde\chi_{i + 1}, \bar\chi_{i + 1}\}^2}}. 
%		 \end{align*}
%		 Furthermore, we require that for all $n \geq n_0$, 
		 %
%		 \begin{align*}
%		 	& n^{-4^{i + 3.5}} \geq 2\pi r_d d n^{-4^{i + 4}} + \frac{n^{-4^{d + 3}} + 2\pi r_d d n^{-4^{i + 3}}}{r_d \delta_{\sin}}, \\
%		 	& \frac{2}{\log n} \leq \frac{\min\{\chi_{i + 1}, \bar\chi_{i + 1}, \tilde\chi_{i + 1}\}^2}{30}. 
%		 \end{align*}
		 %
%		 We comment that we can find $n_0$ that is a function of $(\{\bar{v}_{\theta}: \theta \in \Theta\}, \varphi_0)$, such that the above conditions are satisfied. 
%
%At the current stage, recall that we have already obtained estimates for $(\alpha_i^{\sfb(s_j)})_{1 \leq j \leq |\Theta| - 1}$ and only need to estimate $\alpha_i^{\sfb(s_{|\Theta|})}$. 
%
\begin{align}\label{eq:mechanism3}
\begin{split}
	& \Pi_{j, t} = \mathds{1}_d / d + {\bar r_d }\varphi_0^{-1}( \xi_{i + 1}(\hat\alpha^{s_j}_{(i + 1):(d - 2)})), \qquad 1 \leq j \leq |\Theta| - 1, \\
	& \Pi_{|\Theta|, t} = \mathds{1}_d / d + {\bar r_d}   \varphi_0^{-1} (x \cdot \delta e_i - e \xi_{i + 1}(\hat\alpha_{(i + 1):(d - 2)}^{s_{|\Theta|}})), 
\end{split}
\end{align}
where $\delta > 0$, $e \in (0,1)$, $x \in \{\pm 1\}$, and $e_i \in \RR^{d - 1}$ is a vector such that the $i$-th coordinate of which is one and all the rest coordinates are zero. Once again, we construct an oracle mechanism that is  based on true values and well approximates mechanism \eqref{eq:mechanism3}: 
\begin{align}\label{eq:mechanism3-oracle}
	\begin{split}
	& \Pi_{j, t}^{\ast} = \mathds{1}_d / d + {\bar r_d }\varphi_0^{-1}(\xi_{i + 1}(\alpha^{\sfb_n(s_j)}_{(i + 1):(d - 2)})), \qquad 1 \leq j \leq |\Theta| - 1, \\
	& \Pi_{|\Theta|, t}^{\ast} = \mathds{1}_d / d + \bar r_d \varphi_0^{-1} (x \cdot \delta e_i - e\xi_{i + 1}(\alpha_{(i + 1):(d - 2)}^{\sfb_n(s_{|\Theta|})})). 
\end{split}
\end{align}

\begin{lemma}\label{lemma:mclose}
	Under the conditions of \cref{lemma:alpha-k-next}, for all $\theta \in \Theta$ and $j \in \{1, 2, \cdots, |\Theta|\}$, it holds that
	\begin{align*}
		\left| \langle \Pi_{j, t}, \bar{v}_{\theta} \rangle  - \langle \Pi_{j, t}^{\ast}, \bar{v}_{\theta} \rangle \right| \leq 4 \,\bar r_d  \cdot (d - 2 - i) \cdot n^{-4^{i + 4}}.
	\end{align*}
\end{lemma} 
The proof of \cref{lemma:mclose} is similar to that of \cref{lemma:mechanism-close}, and we skip it for the compactness of presentation. 
We then present a modified conditional sector test for the last coordinate:  

\begin{algorithm}
\caption{Modified conditional sector test for the last coordinate}
\label{alg:8.5}
%\textbf{Input:} $n, \,  \cZ_{i + 1},  \, [|\Theta|] \backslash \cM_{i, |\Theta| - 1} = \{s_{|\Theta|}\}, \, \delta, \, x$;
\textbf{Input:} $n, \,  \hat\balpha_{(i + 1):(d - 2)},  \, s_{|\Theta|}, \, \delta, \, x$;
\begin{algorithmic}[1]
		\State $\mathsf{count}, \mathsf{count}_+ \gets 0$; 
		\State $c \gets \max_{j \in [|\Theta| - 1]}\langle \xi_{i + 1}(\hat\alpha_{(i + 1):(d - 2)}^{ s_{|\Theta|}}), \xi_{i + 1}(\hat\alpha_{(i + 1):(d - 2)}^{s_{j}}) \rangle$, $e \gets \max\{-c, 0\} + (\log n)^{-1}$;
		\While{$\mathsf{count} \leq \lceil \log n \rceil^4$}
			\State $\mathsf{count} \gets \mathsf{count} + 1$;
			\State Start a new round, announce the mechansim as in \cref{eq:mechanism3}, observe agent's reported type $r$;
			\If{$r = |\Theta|$}
				\State $\mathsf{count}_+ \gets \mathsf{count}_+ + 1$; 
			\EndIf 
		\EndWhile
		\State \texttt{\textcolor{blue}{// Create delayed feedback}}
		\State $\ell \gets \lceil \log n \rceil^2$; 
		\For{$i \in [\ell]$}
			\State Start a new round and announce a dummy mechanism $\mathbf{1}_{|\Theta| \times d} / d$; 
		\EndFor
		\If{$\mathsf{count}_+ > 0$}
			\State \Return True;
		\Else
			\State \Return False; 
		\EndIf
\end{algorithmic}
\end{algorithm}
We analyze
\cref{alg:8.5} in \cref{lemma:8.5}. The proof of the lemma is deferred to \cref{sec:proof-8.5}. 
\begin{lemma}\label{lemma:8.5}
	Under the conditions of \cref{lemma:alpha-k-next} with $k = |\Theta| - 1$, the following statements are true:
	\begin{enumerate}
		\item If $\delta \leq \min_{j \in [|\Theta| - 1]} |\tan(\alpha_i^{\sfb_n(s_j)})| \cdot (c_{\ast} + e_{\ast})$, then when $\theta_t = \sfb_n(s_j)$ for some $j \in [|\Theta| - 1]$,  the agent in round $t$ will always report type $j$. Here, 
		\begin{align*}
			& c_{\ast} : = \max_{j \in [|\Theta| - 1]}\langle \xi_{i + 1}(\alpha_{(i + 1):(d - 2)}^{\sfb_n(s_{|\Theta|})}), \xi_{i + 1}(\alpha_{(i + 1):(d - 2)}^{\sfb_n(s_{j})}) \rangle, \\
			& e_{\ast} := \max\left\{-c_{\ast}, 0 \right\} + (\log n)^{-1}. 
		\end{align*}
		\item If $x = \sign(-\alpha_i^{\sfb_n(s_{|\Theta|})} + \pi / 2)$, $\delta \geq |\tan(\alpha_i^{\sfb_n(s_{|\Theta|})})| \cdot (c_{\ast} + e_{\ast}  + n^{-4^{i + 3.5}})$, and $\theta_t = \sfb_n(s_{|\Theta|})$, then the agent in round $t$ will report type $|\Theta|$.
		\item If $x = \sign(-\alpha_i^{\sfb_n(s_{|\Theta|})} + \pi / 2)$, $ \delta \leq |\tan (\alpha_i^{\sfb_n(s_{|\Theta|})})| \cdot (c_{\ast} + e_{\ast} - n^{-4^{i + 3.5}})$, then the agent in round $t$ will never report type $|\Theta|$. 
	\end{enumerate} 
\end{lemma}

In the next lemma, we show that the sign of $\alpha_i^{\sfb_n(s_{|\Theta|})} - \pi / 2$ can be consistently estimated. 

\begin{lemma}\label{lemma:sign}
	Under the conditions of \cref{lemma:alpha-k-next} with $k = |\Theta| - 1$, there exist $n_0 \in \NN_+$ and $C_1, C_2 \in \RR_+$ that depend only on $(\mathscr{P}, \varphi_0)$, such that 
	for all $n \geq n_0$, there exists an algorithm that uses no more than $C_1(\log n)^4$ samples, and with probability at least $1 - C_2n^{-60}$ returns $\sign(\alpha_i^{\sfb_n(s_{|\Theta|})} - \pi / 2)$. 
\end{lemma}

\begin{proof}[Proof of \cref{lemma:sign}]
    We prove \cref{lemma:sign} in Section \ref{proof-lemma:sign}. The algorithm that satisfies the desired properties is stated there as well.  
\end{proof}

Finally, we state the algorithm that estimates $\alpha_i^{\sfb_n(s_{|\Theta|})}$ as \cref{alg:9}.

\begin{algorithm}
\caption{Estimating $ \alpha_i^{\sfb_n(s_{|\Theta|})}$}
\label{alg:9}
\textbf{Input:} $n, \,  \hat\balpha_{(i + 1):(d - 2)},  \, [|\Theta|] \backslash \{{s}_j: j \in [|\Theta| - 1]\} = \{q\}$;
\begin{algorithmic}[1]
		\State $c \gets \max_{j \in [|\Theta| - 1]}\langle \xi_{i + 1}(\hat\alpha_{(i + 1):(d - 2)}^{s_{|\Theta|}}), \, \xi_{i + 1}(\hat\alpha_{(i + 1):(d - 2)}^{s_{j}}) \rangle $, $e \gets \max\{-c, 0\} + (\log n)^{-1}$;
		\State Let $x \in \{\pm 1\}$ be an estimate of $\sign(-\alpha_i^{\sfb_n(s_{|\Theta|})} + \pi / 2)$ generated using the algorithm of \cref{lemma:sign}; 
		\State $\delta_l \gets (\log n)^{-2}$, $\delta_u \gets \min_{j \in [|\Theta| - 1]} |\tanh(\hat \alpha_i^{s_j})| \cdot (c + e) - (\log n)^{-1}$, $L \gets \delta_u - \delta_l$; 
		\While{$L \geq n^{-4^{i + 3}}$}
			\State Run \cref{alg:8.5} with $\delta = (\delta_l + \delta_u) / 2$, which outputs $R$; 
			\If{$R = $True}
				\State $\delta_l \gets (\delta_l + \delta_u) / 2$, $L \gets L / 2$;
			\Else
				\State $\delta_u \gets (\delta_l + \delta_u) / 2$, $L \gets L / 2$;
			\EndIf
		\EndWhile
		\State Return $\tan^{-1}(\frac{\delta_l + \delta_u}{2(c + e)})$; 
\end{algorithmic}
\end{algorithm}

One can verify that \cref{alg:9} achieves the desiderata of \cref{lemma:alpha-k-next} when $k = |\Theta| - 1$. The proof is similar to that for $k \leq |\Theta| - 2$, and we skip it for the compactness of presentation.

\newpage

\section{Supporting lemmas and auxiliary proofs}

\subsection{Supporting lemmas}

\begin{lemma}\label{lemma:cos}
	The following statements are true:
	\begin{enumerate}
		\item For all $x \in [-\pi, \pi]$, it holds that $\cos (x) \leq 1 - x^2 / 30$. 
		\item For $x, y, z \in \R$, if $\arc(x, y) \leq \arc(x, z)$, then 
		\begin{align*}
			\left|\arc(x, y) - \arc(x, z)\right| \geq \cos (x - y) - \cos(x - z) \geq \frac{(\arc(x, y) - \arc(x, z))^2}{30}.
		\end{align*} 
		%\item For $x, y, z \in [0, \pi]$, if $\arc_0(x, y) \leq \arc_0(x, z)$, then 
		%
		%\begin{align*}
		%	\cos(x - y) - \cos(x - z) \geq \frac{(\arc_0(x, y) - \arc_0(y, z))^2}{30}. 
		%\end{align*}
	\end{enumerate}
	
\end{lemma}
\begin{proof}[Proof of \cref{lemma:cos}]
    We first look at the first  claim. By symmetry, it suffices to prove this claim for all $x \in [0, \pi]$. Define $F(x) := 1 - x^2 / 30 - \cos(x)$. Then $F'(x) = \sin(x) - x / 15$, which on $[0, \pi]$ is first nonnegative and then negative. In addition, the sign change occurs on $[\pi / 2, \pi]$. Therefore, 
    \begin{align*}
        \sup_{x \in [-\pi, \pi]} F(x) = \sup_{x \in [\pi / 2, \pi]} F(x)  \geq 1 - \pi^2 / 30 > 0. 
    \end{align*}
    This completes the proof of the first claim. 

    We next establish the second claim.
    The upper bound in the second claim is simple, as $\cos(x - y) = \cos(\arc(x, y))$, $\cos(x - z) = \cos(\arc(x, z))$, and $x \mapsto \cos(x)$ is 1-Lipschitz continuous. To establish the lower bound, note that
    \begin{align*}
        \cos(x - y) - \cos(x - z) = & \cos(\arc(x, y)) - \cos(\arc(x, z)) = \int_{\arc(x, y)}^{\arc(x, z)} \sin (w) \dd w \\
        \geq & \int_0^{\arc(x, z) - \arc(x, y)} \sin (w) \dd w = 1 - \cos(\arc(x, z) - \arc(x, y)),
    \end{align*}
    which by the first claim is no smaller than $(\arc(x, z) - \arc(x, y))^2 / 30$. The proof is complete. 
\end{proof}

\begin{lemma}\label{lemma:cot}
\begin{enumerate}
	\item For all $x \in [0, \pi]$, it holds that $|\cot (x)| \geq |x - \pi / 2|$.
	\item For all $x_1, x_2 \in [0, \pi / 2) \cup (\pi / 2, \pi]$ satisfying $|x_1 - \pi / 2| \geq |x_2 - \pi / 2|$, it holds that 
	\begin{align*}
		|\cos(x_1)| \leq |\cos(x_2)| \cdot \frac{|x_1 - \pi / 2|}{|x_2 - \pi / 2|}. 
	\end{align*}
    Similarly, for $x_1, x_2 \in (0, \pi)$ with $|x_1 - \pi / 2| \leq |x_2 - \pi / 2|$, it holds that 

\end{enumerate} 
\end{lemma}
\begin{proof}[Proof of \cref{lemma:cot}]
    To prove the first claim, we simply notice that $\cot(\pi / 2) = 0$, and $x \mapsto \cot(x)$ is decreasing on $(0, \pi)$ with derivative smaller than $-1$. 

    We then prove the second claim. Using the symmetry about $\pi / 2$, we may without loss assume $x_1, x_2 \in (\pi / 2, \pi]$. Then the claim follows as $x \mapsto \cos(x)$ on $(\pi / 2, \pi)$ is convex and decreasing. 
\end{proof}

\begin{lemma}\label{lemma:rho}
	Recall that $\rho$ is defined in \cref{eq:rho}. Then for all $\alpha, \alpha' \in  [0, \pi]^{d - 3} \times [0, 2\pi)$, 
	\begin{align*}
		\left\|\rho(\alpha) - \rho(\alpha')\right \|_2 \leq   \|\alpha - \alpha'\|_1.
	\end{align*}
\end{lemma}
\begin{proof}[Proof of \cref{lemma:rho}]
	To prove the lemma, we simply compute the Jacobian matrix of $\rho$ and provide a universal upper bound for its $\|\cdot \|_{1 \to 2}$ operator norm. 
	
Direct computation implies that for all $i \in [d - 2]$, 
	\begin{align*}
		\frac{\partial}{\partial \alpha_i} \rho(\alpha) = \prod_{j = 1}^{i - 1} \sin \alpha_j \cdot  \left(0, \cdots, 0, -\sin \alpha_i, \cos \alpha_i \cdot \rho (\pi / 2, \cdots, \pi / 2, \alpha_{i + 1}, \cdots, \alpha_{d - 2})_{(i + 1):(d - 1)}  \right)^{\top}. 
	\end{align*}
	The above vector has Euclidean norm no larger than $1$. Therefore, the matrix
	%
	%{\color{red} Use "," instead of $\vert$?}
	\begin{align*}
		\left[ \frac{\partial}{\partial \alpha_1} \rho(\alpha),  \cdots , \frac{\partial}{\partial \alpha_{d - 2}} \rho(\alpha) \right]
	\end{align*}
	has $\|\cdot\|_{1 \to 2}$ operator norm no larger than $1$ as well. This concludes the proof of the lemma. 
\end{proof}

\subsection{Proof of \cref{lemma:conditional-sector-test}}
\label{sec:proof-lemma:conditional-sector-test}

%Since $\alpha \sim \Unif[u_1, u_2)$, then with probability at least $1 - 6(u_2 - u_1)^{-1} \cdot |\Theta| \cdot (6 (r_d \delta_{sin})^{-1/2} n^{-10} + 20 d^{1/2}( \delta_{\sin} )^{-1/2} n^{-5/2} )$, the following event occurs:
%
%\begin{align}\label{eq:sevent}
%\begin{split}
%	&  \bigcap_{\theta \in \Theta} \left\{ \alpha, \alpha + \delta / 2, \alpha - \delta / 2 \notin \left( \alpha_i^{\theta_p } - 6(r_d \delta_{\sin})^{-1/2} n^{-10} - 20 d^{1/2}( \delta_{\sin} )^{-1/2} n^{-5/2}, \right. \right. \\
%	 & \left. \left.\alpha_{i}^{\theta_p} + 6(r_d \delta_{\sin})^{-1/2} n^{-10} + 20 d^{1/2}( \delta_{\sin} )^{-1/2} n^{-5/2} \right) \right\}.
%\end{split}
%\end{align} 
%
We consider two cases, separated by whether $\alpha_i^{\theta_p}$ falls inside the testing sector $(\alpha - \delta / 2, \alpha + \delta / 2)$. 
\subsubsection*{Case I:  $\alpha_i^{\theta_p}\in (\alpha - \delta / 2, \alpha + \delta / 2)$}

%In the first case, we assume that the target angle $\alpha_i^{\theta_p}$ lies inside the open interval $(\alpha - \delta / 2, \alpha + \delta / 2)$. 
In this case, it suffices to prove that the agent with probability at least $1 - n^{-4^{d + 3}}$ will report type 2 at least once in a total number of $\lceil \log n \rceil^4$ rounds when running \cref{alg:sector_dector2} line 2. 
Note that with probability at least $(1 - f_{\min})^{\lceil \log n \rceil^4}$, type $\theta_p$ appears at least once in $\lceil \log n \rceil^4$ rounds. When $n \geq \exp\big( {4^{(d + 3) / 3}f_{\min}^{-1 / 3}} \big)$, we have
\begin{align*}
    1 - (1 - f_{\min})^{\lceil  \log n \rceil^4  } \geq 1 - (1 - f_{\min})^{4^{d + 3} f_{\min}^{-1} \lceil \log n \rceil} \geq 1 - e^{-4^{d + 3} \lceil \log n \rceil} \geq 1 - n^{-4^{d + 3}}.
\end{align*}
Putting together the above discussions, we see that in order to establish the desired results, it suffices to prove that under the current assumptions, whenever $\theta_t = \theta_p$, the agent will report type 2 in round $t$ if the principal adopts coordination mechanism \cref{eq:simple-mechanism2}.

%To this end, it further suffices to prove whenever $\theta_t = \theta_p$, the agent will report type 2 in round $t$ if the principal adopt coordination mechanism \cref{eq:simple-mechanism2}. %The desired result then follows immediately.  

By \hyperlink{A4}{$\mathsf{(A4)}$} we have 
\begin{align}\label{eq:angle-diff-semi}
\begin{split}
	& \arc (\alpha_i^{\theta_p}, \alpha) \leq \arc(\alpha_i^{\theta_p}, \alpha + \delta) - 6(r_d \delta_{\sin})^{-1/2} n^{-4^{d + 2}} - 20 d^{1/2}( \delta_{\sin} )^{-1/2} n^{-4^{i + 3.5}}, \\
	&  \arc (\alpha_i^{\theta_p}, \alpha) \leq \arc(\alpha_i^{\theta_p}, \alpha - \delta) - 6(r_d \delta_{\sin})^{-1/2} n^{-4^{d + 2}} - 20 d^{1/2}( \delta_{\sin} )^{-1/2} n^{-4^{i + 3.5}}.
\end{split} 
\end{align}
Next, we define a mechanism based on true reward angles, and show that it is close to mechanism \cref{eq:simple-mechanism2}. Recall that $\xi_i$ is defined in \cref{eq:define_mapping_xi_i}. 
\begin{align}\label{eq:simple-mechanism2-star}
\begin{split}
	& \Pi_{1, t}^{\ast} = \mathds{1}_d / d + r_d \varphi_0^{-1}(\xi_i(\alpha - \delta, \alpha^{\theta_p}_{(i + 1): (d - 2)})), \\
	& \Pi_{2, t}^{\ast} = \mathds{1}_d / d + r_d \varphi_0^{-1}(\xi_i(\alpha, \alpha^{\theta_p}_{(i + 1): (d - 2)})), \\
	& \Pi_{3, t}^{\ast} = \mathds{1}_d / d + r_d \varphi_0^{-1}(\xi_i(\alpha + \delta, \alpha^{\theta_p}_{(i + 1): (d - 2)})), \\
	& \Pi_{\ell, t}^{\ast} = \mathds{1}_d / d + r_d \varphi_0^{-1}(\xi_i(\alpha_{i:(d - 2)}^{\sfb_n(j_{\ell - 3})})), \qquad  \mbox{for all } 4 \leq \ell \leq k + 3,  \\
	& \Pi_{\ell, t}^{\ast} = \Pi_{3, t}^{\ast}, \qquad \mbox{for all }k + 4 \leq \ell \leq |\Theta|,
\end{split}
\end{align}
Mechanism \eqref{eq:simple-mechanism2-star} is obtained by replacing the angle estimates with the corresponding truth in mechanism \eqref{eq:simple-mechanism2}. 
In the above display, we let $\cM = \{j_1, j_2, \cdots, j_k\}$. 
Without loss, in \eqref{eq:simple-mechanism2} for $4 \leq \ell \leq k + 3$ we let $\Pi_{\ell}^{\alpha, \delta, s} = \mathds{1}_d / d + r_d \varphi_0^{-1}(\xi_i(\hat \alpha_{i:(d - 2)}^{j_{\ell - 3}}))$. 
%To be precise, we replace $\balpha_{(i + 1):(d - 2)}$ with $\balpha_{(i + 1):(d - 2)}^{\theta_p}$ and replace $(\zeta_{j_{s - 3}^{(i,k)}, i})_{i:(d - 2)}$ with $\balpha_{i:(d - 2)}^{\sfb (j_{s - 3}^{(i,k)}) }$. 
With accurate enough estimation,
%of the previous coordinates ($\balpha_{(i + 1):(d - 2)}$) and the matched vectors ($\{\zeta_{j, i}: j \in \cM_{i,k}\}$), 
we expect the two mechanisms \eqref{eq:simple-mechanism2} and \eqref{eq:simple-mechanism2-star} are also close. We quantitatively delineate such closeness in \cref{lemma:mechanism-close} below.  
\begin{lemma}\label{lemma:mechanism-close}
	Under the conditions of \cref{lemma:conditional-sector-test}, for all $\theta \in \Theta$ and all $\ell \in \{4, 5, \cdots, k + 3\}$, it holds that  
\begin{align*}
	\big| \langle \bar{v}_{\theta}, \Pi_{{\ell}}^{\alpha, \delta, s} \rangle - \langle \bar{v}_{\theta}, \Pi_{{\ell}, t}^{\ast} \rangle  \big| \leq 4 r_d \cdot (d - 1 - i) \cdot n^{-4^{i + 3}}.  
\end{align*}
On the other hand, 
for all $\theta \in \Theta$ and $\ell \in \{1,2,3\} \cup \{k + 4, \cdots, |\Theta|\}$, 
\begin{align*}
	\big| \langle \bar{v}_{\theta}, \Pi_{{\ell}}^{\alpha, \delta, s} \rangle - \langle \bar{v}_{\theta}, \Pi_{{\ell}, t}^{\ast} \rangle  \big| \leq 4 r_d \cdot (d - 2 - i) \cdot n^{-4^{i + 4}}. 
\end{align*}
\end{lemma}
\begin{proof}[Proof of \cref{lemma:mechanism-close}]

By \hyperlink{A1}{$\mathsf{(A1)}$} and  \cref{lemma:rho}, we conclude that for all $\ell \in \{1,2,3\} \cup \{k + 4, \cdots, |\Theta|\}$,
\begin{align}\label{eq:14}
\begin{split}
	& \|\Pi_{{\ell}}^{\alpha, \delta, s} - \Pi_{{\ell}, t}^{\ast}\|_2 \\
	  &   \leq  r_d \cdot \sup_{a \in \{\alpha , \alpha \pm \delta\}} \|\rho(\pi / 2, \cdots, \pi / 2, a, \hat\alpha^s_{(i + 1): (d - 2)}) - \rho(\pi / 2, \cdots, \pi / 2, a, \alpha_{(i + 1): (d - 2)}^{\theta_p}) \|_2 \\
	  &  \leq  r_d \cdot \| \hat\alpha^s_{(i + 1): (d - 2)} - \alpha_{(i + 1): (d - 2)}^{\theta_p} \|_1  \leq  4 r_d \cdot (d - 2 - i) \cdot n^{-4^{i + 4}}. 
\end{split}
\end{align}
Analogously, using \hyperlink{A2}{$\mathsf{(A2)}$} and again \cref{lemma:rho}, we see that for all $\ell \in \{4, \cdots, k + 3\}$, 
\begin{align}\label{eq:15}
\begin{split}
	\|\Pi_{{\ell}}^{\alpha, \delta, s} - \Pi_{{\ell}, t}^{\ast}\|_2 = &  r_d \cdot \| \xi_i(\alpha_{i:(d - 2)}^{\sfb_n(j_{\ell - 3})}) - \xi_i(\hat \alpha_{i:(d - 2)}^{j_{\ell - 3}})  \|_2 \\
	\leq & 4 r_d \cdot  (d - 1 - i) \cdot  n^{-4^{i + 3}}.
\end{split} 
\end{align}
Recall that $\|\bar{v}_{\theta}\|_2 = 1$ by  definition. 
The lemma then follows via applying Cauchy-Schwartz inequality to \cref{eq:14,eq:15}. 
\end{proof}
In what follows, we compare $\langle \bar{v}_{\theta_p}, \Pi_{2}^{\alpha, \delta, s} \rangle$ and $\langle \bar{v}_{\theta_p}, \Pi_{{\ell}}^{\alpha, \delta, s} \rangle$ for all $\ell \in [|\Theta|] \backslash \{2\}$. We prove that the former is significantly larger than the maximum of the later over all $\ell \in [|\Theta|] \backslash \{2\}$, thus the agent will always report type $2$ whenever his true type is $\theta_p$. \cref{lemma:mechanism-close} suggests that we might instead compare $\langle \bar{v}_{\theta_p}, \Pi_{2, t}^{\ast} \rangle$ and $\langle \bar{v}_{\theta_p}, \Pi_{{\ell}, t}^{\ast} \rangle$ to simplify analysis.  

Below we discuss this point with  different values of $\ell$.  

\subsubsection*{(I-a). $\ell \in \{1,3\}\, \cup\,  \{s: k + 4 \leq s \leq |\Theta|, \, s \in \NN_+\}$}

First note that $\ell \geq k + 4$ by definition is equivalent to $\ell = 3$. Hence, it suffices to consider only $\ell \in \{1,3\}$. Direct computation implies that
\begin{align*}
& \langle \bar{v}_{\theta_p}, \Pi_{1, t}^{\ast} \rangle = r_d \cdot \cos(\alpha - \delta - \alpha_i^{\theta_p}) \cdot \prod_{j = 1}^{i - 1} \sin \alpha_j^{\theta_p}, \\
& \langle \bar{v}_{\theta_p}, \Pi_{2, t}^{\ast} \rangle = r_d \cdot \cos(\alpha - \alpha_i^{\theta_p}) \cdot \prod_{j = 1}^{i - 1} \sin \alpha_j^{\theta_p}, \\
& \langle \bar{v}_{\theta_p}, \Pi_{3, t}^{\ast} \rangle = r_d \cdot \cos(\alpha + \delta - \alpha_i^{\theta_p}) \cdot \prod_{j = 1}^{i - 1} \sin \alpha_j^{\theta_p}.
\end{align*}
Note that under case I we always have $\arc(\alpha, \alpha_i^{\theta_p}) \leq \arc(\alpha - \delta, \alpha_i^{\theta_p})$, hence $\cos(\alpha - \alpha_i^{\theta_p}) \geq \cos(\alpha - \delta - \alpha_i^{\theta_p})$, which further implies that 
\begin{align*}
	 \langle \bar{v}_{\theta_p}, \Pi_{2, t}^{\ast} \rangle - \langle \bar{v}_{\theta_p}, \Pi^{\ast}_{1, t} \rangle = \,& r_d \cdot( \cos( \alpha - \alpha_i^{\theta_p}) - \cos(\alpha - \delta  -  \alpha_i^{\theta_p}))  \cdot \prod_{j = 1}^{i - 1} \sin \alpha_j^{\theta_p}  \\
	 \geq \, &  r_d \delta_{\sin} \cdot \left( \cos(\alpha - \alpha_i^{\theta_p}) - \cos(\alpha - \delta - \alpha_i^{\theta_p}) \right). 
\end{align*}
By \cref{lemma:cos} and \cref{eq:angle-diff-semi}, we conclude that the last line of the above equation is  no smaller than
\begin{align*}
	\frac{r_d \delta_{\sin}}{30} \cdot \left( \arc(\alpha_i^{\theta_p}, \alpha ) - \arc(\alpha_i^{\theta_p}, \alpha - \delta) \right)^2, 
\end{align*}
which is further no smaller than $n^{-4^{d + 2.5}} + 4\pi r_d d  \cdot n^{-4^{i + 4}}$. By \cref{lemma:mechanism-close}, we further have 
\begin{align}\label{eq:report1}
	\langle \bar{v}_{\theta_p}, \Pi_{{2}}^{\alpha, \delta, s} \rangle - \langle \bar{v}_{\theta_p}, \Pi_{{1}}^{\alpha, \delta, s}\rangle \geq & \langle \bar{v}_{\theta_p}, \Pi_{2, t}^{\ast} \rangle - \langle \bar{v}_{\theta_p}, \Pi^{\ast}_{1, t} \rangle - 2 \max_{\ell \in \{1,2\}} \big| \langle \bar{v}_{\theta_p}, \Pi_{{\ell}}^{\alpha, \delta, s} \rangle - \langle \bar{v}_{\theta_p}, \Pi_{{\ell}, t}^{\ast} \rangle  \big| \nonumber \\
	\geq & \, n^{-4^{d + 2.5}} + 4\pi r_d d \cdot n^{-4^{i + 4}} - 8 r_d d \cdot n^{-4^{i + 4}} > n^{-4^{d + 3}}.
\end{align}
Similarly, we obtain that 
\begin{align}\label{eq:report2}
	\langle \bar{v}_{\theta_p}, \Pi_{{2}}^{\alpha, \delta, s} \rangle - \langle \bar{v}_{\theta_p}, \Pi_{{3}}^{\alpha, \delta, s}\rangle > n^{-4^{d + 3}}. 
\end{align}
%
%A direct consequence of the above inequality is that $\langle \bar{v}_{\theta_p}, \Pi_{\theta^2, t} \rangle - \langle \bar{v}_{\theta_p}, \Pi_{\theta^{\ell}, t}\rangle \geq n^{-20}$ for all $k + 4 \leq \ell \leq |\Theta|$. 

\subsubsection*{(I-b). $\ell \in \{l: 4 \leq l \leq k + 3, \, l \in \NN_+ \}$}

Next, we prove that $\langle \bar{v}_{\theta_p}, \Pi_{{2}}^{\alpha, \delta, s}  \rangle \geq \langle \bar{v}_{\theta_p}, \Pi_{{\ell}}^{\alpha, \delta, s}  \rangle + n^{-4^{d + 3}}$ for all $4 \leq \ell \leq k + 3$. Straightforward computation leads to the following equations: 
\begin{align}\label{eq:two-cos}
\begin{split}
	& \langle \bar{v}_{\theta_p}, \Pi_{2, t}^{\ast} \rangle = r_d \cos (\alpha_i^{\theta_p} - \alpha) \cdot \prod_{j = 1}^{i - 1} \sin \alpha_j^{\theta_p}, \\
	& \langle \bar{v}_{\theta_p}, \Pi_{\ell, t}^{\ast} \rangle  \leq r_d \cos(\alpha_i^{\theta_p} - \alpha_i^{\sfb_n(j_{\ell - 3})})\cdot \prod_{j = 1}^{i - 1} \sin \alpha_j^{\theta_p}. 
\end{split}
\end{align}
By \hyperlink{A1}{$\mathsf{(A1)}$} and Assumption \ref{assumption:angle} we know $\alpha_i^{\theta_p} \neq \alpha_i^{\sfb_n(j_{\ell - 3})}$, hence  $\arc(\alpha_i^{\sfb_n(j_{\ell - 3})},\alpha_i^{\theta_p}) \geq \bar\chi_i$, where we recall that $\bar\chi_i$ is from Definition \ref{def:chi}. 
Recall 
$\alpha_i^{\theta_p} \in (\alpha - \delta / 2, \alpha + \delta/ 2)$, hence $\arc(\alpha_i^{\theta_p}, \alpha) < \delta / 2$.
As a result, we get $\arc (\alpha_i^{\theta_p}, \alpha_i^{\sfb_n(j_{\ell - 3})}) - \arc(\alpha_i^{\theta_p}, \alpha) \geq \bar\chi_i - \delta / 2 $, which by \hyperlink{A6}{$\mathsf{(A6)}$} is larger than $\bar\chi_i / 2$. 
Using these results and applying \cref{lemma:cos}, we obtain that 
\begin{align*}
	 \langle \bar{v}_{\theta_p}, \Pi_{2, t}^{\ast} \rangle - \langle \bar{v}_{\theta_p}, \Pi_{\ell, t}^{\ast} \rangle \geq \,  &  \delta_{\sin}  r_d\cdot \big(\cos(\alpha_i^{\theta_p} - \alpha) - \cos(\alpha_i^{\theta_p} - \alpha_i^{\sfb_n(j_{\ell - 3})}) \big)  \\
	 \geq \, & \frac{\delta_{\sin} r_d}{30} \cdot \big( \arc(\alpha_i^{\theta_p}, \alpha) - \arc(\alpha_i^{\theta_p}, \alpha_i^{\sfb_n(j_{\ell - 3})}) \big)^2 \\
	 > \, & \frac{\delta_{\sin} r_d \bar\chi_i^2}{120} \geq   n^{-4^{d + 3}} + 8 r_d d n^{-4^{i + 3}},   
\end{align*}
where the last inequality is by \hyperlink{A7}{$\mathsf{(A7)}$}.
By \cref{lemma:mechanism-close}, 
\begin{align}
\label{eq:report3}
\begin{split}
    & \langle \bar{v}_{\theta_p}, \Pi_{2}^{\alpha, \delta, s} \rangle - \langle \bar{v}_{\theta_p}, \Pi_{\ell}^{\alpha, \delta, s} \rangle \\
    & \geq \langle \bar v_{\theta_p}, \Pi^{\ast}_{2, t} \rangle - \langle \bar v_{\theta_p}, \Pi^{\ast}_{\ell, t} \rangle - \max_{l \in \{1,2\}} \big| \langle \bar{v}_{\theta_p}, \Pi_{{l}}^{\alpha, \delta, s} \rangle - \langle \bar{v}_{\theta_p}, \Pi_{{l}, t}^{\ast} \rangle  \big| -\max_{4 \leq l \leq k + 3} \big| \langle \bar{v}_{\theta_p}, \Pi_{{l}}^{\alpha, \delta, s} \rangle - \langle \bar{v}_{\theta_p}, \Pi_{{l}, t}^{\ast} \rangle  \big| \\
    & \geq n^{-4^{d + 3}} + 8 r_d d n^{-4^{i + 3}} - 4r_d d n^{-4^{i + 3}} - 4r_d d n^{-4^{i + 4}} > n^{-4^{d + 3}}. 
\end{split}  
\end{align}
% Similar to how we derive \cref{eq:report1}, we can apply \cref{lemma:mechanism-close} and the triangle inequality, which implies the following lower bound: 
% %
% \begin{align}\label{eq:report3}
% 	\langle \bar{v}_{\theta_p}, \Pi_{2}^{\alpha, \delta, s} \rangle - \langle \bar{v}_{\theta_p}, \Pi_{\ell}^{\alpha, \delta, s} \rangle > n^{-4^{d + 3}}, \qquad \mbox{for all }4 \leq \ell \leq k + 3. 
% \end{align}
%
By \cref{eq:response-D1}, under \hyperlink{A7}{$\mathsf{(A7)}$}, the agent in round $t$ will report $\theta_t'$ only if $\langle \Pi_{\theta_t'}^{\alpha, \delta, s}, \bar{v}_{\theta_t} \rangle > \langle \Pi_{\theta}^{\alpha, \delta, s}, \bar v_{\theta_t} \rangle - \delta_{\sin} / n^{4^{d + 3}} > \langle \Pi_{\theta}^{\alpha, \delta, s}, \bar v_{\theta_t} \rangle  - 1 / n^{4^{d + 3}}$ for all $\theta \in \Theta$. 
Therefore, we arrive at the following conclusion: 
if $\theta_t = \theta_p$ in some round $t$ and $\alpha_i^{\theta_p} \in (\alpha - \delta / 2, \alpha + \delta / 2)$, then in this round the agent will report type $2$. 
Since the probability of the true type being $\theta_p$ is at least $f_{\min}$ by assumption, we can then deduce that the probability of the principal observing at least one $2$ in a total number of $\lceil \log n \rceil^4$ rounds is lower bounded by $1 - (1 - f_{\min})^{\lceil  \log n \rceil^4 }$, which is no smaller than $1 - n^{-4^{d + 3}}$ when $n \geq \exp\big( {4^{(d + 3) / 3}f_{\min}^{-1 / 3}} \big)$.

\subsubsection*{Case II: If $\alpha_i^{\theta_p}\notin (\alpha - \delta / 2, \alpha + \delta / 2)$}

In this case, $\alpha^{\theta_p}_i$ either falls inside the sector $(\alpha + \delta / 2, \pi)$, or falls inside the sector $(0, \alpha - \delta / 2)$ (by \hyperlink{A4}{$\mathsf{(A4)}$} and Assumption \ref{assumption:angle}, we know that $\alpha_i^{\theta_p} \notin \{\alpha \pm \delta / 2, \pi\}$). 
Without loss, here we assume $\alpha_i^{\theta_p} \in (0, \alpha - \delta / 2)$. The other situation can be handled analogously, and we skip it for the sake of simplicity. 

We shall prove that when $\alpha_i^{\theta_p} \in (0, \alpha - \delta / 2)$, the agent will never report type 2 regardless of his true type. This is in stark contrast to case I, in which we have proved that the agent with high probability will report type $2$ at least once in $\lceil \log n \rceil^4$ rounds. 
Our analysis is divided into two parts depending on the agent's true type $\theta_t$. In the first part, we assume there exists $j \in \cM$, such that $\sfb_n(j) = \theta_t$. We discuss all other situations in the second part.

\subsubsection*{(II-a): There exists $ j\in \cM$, such that $\sfb_n(j) = \theta_t$}

By \cref{eq:simple-mechanism2}, we know that there exists $\ell \in \{4, 5, \cdots, k + 3\}$, such that 
\begin{align*}
	\Pi_{\ell}^{\alpha, \delta, s} = \mathds{1}_d / d + r_d \cdot \varphi_0^{-1}(\xi_i (  \hat \alpha_{ i: (d - 2)}^{j} )).
\end{align*} 
Next, we show that $\bar{v}_{\theta_t}$ is more aligned with $\Pi_{\ell}^{\alpha, \delta, s}$ than $\Pi_{2}^{\alpha, \delta, s}$.
Note that 
\begin{align*}
	\Pi_{\ell, t}^{\ast} = \mathds{1}_d / d + r_d \cdot  \varphi_0^{-1}(\xi_i(\alpha_{i:(d - 2)}^{\sfb_n(j)})),
\end{align*}
which further implies that 
\begin{align}\label{eq:22}
	 \langle \bar{v}_{\theta_t},  \Pi^{\ast}_{\ell, t}  \rangle = r_d \cdot  \prod_{q = 1}^{i - 1} \sin \alpha^{\sfb_n(j)}_q. 
\end{align}
On the other hand, 
\begin{align}\label{eq:23}
	 & \langle \bar{v}_{\theta_t}, \Pi^{\ast}_{2, t} \rangle    =   r_d \cdot \langle \xi_i(\alpha_i^{\sfb_n(j)}, \alpha_{(i + 1): (d - 2)}^{\sfb_n(j)}), \, \xi_i(\alpha, \alpha_{(i + 1):(d - 2)}^{\theta_p})  \rangle \cdot \prod_{q = 1}^{i - 1} \sin \alpha^{\sfb_n(j)}_q  \\
	& \qquad  =  r_d \cdot \big( \cos \alpha_i^{\sfb_n(j)} \cos \alpha + \sin \alpha_i^{\sfb_n(j)} \sin \alpha \cdot \langle \xi_{i + 1}(\alpha_{(i + 1):(d - 2)}^{\sfb_n(j)}), \xi_{i + 1}(\alpha^{\theta_p}_{(i + 1):(d - 2)}) \rangle \big) \cdot \prod_{q = 1}^{i - 1} \sin \alpha^{\theta_t}_q. \nonumber %\nonumber \\
	 %\leq & r_d \cdot \cos(\alpha - \alpha_i^{\theta_t}) \cdot \prod_{j = 1}^{i - 1} \sin \alpha^{\theta_t}_j.  
\end{align}
By \hyperlink{A1}{$\mathsf{(A1)}$} we have $\sfb_n(j) \neq \theta_p$. Therefore, 
\begin{align}
\label{eq_67}
	 \langle \xi_{i + 1}(\alpha_{(i + 1):(d - 2)}^{\sfb_n(j)}), \xi_{i + 1}(\alpha^{\theta_p}_{(i + 1):(d - 2)}) \rangle  \leq  \cos (\alpha_{i + 1}^{\sfb_n(j)} - \alpha_{i + 1}^{\theta_p}) \leq 1 - \bar\chi_{i + 1}^2 / 30, 
\end{align}
where the last inequality is by \cref{lemma:cos}. 
%
%By Definition \ref{def:chi}, $\alpha_{i + 1}^{\sfb(j)}, \alpha_{i + 1}^{\theta_p} \in [\tilde \chi_{i + 1}, \pi - \tilde \chi_{i + 1}]$, $||\alpha_{i + 1}^{\sfb(j)} - \pi / 2| - |\alpha_{i + 1}^{\theta_p} - \pi / 2|| \geq \chi_{i + 1}$, and $|\alpha_{i + 1}^{\theta_p} - \alpha_{i + 1}^{\sfb(j)}| \geq \bar\chi_{i + 1}$. Using this result and \cref{lemma:cos}, we conclude that 
%
%\begin{align*}
%	\max \{ |\cos (\alpha_{i + 1}^{\sfb(j)} + \alpha_{i + 1}^{\theta_p})|, |\cos (\alpha_{i + 1}^{\sfb(j)} - \alpha_{i + 1}^{\theta_p})| \} \leq 1 - \frac{\min\{\chi_{i + 1}, \bar\chi_{i + 1}, \tilde\chi_{i + 1}\}^2}{30}. 
%\end{align*}
%

By \hyperlink{A3}{$\mathsf{(A3)}$}, it holds that $|\alpha - \pi / 2| - \delta / 2 \leq |\alpha_i^{\theta_p} - \pi / 2| \leq \pi / 2 - \tilde \chi_i$, where we recall that $\tilde \chi_i$ is from Definition \ref{def:chi}. Using this inequality and \hyperlink{A6}{$\mathsf{(A6)}$}, we see that $|\alpha - \pi / 2| \leq \pi / 2 - \tilde \chi_i / 2$. 
From \hyperlink{A6}{$\mathsf{(A6)}$} we also know that $\tilde \chi_i  < 10^{-1}$, hence $\sin \alpha \geq \sin (\tilde \chi_i / 2) \geq \tilde \chi_i / 4$. 
By definition $\sin \alpha_i^{\sfb_n(j)} \geq \delta_{\sin}$. 
%which further implies that $\sin \alpha \geq \sin (\tilde \chi_i / 2) \geq \tilde \chi_i / 4$ (here we use the fact that $\tilde \chi_i  < 10^{-1}$). 
Combining the above arguments, we arrive at the conclusion that 
\begin{align}
\label{eq_68}
    \sin \alpha \sin \alpha_i^{\sfb_n(j)} \geq \tilde \chi_i \delta_{\sin} / 4.
\end{align}
Putting together \cref{eq:22,eq:23,eq_67,eq_68}, we conclude that 
\begin{align*}
	\langle \bar{v}_{\theta_t}, \Pi_{\ell, t}^{\ast} \rangle - \langle \bar{v}_{\theta_t}, \Pi_{2, t}^{\ast} \rangle \geq   \frac{r_d \delta_{\sin}^2 \tilde \chi_i \bar\chi_{i + 1}^2}{120}, 
\end{align*}
which by \hyperlink{A7}{$\mathsf{(A7)}$} is no smaller than $n^{-4^{d + 3}} + 8 r_d d n^{-4^{i + 3}}$. By \cref{lemma:mechanism-close}, we further obtain that 
\begin{align*}
	\langle \bar{v}_{\theta_t}, \Pi_{\ell}^{\alpha, \delta, s} \rangle - \langle \bar{v}_{\theta_t}, \Pi_{2}^{\alpha, \delta, s} \rangle \geq n^{-4^{d + 3}}. 
\end{align*}
%
%
%Putting together the third and the fifth assumptions of the lemma, we obtain $\pi \geq |\alpha - \alpha_{i}^{\theta_t} | \geq \delta / 2$, thus by \cref{lemma:cos} 
%
%\begin{align*}
%	\cos(\alpha - \alpha_i^{\theta_t}) \leq 1 - \frac{\delta^2}{120}. 
%\end{align*}
%
%Therefore, by the seventh assumption of the lemma 
%
%\begin{align}\label{eq:24}
%\begin{split}
%	\langle \bar{v}_{\theta_t},  \Pi^{\ast}_{s, t}  \rangle - \langle \bar{v}_{\theta_t}, \Pi^{\ast}_{2, t} \rangle \geq & r_d \cdot \big( 1 -  \cos(\alpha - \alpha_i^{\theta_t}) \big) \cdot \prod_{j = 1}^{i - 1} \sin \alpha^{\theta_t}_j \\
%	 \geq & \frac{r_d \delta^2 \delta_{\sin}}{120} \geq   n^{-4^{d + 3}} + 2\pi r_d d \cdot n^{-4^{i + 3}}.
%\end{split}
%\end{align}
%
%Putting together \cref{eq:22,eq:23,eq:24} and \cref{lemma:mechanism-close}, we conclude that 
%
%\begin{align*}
%	\langle \bar{v}_{\theta_t},  \Pi_{s, t}  \rangle -  \langle \bar{v}_{\theta_t}, \Pi_{2, t} \rangle \geq & \langle \bar{v}_{\theta_t}, \Pi^{\ast}_{s, t} \rangle  -  \langle \bar{v}_{\theta_t}, \Pi^{\ast}_{2, t} \rangle - 2 \max_{\ell \in [|\Theta|]} \big| \langle \bar{v}_{\theta}, \Pi_{{\ell}, t} \rangle - \langle \bar{v}_{\theta}, \Pi_{{\ell}, t}^{\ast} \rangle  \big| \\
%	 \geq &  n^{-4^{d + 3}} + 2\pi r_d d \cdot n^{-4^{i + 3}} - 2\pi r_d d \cdot n^{-4^{i + 3}} \geq n^{-4^{d + 3}}. 
%\end{align*}
%
By \cref{eq:response-D1}, under \hyperlink{A7}{$\mathsf{(A7)}$}, the agent in round $t$ will report $\theta_t'$ only if $\langle \Pi_{\theta_t'}^{\alpha, \delta, s}, \bar{v}_{\theta_t} \rangle > \langle \Pi_{\theta}^{\alpha, \delta, s}, \bar v_{\theta_t} \rangle - \delta_{\sin} / n^{4^{d + 3}} > \langle \Pi_{\theta}^{\alpha, \delta, s}, \bar v_{\theta_t} \rangle  - 1 / n^{4^{d + 3}}$ for all $\theta \in \Theta$.
As a result, we are able to conclude that when $\theta_t = \sfb_n(j)$, the agent in round $t$ will never report type $2$. 

\subsubsection*{(II-b): There does not exist $j \in \cM$, such that $\sfb_n(j) = \theta_t$} 

We prove that in this case $\bar{v}_{\theta_t}$ is either more aligned with $\Pi_{1}^{\alpha, \delta, s}$ or is more aligned with $\Pi_{3}^{\alpha, \delta, s}$ comparing to $\Pi_{2}^{\alpha, \delta, s}$. 
We define
\begin{align*}
	F(a) := & \big \langle \bar{v}_{\theta_t},  \mathds{1}_d / d + r_d \varphi_0^{-1} (\xi_i(a, \alpha_{(i + 1):(d - 2)}^{\theta_p})) \big \rangle \\
	= &\, r_d \cdot \big( \cos a \cos \alpha_i^{\theta_t} + \sin a \sin \alpha_i^{\theta_t} \cdot \langle \xi_{i + 1}(\alpha_{(i + 1):(d - 2)}^{\theta_p}), \xi_{i + 1}(\alpha_{(i + 1):(d - 2)}^{\theta_t}) \rangle  \big) \cdot \prod_{j = 1}^{i - 1} \sin \alpha_j^{\theta_t}.  
\end{align*}
One can verify that $F(\alpha) = \langle \bar{v}_{\theta_t}, \Pi_{2, t}^{\ast} \rangle$, $F(\alpha - \delta) = \langle \bar{v}_{\theta_t}, \Pi_{1, t}^{\ast} \rangle$, and $F(\alpha + \delta) = \langle \bar{v}_{\theta_t}, \Pi_{3, t}^{\ast} \rangle$. In the next lemma, we prove that $\max\{F(\alpha + \delta), F(\alpha - \delta)\}$ is larger than $F(\alpha)$, and show that the agent in this case will not report type $2$. 
\begin{lemma}\label{lemma:B2}
	Under the conditions of \cref{lemma:conditional-sector-test} and in case (II-b), it holds that $\max\{F(\alpha + \delta), F(\alpha - \delta )\} \geq F(\alpha) + n^{-4^{d + 3}} + 8 r_d d n^{-4^{i + 4}}$. %This holds for all $\theta_t$ that falls into category (II-b). 
\end{lemma}
Putting together \cref{lemma:B2} and \cref{lemma:mechanism-close}, we conclude that
\begin{align*}
	\max\left\{  \langle \bar{v}_{\theta_t}, \Pi_{1}^{\alpha, \delta, s} \rangle, \langle \bar{v}_{\theta_t},\, \Pi_{3}^{\alpha, \delta, s} \rangle  \right\} - \langle \bar{v}_{\theta_t}, \Pi_{2}^{\alpha, \delta, s} \rangle  \geq n^{-4^{d + 3}}. 
\end{align*}
By \cref{eq:response-D1}, under \hyperlink{A7}{$\mathsf{(A7)}$}, the agent in round $t$ will report $\theta_t'$ only if $\langle \Pi_{\theta_t'}^{\alpha, \delta, s}, \bar{v}_{\theta_t} \rangle > \langle \Pi_{\theta}^{\alpha, \delta, s}, \bar v_{\theta_t} \rangle - \delta_{\sin} / n^{4^{d + 3}} > \langle \Pi_{\theta}^{\alpha, \delta, s}, \bar v_{\theta_t} \rangle  - 1 / n^{4^{d + 3}}$ for all $\theta \in \Theta$.
Hence, the agent in the current setting will never report type $2$. We give proof of \cref{lemma:B2} below.

\begin{proof}[Proof of \cref{lemma:B2}]
	If $\theta_t = \theta_p$, then $F(a) = r_d \cos (a - \alpha_i^{\theta_p}) \cdot \prod_{j = 1}^{i - 1} \sin \alpha_j^{\theta_p}$. Since $\alpha_i^{\theta_p} \notin (\alpha - \delta / 2, \alpha + \delta / 2)$ and we assume \hyperlink{A4}{$\mathsf{(A4)}$}, we can then conclude that 
	\begin{align*}
		& \min \{\arc(\alpha + \delta, \alpha_i^{\theta_p}), \arc(\alpha - \delta , \alpha_i^{\theta_p})\} \\
		  & \qquad \leq \arc(\alpha, \alpha_i^{\theta_p})  - 6(r_d \delta_{\sin})^{-1/2} n^{-4^{d + 2}} - 20 d^{1/2}( \delta_{\sin} )^{-1/2} n^{-4^{i + 3.5}}.
	\end{align*}
	Invoking \cref{lemma:cos}, we obtain that
	\begin{align*}
		\max\{F(\alpha + \delta), F(\alpha - \delta )\} - F(\alpha) \geq &\frac{r_d \delta_{\sin}}{30} \left(6(r_d \delta_{\sin})^{-1/2} n^{-4^{d + 2}} + 20 d^{1/2}( \delta_{\sin} )^{-1/2} n^{-4^{i + 3.5}} \right)^2 \\
		> & n^{-4^{d + 3}} + 4\pi r_d d n^{-4^{i + 4}}, 
	\end{align*}
	thus completing the proof of \cref{lemma:B2} when $\theta_t = \theta_p$.
On the other hand, if $\theta_t \neq \theta_p$, the following statements hold: 
By Definition \ref{def:chi}, 
when $i + 1 \leq d - 3$, we have  $|\alpha_{i + 1}^{\theta_p} - \alpha_{i + 1}^{\theta_t} - \pi| \geq \tilde \chi_{i + 1}$, $|\alpha_{i + 1}^{\theta_p} - \alpha_{i + 1}^{\theta_t}| \geq \chi_{i + 1}$, and $|\alpha_{i + 1}^{\theta_p} - \alpha_{i + 1}^{\theta_t} + \pi| \geq \tilde \chi_{i + 1}$. 
Therefore, by \cref{lemma:cos} we have $|\cos (\alpha_{i + 1}^{\theta_p} - \alpha_{i + 1}^{\theta_t})| \leq 1 - \min\{\chi_{i + 1}^2, \tilde \chi_{i + 1}^2\} / 30$. 
When $i + 1 \leq d - 3$, we have $|\alpha_{i + 1}^{\theta_t} + \alpha_{i + 1}^{\theta_p}| \geq \tilde\chi_{i + 1}$, $|\alpha_{i + 1}^{\theta_t} + \alpha_{i + 1}^{\theta_p} - \pi| \geq \chi_{i + 1}$, and $|\alpha_{i + 1}^{\theta_t} + \alpha_{i + 1}^{\theta_p} - 2\pi| \geq \tilde\chi_{i + 1}$. 
By \cref{lemma:cos}, we have $|\cos (\alpha_{i + 1}^{\theta_p} + \alpha_{i + 1}^{\theta_t})| \leq 1 - \min\{\chi_{i + 1}^2, \tilde \chi_{i + 1}^2\} / 30$. 
When $i + 1 = d - 2$, we have $|\alpha_{d - 2}^{\theta_p} - \alpha_{d - 2}^{\theta_t}| \geq 
\bar\chi_{i + 1}$, $|\alpha_{d - 2}^{\theta_p} - \alpha_{d - 2}^{\theta_t} - \pi| \geq \bar\chi_{i + 1}$, $|\alpha_{d - 2}^{\theta_p} - \alpha_{d - 2}^{\theta_t} - 2\pi| \geq \bar\chi_{i + 1}$, $|\alpha_{d - 2}^{\theta_p} - \alpha_{d - 2}^{\theta_t} + \pi| \geq \bar\chi_{i + 1}$ and $|\alpha_{d - 2}^{\theta_p} - \alpha_{d - 2}^{\theta_t} + 2\pi| \geq \bar\chi_{i + 1}$. 
Therefore, $|\cos(\alpha_{i + 1}^{\theta_p} - \alpha_{i + 1}^{\theta_t})| \leq 1 - \bar\chi_{i + 1}^2 / 30$.  Combining the above arguments, we obtain that 
	\begin{align}\label{eq:xi-inner-product}
	\begin{split}
		& \left| \langle \xi_{i + 1}(\alpha_{(i + 1):(d - 2)}^{\theta_p}), \xi_{i + 1}(\alpha_{(i + 1):(d - 2)}^{\theta_t}) \rangle \right| \\
		  &  \leq \max\left\{ |\cos(\alpha_{i + 1}^{\theta_p} - \alpha_{i + 1}^{\theta_t})|, \, |\cos(\alpha_{i + 1}^{\theta_p} + \alpha_{i + 1}^{\theta_t})|  \right\} \mathbbm{1}\{i < d - 3\} + \left| \cos(\alpha^{\theta_p}_{d - 2} - \alpha_{d - 2}^{\theta_t}) \right| \mathbbm{1}\{i = d - 3\} \\
		  &   \leq 1 - \frac{\min\{\chi_{i + 1}, \tilde\chi_{i + 1}\}^2}{30}  \mathbbm{1}\{i < d - 3\}  - \frac{\bar\chi_{i + 1}^2}{30} \mathbbm{1}\{i = d - 3\}. 
	\end{split}
	\end{align}
Taking the derivative of $F(\cdot)$, we get
\begin{align}\label{eq:Fdiffa}
	F'(a) = r_d \cdot \big( -\sin a \cot \alpha_i^{\theta_t} + \cos a \cdot \langle \xi_{i + 1}(\alpha_{(i + 1):(d - 2)}^{\theta_p}), \xi_{i + 1}(\alpha_{(i + 1):(d - 2)}^{\theta_t}) \rangle \big) \cdot \prod_{j = 1}^{i} \sin \alpha_j^{\theta_t}. 
\end{align}
By \hyperlink{A3}{$\mathsf{(A3)}$}, we know that 
\begin{align}
	|\alpha_i^{\theta_t} - \pi / 2| \geq \min \left\{ |\alpha - \delta / 2 - \pi / 2|, |\alpha + \delta / 2 - \pi / 2| \right\}. \label{eq:dizzy1}
\end{align}
For all $a \in [\alpha - \delta, \alpha + \delta]$, obviously
\begin{align}
	|a - \pi / 2| \leq \max\left\{ |\alpha - \delta - \pi / 2|, |\alpha + \delta - \pi / 2|  \right\}.  \label{eq:dizzy2}
\end{align}
From \cref{eq:dizzy1,eq:dizzy2}, we conclude that: (i) $|\alpha_i^{\theta_t} - \pi / 2| \geq |\alpha - \pi / 2| - \delta / 2$ and (ii) $|a - \pi / 2| \leq |\alpha - \pi / 2| + \delta$ for all $a \in [\alpha - \delta, \alpha + \delta]$. Putting these two arguments together, we obtain that for all $a \in [\alpha - \delta, \alpha + \delta]$, 
\begin{align*}
	|a - \pi / 2| \leq |\alpha_i^{\theta_t} - \pi / 2| + \frac{3\delta}{2}. 
\end{align*}
Invoking \cref{lemma:cot}, we get (recall $\tilde \chi_i$ is from Definition \ref{def:chi})
	\begin{align*}
		 |\cos a| & \leq \max \left\{ \frac{|a - \pi / 2|}{|\alpha_i^{\theta_t} - \pi / 2|}, \, 1\right\} \cdot |\cos \alpha_i^{\theta_t}|  \leq \frac{ \tilde\chi_i + 3\delta / 2}{\tilde\chi_i} \cdot |\cos \alpha_i^{\theta_t}| = \left(1 + \frac{3{\delta}}{2\tilde\chi_i} \right)\cdot | \cos \alpha_i^{\theta_t}|, \\
		 |\cot \alpha_i^{\theta_t} \cdot \sin a| &= |\cos \alpha_i^{\theta_t}| \cdot \frac{|\sin a|}{|\sin \alpha_i^{\theta_t}|} \geq |\cos \alpha_i^{\theta_t}| \cdot \min \left\{ \frac{ \pi / 2 - |a - \pi / 2|}{\pi / 2 - |\alpha_i^{\theta_t} -  \pi / 2|}, 1\right\} \\
		&  \geq |\cos \alpha_i^{\theta_t}| \cdot \frac{\pi / 2 - |\alpha_i^{\theta_t} - \pi / 2| - 3\delta / 2}{\pi / 2 - |\alpha_i^{\theta_t} - \pi / 2|} \geq |\cos \alpha_i^{\theta_t}| \cdot  \frac{\tilde \chi_i - 3\delta / 2}{\tilde \chi_i} 
	\end{align*}
	for all $a \in [\alpha - \delta, \alpha + \delta]$. 
    Plugging the above bounds and \cref{eq:xi-inner-product} into \cref{eq:Fdiffa}, we conclude that for all $a \in [\alpha - \delta, \alpha + \delta]$, 
	\begin{align*}
		|F'(a)| \geq &\, r_d \delta_{\sin} \cdot |\cos \alpha_i^{\theta_t}| \cdot \left( 1 - \mbox{$\frac{3\delta}{2\tilde \chi_i}$} - (1 + \mbox{$\frac{3\delta}{2\tilde \chi_i}$}) \cdot (1 - \mbox{$\frac{\min\{\chi_{i + 1}, \tilde\chi_{i + 1}\}^2}{30}  \mathbbm{1}\{i < d - 3\}  - \frac{\bar\chi_{i + 1}^2}{30} \mathbbm{1}\{i = d - 3\}$})\right) \\
		\geq & \, \frac{r_d \delta_{\sin} \tilde\chi_i}{2} \cdot \left( \mbox{$\frac{\min\{\chi_{i + 1}, \tilde\chi_{i + 1}, \bar\chi_{i + 1}\}^2}{30}$} - \mbox{$\frac{3\delta}{\tilde \chi_i}$} \right) \\
		\geq & \frac{r_d \delta_{\sin} \tilde \chi_i \min \{\chi_{i + 1},  \tilde\chi_{i + 1}, \bar\chi_{i + 1}\}^2}{120}, 
	\end{align*}
    where the last inequality follows from \hyperlink{A6}{$\mathsf{(A6)}$}. 
The above lower bound holds for all $a \in [\alpha - \delta, \alpha + \delta]$, and $F'(a)$ is continuous by \cref{eq:Fdiffa}, thus the sign of $F'(a)$ does not change on $[\alpha - \delta, \alpha + \delta]$. 
As a consequence, 
\begin{align*}
	\max\{F(\alpha + \delta), F(\alpha - \delta)\} \geq & F(\alpha) + \frac{r_d \delta  \delta_{\sin} \tilde \chi_i \min \{\chi_{i + 1},  \tilde\chi_{i + 1}, \bar\chi_{i + 1}\}^2}{120} \\
	 \geq &  F(\alpha) + n^{-4^{d + 3}} + 8 r_d d n^{-4^{i + 4}},
\end{align*}
%
%where the second lower bound is by the seventh assumption of the lemma. 
where the last inequality follows from \hyperlink{A7}{$\mathsf{(A7)}$}. 
This completes the proof of \cref{lemma:B2}.  
\end{proof}

\subsection{Proof of \cref{lemma:sector_dector3}}
\label{sec:proof-lemma:sector_dector3}

We introduce the following mechanism: 
\begin{align}\label{eq:oracle-mechanism}
\begin{split}
	& \Pi_{1, t}^{\ast} = \mathds{1}_d / d + r_d \varphi_0^{-1} (\xi_i(\alpha - \delta, \alpha_{(i + 1):(d - 2)}^{\theta_p})), \\
	& \Pi_{2, t}^{\ast} = \mathds{1}_d / d + r_d \varphi_0^{-1} (\xi_i(\alpha + \delta, \alpha_{(i + 1):(d - 2)}^{\theta_p})), \\
	& \Pi_{\ell, t}^{\ast} = \mathds{1}_d / d + r_d \varphi_0^{-1} (\xi_i(\alpha_{i:(d - 2)}^{\sfb_n(j_{\ell - 2})})), \qquad 3 \leq \ell \leq |\Theta|.
\end{split}
\end{align}
The above mechanism well approximates mechanism \eqref{eq:modified-mechanism}, which is detailed by the following lemma:  
\begin{lemma}\label{lemma:mechanism-close2}
	Under the conditions of \cref{lemma:sector_dector3}, for all $\theta \in \Theta$ and all $\ell \in \{3, 4, \cdots, |\Theta|\}$, it holds that 
	\begin{align*}
		\big| \langle \bar{v}_{\theta}, \Pi_{\ell}^{\alpha, \delta, s} \rangle - \langle \bar{v}_{\theta}, \Pi_{\ell, t}^{\ast} \rangle \big| \leq 4 r_d \cdot (d - 1 - i) \cdot n^{-4^{i + 3}}. 
	\end{align*}
	On the other hand, for all $\theta \in \Theta$ and $\ell \in \{1,2\}$, 
	\begin{align*}
		\big| \langle \bar{v}_{\theta}, \Pi_{\ell}^{\alpha, \delta, s} \rangle - \langle \bar{v}_{\theta}, \Pi_{\ell, t}^{\ast} \rangle \big| \leq 4 r_d \cdot (d - 2 - i) \cdot n^{-4^{i + 4}}, 
	\end{align*}
where we recall $\Pi_{\ell}^{\alpha, \delta, s}$ is from \eqref{eq:modified-mechanism}. 
\end{lemma}
The proofs of \cref{lemma:mechanism-close2}  and that of \cref{lemma:mechanism-close} are almost identical, hence we skip the proof of \cref{lemma:mechanism-close2} for the sake of simplicity. 
%For the sake of compactness, we skip proving \cref{lemma:mechanism-close2} and only prove \cref{lemma:mechanism-close} in this work. 
\cref{lemma:mechanism-close2} implies that we can switch to consider the ``oracle mechanism'' \eqref{eq:oracle-mechanism} without changing the results too much.  
With \cref{lemma:mechanism-close2}, we are then ready to prove \cref{lemma:sector_dector3}. 
	
%\cref{lemma:mechanism-close2} enables us to establish the following lemma. 
%
%\begin{lemma}\label{lemma:E4}
%	Under the conditions of \cref{lemma:sector_dector3}, the following statements are true:
	%
%
%\end{lemma} 

\begin{proof}[Proof of \cref{lemma:sector_dector3}]
Without loss of generality, in this proof we assume $\alpha > \pi/2$; the case $\alpha < \pi/2$ follows by symmetry.
By \cref{eq:response-D1} and \hyperlink{A4p}{$\mathsf{(A4')}$}, we see that during the implementation of \cref{alg:sector_dector3}, the agent in round $t$ will report $\theta_t'$ only if 
\begin{align}
\label{eq:new48}
    \langle \Pi_{t, \theta_t'}, \bar{v}_{\theta_t} \rangle \geq   \langle \Pi_{t, \theta}, \, \bar v_{\theta_t} \rangle  - \frac{C _0 \cdot \gamma^{\ell} }{1- \gamma } > \langle \Pi_{t, \theta}, \, \bar v_{\theta_t} \rangle  - \frac{\delta_{\sin}}{n^{4^{d + 3}}}, \qquad \forall \, \theta \in \Theta. 
\end{align}
\subsubsection*{Proof of claim 1}

In the setting of claim 1, by assumption $\theta_t \neq \theta_p$. 
By \cref{eq:new48}, in order to prove the first claim, it suffices to show that when $\alpha_i^{\theta_t} \in [\alpha - \delta, \pi]$ 
\begin{align}
\label{eq:target232}
    \langle \Pi_{2}^{\alpha, \delta, s}, \bar{v}_{\theta_t} \rangle \geq \langle \Pi_{1}^{\alpha, \delta, s}, \bar{v}_{\theta_t} \rangle + n^{-4^{d + 3}},
\end{align}
and when $\alpha_i^{\theta_t} \in [0, \pi - \alpha + \delta]$
\begin{align}\label{eq:target23}
	\langle \Pi_{1}^{\alpha, \delta, s}, \bar{v}_{\theta_t} \rangle \geq \langle \Pi_{2}^{\alpha, \delta, s}, \bar{v}_{\theta_t} \rangle + n^{-4^{d + 3}}. 
\end{align}
We prove the above claim in \cref{lemma:D6},and defer its proof to \cref{sec:proof-D6}. 
\begin{lemma}\label{lemma:D6}
	Under the conditions of \cref{lemma:sector_dector3} claim 1, if $\alpha_i^{\theta_t} \in [\alpha - \delta, \pi]$ then \cref{eq:target232} holds, and if $\alpha_i^{\theta_t} \in [0, \pi - \alpha + \delta]$ then \cref{eq:target23} holds.   
\end{lemma}
\subsubsection*{Proof of claim 2} 

To prove claim 2, by \cref{eq:new48} it suffices to show
\begin{align}\label{eq:target40}
	\langle \Pi_{2}^{\alpha, \delta, s}, \bar{v}_{\theta_p} \rangle \geq \langle \Pi_{\ell}^{\alpha, \delta, s}, \bar{v}_{\theta_p} \rangle + n^{-4^{d + 3}}
\end{align}
for all $\ell \in [|\Theta|] \backslash \{2\}$.
We divide the proof of \cref{eq:target40} into two cases based on the value of $\ell$. The case $\ell = 1$ is addressed in \cref{lemma:D5}, and the case $\ell \in [|\Theta|] \backslash \{1,2\}$ is addressed in \cref{lemma:D7}.

\begin{lemma}\label{lemma:D5}
	Under the conditions of \cref{lemma:sector_dector3} claim 2,  \cref{eq:target40} holds for $\ell = 1$.  
\end{lemma}
\begin{proof}[Proof of \cref{lemma:D5}]
    We prove \cref{lemma:D5} in Appendix \ref{sec:proof-D5}.  
\end{proof}

\begin{lemma}\label{lemma:D7}
	Under the conditions of \cref{lemma:sector_dector3} claim 2,  \cref{eq:target40} holds all $\ell \in [|\Theta|] \backslash \{1, 2\}$.  
\end{lemma}

\begin{proof}[Proof of \cref{lemma:D7}]
    We prove \cref{lemma:D7} in Appendix \ref{sec:proof-D7}.  
\end{proof}
  
\subsubsection*{Proof of claim 3}  

To prove claim 3, by \cref{eq:new48} we only need to prove 
\begin{align}\label{eq:target41}
	\langle \Pi_{1}^{\alpha, \delta, s}, \bar{v}_{\theta_p} \rangle \geq \langle \Pi_{\ell}^{\alpha, \delta, s}, \bar{v}_{\theta_p}\rangle + n^{-4^{d + 3}}
\end{align}
for all $\ell \in [|\Theta|] \backslash \{1\}$. \cref{eq:target41} from \cref{lemma:E7,lemma:E8}. 

\begin{lemma}\label{lemma:E7}
	Under the conditions of \cref{lemma:sector_dector3} claim 3, \cref{eq:target41} holds for $\ell = 2$. 
\end{lemma}
\begin{proof}[Proof of \cref{lemma:E7}]
    We prove \cref{lemma:E7} in Appendix \ref{sec:proof-E7}. 
\end{proof}

\begin{lemma}\label{lemma:E8}
	Under the conditions of \cref{lemma:sector_dector3} claim 3, \cref{eq:target41} holds for all $\ell \in [|\Theta|] \backslash \{1,2\}$. 
\end{lemma}
\begin{proof}[Proof of \cref{lemma:E8}]
    We prove \cref{lemma:E8} in Appendix \ref{sec:proof-E8}.
\end{proof}
  
\subsubsection*{Proof of claim 4} 

To prove claim 4, it suffices to show when $\theta_t = \sfb_n(j_{\ell})$, 
\begin{align}\label{eq:target42}
	\langle \Pi_{\ell + 2}^{\alpha, \delta, s}, \bar{v}_{\theta_t} \rangle  \geq \langle \Pi_{q}^{\alpha, \delta, s}, \bar{v}_{\theta_t} \rangle + n^{-4^{d + 3}}
\end{align}
for all $q \in [|\Theta|] \backslash \{\ell + 2\}$. 
Note that when $\theta_t = \sfb_n(j_{\ell})$, it holds that $\langle \Pi_{\ell + 2, t}^{\ast},  \bar{v}_{\theta_t}\rangle = r_d \cdot \prod_{j = 1}^{i - 1} \sin \alpha_j^{\theta_t}$. On the other hand, leveraging Cauchy-Schwartz inequality, we get
\begin{align}
\label{eq:three-line} 
\begin{split}
	& \langle \Pi_{q, t}^{\ast}, \bar{v}_{\theta_t} \rangle \leq r_d \cdot \cos(\alpha_i^{\sfb_n(j_{q - 2})} - \alpha_i^{\theta_t}) \cdot \prod_{j = 1}^{i - 1} \sin \alpha_j^{\theta_t}, \qquad \mbox{for all }q \in [|\Theta|] \backslash \{1, 2, \ell + 2\}, \\
	& \langle \Pi_{1, t}^{\ast}, \bar{v}_{\theta_t} \rangle \leq  r_d \cdot \cos (\alpha - \delta - \alpha_i^{\theta_t}) \cdot \prod_{j = 1}^{i - 1} \sin \alpha_j^{\theta_t}, \\
	& \langle \Pi_{2, t}^{\ast}, \bar{v}_{\theta_t} \rangle \leq  r_d \cdot \cos(\alpha + \delta - \alpha_i^{\theta_t}) \cdot \prod_{j = 1}^{i - 1}\sin \alpha_j^{\theta_t}. 
\end{split}
\end{align}
By \hyperlink{A3p}{$\mathsf{(A3')}$}, we have $|\alpha_i^{\theta_t} - \pi / 2| + 2 \delta + \delta^{1/4} \leq |\alpha - \pi / 2|$ and $\arc(\alpha_i^{\theta_t}, \alpha_i^{\sfb_n(j_{q - 2})}) \geq \bar \chi_i$ for all $q \in [|\Theta|] \backslash \{1, 2, \ell + 2\}$. Note that $\alpha, \alpha_i^{\theta_t}, \alpha_i^{\sfb_n(j_{q - 2})} \in [0, \pi]$, hence by \cref{lemma:cos} we have $\max \{\cos (\alpha + \delta - \alpha_i^{\theta_t}), \cos (\alpha - \delta - \alpha_i^{\theta_t})\} \leq 1 - \delta^{1/2} / 30$ and $\cos (\alpha_i^{\sfb_n(j_{q - 2})} - \alpha_i^{\theta_t})\leq 1 - \bar \chi_i^2 / 30$ for all $q \in [|\Theta|] \backslash \{1, 2, \ell + 2\}$.  
Combining these inequalities with \cref{eq:three-line}, we get 
\begin{align}
\label{eq:new80}
	\langle \Pi_{\ell + 2, t}^{\ast}, \bar{v}_{\theta_t} \rangle - \sup_{q \in [|\Theta|] \backslash \{\ell + 2\}} \langle \Pi_{q, t}^{\ast}, \bar{v}_{\theta_t} \rangle \geq \frac{r_d \delta_{\sin} \min \{\delta^{1/2}, \bar\chi_i^2\} }{30} \geq n^{-4^{d + 3}} + 8 r_d d n^{-4^{i + 3}},
\end{align}
where the last inequality follows from \hyperlink{A4p}{$\mathsf{(A4')}$}. 
\cref{eq:target42} then follows from \cref{eq:new80} and \cref{lemma:mechanism-close2}. 

\subsubsection*{Proof of claim 5}

In this case $\theta_t = \theta_p$. 
To prove claim 5, by \cref{eq:new48} we only need to show that when $\alpha_i^{\theta_t} \in [0, \alpha)$, 
\begin{align}
\label{eq:target82}
    \langle \Pi_{1}^{\alpha, \delta, s}, \bar v_{\theta_t} \rangle \geq \langle \Pi_2^{\alpha, \delta, s}, \bar v_{\theta_t} \rangle + n^{-4^{d + 3}}, 
\end{align}
and when $\alpha_i^{\theta_t} \in (\alpha, \pi]$, 
\begin{align}
\label{eq:target83}
    \langle \Pi_2^{\alpha, \delta, s}, \bar v_{\theta_t} \rangle \geq \langle \Pi_1^{\alpha, \delta, s}, \bar v_{\theta_t} \rangle + n^{-4^{d + 3}}. 
\end{align}
We prove the desired claim in \cref{lemma:new-add}.  

\begin{lemma}
\label{lemma:new-add}
Under the conditions of \cref{lemma:sector_dector3} claim 5,  \cref{eq:target82} holds when $\alpha_i^{\theta_t} \in [0, \alpha)$, 
and \cref{eq:target83} holds when $\alpha_i^{\theta_t} \in (\alpha, \pi]$. 
\end{lemma}

\begin{proof}[Proof of \cref{lemma:new-add}]
    We prove \cref{lemma:new-add} in Section \ref{sec:proof-lemma:new-add}. 
\end{proof}

\end{proof}

\subsection{Proof of \cref{lemma:8.5}}\label{sec:proof-8.5}

By Assumption \ref{assumption:angle}, we know that $|c_{\ast}| < 1$. 
If $c_{\ast} \geq 0$, then $c_{\ast} + 2e_{\ast} = c_{\ast} + 2(\log n)^{-1}$. 
If $c_{\ast} < 0$, then $c_{\ast} + 2e_{\ast} = -c_{\ast} + 2(\log n)^{-1}$. 
In this proof, we assume a sufficiently large $n$ such that $c_{\ast} + 2e_{\ast} \in [2(\log n)^{-1}, 1)$. 
Similar to the derivation of \cref{eq:response-D1}, we assume that $n$ is sufficiently large, such that the agent in round $t$ reports $\theta_t'$ only if 
\begin{align*}
        \langle \Pi_{t, \theta_t'}, \bar{v}_{\theta_t} \rangle > \langle \Pi_{t, \theta}, \, \bar v_{\theta_t} \rangle  - \frac{\delta_{\sin}}{n^{4^{d + 3}}}, \qquad \forall \, \theta \in \Theta.
\end{align*}

\subsubsection*{Proof of claim 1}

Recall that 
\begin{align*}
	& c_{\ast} = \max_{j \in [|\Theta| - 1]} \big\langle \xi_{i + 1}(\alpha_{(i + 1):(d - 2)}^{\sfb_n(s_{|\Theta|})}), \, \xi_{i + 1} (\alpha_{(i + 1):(d - 2)}^{\sfb_n(s_j)}) \big\rangle, \\
	& c = \max_{j \in [|\Theta| - 1]} \big \langle \xi_{i + 1}(\hat\alpha_{(i + 1):(d - 2)}^{{s}_{|\Theta|}}),\, \xi_{i + 1} (\hat\alpha_{(i + 1):(d - 2)}^{{s}_{j}}) \big \rangle. 
\end{align*}
We first show that $c_{\ast}$ and $c$ are close. 
By assumption, $\|\alpha_{(i + 1):(d - 2)}^{\sfb_n(s)} - \hat\alpha_{(i + 1):(d - 2)}^s\|_1 \leq 4(d - 2 - i)n^{-4^{i + 4}}$ for all $s \in \Theta$. 
By \cref{lemma:rho}, for all $s \in \Theta$
\begin{align*}
    & \|\xi_{i + 1}(\alpha_{(i + 1):(d - 2)}^{\sfb_n(s)}) - \xi_{i + 1}(\hat\alpha_{(i + 1):(d - 2)}^s)\|_2 \\
    & = \|\rho(\pi / 2, \cdots, \pi / 2, \alpha_{i + 1}^{\sfb_n(s)}, \cdots, \alpha_{i + 1}^{\sfb_n(s)}) - \rho(\pi / 2, \cdots, \pi / 2, \hat\alpha_{i + 1}^{s}, \cdots, \hat\alpha_{i + 1}^{s}) \|_2 \\
    & \leq \|\alpha_{(i + 1):(d - 2)}^{\sfb_n(s)} - \hat\alpha_{(i + 1):(d - 2)}^s\|_1 \\
    & \leq 4(d - 2 - i)n^{-4^{i + 4}}.
\end{align*}
Therefore, for all $j \in [|\Theta| - 1]$, 
\begin{align*}
    & \Big| \big\langle \xi_{i + 1}(\alpha_{(i + 1):(d - 2)}^{\sfb_n(s_{|\Theta|})}), \, \xi_{i + 1} (\alpha_{(i + 1):(d - 2)}^{\sfb_n(s_j)}) \big\rangle - \big \langle \xi_{i + 1}(\hat\alpha_{(i + 1):(d - 2)}^{{s}_{|\Theta|}}),\, \xi_{i + 1} (\hat\alpha_{(i + 1):(d - 2)}^{{s}_{j}}) \big \rangle \Big| \\
    & \leq \|\xi_{i + 1}(\alpha_{(i + 1):(d - 2)}^{\sfb_n(s_{|\Theta|})}) - \xi_{i + 1}(\hat\alpha_{(i + 1):(d - 2)}^{s_{|\Theta|}})\|_2 + \|\xi_{i + 1}(\alpha_{(i + 1):(d - 2)}^{\sfb_n(s_{j})}) - \xi_{i + 1}(\hat\alpha_{(i + 1):(d - 2)}^{s_{j}})\|_2 \\
    & \leq 8(d - 2 - i) n^{-4^{i + 4}}. 
\end{align*}
Hence, $|c - c_{\ast}| \leq 8 (d - 2 - i) n^{-4^{i + 4}}$, which further implies that $|e - e_{\ast}| \leq 8 (d - 2 - i) n^{-4^{i + 4}}$.

Note that when $\theta_t = \sfb_n(s_j)$ for $1 \leq j \leq |\Theta| - 1$, 
\begin{align*}
	& \langle \Pi_{j, t}^{\ast}, \bar{v}_{\theta_t} \rangle = \bar r_d \cdot \prod_{q = 1}^{i} \sin \alpha_q^{\sfb_n(s_j)}, \\
	& \langle \Pi_{l, t}^{\ast},  \bar{v}_{\theta_t}  \rangle = \bar r_d \cdot \prod_{q = 1}^{i} \sin \alpha_q^{\sfb_n(s_j)} \cdot \langle \xi_{i + 1} (\alpha_{(i + 1):(d - 2)}^{\sfb_n(s_j)}),\, \xi_{i + 1}(\alpha_{(i + 1):(d - 2)}^{\sfb_n(s_l)}) \rangle, \qquad 1 \leq l \leq |\Theta| - 1, \, l \neq j, \\
	& \langle \Pi_{|\Theta|, t}^{\ast}, \bar{v}_{\theta_t} \rangle = \bar r_d \cdot \prod_{q = 1}^{i - 1} \sin \alpha_q^{\sfb_n(s_j)} \cdot \left( x \delta \cos \alpha_i^{\sfb_n(s_j)} - e \sin \alpha_i^{\sfb_n(s_j)}\langle \xi_{i + 1} (\alpha_{(i + 1):(d - 2)}^{\sfb_n(s_j)}), \xi_{i + 1}(\alpha_{(i + 1):(d - 2)}^{\sfb_n(s_{|\Theta|})}) \rangle \right).  
\end{align*}
By \cref{eq:xi-inner-product}, we know that $|\langle \xi_{i + 1} (\alpha_{(i + 1):(d - 2)}^{\sfb_n(s_j)}), \xi_{i + 1}(\alpha_{(i + 1):(d - 2)}^{\sfb_n(s_l)}) \rangle| \leq 1 - {\min\{\chi_{i + 1}, \bar\chi_{i + 1}, \tilde \chi_{i + 1}\}^2} /{30}$, hence 
\begin{align*}
	\langle \Pi_{j, t}^{\ast}, \bar{v}_{\theta_t} \rangle - \langle \Pi_{l, t}^{\ast},  \bar{v}_{\theta_t}  \rangle \geq \frac{\bar r_d \delta_{\sin} \min\{\chi_{i + 1}, \bar\chi_{i + 1}, \tilde\chi_{i + 1}\}^2}{30}, \qquad  1 \leq l \leq |\Theta| - 1, \, l \neq j, 
\end{align*}
which is no smaller than $n^{-4^{d + 3}} + 8 \bar r_d d n^{-4^{i + 4}}$ when $n$ is large enough. Putting together the above lower bound, \cref{lemma:mclose} and the triangle inequality, we conclude that $\langle \Pi_{j, t}, \bar{v}_{\theta_t} \rangle - \langle \Pi_{l, t},  \bar{v}_{\theta_t}  \rangle \geq n^{-4^{d + 3}}$ for all $l \in \{1, 2, \cdots, |\Theta| - 1\} \backslash \{j\}$. 

On the other hand, 
\begin{align*}
	& \langle \Pi_{j, t}^{\ast}, \bar{v}_{\theta_t} \rangle - \langle \Pi_{|\Theta|, t}^{\ast}, \bar{v}_{\theta_t} \rangle \\
	 = & \,\bar r_d \cdot \prod_{q = 1}^{i - 1} \sin \alpha_q^{\sfb_n(s_j)} \cdot \left( \sin \alpha_i^{\sfb_n(s_j)} - x\delta \cos \alpha_i^{\sfb_n(s_j)} + e \sin \alpha_i^{\sfb_n(s_j)} \cdot \langle \xi_{i + 1} (\alpha_{(i + 1):(d - 2)}^{\sfb_n(s_j)}), \xi_{i + 1}(\alpha_{(i + 1):(d - 2)}^{\sfb_n(s_{|\Theta|})}) \rangle\right) \\ 
     \overset{(i)}{\geq} & \, \bar r_d \cdot \prod_{q = 1}^{i - 1} \sin \alpha_q^{\sfb_n(s_j)} \cdot \left( \sin \alpha_i^{\sfb_n(s_j)} - (c_{\ast} + e_{\ast}) \sin \alpha_i^{\sfb_n(s_j)} + e \sin \alpha_i^{\sfb_n(s_j)} \cdot \langle \xi_{i + 1} (\alpha_{(i + 1):(d - 2)}^{\sfb_n(s_j)}), \xi_{i + 1}(\alpha_{(i + 1):(d - 2)}^{\sfb_n(s_{|\Theta|})}) \rangle \right) \\
	 \overset{(ii)}{\geq} & \, \bar r_d \cdot \prod_{q = 1}^{i - 1} \sin \alpha_q^{\sfb_n(s_j)}\left( \sin \alpha_i^{\sfb_n(s_j)} - (c_{\ast} +  e_{\ast} ) \sin \alpha_i^{\sfb_n(s_j)} - e_{\ast} \sin \alpha_i^{\sfb_n(s_j)} - 8 d n^{-4^{i + 4}} \sin \alpha_i^{\sfb_n(s_j)} \right) \\
	 \geq & \, \bar r_d \delta_{\sin}  \cdot \left( 1 - c_{\ast} - 2e_{\ast} - 8 d n^{-4^{i + 4}}\right), %\\
	 %\geq & \frac{r_d \delta_{\sin}\tilde \chi_i \min \{\chi_{i + 1}, \tilde\chi_{i + 1}, \bar\chi_{i + 1}\}^2 }{120(\log n)} \geq n^{-4^{d + 3}} + 2\pi r_d d n^{-4^{i + 4}}. 
\end{align*}
In the above display, $(i)$ is because $\delta \leq \min_{j \in [|\Theta| - 1]} |\tan(\alpha_i^{\sfb_n(s_j)})| \cdot (c_{\ast} + e_{\ast})$, hence $-x\delta \cos \alpha_i^{\sfb_n(s_j)} \geq - (c_{\ast} + e_{\ast}) \sin \alpha_i^{\sfb_n(s_j)}$; 
$(ii)$ is because $|e - e_{\ast}| \leq 8(d - 2 - i)n^{-4^{i + 4}}$. 
Note that there exists $\omega > 0$ that is independent of $n$, such that $c_{\ast} + 2e_{\ast} < 1 - \omega$.  
For a sufficiently large $n$, we have $\bar r_d \delta_{\sin}  \cdot ( 1 - c_{\ast} - 2e_{\ast} - 8 d n^{-4^{i + 4}}) \geq n^{-4^{d + 3}} + 8 \bar r_d d n^{-4^{i + 4}}$. 
Using \cref{lemma:mclose}, we further get  $\langle \Pi_{j, t}, \bar{v}_{\theta_t} \rangle - \langle \Pi_{|\Theta|, t}, \bar{v}_{\theta_t} \rangle \geq n^{-4^{d + 3}}$. Putting together these results, we conclude that in the current setting the agent reports type $j$. 
The proof is done.

\subsubsection*{Proof of claim 2}

When $\theta_t = \sfb_n(s_{|\Theta|})$, we have 
\begin{align*}
	& \langle \Pi_{|\Theta|, t}^{\ast}, \bar{v}_{\sfb_n(s_{|\Theta|})} \rangle = \bar r_d \cdot \prod_{q = 1}^{i - 1} \sin \alpha_q^{\sfb_n(s_{|\Theta|})}  \cdot \left( x\delta \cos \alpha_i^{\sfb_n(s_{|\Theta|})} - e \sin \alpha_i^{\sfb_n(s_{|\Theta|})} \right), \\
	& \langle \Pi_{l, t}^{\ast}, \bar{v}_{\sfb_n(s_{|\Theta|})}  \rangle = \bar r_d \cdot \prod_{q = 1}^{i} \sin \alpha_q^{\sfb_n(s_{|\Theta|})}  \cdot \langle \xi_{i + 1}(\alpha_{(i + 1):(d - 2)}^{\sfb_n(s_l)}), \, \xi_{i + 1}(\alpha_{(i + 1):(d - 2)}^{\sfb_n(s_{|\Theta|})}) \rangle, \qquad 1 \leq l \leq |\Theta| - 1.  
\end{align*}
%
%By \cref{eq:xi-inner-product}, we know that $|\langle \xi_{i + 1} (\balpha_{(i + 1):(d - 2)}^{\sfb(s_{|\Theta|})}), \xi_{i + 1}(\balpha_{(i + 1):(d - 2)}^{\sfb(s_l)}) \rangle| \leq 1 - \min\{\chi_{i + 1}, \bar\chi_{i + 1}, \tilde\chi_{i + 1}\}^2 / 30$, hence
Recall that $x = \sign(-\alpha_i^{\sfb_n(s_{|\Theta|})} + \pi / 2)$ and $\delta \geq |\tan(\alpha_i^{\sfb_n(s_{|\Theta|})})| \cdot (c_{\ast} + e_{\ast}  + n^{-4^{i + 3.5}})$. 
Therefore, $x\delta \cos \alpha_i^{\sfb_n(s_{|\Theta|})} \geq (c_{\ast} + e_{\ast}  + n^{-4^{i + 3.5}}) \sin \alpha_i^{\sfb_n(s_{|\Theta|})}$. 
As a consequence, for all $1 \leq l \leq |\Theta| - 1$, 
\begin{align*}
	& \langle \Pi_{|\Theta|, t}^{\ast}, \bar{v}_{\sfb_n(s_{|\Theta|})} \rangle - \langle \Pi_{l, t}^{\ast}, \bar{v}_{\sfb_n(s_{|\Theta|})}  \rangle \\
	 & = \, \bar r_d \prod_{q = 1}^{i - 1} \sin \alpha_q^{\sfb_n(s_{|\Theta|})} \cdot \left( x\delta \cos \alpha_i^{\sfb_n(s_{|\Theta|}) } - e \sin  \alpha_i^{\sfb_n(s_{|\Theta|}) } - \sin  \alpha_i^{\sfb_n(s_{|\Theta|})} \cdot \langle \xi_{i + 1}(\alpha_{(i + 1):(d - 2)}^{\sfb_n(s_l)}), \xi_{i + 1}(\alpha_{(i + 1):(d - 2)}^{\sfb_n(s_{|\Theta|})}) \rangle \right) \\
	 & \geq \bar r_d \prod_{q = 1}^{i} \sin \alpha_q^{\sfb_n(s_{|\Theta|})} \cdot  \left(c_{\ast} + e_{\ast} + n^{-4^{i + 3.5}} - e - \langle \xi_{i + 1}(\alpha_{(i + 1):(d - 2)}^{\sfb_n(s_l)}), \xi_{i + 1}(\alpha_{(i + 1):(d - 2)}^{\sfb_n(s_{|\Theta|})}) \rangle \right)   \\
	  &  \overset{(i)}{\geq} \bar r_d \delta_{\sin} \cdot \left( n^{-4^{i + 3.5}} - 8 r_d d n^{-4^{i + 4}} \right) \\
      & \overset{(ii)}{\geq}  n^{-4^{d + 3}} + 8 r_d d n^{-4^{i + 4}}, 
\end{align*}
where $(i)$ is because for a sufficiently large $n$, 
\begin{align*}
    & c_{\ast} + e_{\ast} + n^{-4^{i + 3.5}} - e - \langle \xi_{i + 1}(\alpha_{(i + 1):(d - 2)}^{\sfb_n(s_l)}), \xi_{i + 1}(\alpha_{(i + 1):(d - 2)}^{\sfb_n(s_{|\Theta|})}) \rangle \\
    & \geq - |c_{\ast} - c| - |e_{\ast} - e| + n^{-4^{i + 3.5}} \\
    & \geq n^{-4^{i + 3.5}} - 16d n^{-4^{i + 4}}, 
\end{align*}
and $(ii)$ is because $i \leq d - 2$. 
Putting together the above lower bound and \cref{lemma:mclose}, we conclude that 
\begin{align*}
	\langle \Pi_{|\Theta|, t}, \bar{v}_{\sfb_n(s_{|\Theta|})} \rangle - \langle \Pi_{l, t}, \bar{v}_{\sfb_n(s_{|\Theta|})}  \rangle \geq n^{-4^{d + 3}}
\end{align*}
for all $l \in [|\Theta| - 1]$. 
This concludes the proof of the second claim. 

\subsubsection*{Proof of claim 3} 

%We first prove $1 - e_{\ast} \geq c_{\ast} + e_{\ast} - n^{-4^{i + 3.5}}$. To this end, we consider two cases. In the first case we assume $c_{\ast} \geq 0$, then by definition $e_{\ast} = (\log n)^{-1}$.
%By \cref{eq:xi-inner-product} we have $|c_{\ast}| \leq 1 - \min \{\chi_{i + 1}, \bar\chi_{i + 1}\}^2 / 30$, hence in this case, for large enough $n$, it holds that $1 - e_{\ast} \geq c_{\ast} + e_{\ast}$ 
%On the other hand, if $c_{\ast} < 0$, then $e_{\ast} = -c_{\ast} + (\log n)^{-1}$. As a result, we obtain $e_{\ast} + c_{\ast} = (\log n)^{-1}$ and $1 - e_{\ast} \geq \min \{\chi_{i + 1}, \bar\chi_{i + 1}\}^2 / 30 - (\log n)^{-1}$, hence $1 - e_{\ast} \geq c_{\ast} + e_{\ast}$ if we choose $n_0$ large enough.  

Since $|\tan \alpha_i^{\sfb_n(s_{|\Theta|})}| < \min_{j \in [|\Theta| - 1]} |\tan \alpha_i^{\sfb_n(s_j)}|$, 
we conclude that if $\delta \leq |\tan (\alpha_i^{\sfb_n(s_{|\Theta|})})| \cdot (c_{\ast} + e_{\ast} - n^{-4^{i + 3.5}})$, 
then $\delta \leq \min_{j \in [|\Theta| - 1]} |\tan(\alpha_i^{\sfb_n(s_j)})| \cdot (c_{\ast} + e_{\ast})$. 
By the first claim of the lemma, we see that if $\theta_t = \sfb_n(s_j)$ for some $j \in [|\Theta| - 1]$,
then the agent in round $t$ reports type $j$. 
Next, we show that if $\theta_t = \sfb_n(s_{|\Theta|})$, then the agent will not report type $|\Theta|$. To this end, it suffices to prove 
\begin{align}\label{eq:target-345}
	\langle \Pi^{\ast}_{l_{\ast}, t}, \bar{v}_{\sfb_n(s_{|\Theta|})} \rangle - \langle \Pi^{\ast}_{|\Theta|, t}, \bar{v}_{\sfb_n(s_{|\Theta|})} \rangle \geq n^{-4^{d + 3}} + 8 \bar r_d d n^{-4^{i + 4}}, 
\end{align}
where $l_{\ast} = \argmax_{1 \leq l \leq |\Theta| - 1}  \langle \xi_{i + 1}(\alpha_{(i + 1):(d - 2)}^{\sfb_n(s_l)}), \xi_{i + 1}(\alpha_{(i + 1):(d - 2)}^{\sfb_n(s_{|\Theta|})}) \rangle$.
Note that 
\begin{align*}
	& \langle \Pi_{|\Theta|, t}^{\ast}, \bar{v}_{\sfb_n(s_{|\Theta|})} \rangle = \bar r_d \cdot \prod_{q = 1}^{i - 1} \sin \alpha_q^{\sfb_n(s_{|\Theta|})}  \cdot \left( x\delta \cos \alpha_i^{\sfb_n(s_{|\Theta|})} - e \sin \alpha_i^{\sfb_n(s_{|\Theta|})} \right), \\
	& \langle \Pi_{l_{\ast}, t}^{\ast}, \bar{v}_{\sfb_n(s_{|\Theta|})}  \rangle = \bar r_d \cdot \prod_{q = 1}^{i} \sin \alpha_q^{\sfb_n(s_{|\Theta|})}  \cdot \langle \xi_{i + 1}(\alpha_{(i + 1):(d - 2)}^{\sfb_n(s_{l_{\ast}})}), \xi_{i + 1}(\alpha_{(i + 1):(d - 2)}^{\sfb_n(s_{|\Theta|})}) \rangle = \bar r_d c_{\ast}  \prod_{q = 1}^{i} \sin \alpha_q^{\sfb_n(s_{|\Theta|})}. 
\end{align*}
Since $x = \sign(-\alpha_i^{\sfb_n(s_{|\Theta|})} + \pi / 2)$ and $ \delta \leq |\tan (\alpha_i^{\sfb_n(s_{|\Theta|})})| (c_{\ast} + e_{\ast} - n^{-4^{i + 3.5}})$, 
we have $-x\delta \cos\alpha_i^{\sfb_n(s_{|\Theta|})}  \geq -(c_{\ast} + e_{\ast} - n^{-4^{i + 3.5}}) \sin \alpha_i^{\sfb_n(s_{|\Theta|})}$. 
Combining the above analysis, we see that for a sufficiently large $n$, 
\begin{align*}
	& \langle \Pi_{l_{\ast}, t}^{\ast}, \bar{v}_{\sfb_n(s_{|\Theta|})}  \rangle - \langle \Pi_{|\Theta|, t}^{\ast}, \bar{v}_{\sfb_n(s_{|\Theta|})} \rangle \\
	 & = \bar r_d \cdot \prod_{q = 1}^{i - 1} \sin \alpha_q^{\sfb_n(s_{|\Theta|})} \cdot \left( c_{\ast} \sin \alpha_i^{\sfb_n(s_{|\Theta|})} - x\delta \cos \alpha_i^{\sfb_n(s_{|\Theta|})} + e \sin \alpha_i^{\sfb_n(s_{|\Theta|})}  \right) \\
 & \geq \bar r_d \cdot \prod_{q = 1}^{i - 1} \sin \alpha_q^{\sfb_n(s_{|\Theta|})} \cdot \left( c_{\ast} \sin \alpha_i^{\sfb_n(s_{|\Theta|})} - (c_{\ast} + e_{\ast} - n^{-4^{i + 3.5}}) \sin \alpha_i^{\sfb_n(s_{|\Theta|})} + e \sin \alpha_i^{\sfb_n(s_{|\Theta|})} \right) \\
 & \geq \bar r_d \cdot \prod_{q = 1}^{i} \sin \alpha_q^{\sfb_n(s_{|\Theta|})} \cdot (e - e_{\ast} + n^{-4^{i + 3.5}}) \\
 & \geq \bar r_d \cdot \prod_{q = 1}^{i} \sin \alpha_q^{\sfb_n(s_{|\Theta|})}  \cdot ( n^{-4^{i + 3.5}} - 8 d n^{-4^{i + 4}}) \\
 & \geq n^{-4^{d + 3}} + 8 \bar r_d d n^{-4^{i + 4}}.
\end{align*}
This completes the proof of \cref{eq:target-345}. 
The lemma follows from the above lower bound and \cref{lemma:mclose}.

 \subsection{Proof of \cref{lemma:sign}}
 \label{proof-lemma:sign}

Similar to the derivation of \cref{eq:response-D1}, we assume that $n$ is sufficiently large, such that the agent in round $t$ reports $\theta_t'$ only if 
\begin{align*}
        \langle \Pi_{t, \theta_t'}, \bar{v}_{\theta_t} \rangle > \langle \Pi_{t, \theta}, \, \bar v_{\theta_t} \rangle  - \frac{\delta_{\sin}}{n^{4^{d + 3}}}, \qquad \forall \, \theta \in \Theta.
\end{align*}
In round $t$ we employ the following mechanism: 
\begin{align}
\label{eq:type-mechanism}
	\Pi_{j, t} = \mathds{1}_d / d + r_d \varphi_0^{-1}\big( \xi_{i + 1}(\hat\alpha_{(i + 1):(d - 2)}^{{s}_j}) \big), \qquad j \in [|\Theta|]. 
\end{align} 
For $j \in [|\Theta|]$, we define $\Pi_{j, t}^{\ast} = \mathds{1}_d / d + r_d \varphi_0^{-1}\big( \xi_{i + 1}(\alpha_{(i + 1):(d - 2)}^{\sfb_n({s}_j)}) \big)$.
Similar to the derivation of \cref{lemma:mechanism-close}, we see that under the current conditions, for all $j \in [|\Theta|]$ and $\theta \in \Theta$
\begin{align}
\label{eq:new93}
    \big| \langle \bar v_{\theta}, \Pi_{j, t}^{\ast} \rangle - \langle \bar v_{\theta}, \Pi_{j, t} \rangle \big| \leq 4r_d \cdot(d - 2 - i) \cdot n^{-4^{i + 4}}. 
\end{align}
In round $t$ the agent has type $\theta_t$. 
Suppose $\theta_t = \sfb_n(s_j)$ for some $j \in [|\Theta|]$, then 
\begin{align*}
    & \langle \bar v_{\theta_t}, \, \Pi_{j, t}^{\ast} \rangle = r_d \prod_{q = 1}^i \sin \alpha_i^{\sfb_n(s_j)}, \\
    & \langle \bar v_{\theta_t}, \, \Pi_{l, t}^{\ast} \rangle = r_d \prod_{q = 1}^i \sin \alpha_i^{\sfb_n(s_j)} \cdot \big \langle \xi_{i + 1}(\alpha_{(i + 1):(d - 2)}^{\sfb_n(s_j)}), \, \xi_{i + 1}(\alpha_{(i + 1):(d - 2)}^{\sfb_n(s_l)}) \big \rangle \qquad \mbox{for all }l \in [|\Theta|] \backslash \{j\}. 
\end{align*}
Therefore, for all $l \in [|\Theta|] \backslash \{j\}$
\begin{align}
\label{eq:new94}
\begin{split}
    \langle \bar v_{\theta_t}, \, \Pi_{j, t}^{\ast} \rangle - \langle \bar v_{\theta_t}, \, \Pi_{l, t}^{\ast} \rangle = & r_d \prod_{q = 1}^i \sin \alpha_i^{\sfb_n(s_j)} \cdot \Big( 1- \big \langle \xi_{i + 1}(\alpha_{(i + 1):(d - 2)}^{\sfb_n(s_j)}), \, \xi_{i + 1}(\alpha_{(i + 1):(d - 2)}^{\sfb_n(s_l)}) \big \rangle \Big) \\
    \overset{(i)}{\geq} &\, \frac{r_d \delta_{\sin} \min \{\chi_{i + 1}, \bar\chi_{i + 1}, \tilde\chi_{i + 1}\}^2}{30} \\
    \overset{(ii)}{\geq} & \,  n^{-4^{d + 3}} + 8r_d d n^{-4^{i + 4}}, 
\end{split}
\end{align}
where $(i)$ is by \cref{eq:xi-inner-product}, 
and $(ii)$ holds for a sufficiently large $n$. 
Combining \cref{eq:new93,eq:new94}, we conclude that for all $l \in [|\Theta|] \backslash \{j\}$, $\langle \bar v_{\theta_t}, \Pi_{j, t} \rangle - \langle \bar v_{\theta_t}, \Pi_{l, t} \rangle \geq n^{-4^{d + 3}}$ for all $l \in [|\Theta|] \backslash \{j\}$, hence the agent reports type $j$ if he has true type $\sfb_n(s_j)$.  
As a consequence, by repeatedly implementing mechanism \eqref{eq:type-mechanism} we can estimate the type prior distribution $f$. 
We present this procedure as \cref{alg:est-prob} below. 

\begin{algorithm}
\caption{Estimating the type prior distribution}
\label{alg:est-prob}
%\textbf{Input:} $n, \,  \cZ_{i + 1},  \, [|\Theta|] \backslash \cM_{i, |\Theta| - 1} = \{s_{|\Theta|}\}, \, \delta, \, x$;
\textbf{Input:} $n, \,  \hat\balpha_{(i + 1):(d - 2)}$;
\begin{algorithmic}[1] 
		\State $c_{\sf tot} \gets 0$, $c_j \gets 0$ for all $j \in [|\Theta|]$;
		\While{$c_{\sf tot} \leq \lceil \log n \rceil^4$}
			\State $c_{\sf tot} \gets c_{\sf tot} + 1$;
			\State Start a new round, announce the mechansim as in \cref{eq:type-mechanism}, and observe the agent's reported type $r$;
			\State $c_r \gets c_r + 1$;
		\EndWhile
		\State \texttt{\textcolor{blue}{// Create delayed feedback}}
		\State $\ell \gets \lceil \log n \rceil^2$; 
		\For{$i \in [\ell]$}
			\State Start a new round and announce a dummy mechanism $\mathds{1}_{|\Theta| \times d} / d$; 
		\EndFor
		\Return $\hat{p}_j = c_j / c_{\sf tot}$ for all $j \in [|\Theta|]$; 
\end{algorithmic}
\end{algorithm}

The output $\hat{p}_j$ concentrates around  $f(\sfb_n(s_j))$ by Hoeffding's inequality. 
Given the estimates $(\hat{p}_j)_{j \in [|\Theta|]}$, we can estimate the sign of $\alpha_i^{\sfb_n(s_{|\Theta|})} - \pi / 2$. 
For this purpose we use the following mechanism:
\begin{align}\label{eq:simple-two-point}
	\widetilde\Pi_{1, t} = (1 - r)\mathds{1}_d / d + r e_i, \qquad \widetilde\Pi_{j, t} = (1 + r)\mathds{1}_d / d - r e_i\,\,\, \mbox{ for all }j \in \{2, 3, \cdots, |\Theta|\}, 
\end{align}
where $r = 1 / (2d)$, and we recall that $e_i$ is the $i$-th standard basis in $\RR^d$.
With mechanism \cref{eq:simple-two-point}, we then implement \cref{alg:est-sign}. 
\begin{algorithm}
\caption{Estimating the sign of $\alpha_i^{\sfb_n(s_{|\Theta|})} - \pi / 2$}
\label{alg:est-sign}
%\textbf{Input:} $n, \,  \cZ_{i + 1},  \, [|\Theta|] \backslash \cM_{i, |\Theta| - 1} = \{s_{|\Theta|}\}, \, \delta, \, x$;
\textbf{Input:} $n, \,  \hat\balpha_{(i + 1):(d - 2)}$;
\begin{algorithmic}[1] 
		\State $c_{\sf tot} \gets 0$, $c_+ \gets 0$, and $c_- \gets 0$;
		\While{$c_{\sf tot} \leq \lceil \log n \rceil^4$}
			\State $c_{\sf tot} \gets c_{\sf tot} + 1$;
			\State Start a new round, announce the mechansim as in \cref{eq:simple-two-point}, observe the agent's reported type $r$;
			\If{$r = 1$}
			\State $c_+ \gets c_+ + 1$;
			\Else
			\State $c_- \gets c_- + 1$;
			\EndIf
		\EndWhile
		\State \texttt{\textcolor{blue}{// Create delayed feedback}}
		\State $\ell \gets \lceil \log n \rceil^2$; 
		\For{$i \in [\ell]$}
			\State Start a new round and announce a dummy mechanism $\mathds{1}_{|\Theta| \times d} / d$; 
		\EndFor
		\Return $(c_+, c_-)$; 
\end{algorithmic}
\end{algorithm}

When $\theta_t = \sfb_n(s_j)$, then 
\begin{align*}
    & \langle \widetilde \Pi_{1, t}, \bar v_{\theta_t} \rangle = r \cos \alpha_i^{\sfb_n(s_j)} \cdot \prod_{q = 1}^{i - 1} \sin \alpha_q^{\sfb_n(s_j)}, \\
    & \langle \widetilde \Pi_{l, t}, \bar v_{\theta_t} \rangle = -r \cos \alpha_i^{\sfb_n(s_j)} \cdot \prod_{q = 1}^{i - 1} \sin \alpha_q^{\sfb_n(s_j)} \qquad \mbox{for all } l \in [|\Theta|] \backslash \{1\}.  
\end{align*}
If $\sign(\alpha_i^{\sfb_n(s_{j})} - \pi / 2) = +$, then by \cref{lemma:cos}, for all $2 \leq l \leq |\Theta|$ we have
\begin{align*}
    \langle \widetilde \Pi_{1, t}, \bar v_{\theta_t} \rangle - \langle \widetilde\Pi_{l, t}, \bar v_{\theta_t} \rangle = 2r \cos \alpha_i^{\sfb_n(s_j)} \cdot \prod_{q = 1}^{i - 1} \sin \alpha_q^{\sfb_n(s_j)} \leq - \frac{r \delta_{\sin} \min \{\chi_{i}, \bar\chi_{i}, \tilde\chi_{i}\}^2 }{15}.
\end{align*}
On the other hand, if $\sign(\alpha_i^{\sfb_n(s_{j})} - \pi / 2) = -$, then for all $2 \leq l \leq |\Theta|$
\begin{align*}
    \langle \widetilde \Pi_{1, t}, \bar v_{\theta_t} \rangle - \langle \widetilde\Pi_{l, t}, \bar v_{\theta_t} \rangle = 2r \cos \alpha_i^{\sfb_n(s_j)} \cdot \prod_{q = 1}^{i - 1} \sin \alpha_q^{\sfb_n(s_j)} \geq  \frac{r \delta_{\sin} \min \{\chi_{i}, \bar\chi_{i}, \tilde\chi_{i}\}^2 }{15}.
\end{align*}
For a sufficiently large $n$ we have $\frac{r \delta_{\sin} \min \{\chi_{i}, \bar\chi_{i}, \tilde\chi_{i}\}^2 }{15} \geq n^{-4^{d + 3}} + 8r_d dn^{-4^{i + 4}}$. 
Therefore, if $\theta_t = \sfb_n(s_j)$ and $\sign(\alpha_i^{\sfb_n(s_{j})} - \pi / 2) = +$, then the agent will not report type 1, and if $\theta_t = \sfb_n(s_j)$ and  $\sign(\alpha_i^{\sfb_n(s_{j})} - \pi / 2) = -$, then the agent will report type 1. 
As a consequence, if $\sign(\alpha_i^{\sfb_n(s_{|\Theta|})} - \pi / 2) = -$ then $c^+ / c_{\sf tot} \approx \sum_{\hat\alpha_i^{s_j} \leq \pi / 2, j \in [|\Theta| - 1]} \hat p_j + \hat p_{|\Theta|}$. 
On the other hand, if $\sign(\alpha_i^{\sfb_n(s_{|\Theta|})} - \pi / 2) = +$, then $c^+ / c_{\sf tot} \approx \sum_{\hat\alpha_i^{s_j} \leq \pi / 2, j \in [|\Theta| - 1]} \hat p_j$.
We estimate the sign of $\alpha_i^{\sfb_n(s_{|\Theta|})} - \pi / 2$ as negative if 
\begin{align*}
	\Big| \frac{c_+}{c_{\sf tot}} - \sum_{\hat\alpha_{i}^{{s}_j} \leq \pi / 2, \, j \in [|\Theta| - 1]  }\hat{p}_j \Big| \geq \frac{1}{\sqrt{\log n}},
\end{align*}
and estimate this sign as positive otherwise. 
The rest of the proof follows from simple application of Hoeffding's inequality.

\subsection{Proof of \cref{lemma:D6}}
\label{sec:proof-D6}

Define 
\begin{align*}
	F(a) := & \langle \bar v_{\theta_t}, \, \Pi_{1, t}^{\ast} \rangle = \langle \bar{v}_{\theta_t}, \, \mathds{1}_d / d + r_d \varphi_0^{-1}(\xi_i(a, \alpha^{\theta_p}_{(i + 1):(d - 2)})) \rangle \\
	= & r_d \cdot \left( \cos a \cos \alpha_i^{\theta_t} + \sin a \sin \alpha_i^{\theta_t} \cdot \langle \xi_{i + 1}(\alpha_{(i + 1):(d - 2)}^{\theta_p}), \xi_{i + 1}(\alpha_{(i + 1):(d - 2)}^{\theta_t}) \rangle \right) \cdot \prod_{j = 1}^{i - 1} \sin \alpha_j^{\theta_t}. 
\end{align*}
Note that $F(\alpha - \delta) = \langle \bar{v}_{\theta_t}, \Pi_{1, t}^{\ast}\rangle$ and $F(\alpha + \delta) = \langle \bar{v}_{\theta_t}, \Pi_{2, t}^{\ast} \rangle$, where we recall that  $\Pi_{1, t}^{\ast}$ and $\Pi_{2, t}^{\ast}$ are from  \cref{eq:oracle-mechanism}. To compare $F(\alpha - \delta)$ and $F(\alpha + \delta)$, we take the derivative of $F$:
\begin{align*}
	F'(a) = r_d \cdot \left( - \sin a \cot \alpha_i^{\theta_t} + \cos a \cdot \langle \xi_{i + 1}(\alpha_{(i + 1):(d - 2)}^{\theta_p}), \xi_{i + 1}(\alpha_{(i + 1):(d - 2)}^{\theta_t}) \rangle \right) \cdot \prod_{j = 1}^i \sin \alpha_j^{\theta_t}. 
\end{align*}
We separately consider $\alpha_i^{\theta_t} \in [\alpha - \delta, \pi]$ and $\alpha_i^{\theta_t} \in [0, \pi - \alpha + \delta]$ below. 
\subsubsection*{Case I: $\alpha_i^{\theta_t} \in [\alpha - \delta, \pi]$}
Since $\alpha > \pi / 2$ and $\alpha_i^{\theta_t} \in [\alpha - \delta, \pi]$, we conclude that for all $a \in [\alpha - \delta, \alpha + \delta]$, $|a - \pi / 2| \leq |\alpha_i^{\theta_t} - \pi / 2| + 2\delta$. By Definition \ref{def:chi} we have $ |\alpha_i^{\theta_t} - \pi / 2| \geq \tilde \chi_i$ and $\pi / 2 - |\alpha_i^{\theta_t} - \pi / 2| \geq \tilde \chi_i$. 
Combining this result and  \cref{lemma:cot}, we see that  
\begin{align*}
	 |\cos a| &\leq  \max\left\{ \frac{|a - \pi / 2|}{|\alpha_i^{\theta_t} - \pi / 2|}, \, 1 \right\} \cdot |\cos \alpha_i^{\theta_t}| \leq \frac{\tilde \chi_i + 2\delta}{\tilde \chi_i} \cdot |\cos \alpha_i^{\theta_t}| = \left( 1 + 2\tilde \chi_i^{-1}{\delta} \right) \cdot |\cos \alpha_i^{\theta_t}|, \\
	|\cot \alpha_i^{\theta_t} \cdot \sin a| &= |\cos \alpha_i^{\theta_t}| \cdot \frac{|\sin a|}{|\sin \alpha_i^{\theta_t}|} \geq |\cos \alpha_i^{\theta_t}| \cdot \min \left\{ \frac{\pi / 2 - |a - \pi / 2|}{\pi / 2 - |\alpha_i^{\theta_t} - \pi / 2|}, \, 1 \right\} \\
	& \geq |\cos \alpha_i^{\theta_t}| \cdot  \frac{\pi / 2 - |\alpha_i^{\theta_t} - \pi / 2| - 2\delta}{\pi / 2 - |\alpha_i^{\theta_t} - \pi / 2|} \geq |\cos \alpha_i^{\theta_t}| \cdot \frac{\tilde \chi_i - 2\delta}{\tilde \chi_i}
\end{align*}
for all $a \in [\alpha - \delta, \alpha + \delta]$.
By Definition \ref{def:chi}, 
when $i + 1 \leq d - 3$, we have  $|\alpha_{i + 1}^{\theta_p} - \alpha_{i + 1}^{\theta_t} - \pi| \geq \tilde \chi_{i + 1}$, $|\alpha_{i + 1}^{\theta_p} - \alpha_{i + 1}^{\theta_t}| \geq \chi_{i + 1}$, and $|\alpha_{i + 1}^{\theta_p} - \alpha_{i + 1}^{\theta_t} + \pi| \geq \tilde \chi_{i + 1}$. 
Therefore, by \cref{lemma:cos} we have $|\cos (\alpha_{i + 1}^{\theta_p} - \alpha_{i + 1}^{\theta_t})| \leq 1 - \min\{\chi_{i + 1}^2, \tilde \chi_{i + 1}^2\} / 30$. 
Note that $|\alpha_{i + 1}^{\theta_t} + \alpha_{i + 1}^{\theta_p}| \geq \tilde\chi_{i + 1}$, $|\alpha_{i + 1}^{\theta_t} + \alpha_{i + 1}^{\theta_p} - \pi| \geq \chi_{i + 1}$, and $|\alpha_{i + 1}^{\theta_t} + \alpha_{i + 1}^{\theta_p} - 2\pi| \geq \tilde\chi_{i + 1}$. 
By \cref{lemma:cos}, we have $|\cos (\alpha_{i + 1}^{\theta_p} + \alpha_{i + 1}^{\theta_t})| \leq 1 - \min\{\chi_{i + 1}^2, \tilde \chi_{i + 1}^2\} / 30$. 
When $i + 1 = d - 2$, we have $|\alpha_{d - 2}^{\theta_p} - \alpha_{d - 2}^{\theta_t}| \geq 
\bar\chi_{i + 1}$, $|\alpha_{d - 2}^{\theta_p} - \alpha_{d - 2}^{\theta_t} - \pi| \geq \bar\chi_{i + 1}$, $|\alpha_{d - 2}^{\theta_p} - \alpha_{d - 2}^{\theta_t} - 2\pi| \geq \bar\chi_{i + 1}$, $|\alpha_{d - 2}^{\theta_p} - \alpha_{d - 2}^{\theta_t} + \pi| \geq \bar\chi_{i + 1}$ and $|\alpha_{d - 2}^{\theta_p} - \alpha_{d - 2}^{\theta_t} + 2\pi| \geq \bar\chi_{i + 1}$. 
Therefore, $|\cos(\alpha_{i + 1}^{\theta_p} - \alpha_{i + 1}^{\theta_t})| \leq 1 - \bar\chi_{i + 1}^2 / 30$. 
As a consequence, 
%$|\alpha_{i + 1}^{\theta_t} + \alpha_{i + 1}^{\theta_p} - \pi / 2| \geq \chi_{i + 1}$ and $|\alpha_{i + 1}^{\theta_t} + \alpha_{i + 1}^{\theta_p} - 3\pi / 2| \geq \chi_{i + 1}$. 
%When $i + 1 = d - 2$, we have $\cZ_{\pi} (\alpha_{i + 1}^{\theta_p} - \alpha_{i + 1}^{\theta_t}) \geq \bar \chi_{i + 1}$ and $\cZ_{\pi} (\alpha_{i + 1}^{\theta_p} - \alpha_{i + 1}^{\theta_t}) \leq \pi -  \bar \chi_{i + 1}$. 
%By \cref{lemma:cos}, when $\theta_t \neq \theta_p$, it holds that 
%
\begin{align*}
	\begin{split}
		& \left| \langle \xi_{i + 1}(\alpha_{(i + 1):(d - 2)}^{\theta_p}), \xi_{i + 1}(\alpha_{(i + 1):(d - 2)}^{\theta_t}) \rangle \right| \\
		  &  \leq \max\left\{ |\cos(\alpha_{i + 1}^{\theta_p} - \alpha_{i + 1}^{\theta_t})|, \, |\cos(\alpha_{i + 1}^{\theta_p} + \alpha_{i + 1}^{\theta_t})|  \right\} \mathbbm{1}\{i < d - 3\} + \left| \cos(\alpha^{\theta_p}_{d - 2} - \alpha_{d - 2}^{\theta_t}) \right| \mathbbm{1}\{i = d - 3\} \\
		  &   \leq 1 - \frac{\min\{\chi_{i + 1}, \tilde\chi_{i + 1}\}^2}{30}  \mathbbm{1}\{i < d - 3\}  - \frac{\bar\chi_{i + 1}^2}{30} \mathbbm{1}\{i = d - 3\}. 
	\end{split}
\end{align*}
Note that $\tilde\chi_i \leq \min\{|\alpha_i^{\theta_t} - \pi / 2|, \, 0.1\}$, hence $|\cos(\alpha_i^{\theta_t})| \geq \tilde\chi_i / 2$. 
As a result, for all $a \in [\alpha - \delta, \alpha + \delta]$, 
\begin{align*}
	|F'(a)| \geq &\, r_d \delta_{\sin} \cdot |\cos \alpha_i^{\theta_t}| \cdot \left( 1 - 2\delta \tilde \chi_i^{-1} - (1 + 2 \tilde \chi_i^{-1}{\delta}) \cdot (\mbox{$1 - \frac{\min\{\chi_{i + 1}, \bar\chi_{i + 1}, \tilde\chi_{i + 1}\}^2}{30}$})  \right) \\
	\geq & \, \frac{r_d \delta_{\sin} \tilde \chi_i}{2} \cdot \left(\mbox{$ \frac{\min\{\chi_{i + 1}, \bar\chi_{i + 1}, \tilde\chi_{i + 1}\}^2}{30}$} - 4\delta \tilde \chi_i^{-1} \right) \\
	\geq & \frac{r_d \delta_{\sin} \tilde\chi_i \min \{\chi_{i + 1}, \tilde \chi_{i + 1},  \bar\chi_{i + 1}\}^2}{120},
\end{align*}
where the last inequality follows from \hyperlink{A5p}{$\mathsf{(A5')}$}. 
Since the lower bound above holds for all $a \in [\alpha - \delta, \alpha + \delta]$ and $a \mapsto F'(a)$ is continuous, we conclude that the sign of $F'(a)$ remains unchanged for all  $a \in [\alpha - \delta, \alpha + \delta]$. 
In particular, under the current set of conditions we have $F'(a) > 0$ for all $a \in [\alpha - \delta, \alpha + \delta]$. Therefore, 
\begin{align*}
	F(\alpha + \delta) - F(\alpha - \delta) \geq \frac{r_d \delta \delta_{\sin} \tilde\chi_i \min \{\chi_{i + 1}, \bar\chi_{i + 1}, \tilde\chi_{i + 1}\}^2}{60}.  
\end{align*}
By \hyperlink{A4p}{$\mathsf{(A4')}$}, we have
\begin{align*}
	F(\alpha + \delta) \geq F(\alpha - \delta)   + n^{-4^{d + 3}} + 8 r_d d n^{-4^{i + 4}}. 
\end{align*}
Combining triangle inequality, \cref{lemma:mechanism-close2}, and the above inequality, we get
\begin{align*}
	\langle \Pi_{2}^{\alpha, \delta, s}, \bar{v}_{\theta_t} \rangle \geq \langle \Pi_{1}^{\alpha, \delta, s}, \bar{v}_{\theta_t} \rangle + n^{-4^{d + 3}}. 
\end{align*}
By \cref{eq:new48}, we conclude that the agent will never report type 1, hence completing the proof for case I. 

\subsubsection*{Case II: $\alpha_i^{\theta_t} \in [0, \pi - \alpha + \delta]$}

Since $\alpha > \pi / 2$, we have $|a - \pi / 2| \leq |\alpha_i^{\theta_t} - \pi  /2| + 2\delta$ for all $a \in [\alpha - \delta, \alpha + \delta]$. Using exactly the same argument as above, we get 
\begin{align*}
    |F'(a)| \geq \frac{r_d \delta_{\sin} \tilde\chi_i \min \{\chi_{i + 1}, \tilde \chi_{i + 1},  \bar\chi_{i + 1}\}^2}{120}
\end{align*}
for all $a \in [\alpha - \delta, \alpha + \delta]$.
By \hyperlink{A3p}{$\mathsf{(A3')}$} we have $|\alpha - \pi / 2| \geq 2 \delta$. 
Since $\alpha_i^{\theta_t} \in [0, \pi - \alpha + \delta]$ and $\alpha > \pi / 2$, we have $\alpha_i^{\theta_t} \in [0, \pi / 2)$. 
In this case, $-\sin \alpha \cot \alpha_i^{\theta_t} \leq 0$, which further implies that $F'(a) \leq  -{r_d \delta_{\sin} \tilde\chi_i \min \{\chi_{i + 1}, \tilde\chi_{i + 1}, \bar\chi_{i + 1}\}^2} / 120$ for all $a \in [\alpha - \delta, \alpha + \delta]$.  
As a consequence, 
\begin{align*}
    F(\alpha - \delta) - F(\alpha + \delta) \geq \frac{r_d \delta \delta_{\sin} \tilde \chi_i \min \{\chi_{i + 1}, \tilde\chi_{i + 1},  \bar{\chi}_{i + 1}\}^2 }{60}. 
\end{align*}
Under \hyperlink{A4p}{$\mathsf{(A4')}$}, it holds that $F(\alpha - \delta) - F(\alpha + \delta) \geq n^{-4^{d + 3}} + 8r_d d n^{-4^{i + 4}}$. 
Using \cref{lemma:mechanism-close2}, the triangle inequality and  the above lower bound, we conclude that in the current setting the agent will never report type 2. 
The proof is done.

\subsection{Proof of \cref{lemma:D5}}
\label{sec:proof-D5}

Recall that $\Pi_{1, t}^{\ast}$ and $\Pi_{2, t}^{\ast}$ are from \cref{eq:oracle-mechanism}. 
When $\theta_t = \theta_p$, 
\begin{align*}
	 \langle \Pi_{1, t}^{\ast}, \bar{v}_{\theta_t} \rangle = r_d \cdot \cos(\alpha - \delta - \alpha_i^{\theta_t}) \cdot \prod_{j = 1}^{i - 1} \sin \alpha_j^{\theta_t}, \qquad  \langle \Pi_{2, t}^{\ast}, \bar{v}_{\theta_t} \rangle = r_d \cdot \cos(\alpha + \delta - \alpha_i^{\theta_t}) \cdot \prod_{j = 1}^{i - 1} \sin \alpha_j^{\theta_t}. 
\end{align*}
Since $\alpha > \pi / 2$ and $\alpha_i^{\theta_t} > \alpha$, by \hyperlink{A6p}{$\mathsf{(A6')}$} we have 
$$\arc(\alpha_i^{\theta_t}, \alpha - \delta) \geq \arc(\alpha_i^{\theta_t}, \alpha + \delta) + 6 (r_d \delta_{\sin})^{-1/2} n^{-4^{d + 2}} + 20 d^{1/2} (\delta_{\sin})^{-1/2} n^{-4^{i + 3.5}}. $$
Invoking \cref{lemma:cos}, we obtain that 
\begin{align*}
	\langle \Pi_{2, t}^{\ast}, \bar{v}_{\theta_t} \rangle - \langle \Pi_{1, t}^{\ast}, \bar{v}_{\theta_t} \rangle & = r_d \cdot \left( \cos(\alpha + \delta - \alpha_i^{\theta_t}) - \cos (\alpha - \delta - \alpha_i^{\theta_t}) \right) \cdot \prod_{j = 1}^{i - 1} \sin \alpha_j^{\theta_t} \\
	& \geq      \frac{r_d \delta_{\sin}}{30} \cdot \left( 6 (r_d \delta_{\sin})^{-1/2} n^{-4^{d + 2}} + 20 d^{1/2} (\delta_{\sin})^{-1/2} n^{-4^{i + 3.5}} \right)^2 \\
	 & \geq n^{-4^{d + 3}} + 8 r_d d n^{-4^{i + 4}}. 
\end{align*}
Combining the above lower bound,  \cref{lemma:mechanism-close2} and the triangle inequality, we conclude that 
\begin{align*}
	\langle \Pi_{2}^{\alpha, \delta, s}, \bar{v}_{\theta_t} \rangle - \langle \Pi_{1}^{\alpha, \delta, s}, \bar{v}_{\theta_t} \rangle \geq n^{-4^{d + 3}}. 
\end{align*}

\subsection{Proof of \cref{lemma:D7}}
\label{sec:proof-D7}

%Note that the case $\ell = 1$ is already covered by \cref{lemma:D5}. In what follows, we consider $\ell \geq 3$. 
By \hyperlink{A1p}{$\mathsf{(A1')}$} and \hyperlink{A2p}{$\mathsf{(A2')}$},
we know that $\theta_p \in \{\sfb_n(s_{|\Theta| - 1}),\, \sfb_n (s_{|\Theta|})\}$. 
Further by \hyperlink{A3p}{$\mathsf{(A3')}$},  
we know that $\alpha_i^{\theta_p} > \alpha \geq \alpha_i^{\sfb_n(j_{\ell - 2})} + \delta^{1/4} + 2\delta$ for all $3 \leq \ell \leq |\Theta|$ (recall that $j_{\ell - 2}$ is defined in \cref{eq:modified-mechanism}).
Therefore, $\arc(\alpha_i^{\theta_p}, \alpha + \delta) \leq \arc(\alpha_i^{\theta_p}, \alpha_i^{\sfb_n(j_{\ell - 2})}) + \delta^{1/4}$. 
Note that
\begin{align*}
	& \langle \Pi^{\ast}_{\ell, t}, \bar{v}_{\theta_p} \rangle \\
	  &\qquad  =\, r_d \cdot \left( \cos \alpha_i^{\theta_p} \cos \alpha_i^{\sfb_n(j_{\ell - 2})} + \sin \alpha_i^{\theta_p} \sin \alpha_i^{\sfb_n(j_{\ell - 2})} \langle \xi_{i + 1}(\alpha_{(i + 1):(d - 2)}^{\theta_p}), \xi_{i + 1}(\alpha_{(i + 1):(d - 2)}^{\sfb_n(j_{\ell - 2})})   \rangle \right) \cdot \prod_{j = 1}^{i - 1} \sin \alpha_j^{\theta_p}, \\
	&  \langle \Pi^{\ast}_{2, t}, \bar{v}_{\theta_p} \rangle = r_d \cdot \cos(\alpha_i^{\theta_p} - \alpha - \delta) \cdot \prod_{j = 1}^{i - 1} \sin \alpha_j^{\theta_p}. 
\end{align*}
Therefore, by \cref{lemma:cos} we have 
\begin{align*}
	\langle \Pi^{\ast}_{2, t}, \bar{v}_{\theta_p} \rangle - \langle \Pi^{\ast}_{\ell, t}, \bar{v}_{\theta_p} \rangle \geq & \, r_d \delta_{\sin} \cdot \left( \cos (\alpha_i^{\theta_p} - \alpha - \delta) - \cos (\alpha_i^{\theta_p} - \alpha_i^{\sfb_n(j_{\ell - 2})}) \right) \\
	\geq & \, \frac{r_d \delta_{\sin} \delta^{1/2} }{30} \geq n^{-4^{d + 3}} + 8 r_d d n^{-4^{i + 3}}, 
\end{align*}
where the last lower bound is by \hyperlink{A4p}{$\mathsf{(A4')}$}. 
Invoking \cref{lemma:mechanism-close2} and triangle inequality, we get 
\begin{align*}
	\langle \Pi_{2}^{\alpha, \delta, s}, \bar{v}_{\theta_p} \rangle - \langle \Pi_{\ell}^{\alpha, \delta, s}, \bar{v}_{\theta_p} \rangle \geq n^{-4^{d + 3}} 
\end{align*}
for all $\ell \in [|\Theta|] \backslash \{1,2\}$, thus
completing the proof of the lemma.

\subsection{Proof of \cref{lemma:E7}}
\label{sec:proof-E7}

When $\theta_t = \theta_p$, we have
\begin{align*}
	 \langle \Pi_{1, t}^{\ast}, \bar{v}_{\theta_t} \rangle = r_d \cdot \cos(\alpha - \delta - \alpha_i^{\theta_t}) \cdot \prod_{j = 1}^{i - 1} \sin \alpha_j^{\theta_t}, \qquad  \langle \Pi_{2, t}^{\ast}, \bar{v}_{\theta_t} \rangle = r_d \cdot \cos(\alpha + \delta - \alpha_i^{\theta_t}) \cdot \prod_{j = 1}^{i - 1} \sin \alpha_j^{\theta_t}. 
\end{align*}
By assumption we have $\alpha - \delta \leq \alpha_i^{\theta_p} \leq \alpha $. Using this and \hyperlink{A6p}{$\mathsf{(A6')}$}, we get
$$\arc(\alpha_i^{\theta_p}, \alpha + \delta) \geq \arc(\alpha_i^{\theta_t}, \alpha - \delta) + 6 (r_d \delta_{\sin})^{-1/2} n^{-4^{d + 2}} + 20 d^{1/2} (\delta_{\sin})^{-1/2} n^{-4^{i + 3.5}}. $$
By \cref{lemma:cos}, we have 
\begin{align*}
	\langle \Pi_{1,t}^{\ast}, \bar{v}_{\theta_p} \rangle - \langle \Pi_{2,t}^{\ast}, \bar{v}_{\theta_p} \rangle = & r_d \cdot \left( \cos(\alpha - \delta - \alpha_i^{\theta_t}) - \cos (\alpha + \delta - \alpha_i^{\theta_t}) \right) \cdot \prod_{j = 1}^{i - 1} \sin \alpha_j^{\theta_t} \\
	\geq & \frac{r_d \delta_{\sin}}{30} \cdot \left( 6 (r_d \delta_{\sin})^{-1/2} n^{-4^{d + 2}} + 20 d^{1/2} (\delta_{\sin})^{-1/2} n^{-4^{i + 3.5}} \right)^2 \\
	\geq & n^{-4^{d + 3}} + 8 r_d d n^{-4^{i + 4}}. 
\end{align*}
Putting together the above lower bound and \cref{lemma:mechanism-close2}, we get
  $\langle \Pi_{1}^{\alpha, \delta, s}, \bar{v}_{\theta_p} \rangle - \langle \Pi_{2}^{\alpha, \delta, s}, \bar{v}_{\theta_p} \rangle \geq n^{-4^{d + 3}}$. This completes the proof of the lemma. 

\subsection{Proof of \cref{lemma:E8}}
\label{sec:proof-E8}

Recall that $\alpha_i^{\theta_p} \in [\alpha - \delta, \alpha]$ and $\alpha > \pi / 2$. 
By \hyperlink{A3p}{$\mathsf{(A3')}$} we have 
\begin{align*}
	\alpha \geq \alpha_i^{\theta_p} \geq \alpha - \delta \geq \alpha_i^{\sfb_n(j_{\ell - 2})} + \delta + \delta^{1/4},
\end{align*}
which further implies that $\arc(\alpha_i^{\theta_p}, \alpha - \delta)  + \delta^{1/4} \leq \arc(\alpha_i^{\theta_p}, \alpha_i^{\sfb_n(j_{\ell - 2})})$. 
By \cref{lemma:cos}, we have 
\begin{align*}
	\langle \Pi_{1,t}^{\ast}, \bar{v}_{\theta_p} \rangle - \langle \Pi_{\ell,t}^{\ast}, \bar{v}_{\theta_p} \rangle \geq & \, r_d \delta_{\sin} \cdot \left( \cos (\alpha_i^{\theta_p} - \alpha + \delta) - \cos (\alpha_i^{\theta_p} - \alpha_i^{\sfb_n(j_{\ell - 2})}) \right) \\
	\geq & \frac{r_d \delta_{\sin} \delta^{1/2}}{30} \geq n^{-4^{d + 3}} + 8 r_d d n^{-4^{i + 3}}, 
\end{align*}
where the last lower bound is by \hyperlink{A4p}{$\mathsf{(A4')}$}.
Applying \cref{lemma:mechanism-close2} and the triangle inequality, we get
\begin{align*}
	\langle \Pi_{1}^{\alpha, \delta, s}, \bar{v}_{\theta_t} \rangle - \langle \Pi_{\ell}^{\alpha, \delta, s}, \bar{v}_{\theta_t} \rangle \geq n^{-4^{d + 3}}  
\end{align*}
for all $\ell \in [|\Theta|] \backslash \{1, 2\}$. 
The proof is complete.

\subsection{Proof of \cref{lemma:new-add}}
\label{sec:proof-lemma:new-add}

\subsubsection*{Case I: $\alpha_i^{\theta_t} \in [0, \alpha)$}

Since $\alpha_i^{\theta_t} < \alpha$, by \hyperlink{A6p}{$\mathsf{(A6')}$}, we have 
\begin{align*}
    \arc(\alpha + \delta, \alpha_i^{\theta_t}) \geq \arc(\alpha - \delta, \alpha_i^{\theta_t}) + 6(r_d \delta_{\sin})^{-1/2} n^{-4^{d + 2}} + 20 d^{1/2}( \delta_{\sin} )^{-1/2} n^{-4^{i + 3.5}}. 
\end{align*}
When $\theta_t = \theta_p$, it holds that 
\begin{align*}
	 \langle \Pi_{1, t}^{\ast}, \bar{v}_{\theta_t} \rangle = r_d \cdot \cos(\alpha - \delta - \alpha_i^{\theta_t}) \cdot \prod_{j = 1}^{i - 1} \sin \alpha_j^{\theta_t}, \qquad  \langle \Pi_{2, t}^{\ast}, \bar{v}_{\theta_t} \rangle = r_d \cdot \cos(\alpha + \delta - \alpha_i^{\theta_t}) \cdot \prod_{j = 1}^{i - 1} \sin \alpha_j^{\theta_t}.
\end{align*}
By \cref{lemma:cos}, we get 
\begin{align}
\label{eq:last-eq}
\begin{split}
    \langle \Pi_{1, t}^{\ast}, \bar{v}_{\theta_t} \rangle - \langle \Pi_{2, t}^{\ast}, \bar{v}_{\theta_t} \rangle = & r_d \cdot \big( \cos(\alpha - \delta - \alpha_i^{\theta_t}) - \cos(\alpha + \delta - \alpha_i^{\theta_t}) \big) \cdot \prod_{j = 1}^{i - 1} \sin \alpha_j^{\theta_t} \\
    \geq & \frac{r_d \delta_{\sin} (6(r_d \delta_{\sin})^{-1/2} n^{-4^{d + 2}} + 20 d^{1/2}( \delta_{\sin} )^{-1/2} n^{-4^{i + 3.5}})^2}{30} \\
    \geq & n^{-4^{d + 3}} + 8r_d d n^{-4^{i + 4}}.  
\end{split}
\end{align}
Combining \cref{eq:last-eq} and \cref{lemma:mechanism-close2}, we obtain that 
\begin{align*}
    \langle \Pi_1^{\alpha, \delta, s}, \bar v_{\theta_t} \rangle - \langle \Pi_2^{\alpha, \delta, s}, \bar v_{\theta_t} \rangle \geq n^{-4^{d + 3}}. 
\end{align*}
The proof is done. 

\subsubsection*{Case I: $\alpha_i^{\theta_t} \in (\alpha, \pi]$}

Since $\alpha_i^{\theta_t} > \alpha$, from \hyperlink{A6p}{$\mathsf{(A6')}$} we know that 
\begin{align*}
    \arc(\alpha - \delta, \alpha_i^{\theta_t}) \geq \arc(\alpha + \delta, \alpha_i^{\theta_t}) + 6(r_d \delta_{\sin})^{-1/2} n^{-4^{d + 2}} + 20 d^{1/2}( \delta_{\sin} )^{-1/2} n^{-4^{i + 3.5}}.  
\end{align*}
By \cref{lemma:cos}, 
\begin{align}
\label{eq:last-eq2}
\begin{split}
    \langle \Pi_{2, t}^{\ast}, \bar{v}_{\theta_t} \rangle - \langle \Pi_{1, t}^{\ast}, \bar{v}_{\theta_t} \rangle = & r_d \cdot \big( \cos(\alpha + \delta - \alpha_i^{\theta_t}) - \cos(\alpha - \delta - \alpha_i^{\theta_t}) \big) \cdot \prod_{j = 1}^{i - 1} \sin \alpha_j^{\theta_t} \\
    \geq & \frac{r_d \delta_{\sin} (6(r_d \delta_{\sin})^{-1/2} n^{-4^{d + 2}} + 20 d^{1/2}( \delta_{\sin} )^{-1/2} n^{-4^{i + 3.5}})^2}{30} \\
    \geq & n^{-4^{d + 3}} + 8r_d d n^{-4^{i + 4}}.  
\end{split}
\end{align}
Putting together \cref{eq:last-eq} and \cref{lemma:mechanism-close2}, 
\begin{align*}
    \langle \Pi_2^{\alpha, \delta, s}, \bar v_{\theta_t} \rangle - \langle \Pi_1^{\alpha, \delta, s}, \bar v_{\theta_t} \rangle \geq n^{-4^{d + 3}}. 
\end{align*}
The proof is done. 

\end{appendices}

\end{document}